%% file: main.tex
\documentclass[10pt]{article}
\input{commands}

\title{Posterior Refinement:\\Fast Language Generation via Any-Order Flow Maps}

\author[\ast 1]{Manan Agarwal}
\author[\ast 1]{Sheel Shah}
\author[2]{Chanhyuk Lee}
\author[2]{Jaehoon Yoo}
\midauthor[1]{Jerry Huang}
\midauthor[2]{Seunghoon Hong}
\midauthor[1]{Aditi Raghunathan}
\seniorauthor[\dag 2]{Jinwoo Kim}
\seniorauthor[\dag 1]{Nicholas M.~Boffi}
\affiliation[1]{Carnegie Mellon University}
\affiliation[2]{KAIST}
\contribution[\ast]{Equal contribution; order decided by a coin flip}
\contribution[\dagger]{Equal advising}

\begin{document}
\vspace{1em}
\begin{abstract}
\input{abstract}
\end{abstract}
\maketitle

\input{intro}

\input{background}
\input{algo}
\input{results}
\input{conclusion}
\input{acknowledgements}

% Bibliography
% \newpage
\bibliographystyle{unsrtnat}
\bibliography{main}

\newpage
\appendix
\input{appendix}

\end{document}

%% file: commands.tex
\usepackage{bofflab}

\definecolor{titleteal}{HTML}{F4F6F6}
\tcbset{bofflab-title/.style={
  bofflab-common,
  colback=fig1teal!4!white,
  boxrule=0pt,
  frame hidden,
  left=0.5cm,right=0.5cm,top=0.5cm,bottom=0.5cm
}}

\usepackage{microtype}
\usepackage{parskip}

\usepackage{booktabs}
\usepackage{caption}
\usepackage{multirow}
\usepackage{wrapfig}
\usepackage{makecell}
\usepackage{adjustbox}
\usepackage{fullpage}
\usepackage{ragged2e}
\usepackage{algorithm}
\usepackage{algorithmic}
\usepackage{listings}
\usepackage{subcaption}
\input{pseudocode_style}

\usepackage[dvipsnames]{xcolor}

\usepackage{amsmath}
\usepackage{amssymb}
\usepackage{mathtools}
\usepackage{amsthm}
\usepackage{cancel}
\usepackage{enumitem}

\DeclareMathOperator*{\argmax}{argmax}
\setlist[enumerate,1]{leftmargin=2em, itemsep=1em, parsep=0pt}
\setlist[itemize,1]{leftmargin=2em}

\newcommand{\ours}{\text{FMLM}+}
\newcommand{\oursPR}{\text{PR}}
\newcommand{\mask}{[\texttt{MASK}]}

\usepackage[capitalize,noabbrev]{cleveref}
\usepackage{thmtools,thm-restate}
\crefformat{section}{Section~#2#1#3}
\crefformat{subsection}{Section~#2#1#3}
\crefformat{subsubsection}{Section~#2#1#3}
\crefformat{equation}{(#2#1#3)}
\crefrangeformat{equation}{(#3#1#4) to (#5#2#6)}
\crefmultiformat{equation}{(#2#1#3)}{ and (#2#1#3)}{, (#2#1#3)}{ and (#2#1#3)}
\crefname{appendix}{Appendix}{Appendices}
\Crefname{appendix}{Appendix}{Appendices}
\crefformat{appendix}{Appendix~#2#1#3}
\AddToHook{cmd/appendix/before}{%
\crefalias{section}{appendix}%
\crefalias{subsection}{appendix}%
\crefalias{subsubsection}{appendix}
}
\crefname{table}{Table}{Tables}
\crefname{figure}{Figure}{Figures}
\crefname{lemma}{Lemma}{Lemmas}
\crefname{proposition}{Proposition}{Propositions}
\crefname{assumption}{Assumption}{Assumptions}

\makeatletter
\def\thm@space@setup{%
  \thm@preskip=0.8em plus 0.2em minus 0.2em
  \thm@postskip=0.8em plus 0.2em minus 0.2em
}
\makeatother

\theoremstyle{plain}
\newtheorem{theorem}{Theorem}[section]
\newtheorem{proposition}[theorem]{Proposition}

\theoremstyle{definition}

\makeatletter
\let\c@table\c@figure
\let\c@algorithm\c@figure
\makeatother

\usepackage{pifont}
\usepackage{array}
\usepackage{xurl}

\newcommand{\R}{\mathbb{R}}
\newcommand{\E}{\mathbb{E}}
\newcommand{\calL}{\mathcal{L}}

\newcommand{\kl}[2]{\mathsf{KL}(#1\:\Vert\:#2)}

\newcommand{\pdistill}{p_{\mathrm{distill}}}

\newcolumntype{C}[1]{>{\centering\arraybackslash}p{#1}}
\newcolumntype{L}[1]{>{\raggedright\arraybackslash}p{#1}}

\usepackage[textsize=tiny]{todonotes}

\definecolor{d62728}{RGB}{214,39,40}
\definecolor{1f77b4}{RGB}{31,119,180}
\definecolor{9467bd}{RGB}{148,103,189}
\definecolor{FigNY}{HTML}{1a1a3e}    % dark purple for new-york mode
\definecolor{FigSD}{HTML}{b0a8d0}    % light purple for san-diego mode
\usepackage{colortbl}
\definecolor{pinegreen}{rgb}{0.0, 0.47, 0.44}
\definecolor{cornellred}{rgb}{0.7, 0.11, 0.11}
\definecolor{cadmiumgreen}{rgb}{0.0, 0.42, 0.24}
\definecolor{spirodiscoball}{rgb}{0.06, 0.75, 0.99}
\definecolor{blizzardblue}{rgb}{0.73, 0.96, 0.99}
\definecolor{aliceblue}{rgb}{0.91, 0.94, 0.97}
\definecolor{Red7}{rgb}{0.941, 0.243, 0.243}
\definecolor{Green7}{RGB}{55, 178, 77}
\definecolor{darkgray}{HTML}{555555}     % for caption metadata text

\definecolor{fig1grey}{HTML}{B0B0B0}        % MDM baseline
\definecolor{fig1ochrelight}{HTML}{E5B07E}  % FMLM scratch
\definecolor{fig1ochre}{HTML}{C97B3A}       % FMLM distill
\definecolor{fig1teal}{HTML}{1F4E4A}        % FMLM init - hero

\definecolor{tealtint}{HTML}{E8EFEE}        % ~5% teal on white
\definecolor{ochretint}{HTML}{FBF1E6}       % ~10% light ochre on white
\definecolor{greytint}{HTML}{F2F2F2}

\colorlet{darkteal}{tealtint}
\colorlet{darkblue}{tealtint}

\newtcolorbox{tealbox}{%
    enhanced,
    colback=tealtint,
    colframe=tealtint,
    boxrule=0pt,
    arc=4pt,
    left=8pt, right=8pt, top=8pt, bottom=8pt,
    boxsep=0pt,
    before skip=0pt, after skip=0pt,
}

\newtcolorbox{baselinebox}{%
    enhanced,
    colback=greytint,
    colframe=fig1grey,
    boxrule=0.4pt,
    arc=3pt,
    left=8pt, right=8pt, top=6pt, bottom=6pt,
    boxsep=0pt,
}

\newtcolorbox{samplebox}[1]{
	enhanced,
	colback=fig1teal!5!white,
	colframe=black!60,
	boxrule=0.6pt,
	arc=6pt,
	left=4pt, right=4pt, top=4pt, bottom=4pt,
	boxsep=3pt,
	before upper={\noindent\textbf{\small #1}\par\smallskip},
}

\DeclareUnicodeCharacter{FFFD}{ }

%% file: pseudocode_style.tex
\definecolor{codekey}{HTML}{1F4E4A}   % deep teal — keywords (hero)
\definecolor{codefunc}{HTML}{C97B3A}  % ochre — function calls
\definecolor{codecomment}{HTML}{888888} % grid grey — comments
\definecolor{codeop}{HTML}{B0B0B0}    % light grey — operators

\lstdefinelanguage{PythonFuncColor}{
  language=Python,
  keywordstyle=\color{codekey}\bfseries,
  commentstyle=\color{codecomment},
  stringstyle=\color{codefunc},
  showstringspaces=false,
  basicstyle=\ttfamily\small,
  literate=
    {+}{{\color{codeop}+ }}{1}
    {-}{{\color{codeop}- }}{1}
    {*}{{\color{codeop}* }}{1}
    {/}{{\color{codeop}/ }}{1}
    {sample_s_u_t}{{\color{codefunc}sample\_s\_u\_t}}{1}
    {randn_like}{{\color{codefunc}randn\_like}}{1}
    {random_mask}{{\color{codefunc}random\_mask}}{1}
    {one_hot}{{\color{codefunc}one\_hot}}{1}
    {argmax}{{\color{codefunc}argmax}}{1}
    {top_k}{{\color{codefunc}top\_k}}{1}
    {flowmap}{{\color{codefunc}flowmap}}{1}
    {randn}{{\color{codefunc}randn}}{1}
    {zeros}{{\color{codefunc}zeros}}{1}
    {ones}{{\color{codefunc}ones}}{1}
    {max}{{\color{codefunc}max}}{1}
    {sg}{{\color{codefunc}sg}}{1}
    {where}{{\color{codefunc}where}}{1}
}

%% file: abstract.tex
Non-autoregressive generation offers a powerful paradigm for iterative refinement, allowing models to recursively critique, erase and regenerate arbitrary subsets of tokens.
However, existing non-autoregressive models fail to realize this potential. Masked Diffusion Models (MDMs) suffer from factorization error, causing sample quality to collapse when generating multiple tokens simultaneously.
Flow Map Language Models (FMLMs) circumvent this bottleneck via joint sequence transport for excellent few-step generation, but sacrifice the inference-time flexibility of MDMs.
We introduce $\ours$, a framework that bridges this gap by equipping FMLM with masking-style noise schedules.
While generating the full sequence in a single step, $\ours$ simultaneously scores the global consistency of each token \emph{a posteriori}.
We leverage this to introduce \emph{Posterior Refinement}, a novel inference-time refinement strategy that enables the model to adaptively self-correct its outputs, matching the performance of discrete baselines with $32 \times$ fewer NFEs.
Across diverse benchmarks, we demonstrate that $\ours$ with Posterior Refinement improves the speed--quality tradeoff over both MDM and FMLM families, providing a scalable foundation for high-fidelity language modeling.

\vspace{1.5em}%
\textbf{\sffamily\bfseries Code: }\url{https://github.com/MananAg007/posterior-refinement}

\textbf{\sffamily\bfseries Project Page: }\url{https://posterior-refinement.github.io/}

\textbf{\sffamily\bfseries Correspondence: } \href{mailto:mananaga@cs.cmu.edu,sheels@cs.cmu.edu}{\color{black}{\{mananaga, sheels\}@cs.cmu.edu}}

%% file: intro.tex
\vspace{-0.4em}
\section{Introduction}
\label{sec:intro}
\vspace{-0.4em}

\begin{wrapfigure}[17]{r}{0.42\textwidth}
    \vspace{-0.3em}
	\centering
	\includegraphics[width=\linewidth]{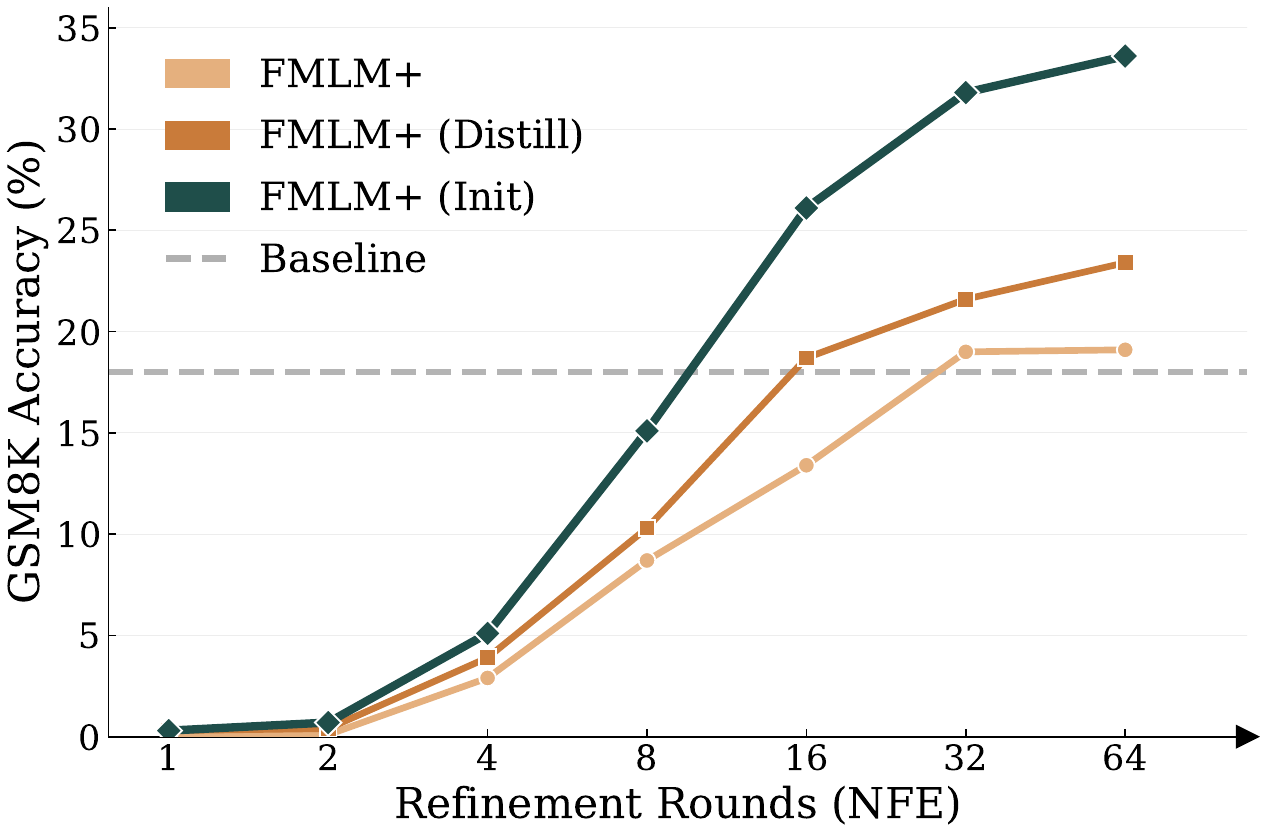}
	\caption{
		\textbf{Iterative refinement.} $\ours$ \emph{unlocks} self-correction capabilities, outperforming all diffusion baselines with 1024 function evaluations using as few as $32$ rounds of Posterior Refinement.
    }
	\label{fig:gsm8k_iterative_refinement}
\end{wrapfigure}

Contemporary language models are predominantly autoregressive (AR), necessitating $L$ sequential forward passes to generate a sequence of length $L$~\citep{achiam2023gpt, team2023gemini, guo2025deepseek}.
While this paradigm has scaled remarkably well, its inherent sequential bottleneck makes inference slow and memory-bound~\citep{gu2017non, 3691938.3691945}. These inefficiencies cascade into downstream applications like inference-time scaling, reasoning, and reinforcement learning, which are inference-bottlenecked, ultimately hindering scalability~\citep{guo2025deepseek}.

Masked diffusion models (MDMs) have emerged as the leading non-autoregressive alternative for language modeling~\citep{austin2021structured, lou2023discrete, sahoo2024simple}, with recent efforts scaling them to billions of parameters~\citep{nie2025large, ye2025dream, khanna2025mercury, googledeepmind2025gemini, song2025seed}.
In contrast to fixed causal orderings, MDMs model any-order factorizations of the sequence by training on arbitrary masking patterns.
During inference, this permits adaptive unmask-and-remask decoding strategies that commit high-confidence tokens first~\citep{Chang2022MaskGITMG, zheng2023rdm, ye2024beyond}.
Such a flexible generative procedure has proven useful in planning-oriented tasks~\citep{kim2025trainworstplanbest}.

Complementing discrete formulations, a distinct line of work explores diffusion and flows within continuous representations for non-autoregressive language modeling, spanning fixed embeddings~\citep{gong2022diffuseq, yuan2022seqdiffuseq, elf2026}, learned latent spaces~\citep{dieleman2022continuous, chen2026langflowcontinuousdiffusionrivals}, and one-hot relaxations~\citep{mahabadi2024tess, tae2025tess, davis2025generalised}.
Recently, continuous-space approaches have been shown to be competitive with their discrete-diffusion counterparts in language modeling~\citep{lee2025flowmaplm, chen2026langflowcontinuousdiffusionrivals, elf2026, yang2026continuousdiffusionscalescompetitively}.
In particular, Flow Map Language Models (FMLMs)~\citep{lee2025flowmaplm} train flow maps~\citep{boffi2025flowmapmatchingstochastic} over one-hot embeddings in Euclidean space, resolving the inherent factorization error in MDMs~\citep{deschenaux2024beyond, zheng2024masked, kang2025parallelbench} by learning a joint transport from Gaussian noise to one-hot vectors.
This allows FMLMs to substantially improve the quality of few-step language generation, including generating text in just a \emph{single} function evaluation~\citep{lee2025flowmaplm}.

Beyond fast, parallel inference, non-autoregressive generation offers a more profound theoretical advantage over AR modeling: the capacity for \emph{iterative refinement} of generated text.
While traditional causal language models are fundamentally constrained by a one-way, left-to-right generation, an iterative refinement paradigm allows the model to globally critique, erase, and regenerate arbitrary subsets of tokens to continually improve its output.
This capacity directly mirrors human cognition; when composing complex text, humans rarely write perfectly in a single, unyielding left-to-right pass. Instead, we efficiently formulate a rough draft to capture the global context, and then iteratively assess, delete, and rewrite specific parts of text while keeping the overall picture in context~\citep{flower1981cognitive}.

However, existing non-autoregressive paradigms fail to fully realize this potential.
While MDMs offer flexible, any-order inference, their quality rapidly degrades when multiple tokens are unmasked simultaneously, rendering few-step generation infeasible~\citep{deschenaux2024beyond, zheng2024masked, kang2025parallelbench, lee2025flowmaplm}.
On the contrary, FMLMs resolve this factorization error by learning the joint transport from Gaussian noise to one-hot vectors, yet they lack the inference-time flexibility native to MDMs because they evolve all tokens simultaneously.
We argue that the respective strengths of these two paradigms constitute the core primitives for effective iterative refinement, which requires both --- efficient drafting of coherent multi-token blocks and the flexibility to retain arbitrary token subsets while regenerating the rest.

Motivated by this insight, we introduce $\ours$, a unified framework that equips FMLMs with masking-style noise schedules.
By coupling single-step sequence generation from FMLMs with flexible conditioning from MDMs, this architecture naturally supports iterative refinement.
Crucially, because the entire sequence is transported jointly, our approach yields an \emph{a posteriori} confidence score for each token \emph{conditioned on the full draft}.
This contrasts sharply with the \emph{a priori} confidence derived from independent token-wise marginals in MDMs.
We leverage this property to develop \emph{Posterior Refinement} ($\oursPR$) --- an iterative refinement strategy that recursively retains high-confidence tokens and regenerates the rest to keep improving sample quality with each additional round of refinement (\cref{fig:gsm8k_iterative_refinement}).
Subsequently, $\ours$ with Posterior Refinement establishes a new state-of-the-art that consistently outperforms both MDM and FMLM baselines across all considered benchmarks with far fewer function evaluations (NFEs). 

Our contributions are as follows:
\begin{enumerate}
    \item \textbf{A unifying framework.} We introduce $\ours$, a generalization of FMLMs to masking-style noise schedules that combines the few-step efficiency of FMLMs with the any-order flexibility of MDMs. We further establish a connection between MDMs and $\ours$, providing a direct path to scale $\ours$ leveraging advances in large-scale MDM pre-training~\citep{nie2025large, ye2025dream}.
    \item \textbf{Posterior refinement.} We introduce \emph{Posterior Refinement}, an \emph{after-the-fact} confidence-based strategy that enables $\ours$ to recursively self-correct its output while generating all remaining tokens with \textit{single} function evaluations (\cref{fig:sudoku_iterative_refinement}). We identify how existing a-priori confidence heuristics of MDMs fail on simple toy problems and demonstrate how Posterior Refinement is able to rectify these shortcomings.
    \item \textbf{Extensive evaluation.} We demonstrate that $\ours$ navigates the speed--quality trade-off more effectively than both MDM and FMLM families across four benchmarks: TinyStories~\citep{eldan2023tinystoriessmalllanguagemodels}, OpenWebText~\citep{Gokaslan2019OpenWeb}, GSM8K~\citep{cobbe2021trainingverifierssolvemath}, and Sudoku, matching or surpassing the strongest baselines with up to $32\times$ fewer NFEs.
\end{enumerate}

%% file: background.tex
\newpage
\begin{figure}[!htb]
	\centering
	\includegraphics[width=0.95\linewidth]{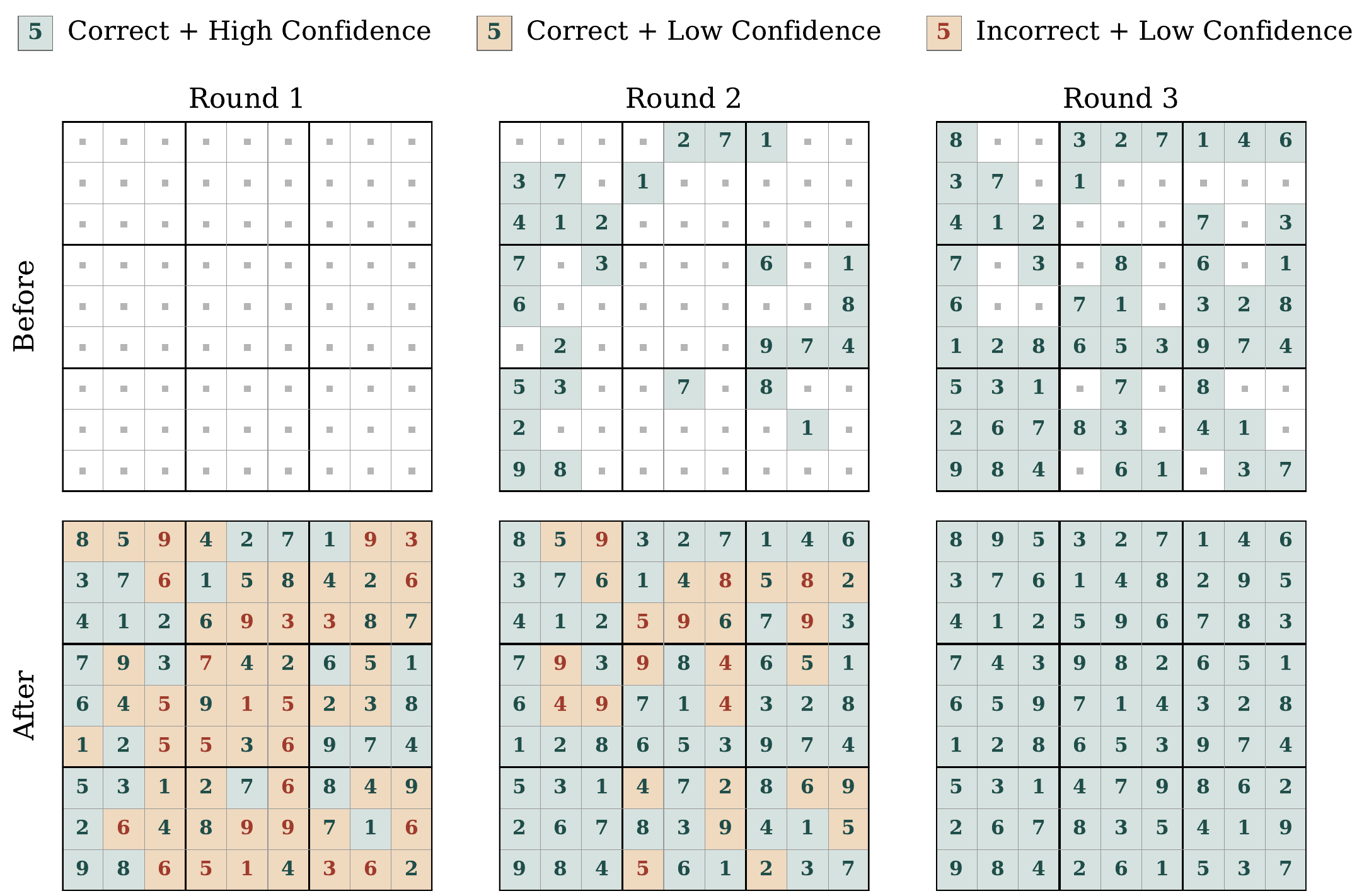}
	\caption{
    	\textbf{Posterior Refinement with $\ours$.} 
        By evaluating token consistency against the full generated sequence, $\ours$ effectively identifies its own errors. Crucially, incorrect tokens consistently form a subset of low-confidence generations, allowing Posterior Refinement to reliably filter and revise errors.
    }
	\label{fig:sudoku_iterative_refinement}
\end{figure}
\section{Background}
\label{sec:background}

Let $\mathcal{V}$ denote a vocabulary, with $|\mathcal{V}| = V$, and let ${\bf y} = ({\bf y}^l)_{l=1}^L \in \mathcal{V}^L$ be a sequence of length $L$ drawn from a data distribution $p({\bf y})$.

\paragraph{Autoregressive Models.} The AR paradigm models the data distribution through a left-to-right factorization
$p({\bf y})=p({\bf y}^1)p({\bf y}^2|{\bf y}^1)\hdots p({\bf y}^L|{\bf y}^{<L})$,
and learns the conditional distributions $p({\bf y}^l \mid {\bf y}^{<l})$ over tokens given their preceding context~\citep{jordan1986attractor, elman1990finding, bengio2003neural}.
Generation proceeds sequentially by sampling one token at a time conditioned on previously generated tokens.

\paragraph{Masked Diffusion Models.}
\label{par:MDM_bg}
Masked Diffusion Models (MDMs)~\citep{austin2021structured, lou2023discrete, sahoo2024simple} offer an alternative paradigm for discrete generative modeling, parameterizing marginals over clean tokens given a partially masked input.
MDMs extend the vocabulary $\mathcal{V}$ with an extra $\mask$ token. A monotonic fraction $\alpha(t)$ of clean tokens ${\bf y}$ is masked to produce ${\bf y}_t$.
The token-wise marginal $p_\theta({\bf y}^l \mid {\bf y}_t)$ is then trained via a classification loss over the masked positions,
\begin{align}
    \label{eq:bg:mdm_loss}
    \mathcal{L}_{\text{MDM}}(p_\theta) = -\E_t \left[ w(t) \sum_{l=1}^L \mathbf{1}[{\bf y}_t^l = \texttt{MASK}] \log p_\theta({\bf y}^l \mid {\bf y}_t)\right],
\end{align}
where $w(t) > 0$ is some weighting function.

At inference, MDMs start from a fully masked sequence and generate tokens via an iterative unmask--remask procedure; at each step they first categorically sample from the marginal distribution for each token independently and then remask some tokens, repeating until all tokens are unmasked. A natural decoding strategy for MDMs is to commit tokens in decreasing order of confidence, usually using the maximum marginal probability $\max_{v} p_\theta({\bf y}^l = v \mid {\bf y}_t)$~\citep{zheng2023rdm, ye2024beyond}.
\paragraph{Flow Matching and Flow Maps.}
Flow matching~\citep{lipman2022flow, liu2022flow, albergo2023stochastic} learns a continuous transport from a simple source distribution (e.g., Gaussian noise) to the target data distribution by regressing the velocity field of a probability flow ODE along a prescribed source–target interpolant.
We define ${\bf x}_1 = f({\bf y})$ as the one-hot representation of clean data, and sample ${\bf x}_0 \sim \mathsf{N}(0, I)$ as isotropic Gaussian noise of matching shape to construct the linear interpolant in Euclidean space:
\begin{align}
    \label{eq:bg:interpolant}
    I_t \coloneqq (1-t) {\bf x}_0 + t {\bf x}_1, \qquad t \in [0,1].
\end{align}
This interpolant admits a deterministic probability flow ODE
\begin{align}
    \label{eq:bg:flow}
    \dot{\bf x}_t = b_t({\bf x}_t), \qquad b_t({\bf x}) = \E[\dot{I}_t \mid I_t = {\bf x}] = \E[{\bf x}_1 - {\bf x}_0 \mid I_t = {\bf x}],
\end{align}
whose solution transports samples from the source distribution to the target data distribution~\citep{lipman2022flow, liu2022flow, albergo2023stochastic}.
In practice, we parameterize the flow velocity $b_t$ with a neural network, and inference requires a numerical integration of the probability flow ODE to generate samples.

The flow map $X_{s,t}: [0,1]^2 \times \R^{L \times V} \to \R^{L \times V}$ associated with the velocity field $b_t$ is defined as the solution operator that transports the state at time $s$ to the state at time $t$~\citep{boffi2025flowmapmatchingstochastic}:
\begin{align}
    \label{eq:bg:flow_map}
    X_{s,t}({\bf x}_s) = {\bf x}_s + \int_s^t b_\tau({\bf x}_\tau) d\tau = {\bf x}_s + (t-s)v_{s,t}({\bf x}_s),
\end{align}
where $v_{s,t}({\bf x}_s)$ is the average velocity (mean flow) between times $s$ and $t$~\citep{geng2025mean}.
In the limit $s \to t$, the average velocity recovers the instantaneous velocity, $\lim_{s \to t} v_{s,t}({\bf x}) = b_t({\bf x})$.
Crucially, this allows us to arbitrarily jump along the flow, including generating from the target distribution in a \emph{single} function evaluation ($X_{0,1}$)~\citep{frans2024one, geng2025mean}.

\paragraph{Flow Map Language Models.}
\label{par:FMLM_bg}
Flow map language models (FMLM)~\citep{lee2025flowmaplm} enable few-step language generation by parameterizing the \emph{two-time denoiser} (TTD) $\delta_{s,t}: [0,1]^2 \times \R^{L \times V} \to \R^{L \times V}$,
\begin{align}
    \label{eq:bg:ttd}
    \delta_{s,t}({\bf x}) &\coloneqq {\bf x} + (1-s) v_{s,t}({\bf x}),
\end{align}
which extrapolates the flow-map transport to $t = 1$ and satisfies the boundary identity $\delta_{t,t}({\bf x}) = \E[{\bf x}_1 \mid I_t = {\bf x}]$~\citep{pixelmeanflows, roos2026categorical, lee2025flowmaplm}.
The TTD takes values on the per-token simplex, $\delta_{s,t}({\bf x})^l \in \Delta^{V-1}$ for $l = 1, \dots, L$, and obeys the semigroup identity
\begin{align}
    \label{eq:bg:ttd_semigroup}
    \delta_{s,t}({\bf x}) = \gamma\, \delta_{s,u}({\bf x}) + (1-\gamma)\, \delta_{u,t}(X_{s,u}({\bf x})),
    \qquad
    \gamma = \frac{(1-t)(u-s)}{(1-u)(t-s)},
\end{align}
for every $0 \le s \le u \le t \le 1$~\citep{lee2025flowmaplm}.
In practice, we parameterize ${\delta}_{s,t}$ with a token-wise softmax head and train it via a token-wise cross-entropy objective,
\begin{align}
    \label{eq:bg:ttd_ce}
    \calL_{\mathsf{FMLM}}({\delta})
    &\coloneqq
    \underbrace{
    \E_t\E_{{\bf x}_0,{\bf x}_1}\left[ \text{CE}\left( {\delta}_{s,s}(I_s), {\bf x}_1\right)
    \right]}_{\calL_{\mathrm{diag}}}
    +
    \underbrace{
    \E_{s,u,t}\E_{{\bf x}_0,{\bf x}_1}\left[ \text{CE} \left( {\delta}_{s,t}(I_s), \bar{\delta}_{s,t} \right)
    \right]}_{\calL_{\mathrm{off}}},
\end{align}
where $X_{s,u}({\bf x}) \coloneqq \frac{1-u}{1-s}\,{\bf x} + \frac{u-s}{1-s}\,\delta_{s,u}({\bf x})$, $\bar{\delta}_{s,t} \coloneqq \text{stopgrad}\left({\gamma\, \delta_{s,u}(I_s) + (1-\gamma)\, \delta_{u,t}(X_{s,u}(I_s))}\right)$.
Intuitively, this loss trains the two-time denoiser to reconstruct clean data along the diagonal $s = t$ ($\mathcal{L}_{\mathrm{diag}}$) and to progressively distill a larger off-diagonal jump ($s \rightarrow t$) using two smaller jumps ($s \rightarrow u, u \rightarrow t$) as a teacher ($\mathcal{L}_{\mathrm{off}}$).

During inference, we can arbitrarily jump along the flow by recovering the flow map from the TTD,
\begin{align}
    \label{eq:bg:ttd_recovery}
    X_{s,t}({\bf x}) = \frac{1-t}{1-s}\,{\bf x} + \frac{t-s}{1-s}\,\delta_{s,t}({\bf x}).
\end{align}
This allows us to generate samples in as few as a \emph{single} function evaluation ($\delta_{0,1}$).

%% file: algo.tex
\newpage
\section{FMLM+}
\label{sec:algo}

Continuous diffusion and flow-based models typically denoise the full sequence in parallel~\citep{ho2020denoising, song2020score, lipman2022flow}.
While effective for unconditional generation, this formulation has no native mechanism for conditioning on clean context.
Unlike autoregressive or masked models, where conditioning is a structural primitive, continuous-flow models often require architectural modifications, such as cross-attention or auxiliary guidance models, to incorporate conditional information~\citep{ramesh2022hierarchicaltextconditionalimagegeneration, blackforestlabs2024flux1dev}.
To address this limitation in flow-based language modeling, we introduce $\ours$, a generalization of FMLMs that natively learns to generate under diverse conditioning patterns (\Cref{sec:algo:training}).
We show that this allows for flexible decoding strategies at inference time (\Cref{sec:algo:inference}), such as FMLM- and MDM- like inference, and further unlocks novel iterative refinement schemes (\Cref{sec:algo:pr}).

\subsection{Training}
\label{sec:algo:training}

The key idea underlying $\ours$ is to generalize FMLMs by lifting the noising process to a masking-style design, preserving a set of clean tokens ${\bf y}^{\mathcal{C}}$ at context positions $\mathcal{C} \subset \{1, \cdots, L\}$ while noising the rest:
\begin{align}
    \label{eq:algo:interpolant}
    I_t^l \coloneqq
    \begin{cases} 
        \mathbf{x}_1^l, & \text{if } l \in \mathcal{C} \\
        (1 - t) \mathbf{x}_0^l + t \mathbf{x}_1^l, & \text{else } (l \in [L]\setminus\mathcal{C}).
    \end{cases}
\end{align}
This enables $\ours$ to condition on arbitrary clean context without additional architectural components.
In practice, this is implemented by converting the shared time conditioning in the DiT architecture~\citep{peebles2023scalable} to a per-position conditioning with $t=1$ for clean positions.
For brevity, we drop this implicit conditioning on ${\bf y}^\mathcal{C}$ throughout the remainder of the paper, writing $\delta_{s,t}(\cdot)$ instead of $\delta_{s,t}(\cdot | {\bf y}^\mathcal{C})$.

We train the two-time denoiser $\delta_{s,t}$ with the following progressive self-distillation objective:
\begin{align}
    \label{eq:algo:training}
    \calL_{\ours}(\delta) = \E_{\mathcal{C}}\Big[
    \underbrace{\E_{t}\,\E_{{\bf x}_0,{\bf x}_1}\!\big[\mathrm{CE}(\delta_{s,s}(I_s),{\bf x}_1)\odot m\big]}_{\calL_{\mathrm{diag}}}
    +
    \underbrace{\E_{s,u,t}\,\E_{{\bf x}_0,{\bf x}_1}\!\big[\mathrm{CE}(\delta_{s,t}(I_s),\bar{\delta}_{s,t})\odot m\big]}_{\calL_{\mathrm{off}}}
    \Big].
\end{align}
Compared to the vanilla FMLM objective \cref{eq:bg:ttd_ce}, we have an additional expectation over $\mathcal{C}$, the set of context tokens, and use $m\in\{0,1\}^L$ with $m^l=\mathbf{1}[\,l\notin\mathcal{C}\,]$ to mask the loss from tokens that are already clean.
To avoid extra complexity, we drop the time reparameterization trick from~\citet{lee2025flowmaplm} and use the self-distillation objective~\eqref{eq:algo:training} instead of the two-stage training in~\citet{lee2025flowmaplm} for all $\ours$ runs.

\paragraph{Leveraging Pretrained MDMs.}
\label{sec:algo:mdlm_boundary}
To accelerate the training of $\ours$, we propose a novel method to leverage pretrained MDMs based on the following observation.
At $s=t=0$, the off-diagonal term $\mathcal{L}_{\mathrm{off}}$ in the $\ours$ training objective \cref{eq:algo:training} vanishes and the diagonal term $\mathcal{L}_{\mathrm{diag}}$ reduces to:
\begin{align}
    \label{eqn:mdm_equivalence}
    \E_\mathcal{C} \left[ \E_{{\bf x}_0, {\bf x}_1}\left[ \text{CE}(\delta_{0,0}({\bf x}_0), {\bf x}_1) \odot m \right] \right].
\end{align}
This token-wise classification objective predicts the ground-truth tokens using only the context ${\bf y}^\mathcal{C}$ and pure Gaussian noise.
Because the model must rely exclusively on the context, this formulation is equivalent to MDM training objective \cref{eq:bg:mdm_loss}, which similarly predicts targets from ${\bf y}^\mathcal{C}$ and $\mask$ tokens.

We use this equivalence to accelerate $\ours$ training in two ways:

\begin{enumerate}
    \item \textbf{Distillation.} We utilize an MDM as a teacher model, training the $\ours$ student to match the MDM's token probabilities for the same context. This approach reduces gradient variance and accelerates convergence~\citep{Hinton2015DistillingTK}. Concretely, we use:
    \begin{align}
        \label{eq:distill_loss}
        \mathcal{L}_{\text{Distill}}(\delta) = \mathcal{L}_{\text{\ours}}(\delta) + \lambda \E_{\mathcal{C}}\E_{{\bf x}_0,{\bf x}_1} \left[ \text{CE}(\delta_{0,0}(I_0), p_\theta({\bf y}_0))\right] .
    \end{align}
    \item \textbf{Initialization.} We initialize $\ours$ weights directly from a pretrained MDM checkpoint, leveraging the similarity between the masked diffusion and continuous flow objectives~\citep{transferlearning}. We show how the DiT architecture can mechanistically adapt MDM weights to handle noisy input in~\cref{sec:exp:dit}.
\end{enumerate}

\begin{figure}[!htb]
	\centering
	\includegraphics[width=0.95\linewidth]{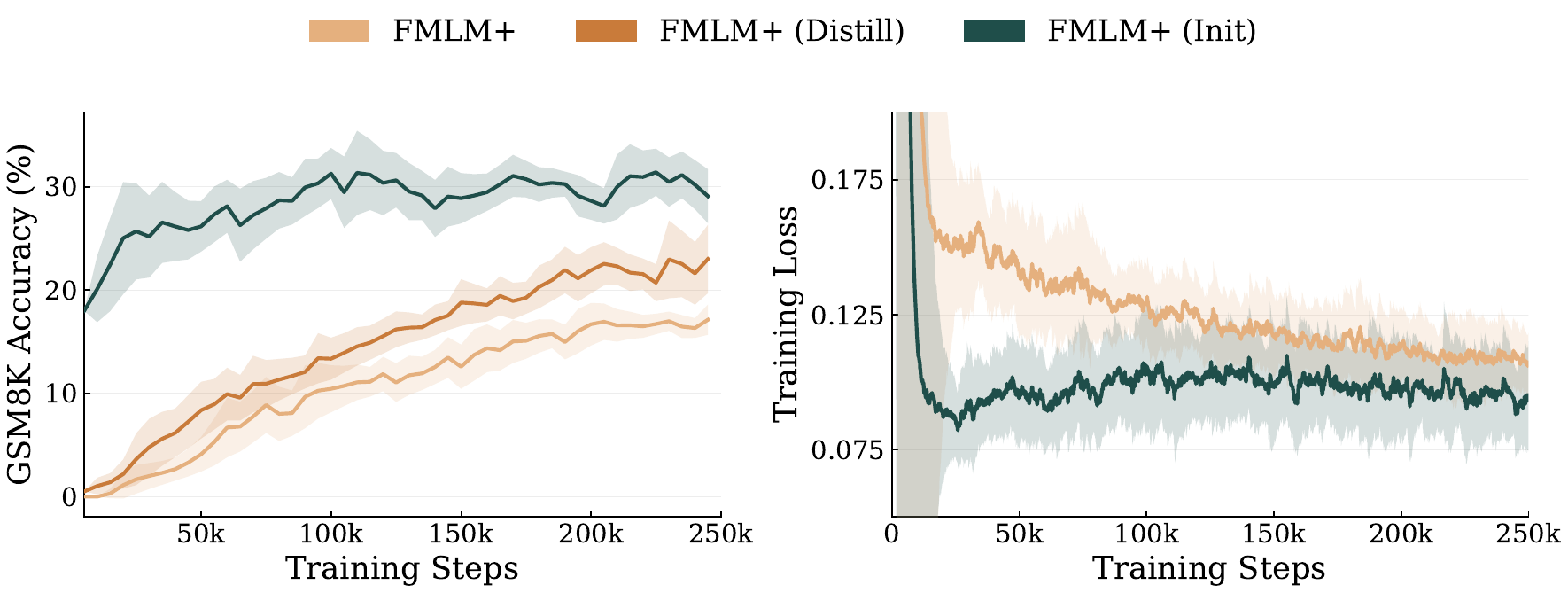}
	\caption{
        \textbf{Improved Training Techniques for $\ours$.} \emph{Left:} downstream accuracy ($\uparrow$) with 32 rounds of Posterior Refinement; both teacher-based variants outperform training from scratch. \emph{Right:} training loss ($\downarrow$); warm-started training converges faster and achieves a better optimum. Both metrics are smoothened with an exponential moving average ($\lambda = 0.99$).
    }
	\label{fig:fmlm_training}
\end{figure}

Empirically, both distillation and initialization strategies significantly improve training efficiency and final performance.
We find that warm starting from pre-trained MDM weights helps the model converge to a better optimum than training from scratch, highlighting the role of stochasticity and parameter initialization in training FMLMs (\cref{fig:fmlm_training}).

\subsection{Inference}
\label{sec:algo:inference}

\begin{wrapfigure}[17]{r}{0.45\textwidth}
  \centering

  \begin{minipage}{0.45\textwidth}
    \refstepcounter{algorithm}
    \hrule height 0.8pt
    \vspace{2pt}
    \noindent\textbf{Algorithm~\thealgorithm} $\ours$ Inference.
    \label{alg:inference_code}
    \vspace{2pt}
    \hrule height 0.4pt

    \footnotesize
    \input{pseudocode_inference}

    \hrule height 0.8pt
  \end{minipage}

  % \vspace{-0.5em}
\end{wrapfigure}

The core advantage of training $\ours$ on diverse conditioning patterns is that it allows us to flexibly choose between various decoding strategies during inference.
As a simple example, we can recover the standard FMLM inference by defining a time grid $0 = t_0 < t_1 < \cdots < t_N = 1$ and denoising all unknown tokens simultaneously, jumping in $N$ steps along the flow trajectory via~\cref{eq:bg:ttd_recovery}.

More importantly, this allows us to construct a mechanism for iterative refinement as a loop that performs parallel generation followed by committing a subset of tokens (\Cref{alg:inference_code}).
At each iteration $i$, we propose a candidate $\hat{\bf x}_1$ via FMLM inference.
Then, we select a subset $\mathcal{U}$ to retain from this candidate sequence and round them into discrete tokens via an $\argmax$ operation. Finally, we add them to our context set $\mathcal{C} \coloneqq \mathcal{C} \cup \mathcal{U}$ and re-noise all remaining tokens $[L] \setminus \mathcal{C}$ for the next iteration.
This, for example, allows for block-diffusion-style inference~\citep{arriola2025block} by choosing $\mathcal{U}$ as the next contiguous block of $B$ tokens.

We may further perform MDM-like, confidence-based inference by noting the equivalence between $\delta_{0,0}$ and MDMs \eqref{eqn:mdm_equivalence}.
Consequently, by sampling tokens from $\delta_{0,0}({\bf x}_0)$ and
by prioritizing positions with sharper $\delta_{0,0}({\bf x}_0)$, we recover MDM's confidence-based sampling~\citep{zheng2023rdm, ye2024beyond, kim2025trainworstplanbest}.
We note that this approach, applied either to $\ours$ with $\delta_{0,0}$ or MDMs, generates tokens and estimates confidence \emph{a priori}, using only the partial context available at that step and without knowing what tokens will be generated at other positions.

\subsection{Posterior Refinement}
\label{sec:algo:pr}

Beyond recovering existing inference schemes as in \Cref{sec:algo:inference}, we show that $\ours$ enables a novel class of after-the-fact refinement strategies native to neither the MDM nor the FMLM paradigms.

During inference, $\ours$ steps along the flow trajectory to simultaneously decode all positions.
At the final integration step, because the model operates on a nearly complete sequence, the model implicitly evaluates the conditional probability of each token given the rest of the generated text. The one-step generation $\delta_{0,1}$ also captures this behavior since, ideally, the flow from $t=0$ to $t=1$ is equivalent to composing the transport from $0$ to $t$ and the flow map jump from $t$ to $1$.
This enables us to compute token confidence conditioned on the fully generated sequence, effectively evaluating each token's fit within the global context.
Reflecting the \emph{a-posteriori} nature of these estimates, we term our approach \emph{Posterior Refinement} (PR).

PR instantiates the inference loop in \cref{sec:algo:inference}, choosing the retained set $\mathcal{U}$ using the confidence of the predicted endpoint.
Concretely, at each iteration, we use the final flow map jump $\hat{\bf x}_1 = \delta_{t_{N-1},1}({\bf x}_{t_{N-1}})$ to set
\begin{align}
    \label{eq:algo:pr_update_set}
    \mathcal{U} = \{ l \in [L] \setminus \mathcal{C} : p_{\text{max}}^l > p\},
\end{align}
where $p_{\text{max}}^l \coloneqq \max_v (\hat{\bf x}_1^{(i)})^{l,v}$ is a measure of the a-posteriori confidence and $p$ is some threshold for selecting confident tokens.
To avoid stalling when no position clears the threshold, we commit at least $k$ positions per round, falling back to the top-$k$ by confidence.

\paragraph{Interpretation.}
\label{sec:algo:interpretation}
\begin{wrapfigure}[21]{r}{0.53\textwidth}
    \vspace{-2em}
    \centering
    \includegraphics[width=\linewidth]{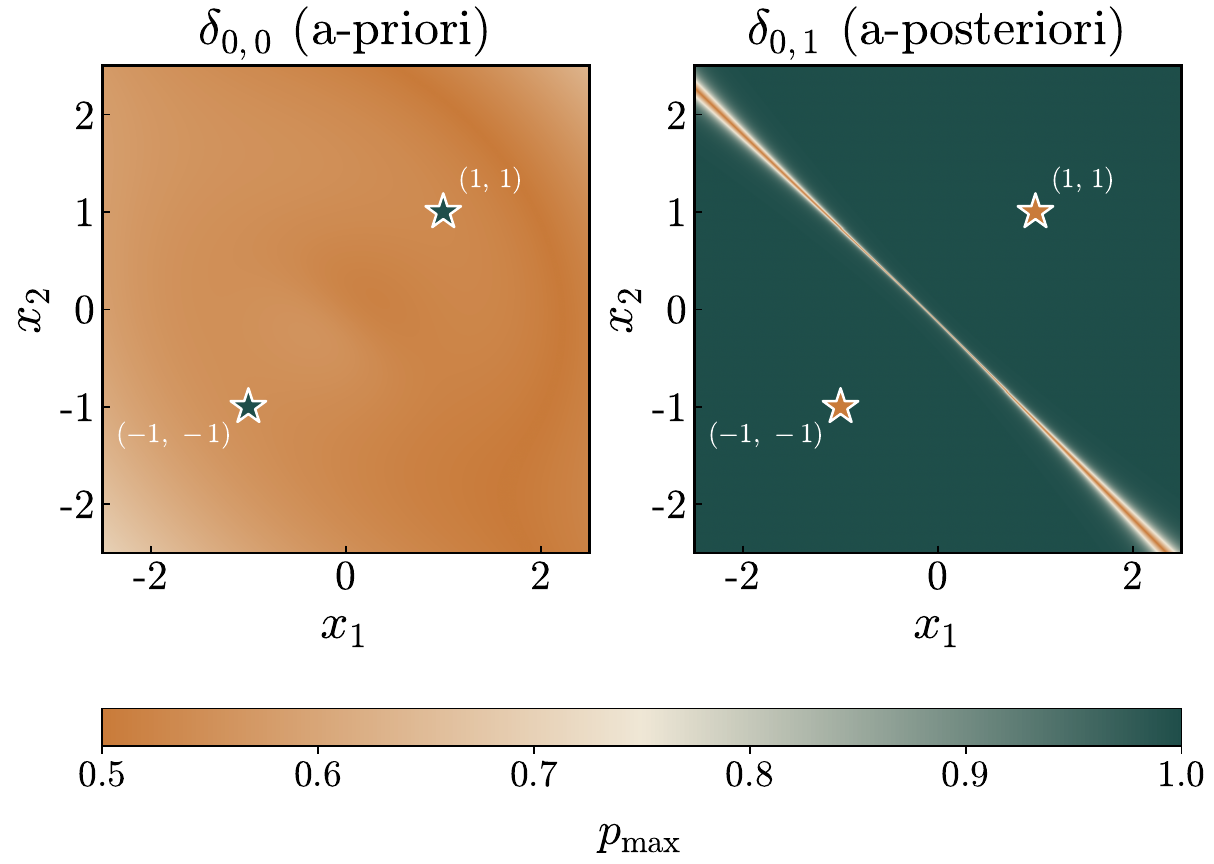}
	\caption{
    \textbf{Two types of confidence.} On a 2-mode toy problem, the one-step FMLM $\delta_{0,1}$ outputs high $p_{\max}$ for most inputs. Low-confidence regions concentrate near decision boundaries, where the endpoint is ambiguous, indicating $\delta_{0,1}$ captures the epistemic confidence of the model. In contrast, $\delta_{0,0}$ is flat for all inputs, reflecting its aleatoric nature.
    }
    \label{fig:denoiser_field}
    \vspace{-0.5em}
\end{wrapfigure}
An ideal FMLM would perfectly sample from its training distribution, always producing one-hot vectors.
In practice, however, the model outputs categorical distributions that are not strictly one-hot. To resolve this,~\citet{lee2025flowmaplm} project each token to the nearest vertex of the probability simplex via the argmax operator.
We interpret the \emph{rounding error} of this projection as a proxy for the model's confidence.
This interpretation aligns with our empirical observations in Sudoku, where we find that incorrect tokens consistently exhibit high rounding error (\cref{fig:sudoku_iterative_refinement}).

Concretely, this rounding error can be quantified using either the $\ell_1$ distance or the Kullback-Leibler (KL) divergence between the generated distribution and the target one-hot distribution.
Under both metrics, the rounding error is a monotonically decreasing function of $p_{\max}^l$.
Consequently, applying a threshold to $p_{\max}^l$ is mathematically equivalent to thresholding the rounding error under these metrics (\cref{prop:rounding_monotone}).

\paragraph{Why a-priori confidence fails.}
Because both the MDM denoiser $p_\theta$ and the FMLM denoiser $\delta_{0,0}$ predict token-wise marginals, independently sampling high-confidence tokens from these marginals cannot circumvent factorization error.
While this heuristic performs well on unimodal target distributions such as unique Sudoku solutions, it catastrophically fails on multi-modal distributions (\cref{tab:uncond-sudoku-nfe}).
We identify that MDMs, and equivalently $\delta_{0,0}$, conflate aleatoric uncertainty, the inherent randomness of the data distribution, with epistemic uncertainty, the model's internal confidence~\citep{hullermeier2021aleatoric}.
We hypothesize that by failing to perfectly learn the joint transport from Gaussian noise to clean one-hot data, the FMLM trajectory is able to surface the true epistemic uncertainty of the model through errors in its predictions (\cref{fig:denoiser_field}).

%% file: pseudocode_inference.tex
% Posterior Refinement (inference). Style from pseudocode_style.tex.
\begin{lstlisting}[language=PythonFuncColor, escapechar=`]
# delta: Two Time Denoiser
# C: commit criterion; R rounds
x = randn(L, V)
s, t = zeros(L), ones(L)
for _ in range(R):
    x_1    = delta(x, s, t, flow_steps)
    tokens = x_1.argmax(axis=1)

    commit = C(tokens, x_1)

    x = where(commit,
        one_hot(tokens), randn(L, V))
    s = where(commit, 1, 0)
return x_1
\end{lstlisting}

%% file: results.tex
% \vspace{em}
\newpage
\section{Experiments}
\label{sec:results}

We evaluate $\ours$ on four benchmarks: TinyStories~\citep{eldan2023tinystoriessmalllanguagemodels} and OpenWebText~\citep{Gokaslan2019OpenWeb} for unconditional language generation, TinyGSM~\citep{liu2023tinygsmachieving80gsm8k} for conditional reasoning, and Sudoku for structured constraint satisfaction.
Following~\citet{kim2025trainworstplanbest}, we train on code solutions from the synthetic TinyGSM corpus and evaluate zero-shot on the GSM8K~\citep{cobbe2021trainingverifierssolvemath} test split.
We adopt a DiT~\citep{peebles2023scalable} backbone for all baselines; further details on model architecture and datasets are provided in~\cref{sec:app:setup}.

\paragraph{Training.} We train $\ours$ with AdamW under bfloat16 mixed precision, optimizing the Two-Time Denoiser self-distillation objective in~\cref{eq:algo:training}.
For the MDM-distillation and -initialization experiments (\Cref{sec:exp:dit}), we additionally train an MDLM model~\citep{sahoo2024simple} with the same architecture to act as a teacher and an initialization checkpoint respectively.
\Cref{sec:app:training} specifies all task-specific training hyperparameters.

\paragraph{Evaluation.}
On TinyStories and OpenWebText, we report generative perplexity under GPT-2 Large~\citep{radford2019language} together with per-sample unigram entropy; where low entropy signals repetition collapse~\citep{zheng2024masked}.
We report solve-rate accuracy using a sampled Python solution graded by execution for GSM8K and a valid completion respecting the clue set for Sudoku.

\subsection{Posterior Refinement}
\begin{wraptable}[14]{r}{0.55\linewidth}
    % \vspace{-1em}
    \centering
    \footnotesize
    \setlength{\tabcolsep}{8pt}
    \begin{tabular}{lccccc}
        \toprule
         & \multicolumn{5}{c}{{\textbf{Unconditional Sudoku}} (\%, $\uparrow$)} \\
        \cmidrule(lr){2-6}
        Method & 1 & 3 & 9 & 27 & 81 \\
        \midrule
        MDLM (Ancestral) & 0 & 0 & 0 & 7.7 & 33.8 \\
        MDLM (Confidence) & 0 & 0 & 0 & 0.0 & 97.2 \\
        \ours{} ($\delta_{\bf 0, \bf 0}$) & 0 & 0 & 0 & 0.0 & 98.8 \\
        \rowcolor{darkteal} \ours{} ($\delta_{\bf 0,\bf 1}$) & \textbf{0} & \textbf{22.2} & \textbf{87.1} & \textbf{98.8} & \textbf{99.6} \\
        \bottomrule
    \end{tabular}
    \caption{ 
    \textbf{A-priori confidence fails on multi-modal data.} On unconditional Sudoku generation, MDLM and $\ours$ with a-priori confidence (\Cref{sec:algo:inference}) fail catastrophically when sampling multiple tokens at a time. $\ours$ with $\oursPR$ (\Cref{sec:algo:pr}) achieves $87\%$ accuracy with only a few NFEs.
    }
    \label{tab:uncond-sudoku-nfe}
\end{wraptable}
We compare Posterior Refinement ($\oursPR$) against a-priori confidence-based sampling methods in \cref{tab:uncond-sudoku-nfe}, and against a range of representative language modeling methods in \cref{tab:main}.
$\oursPR$ delivers large efficiency gains across all considered benchmarks.
As a result, $\ours$ achieves 71\% and 19\% accuracy on Hard Sudoku and GSM8K problems with just 4 and 32 NFEs respectively, outperforming all non-autoregressive baselines in both speed ($32 \times$ speedup) and accuracy.
On TinyStories and OWT, $\ours$ matches or surpasses all baselines with up to $8 \times$ fewer NFEs. 
We ablate over the choice of hyperparameters for $\oursPR$ in \cref{sec:app:ablations} and provide qualtitative samples in \cref{sec:app:samples}.
\vspace{1em}
\begin{table}[!htbp]
    \centering
    \footnotesize
    \setlength{\tabcolsep}{4pt}

    \label{tab:pr_results}
    \adjustbox{max width=\linewidth}{%
    \begin{tabular}{l l
        c cc
        c cc
        c cccc
        c c}
        \toprule
        & &
        \multicolumn{3}{c}{\textbf{TinyStories}} &
        \multicolumn{3}{c}{\textbf{OpenWebText}} &
        \multicolumn{5}{c}{\textbf{Sudoku}} &
        \multicolumn{2}{c}{\textbf{GSM8K}} \\
        \cmidrule(lr){3-5}
        \cmidrule(lr){6-8}
        \cmidrule(lr){9-13}
        \cmidrule(lr){14-15}
        Category & Method &
        NFE & Gen-PPL $\downarrow$ & Ent. &
        NFE & Gen-PPL $\downarrow$ & Ent. &
        NFE & Easy & Med.\  & Hard & &
        NFE & Acc.\ (\%,$\uparrow$) \\
        \midrule
        \multirow{2}{*}{AR}
        & Sample
          & 128 & 8.89  & 4.01
          & 1024 & 35.45 & 5.58
          & 128  & 13.9 &  5.1 &  0.6 &
          & 512  & 53.9 \\
        & Greedy
          & 128 & 5.34  & 4.06
          & 1024 & 1.07  & 1.65
          & 128  & 14.6 &  5.1 &  1.0 &
          & 512  & 63.3 \\
        \midrule
        \multirow{2}{*}{Discrete}
        & MDLM
          & 128 & 18.74 & 4.03
          & 1024 & 105.15 & 5.63
          & 128  & 92.0 & 77.1 & 30.2 &
          & 1024 & 18.0 \\
        & Duo
          & 128 & 22.73 & 4.05
          & 1024 & 77.69  & 5.55
          & 128  & 96.3 & 84.7 & 58.4 &
          & 1024 & 17.2 \\
        \midrule
        \multirow{4}{*}{Continuous}
        & CANDI
          & 128  & 46.32  & 4.04
          & 1024 & 143.13 & 5.71
          & 128  & 79.3 & 45.9 & 16.7 &
          & 1024 &  0.2 \\
        & FLM
          & 128  & 57.50  & 4.17
          & 1024 &  62.23 & 5.33
          & 128  & 94.2 & 82.7 & 44.5 &
          & 1024 &  0.3 \\
        & S-FLM
          & 128  & 95.25  & 4.08
          & 1024 & 123.87 & 5.52
          & 128  & 94.8 & 85.2 & 45.0 &
          & 1024 & 18.0 \\
        \rowcolor{darkteal}
        & $\ours$
          & \textbf{32}  & \textbf{17.53} & 3.96
          & \textbf{128} & \textbf{66.6}  & 5.21
          & \textbf{4}   & \textbf{97.9}  & \textbf{92.0} & \textbf{71.2} &
          & \textbf{32}  & \textbf{19.0} \\
        \bottomrule
    \end{tabular}%
    }
    \caption{
        \textbf{Posterior Refinement.}
        Generative perplexity and unigram entropy on unconditional benchmarks, and
        accuracy (\%, $\uparrow$) on conditional benchmarks (Sudoku split into Easy/Med./Hard; GSM8K zero-shot). $\ours$, which uses $\oursPR$, matches or surpasses the strongest baselines with $8$--$32{\times}$ fewer function evaluations.
    }
     % \vspace{-1em}
     \label{tab:main}
\end{table}
 % \vspace{-1em}
 \newpage
\subsection{Leveraging Pretrained MDMs}
\label{sec:exp:dit}
\begin{wraptable}[9]{r}{0.50\linewidth}
    \vspace{-3em}
    \centering
    \footnotesize
    \setlength{\tabcolsep}{4pt}
    \begin{tabular}{l ccccccc}
        \toprule
         & \multicolumn{7}{c}{\textbf{GSM8K} (\%, $\uparrow$)} \\
        \cmidrule(lr){2-8}
        Method & 1 & 2 & 4 & 8 & 16 & 32 & 64 \\
        \midrule
        MDLM (Teacher)    & 0.0 & 0.0  & 0.1  & 0.6  & 4.7 & 9.0 & 13.4 \\
        \midrule
        $\ours$           & 0.0 & 0.1 & 2.9 & 8.7 & 13.4 & 19.0 & 19.1 \\
        $\ours$ (Distill)  & 0.3 & 0.4 & 3.9 & 10.3 & 18.7 & 21.6 & 23.4\\
        $\ours$ (Init) & \textbf{0.3} & \textbf{0.7} & \textbf{5.1} & \textbf{15.1} & \textbf{26.1} & \textbf{31.8} & \textbf{33.6} \\
        \bottomrule
    \end{tabular}
    \caption{
        \textbf{Leveraging pretrained MDMs.} Accuracy (\%, $\uparrow$) on zero-shot GSM8K, comparing three training strategies of $\ours$  with Posterior Refinement.
    }
    \label{tab:mdm_connection}
\end{wraptable}
In \cref{tab:mdm_connection}, we compare the two $\ours$ training approaches proposed in~\cref{sec:algo:mdlm_boundary} that leverage a pretrained MDM against $\ours$ and MDM trained from scratch.
We find that distilling and initializing from a pretrained MDM can be highly beneficial.
We note that $\ours$ in \cref{tab:main} are trained from scratch for fair comparison with baselines, although we expect that leveraging MDMs as in \cref{tab:mdm_connection} would lead to further improvements.

\begin{wrapfigure}[19]{r}{0.43\textwidth}
    \vspace{0.5em}
    \centering
    \includegraphics[width=\linewidth]{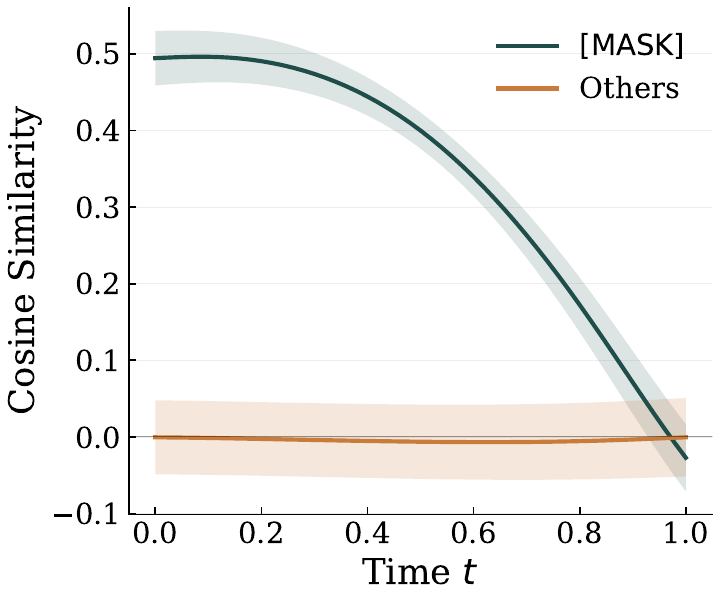}
    \caption{\textbf{adaLN activations.}
   At $t{=}0$, noisy inputs align with the
    MDM's \texttt{$\mask$} activation while remaining uncorrelated with
    other tokens.}
    \label{fig:adaLN}
    % \vspace{-0.5em}
\end{wrapfigure}

% \paragraph{adaLN absorbs the shift in input.}
The strong performance of $\ours$ (Init) in \cref{tab:mdm_connection} is surprising, since it is unclear why pretraining on $\mask$ tokens should benefit processing Gaussian-noise inputs.
We explain this by inspecting two components of the DiT backbone, the first adaptive layer normalization (adaLN)~\citep{perez2018adaln} for time conditioning, and the token embedding layer, before and after $\ours$ training.
For the first DiT block, we measure the cosine similarity between $\ours$'s post-adaLN activations and those of the MDM teacher's (\Cref{fig:adaLN}).
At $t{=}0$ (pure Gaussian noise), $\ours$ activations remain highly aligned with the teacher's $\mask$ activations.
Thus, we conclude that the first adaLN layer maps the Gaussian-noise input onto the same representation the teacher produced for $\mask$ tokens, presenting downstream layers with an in-distribution input and preserving the pretrained knowledge.
% \paragraph{Embedding geometry is preserved.}
In \Cref{tab:embed-cos}, we further compute the per-token cosine similarity between each $\ours$ input embedding and its MDM-init counterpart.
Every token embedding stays aligned with its initialization after training, indicating that $\ours$ adapts the embedding layer to be robust to Gaussian noise corruptions.

%% file: conclusion.tex
\section{Conclusion}
\label{sec:conc}

We introduce $\ours$, a generalization of flow map language models that equips continuous-flow language modeling with the any-order inference flexibility of masked diffusion, paving the way towards effective iterative refinement.
We further introduce \emph{Posterior Refinement}, a novel, after-the-fact refinement strategy that re-noises and regenerates low-confidence tokens to self-correct in parallel.
Empirically, $\ours$ improves the speed-quality tradeoff over both MDM and FMLM families across unconditional and conditional generation settings, matching or surpassing the strongest baselines with up to $32{\times}$ fewer function evaluations.
Finally, we show that the masked diffusion objective is equivalent to a boundary case of the $\ours$ training objective, and use this correspondence to derive two practical mechanisms based on distillation and transfer learning for accelerating $\ours$ training using pretrained MDMs.

\paragraph{Future directions.}
The correspondence between MDM and $\ours$ offers a path to scale $\ours$ to the billion-parameter regime using the growing pool of large pretrained masked diffusion models (LLaDA-8B \citep{nie2025large}, Dream \citep{ye2025dream}, and others).
We also believe that the conditional nature of $\ours$ can be leveraged to improve generation quality using techniques such as classifier-free guidance~\citep{ho2021classifierfree}.
Finally, Posterior Refinement exposes a form of inference-time scaling native to $\ours$, and we expect novel test-time strategies can be designed around such posterior-confidence signals.
We leave a systematic study of these directions to future work.

\paragraph{Limitations.}
Methodologically, we follow the progressive self-distillation objective from~\citet{lee2025flowmaplm} and do not explore alternatives such as the Lagrangian, Eulerian and Mean Flow variants~\citep{boffi2025flowmapmatchingstochastic, geng2025mean}.
We conduct our experiments on a small scale on largely synthetic and small-sized datasets.
Thus, the scalability of our approach needs a more broad study with evaluation on downstream tasks.
While our current formulation of PR holds committed tokens fixed, alternative formulations can be explored to allow $\ours$ to rectify previous errors.
Finally, we note that training from scratch is suboptimal compared to initializing with MDM weights, even at convergence.
This implies that there is room to improve $\ours$ training with better model architecture and low-variance loss functions.

%% file: acknowledgements.tex
\section*{Acknowledgements}
\label{sec:ack}
The authors would like to thank \href{https://modal.com/}{Modal Labs} for their generous compute grants, which proved invaluable in supporting this work.

%% file: appendix.tex
\newpage
\section{Theoretical details}
\label{sec:app:theory}

\subsection{Rounding error in terms of confidence}
\label{sec:app:rounding}

In \cref{sec:algo:interpretation} we interpret the \emph{rounding error} that occurs when model prediction is projected to its nearest one-hot vertex as a proxy for the model's confidence. We claim that thresholding this error is equivalent to thresholding the top probability $p_{\max}^l$ and show it here.

Fix a position $l$ and write $q \coloneqq Q^l \in \Delta^{|V|-1}$ for the predicted categorical distribution, with top probability $p_{\max} \coloneqq \max_v q^v$ attained at $\hat{v} \coloneqq \argmax_v q^v$. Committing the token rounds $q$ to the one-hot vertex $e_{\hat{v}}$, and the rounding error is the discrepancy between $q$ and $e_{\hat{v}}$, measured either by the $\ell_1$ distance or by the KL divergence.

\begin{proposition}[Rounding error is monotone in confidence]
\label{prop:rounding_monotone}
With the notation above,
\begin{align}
    \label{eq:app:rounding}
    \bigl\|q - e_{\hat{v}}\bigr\|_1 = 2\,(1 - p_{\max}),
    \qquad
    \kl{e_{\hat{v}}}{q} = -\log p_{\max}.
\end{align}
Both quantities are strictly decreasing in $p_{\max} \in (0,1]$, vanishing exactly at $p_{\max} = 1$. Consequently, for any threshold $p \in (0,1]$,
\begin{align}
    \label{eq:app:threshold_equiv}
    p_{\max} \ge p
    \;\iff\;
    \bigl\|q - e_{\hat{v}}\bigr\|_1 \le 2(1-p)
    \;\iff\;
    \kl{e_{\hat{v}}}{q} \le -\log p,
\end{align}
so thresholding the confidence $p_{\max}$ and thresholding the rounding error under either metric retain exactly the same set of positions.
\end{proposition}

\begin{proof}
The mass that $q$ places off the top token is $\sum_{v \ne \hat{v}} q^v = 1 - p_{\max}$. For the $\ell_1$ distance,
\begin{align}
    \bigl\|q - e_{\hat{v}}\bigr\|_1
    = \bigl|q^{\hat{v}} - 1\bigr| + \sum_{v \ne \hat{v}} \bigl|q^v\bigr|
    = (1 - p_{\max}) + \sum_{v \ne \hat{v}} q^v
    = 2\,(1 - p_{\max}).
\end{align}
For the KL divergence, only the $v = \hat{v}$ term survives, since $e_{\hat{v}}^v = 0$ for $v \ne \hat{v}$:
\begin{align}
    \kl{e_{\hat{v}}}{q}
    = \sum_v e_{\hat{v}}^v \log\frac{e_{\hat{v}}^v}{q^v}
    = 1 \cdot \log\frac{1}{q^{\hat{v}}}
    = -\log p_{\max}.
\end{align}
Both $2(1 - p_{\max})$ and $-\log p_{\max}$ are continuous and strictly decreasing on $(0,1]$, so each is invertible; applying the inverse to either rounding-error threshold recovers the confidence threshold, giving the equivalences in \cref{eq:app:threshold_equiv}.
\end{proof}
\section{Setup}
\label{sec:app:setup}

We document the configurations used to produce all reported results (summarized in \Cref{tab:hps}). Following prior work~\citep{sahoo2024simple, lee2025flowmaplm}, we adopt a DiT~\citep{peebles2023scalable} backbone for $\ours$. All experiments are run on $4{\times}$ NVIDIA H100 GPUs, in bfloat16 mixed precision, with AdamW ($\beta_1{=}0.9$, $\beta_2{=}0.999$, $\epsilon{=}10^{-8}$), no weight decay, and gradient-norm clipping at $1.0$. Learning rates follow a $2.5\text{k}$-step linear warmup to a constant peak. We evaluate using exponential moving average (EMA, decay $0.9999$) parameters.

\subsection{Datasets}
\label{sec:app:data}

\paragraph{TinyStories (unconditional generation).}
We train on TinyStories \citep{eldan2023tinystoriessmalllanguagemodels}, a corpus of GPT-3.5/4-generated short stories restricted to a limited vocabulary. We tokenize with the GPT-2 byte-pair tokenizer and pack examples into length-$128$ sequences with \texttt{<EOS>}-insertion at sequence boundaries. We report generative perplexity under GPT-2 Large following~\citep{sahoo2024simple,deschenaux2024beyond,lee2025flowmaplm}, alongside per-sample token entropy as a check against entropy collapse. Results appear in \cref{tab:main}.

\paragraph{OpenWebText (unconditional generation).}
We also evaluate on OpenWebText~\citep{Gokaslan2019OpenWeb}, an open reproduction of the WebText corpus used to train GPT-2. We tokenize with the GPT-2 byte-pair tokenizer and pack documents into length-$1024$ sequences with \texttt{<EOS>}-insertion at document boundaries. As with TinyStories, we report generative perplexity under GPT-2 Large~\citep{sahoo2024simple,deschenaux2024beyond,lee2025flowmaplm}, alongside per-sample token entropy as a check against entropy collapse. Results appear in \cref{tab:main}.

\paragraph{TinyGSM / GSM8K (conditional reasoning).}
We follow the exact protocol of S-FLM~\citep{deschenaux2026language} so that our numbers are directly comparable to its baseline suite (MDLM, Duo, CANDI, FLM, S-FLM). We train on TinyGSM~\citep{liu2023tinygsmachieving80gsm8k}, a synthetic corpus of 12.3M grade-school math word problems generated to imitate the GSM8K distribution. We tokenize with the SmolLM-135M tokenizer and pack each $\langle\text{question},\text{Python solution}\rangle$ pair to length $512$, with loss supervised only on the solution tokens. Models are evaluated zero-shot on the GSM8K test split by sampling a Python solution conditioned on the prompt and grading by execution accuracy.

\paragraph{Sudoku (conditional).}
Following S-FLM~\citep{deschenaux2026language}, each puzzle is encoded as a length-$180$ sequence over a $12$-token vocabulary, with loss supervised only on the solution positions. Models are evaluated on the rate at which they produce a valid completion of the clue set across Easy/Medium/Hard splits. Results appear in \Cref{tab:main}.

\vspace{-1em}
\subsection{Architecture}
\label{sec:app:arch}

All models use the Diffusion Transformer (DiT) backbone~\citep{peebles2023scalable} with adaLN time conditioning. Following \citet{lee2025flowmaplm}, $\ours$ uses a \emph{double} time embedding ($s$ and $t$ are both injected through adaLN) so the same transformer can serve as the two-time denoiser at training and as a single-time denoiser at $s{=}t$. The MDM baselines and the MDM teachers used for distillation/initialization use the standard single-time DiT. The exact per-task configurations are listed in \cref{tab:hps}; for the warm-start and distillation variants, $\ours$ reuses the architecture of the corresponding MDM so that initialization is a parameter-identity copy.

\subsection{$\ours$ training objective}
\label{sec:app:training}

\begin{wrapfigure}[16]{r}{0.45\textwidth}
  \vspace{-1.0em}
  \centering
  \begin{minipage}{0.45\textwidth}
    \refstepcounter{algorithm}
    \hrule height 0.8pt
    \vspace{2pt}
    \noindent\textbf{Algorithm~\thealgorithm} $\ours$ Training.
    \label{alg:training_code}
    \vspace{2pt}
    \hrule height 0.4pt
    \vspace{0.6em}

    \footnotesize
    \input{pseudocode_training}

    \vspace{0.6em}
    \hrule height 0.8pt
  \end{minipage}

  \vspace{-0.5em}
\end{wrapfigure}

We train $\ours$ with the progressive self-distillation (PSD) objective in \cref{eq:algo:training}, instantiated as \cref{alg:training_code}. We fix the following across all datasets:
\begin{itemize}
    \item \textbf{Corruption schedule}: Noisy positions share a single sampled time $t\!\sim\!\mathcal{U}[0,1]$, while the rest are pinned to $t{=}1$ as clean context.
    \item \textbf{Progressive self-distillation (PSD)}: the off-diagonal target uses the PSD construction of \citet{lee2025flowmaplm} with stop-gradient on the teacher branch.
    \item \textbf{Time Reparameterization}: we drop the time-reparameterization trick from \citet{lee2025flowmaplm}, and follow a simple linear-schedule.
    \item \textbf{Diagonal fraction}: with probability $0.5$ a step uses the diagonal cross-entropy term ($s{=}t$); otherwise the off-diagonal semigroup term. Both supervise only the noisy tokens.
    \item \textbf{Antithetic sampling}: each minibatch pairs $(s,u,t)$ with $(1{-}t,1{-}u,1{-}s)$ to reduce variance along the time axis.

\end{itemize}

\paragraph{MDM warm-start and distillation.}
For the \emph{FMLM+ (Init)} variant we initialize $\ours$ weights from a same-architecture MDM checkpoint trained on the same dataset. The initialization is a parameter-identity copy: the only architectural difference is that $\ours$ activates its second time embedding (initialized to zero so the network behaves identically at step zero). For the \emph{FMLM+ (Distill)} variant we additionally hold a frozen MDM teacher and replace the diagonal target with the teacher's posterior with probability $\pdistill$. The MDM teachers are trained from scratch with the standard MDLM loss~\citep{sahoo2024simple} on the same data, optimizer, and architecture used by the corresponding $\ours$ run.

% \subsection{Inference protocols}
% \label{sec:app:inference}

% All samplers share the structure of \cref{alg:posterior_refinement_code}: noise tokens at $t{=}0$, prompt tokens at $t{=}1$, then iteratively predict $\hat{\bf x}_1$ via a flow-map call and commit a subset of positions per step. The strategies differ only in how each iteration's update set $\mathcal{C}'$ is chosen:
% \begin{itemize}
%     \item \textbf{Parallel}: $\mathcal{C}' = [L]\setminus \mathcal{C}$ in a single step; the entire sequence is committed at once.
%     \item \textbf{Block-Diffusion}: $\mathcal{C}'$ is the next contiguous block of size $B$ in left-to-right order, with $S$ flow-map steps per block; the total NFE is $\lceil L/B\rceil \cdot S$.
%     \item \textbf{Confidence}: at each step, $\mathcal{C}'$ is the top-$k$ positions ranked by the model's per-token max-softmax confidence~\citep{ghazvininejad2019maskpredict}, with $k$ chosen so the schedule unmasks all positions in a fixed budget of steps.
% \end{itemize}
% For all samplers we apply a final argmax projection at $t{=}1$. The block-decoding sweep in \cref{fig:speed_tradeoff} and \Cref{tab:tinystories-sweep} varies $(B, S)$ over $\{1,2,4,8,16,32,64,128\}$ with iso-NFE diagonals $B\cdot S$.

\subsection{Hyperparameters}
\label{sec:app:hps}

\begin{table}[!htbp]
\centering
\label{tab:hps}
\footnotesize
\setlength{\tabcolsep}{6pt}
\adjustbox{max width=\textwidth}{%
\begin{tabular}{l cccc}
\toprule
 & TinyStories & OpenWebText & TinyGSM / GSM8K & Sudoku \\
\midrule
\multicolumn{5}{l}{\textit{Architecture (DiT)}} \\
Hidden size           & 512  & 768  & 768  & 512 \\
adaLN cond.\ dim      & 128  & 128  & 128  & 128 \\
Layers                & 8    & 12   & 12   & 8 \\
Attention heads       & 8    & 12   & 12   & 8 \\
FFN multiplier        & 4    & 4    & 4    & 4 \\
Dropout               & 0.1  & 0.1  & 0.1  & 0.1 \\
Sequence length $L$   & 128  & 1024 & 512  & 180 \\
Vocabulary $|V|$      & 50257 (GPT-2) & 50257 (GPT-2) & 49152 (SmolLM-135M) & 12 \\
Parameters (approx.)  & 80M  & 170M & 168M & 29M \\
\midrule
\multicolumn{5}{l}{\textit{Optimization}} \\
Optimizer             & \multicolumn{4}{c}{AdamW ($\beta_1{=}0.9$, $\beta_2{=}0.999$, $\epsilon{=}10^{-8}$)} \\
Weight decay          & \multicolumn{4}{c}{$0$} \\
Gradient clipping     & \multicolumn{4}{c}{$1.0$} \\
Schedule              & \multicolumn{4}{c}{Constant w/ linear warmup} \\
Warmup steps          & \multicolumn{4}{c}{$2\,500$} \\
Peak LR               & $3{\times}10^{-4}$ & $3{\times}10^{-4}$ & $3{\times}10^{-4}$ & $3{\times}10^{-4}$ \\
Total steps           & $100\,000$ & $1\,000\,000$ & $250\,000$ & $20\,000$ \\
Global batch size     & 512  & 512  & 512  & 256 \\
Per-GPU batch size    & 128  & 32   & 64   & 256 \\
Precision             & \multicolumn{4}{c}{bfloat16 mixed} \\
EMA decay             & \multicolumn{4}{c}{$0.9999$} \\
\bottomrule
\end{tabular}%
}
\caption{\textbf{Per-task hyperparameters.} All $\ours$ runs reuse the architecture and optimizer of the corresponding MDM baseline; the only per-method difference is the training and inference schedule.}
\end{table}

\section{Additional Experiments}
\subsection{Leveraging Pretrained MDMs}
\begin{table}[!htbp]
\centering
\footnotesize
\setlength{\tabcolsep}{6pt}
\adjustbox{max width=\textwidth}{%
\begin{tabular}{lccccc}
\toprule
Cosine similarity & $<0$ & $[0,0.25)$ & $[0.25,0.5)$ & $[0.5,0.75)$ & $[0.75,1]$ \\
\midrule
\% of vocabulary & $0$ & $18$ & $47$ & $28$ & $7$ \\
\bottomrule
\end{tabular}%
}
\caption{\textbf{Embedding preservation.} Distribution of per-token cosine similarities between the learned $\ours$ embedding matrix and the corresponding MDM initialization. All $49{,}152$ vocabulary embeddings retain positive cosine similarity after training.}
\label{tab:embed-cos}
\end{table}
\newpage
\section{Ablations}
\label{sec:app:ablations}

\paragraph{Commit threshold.} We sweep the commit threshold $p \in \{0.9, 0.99, 0.999\}$ for Posterior Refinement on Sudoku-Hard and GSM8K across NFE budgets. On GSM8K, $p{=}0.9$ saturates by 16 NFEs at 8.6\%, while $p{=}0.999$ continues improving through 64 NFEs, reaching 19.1\%. Similarly, on Sudoku-Hard, $p{=}0.999$ is a better threshold from 4 NFEs onward, reaching 81.4\% at 16 NFEs. We use $p{=}0.999$ for the headline results in~\cref{tab:main}.

\begin{table}[!htpb]
\centering
\footnotesize
\setlength{\tabcolsep}{4pt}

\label{tab:gsm8k-threshold-sweep}
\adjustbox{max width=\textwidth}{%
\begin{tabular}{l ccccc ccccccc}
\toprule
 & \multicolumn{5}{c}{\textbf{Sudoku-Hard} (NFE)} & \multicolumn{7}{c}{\textbf{GSM8K} (NFE)} \\
\cmidrule(lr){2-6} \cmidrule(lr){7-13}
Method & 1 & 2 & 4 & 8 & 16 & 1 & 2 & 4 & 8 & 16 & 32 & 64 \\
\midrule
$\ours$ (PR = 0.9)   & 4.1 & 21.2 & 20.2 & 19.9 & 20.9 & 0 & \textbf{1.6} & 4.6 & 6.5 & 7.8 & 9.0 & 8.5 \\
$\ours$ (PR = 0.99)  & 4.1 & \textbf{35.6} & 51.9 & 54.6 & 56.3 & 0 & 0.6 & \textbf{5.7} & \textbf{9.7} & \textbf{14.5} & 16.6 & 15.0 \\
$\ours$ (PR = 0.999) & 4.1 & 26.0 & \textbf{71.2} & \textbf{78.1} & \textbf{81.4} & 0 & 0.1 & 2.9 & 8.7 & 13.4 & \textbf{19.0} & \textbf{19.1} \\
\bottomrule
\end{tabular}%
}
\caption{\textbf{Posterior Refinement threshold sweep.} Solve-rate accuracy (\%, $\uparrow$) on Sudoku-Hard and zero-shot GSM8K across NFE budgets, for commit thresholds $p \in \{0.9, 0.99, 0.999\}$.}
\vspace{4pt}
\end{table}

On TinyStories (\Cref{tab:tinystories-threshold-sweep}), the threshold has no effect on small NFE budgets. From 8 NFEs onward, higher thresholds start performing better: $p{=}0.999$ reaches Gen.\ PPL $7.40$ at 32 NFEs, versus $7.69$ for $p{=}0.99$ and $9.71$ for $p{=}0.9$, with entropy stable across all settings ($\approx 3.8$--$3.9$). We use $p{=}0.999$ for the headline TinyStories result in~\cref{tab:main}.

\begin{table}[!htpb]
\centering
\footnotesize
\setlength{\tabcolsep}{4pt}

\label{tab:tinystories-threshold-sweep}
\adjustbox{max width=\textwidth}{%
\begin{tabular}{l cccccc}
\toprule
NFE & 1 & 2 & 4 & 8 & 16 & 32 \\
\midrule
$\ours$ (PR = 0.9)   & 96.28 / 3.97 & 35.14 / 3.98 & \textbf{16.92 / 3.9}1 & 11.58 / 3.86 & 10.21 / 3.86 & 9.71 / 3.83 \\
$\ours$ (PR = 0.99)  & 96.28 / 3.97 & 35.14 / 3.98 & 17.76 / 3.91 & \textbf{11.42 / 3.86} & 8.74 / 3.83 & 7.69 / 3.81 \\
$\ours$ (PR = 0.999) & 96.28 / 3.97 & 35.14 / 3.98 & 17.76 / 3.91 & 11.57 / 3.86 & \textbf{8.94 / 3.85} & \textbf{7.40 / 3.83} \\
\bottomrule
\end{tabular}%
}
\caption{\textbf{Posterior Refinement threshold sweep on TinyStories.} Generative perplexity ($\downarrow$) / per-sample entropy ($\uparrow$) under GPT-2 Large across NFE budgets, for commit thresholds $p \in \{0.9, 0.99, 0.999\}$.}
\vspace{-1em}
\end{table}

\paragraph{Design choices.} We ablate hyperparams of PR sampling on GSM8K (32 NFEs, $p{=}0.999$). We find that resampling renoised cells from fresh noise is important. Re-using the same initial noise can cause the model to converge to the same (rejected) predictions, affecting accuracy. We also note that setting both threshold and Top-K help improve the performance.

\begin{table}[!htpb]
\centering
\footnotesize
\setlength{\tabcolsep}{4pt}
\label{tab:pr-design-ablation}
\begin{tabular}{lc}
\toprule
Configuration & Acc.\ (\%) \\
\midrule
\multicolumn{2}{l}{\textit{Renoising}} \\
Same Initial Noise & 14.4 \\
\rowcolor{darkteal} \textbf{New Initial Noise } & \textbf{19.0} \\
\midrule
\multicolumn{2}{l}{\textit{Commit schedule}} \\
Threshold & 12.8 \\
Top-K & 7.8 \\
\rowcolor{darkteal} \textbf{Threshold + Top-K } & \textbf{19.0} \\
\midrule
\multicolumn{2}{l}{\textit{Top-K Value}} \\
Top-1 & 17.8 \\
Top-2 & 17.8 \\
Top-4 & 16.1 \\
Top-8 & 11.4 \\
\rowcolor{darkteal} \textbf{Top-K } & \textbf{19.0} \\
\bottomrule
\end{tabular}
\vspace{4pt}
\caption{\textbf{Posterior Refinement design ablations.} GSM8K accuracy (\%, $\uparrow$) at 32 NFEs, $p{=}0.999$. Top-K is dynamically set to follow a linear schedule in the remaining number of tokens and budget.}
\end{table}

% \subsection{Compute}
% \label{sec:app:compute}
% A single $\ours$ run on TinyStories takes ${\sim}10$ GPU-hours on $4{\times}$A6000; TinyGSM/GSM8K runs take ${\sim}80$ GPU-hours; Conditional Sudoku runs take ${\sim}20$ GPU-hours. The full set of results in the paper, including baselines and the three $\ours$ variants on each task, used roughly $1.2$k GPU-hours in aggregate. Inference for the block-decoding sweep in \Cref{fig:speed_tradeoff} adds a further ${\sim}30$ GPU-hours, dominated by the $128$-step configurations. \manan{confirm this}

%
\newpage
\section{Posterior Refinement: Qualitative Samples}
We show qualitative Posterior Refinement trajectories for
$\ours$ on TinyGSM, OpenWebText, and TinyStories. For each sample, we
snapshot a few refinement rounds and color tokens by their commit
state: tokens committed at high confidence (committed) are shown in
\textcolor{fig1teal}{\textbf{teal}}, and tokens still being refined in
\textcolor{fig1ochre}{\textbf{ochre}}. For TinyGSM, the conditioning
question is shown in \textbf{black}. Each box header reports the
refinement round and the number of tokens committed so far.
\label{sec:app:samples}
\input{samples}

%% file: pseudocode_training.tex
% Training step for the two-time denoiser. Style from pseudocode_style.tex.
\begin{lstlisting}[language=PythonFuncColor, escapechar=`]
mask   = random_mask(y.shape)

# clean context pinned to t = 1
s, u, t = where(mask, sample_times(), 1.0)                
x_1     = one_hot(y)
x_s     = (1-s) * randn(L, V) + s * x_1

# self distillation target
d_su   = delta(x_s, s, u)
X_su = flowmap(x_s, d_su, s, u)
d_ut   = delta(X_su, u, t)
target = lerp(d_su, d_ut, gamma(s, u, t))    

diag    = CE(delta(x_s, s, s), x_1)
offdiag = CE(delta(x_s, s, t), sg(target))

# supervise noised positions only
loss = (diag + offdiag) * mask
\end{lstlisting}

%% file: samples.tex
% ===================== TinyGSM =====================
\begin{figure}[H]
\centering
\begin{samplebox}{\normalfont\textbf{$\ours$ (PR), Round: 1} \hfill \normalfont\scriptsize \textcolor{darkgray}{Committed: \textbf{428/483}}}
\scriptsize\linespread{0.9}\selectfont
\textcolor{black}{\texttt{A}} \textcolor{black}{\texttt{robe}} \textcolor{black}{\texttt{takes}} \textcolor{black}{\texttt{2}} \textcolor{black}{\texttt{bolts}} \textcolor{black}{\texttt{of}} \textcolor{black}{\texttt{blue}} \textcolor{black}{\texttt{fiber}} \textcolor{black}{\texttt{and}} \textcolor{black}{\texttt{half}} \textcolor{black}{\texttt{that}} \textcolor{black}{\texttt{much}} \textcolor{black}{\texttt{white}} \textcolor{black}{\texttt{fiber}}\textcolor{black}{\texttt{.}}  \textcolor{black}{\texttt{How}} \textcolor{black}{\texttt{many}} \textcolor{black}{\texttt{bolts}} \textcolor{black}{\texttt{in}} \textcolor{black}{\texttt{total}} \textcolor{black}{\texttt{does}} \textcolor{black}{\texttt{it}} \textcolor{black}{\texttt{take}}\textcolor{black}{\texttt{?}}\\
\textcolor{fig1teal}{\texttt{def}} \textcolor{fig1teal}{\texttt{simple}}\textcolor{fig1teal}{\texttt{\_}}\textcolor{fig1teal}{\texttt{math}}\textcolor{fig1teal}{\texttt{\_}}\textcolor{fig1teal}{\texttt{problem}}\textcolor{fig1teal}{\texttt{()}} \textcolor{fig1teal}{\texttt{->}} \textcolor{fig1teal}{\texttt{int}}\textcolor{fig1teal}{\texttt{:}}\\
\mbox{\texttt{~~~~}} \textcolor{fig1ochre}{\texttt{"""}}\\
\mbox{\texttt{~~~~}} \textcolor{fig1teal}{\texttt{A}} \textcolor{fig1ochre}{\texttt{robe}} \textcolor{fig1teal}{\texttt{takes}} \textcolor{fig1teal}{\texttt{2}} \textcolor{fig1teal}{\texttt{bolts}} \textcolor{fig1teal}{\texttt{of}} \textcolor{fig1teal}{\texttt{blue}} \textcolor{fig1teal}{\texttt{fiber}} \textcolor{fig1teal}{\texttt{and}} \textcolor{fig1teal}{\texttt{half}} \textcolor{fig1teal}{\texttt{that}} \textcolor{fig1teal}{\texttt{much}} \textcolor{fig1teal}{\texttt{white}} \textcolor{fig1teal}{\texttt{fiber}}\textcolor{fig1teal}{\texttt{.}}\\
\mbox{\texttt{~~~~}}\\
\mbox{\texttt{~~~~}} \textcolor{fig1teal}{\texttt{How}} \textcolor{fig1teal}{\texttt{many}} \textcolor{fig1teal}{\texttt{bolts}} \textcolor{fig1teal}{\texttt{in}} \textcolor{fig1teal}{\texttt{total}} \textcolor{fig1teal}{\texttt{does}} \textcolor{fig1teal}{\texttt{it}} \textcolor{fig1teal}{\texttt{take}}\textcolor{fig1teal}{\texttt{?}}\\
\mbox{\texttt{~~~~}} \textcolor{fig1ochre}{\texttt{"""}}\\
\mbox{\texttt{~~~~}} \textcolor{fig1ochre}{\texttt{blue}}\textcolor{fig1ochre}{\texttt{\_}}\textcolor{fig1ochre}{\texttt{fiber}} \textcolor{fig1ochre}{\texttt{=}} \textcolor{fig1ochre}{\texttt{2}}\\
\mbox{\texttt{~~~~}} \textcolor{fig1ochre}{\texttt{white}}\textcolor{fig1ochre}{\texttt{\_}}\textcolor{fig1ochre}{\texttt{fiber}} \textcolor{fig1ochre}{\texttt{=}} \textcolor{fig1ochre}{\texttt{blue}}\textcolor{fig1ochre}{\texttt{\_}}\textcolor{fig1ochre}{\texttt{fiber}} \textcolor{fig1ochre}{\texttt{/}}\\
\mbox{\texttt{~~~~}}\\
\mbox{\texttt{~~~~}} \textcolor{fig1ochre}{\texttt{total}}\textcolor{fig1ochre}{\texttt{\_}}\textcolor{fig1ochre}{\texttt{bol}}\textcolor{fig1ochre}{\texttt{ts}} \textcolor{fig1ochre}{\texttt{=}} \textcolor{fig1ochre}{\texttt{blue}}\textcolor{fig1ochre}{\texttt{\_}}\textcolor{fig1ochre}{\texttt{fiber}} \textcolor{fig1ochre}{\texttt{+}} \textcolor{fig1ochre}{\texttt{white}}\textcolor{fig1ochre}{\texttt{\_}}\textcolor{fig1ochre}{\texttt{fiber}}\\
\mbox{\texttt{~~~~}}\\
\mbox{\texttt{~~~~}} \textcolor{fig1ochre}{\texttt{result}} \textcolor{fig1ochre}{\texttt{=}} \textcolor{fig1ochre}{\texttt{total}}\textcolor{fig1ochre}{\texttt{\_}}\textcolor{fig1ochre}{\texttt{bol}}\textcolor{fig1ochre}{\texttt{ts}}\\
\\
\mbox{\texttt{~~~~}} \textcolor{fig1ochre}{\texttt{return}} \textcolor{fig1ochre}{\texttt{result}}\textcolor{fig1ochre}{\texttt{[PAD]}}\textcolor{fig1ochre}{\texttt{[PAD]}}\textcolor{fig1ochre}{\texttt{[PAD]}}\textcolor{fig1ochre}{\texttt{[PAD]}}\textcolor{fig1ochre}{\texttt{[PAD]}}\textcolor{fig1ochre}{\texttt{[PAD]}}\textcolor{fig1ochre}{\texttt{[PAD]}}\textcolor{fig1ochre}{\texttt{[PAD]}}\textcolor{fig1ochre}{\texttt{[PAD]}}\textcolor{fig1ochre}{\texttt{[PAD]}}
\end{samplebox}
\begin{samplebox}{\normalfont\textbf{$\ours$ (PR), Round: 10} \hfill \normalfont\scriptsize \textcolor{darkgray}{Committed: \textbf{465/483}}}
\scriptsize\linespread{0.9}\selectfont
\textcolor{black}{\texttt{A}} \textcolor{black}{\texttt{robe}} \textcolor{black}{\texttt{takes}} \textcolor{black}{\texttt{2}} \textcolor{black}{\texttt{bolts}} \textcolor{black}{\texttt{of}} \textcolor{black}{\texttt{blue}} \textcolor{black}{\texttt{fiber}} \textcolor{black}{\texttt{and}} \textcolor{black}{\texttt{half}} \textcolor{black}{\texttt{that}} \textcolor{black}{\texttt{much}} \textcolor{black}{\texttt{white}} \textcolor{black}{\texttt{fiber}}\textcolor{black}{\texttt{.}}  \textcolor{black}{\texttt{How}} \textcolor{black}{\texttt{many}} \textcolor{black}{\texttt{bolts}} \textcolor{black}{\texttt{in}} \textcolor{black}{\texttt{total}} \textcolor{black}{\texttt{does}} \textcolor{black}{\texttt{it}} \textcolor{black}{\texttt{take}}\textcolor{black}{\texttt{?}}\\
\textcolor{fig1teal}{\texttt{def}} \textcolor{fig1teal}{\texttt{simple}}\textcolor{fig1teal}{\texttt{\_}}\textcolor{fig1teal}{\texttt{math}}\textcolor{fig1teal}{\texttt{\_}}\textcolor{fig1teal}{\texttt{problem}}\textcolor{fig1teal}{\texttt{()}} \textcolor{fig1teal}{\texttt{->}} \textcolor{fig1teal}{\texttt{int}}\textcolor{fig1teal}{\texttt{:}}\\
\mbox{\texttt{~~~~}} \textcolor{fig1teal}{\texttt{'''}}\\
\mbox{\texttt{~~~~}} \textcolor{fig1teal}{\texttt{A}} \textcolor{fig1teal}{\texttt{robe}} \textcolor{fig1teal}{\texttt{takes}} \textcolor{fig1teal}{\texttt{2}} \textcolor{fig1teal}{\texttt{bolts}} \textcolor{fig1teal}{\texttt{of}} \textcolor{fig1teal}{\texttt{blue}} \textcolor{fig1teal}{\texttt{fiber}} \textcolor{fig1teal}{\texttt{and}} \textcolor{fig1teal}{\texttt{half}} \textcolor{fig1teal}{\texttt{that}} \textcolor{fig1teal}{\texttt{much}} \textcolor{fig1teal}{\texttt{white}} \textcolor{fig1teal}{\texttt{fiber}}\textcolor{fig1teal}{\texttt{.}}\\
\mbox{\texttt{~~~~}}\\
\mbox{\texttt{~~~~}} \textcolor{fig1teal}{\texttt{How}} \textcolor{fig1teal}{\texttt{many}} \textcolor{fig1teal}{\texttt{bolts}} \textcolor{fig1teal}{\texttt{in}} \textcolor{fig1teal}{\texttt{total}} \textcolor{fig1teal}{\texttt{does}} \textcolor{fig1teal}{\texttt{it}} \textcolor{fig1teal}{\texttt{take}}\textcolor{fig1teal}{\texttt{?}}\\
\mbox{\texttt{~~~~}} \textcolor{fig1teal}{\texttt{'''}}\\
\mbox{\texttt{~~~~}} \textcolor{fig1ochre}{\texttt{blue}}\textcolor{fig1teal}{\texttt{\_}}\textcolor{fig1ochre}{\texttt{fiber}} \textcolor{fig1teal}{\texttt{=}} \textcolor{fig1teal}{\texttt{2}}\\
\mbox{\texttt{~~~~}} \textcolor{fig1ochre}{\texttt{white}}\textcolor{fig1teal}{\texttt{\_}}\textcolor{fig1ochre}{\texttt{fiber}} \textcolor{fig1teal}{\texttt{=}} \textcolor{fig1ochre}{\texttt{blue}}\textcolor{fig1ochre}{\texttt{0}}\textcolor{fig1ochre}{\texttt{fiber}}\textcolor{fig1ochre}{\texttt{5}} \textcolor{fig1ochre}{\texttt{blue}}\textcolor{fig1ochre}{\texttt{\_}}\textcolor{fig1ochre}{\texttt{fiber}}\\
\mbox{\texttt{~~~~}} \textcolor{fig1teal}{\texttt{total}}\textcolor{fig1teal}{\texttt{\_}}\textcolor{fig1ochre}{\texttt{bol}}\textcolor{fig1ochre}{\texttt{ts}} \textcolor{fig1teal}{\texttt{=}} \textcolor{fig1ochre}{\texttt{blue}}\textcolor{fig1teal}{\texttt{\_}}\textcolor{fig1ochre}{\texttt{fiber}} \textcolor{fig1teal}{\texttt{+}} \textcolor{fig1ochre}{\texttt{white}}\textcolor{fig1teal}{\texttt{\_}}\textcolor{fig1ochre}{\texttt{fiber}}\\
\mbox{\texttt{~~~~}} \textcolor{fig1teal}{\texttt{result}} \textcolor{fig1teal}{\texttt{=}} \textcolor{fig1teal}{\texttt{total}}\textcolor{fig1teal}{\texttt{\_}}\textcolor{fig1ochre}{\texttt{bol}}\textcolor{fig1teal}{\texttt{ts}}\\
\mbox{\texttt{~~~~}} \textcolor{fig1teal}{\texttt{return}} \textcolor{fig1teal}{\texttt{result}}\textcolor{fig1teal}{\texttt{[PAD]}}\textcolor{fig1teal}{\texttt{[PAD]}}\textcolor{fig1teal}{\texttt{[PAD]}}\textcolor{fig1teal}{\texttt{[PAD]}}\textcolor{fig1teal}{\texttt{[PAD]}}\textcolor{fig1teal}{\texttt{[PAD]}}\textcolor{fig1teal}{\texttt{[PAD]}}\textcolor{fig1teal}{\texttt{[PAD]}}\textcolor{fig1teal}{\texttt{[PAD]}}\textcolor{fig1teal}{\texttt{[PAD]}}
\end{samplebox}
\begin{samplebox}{\normalfont\textbf{$\ours$ (PR), Round: 20} \hfill \normalfont\scriptsize \textcolor{darkgray}{Committed: \textbf{482/483}}}
\scriptsize\linespread{0.9}\selectfont
\textcolor{black}{\texttt{A}} \textcolor{black}{\texttt{robe}} \textcolor{black}{\texttt{takes}} \textcolor{black}{\texttt{2}} \textcolor{black}{\texttt{bolts}} \textcolor{black}{\texttt{of}} \textcolor{black}{\texttt{blue}} \textcolor{black}{\texttt{fiber}} \textcolor{black}{\texttt{and}} \textcolor{black}{\texttt{half}} \textcolor{black}{\texttt{that}} \textcolor{black}{\texttt{much}} \textcolor{black}{\texttt{white}} \textcolor{black}{\texttt{fiber}}\textcolor{black}{\texttt{.}}  \textcolor{black}{\texttt{How}} \textcolor{black}{\texttt{many}} \textcolor{black}{\texttt{bolts}} \textcolor{black}{\texttt{in}} \textcolor{black}{\texttt{total}} \textcolor{black}{\texttt{does}} \textcolor{black}{\texttt{it}} \textcolor{black}{\texttt{take}}\textcolor{black}{\texttt{?}}\\
\textcolor{fig1teal}{\texttt{def}} \textcolor{fig1teal}{\texttt{simple}}\textcolor{fig1teal}{\texttt{\_}}\textcolor{fig1teal}{\texttt{math}}\textcolor{fig1teal}{\texttt{\_}}\textcolor{fig1teal}{\texttt{problem}}\textcolor{fig1teal}{\texttt{()}} \textcolor{fig1teal}{\texttt{->}} \textcolor{fig1teal}{\texttt{int}}\textcolor{fig1teal}{\texttt{:}}\\
\mbox{\texttt{~~~~}} \textcolor{fig1teal}{\texttt{'''}}\\
\mbox{\texttt{~~~~}} \textcolor{fig1teal}{\texttt{A}} \textcolor{fig1teal}{\texttt{robe}} \textcolor{fig1teal}{\texttt{takes}} \textcolor{fig1teal}{\texttt{2}} \textcolor{fig1teal}{\texttt{bolts}} \textcolor{fig1teal}{\texttt{of}} \textcolor{fig1teal}{\texttt{blue}} \textcolor{fig1teal}{\texttt{fiber}} \textcolor{fig1teal}{\texttt{and}} \textcolor{fig1teal}{\texttt{half}} \textcolor{fig1teal}{\texttt{that}} \textcolor{fig1teal}{\texttt{much}} \textcolor{fig1teal}{\texttt{white}} \textcolor{fig1teal}{\texttt{fiber}}\textcolor{fig1teal}{\texttt{.}}\\
\mbox{\texttt{~~~~}}\\
\mbox{\texttt{~~~~}} \textcolor{fig1teal}{\texttt{How}} \textcolor{fig1teal}{\texttt{many}} \textcolor{fig1teal}{\texttt{bolts}} \textcolor{fig1teal}{\texttt{in}} \textcolor{fig1teal}{\texttt{total}} \textcolor{fig1teal}{\texttt{does}} \textcolor{fig1teal}{\texttt{it}} \textcolor{fig1teal}{\texttt{take}}\textcolor{fig1teal}{\texttt{?}}\\
\mbox{\texttt{~~~~}} \textcolor{fig1teal}{\texttt{'''}}\\
\mbox{\texttt{~~~~}} \textcolor{fig1teal}{\texttt{blue}}\textcolor{fig1teal}{\texttt{\_}}\textcolor{fig1teal}{\texttt{fiber}} \textcolor{fig1teal}{\texttt{=}} \textcolor{fig1teal}{\texttt{2}}\\
\mbox{\texttt{~~~~}} \textcolor{fig1teal}{\texttt{white}}\textcolor{fig1teal}{\texttt{\_}}\textcolor{fig1teal}{\texttt{fiber}} \textcolor{fig1teal}{\texttt{=}} \textcolor{fig1teal}{\texttt{blue}}\textcolor{fig1teal}{\texttt{\_}}\textcolor{fig1teal}{\texttt{fiber}}\textcolor{fig1ochre}{\texttt{*}}\textcolor{fig1teal}{\texttt{0}}\textcolor{fig1teal}{\texttt{.}}\textcolor{fig1teal}{\texttt{5}}\\
\mbox{\texttt{~~~~}} \textcolor{fig1teal}{\texttt{total}}\textcolor{fig1teal}{\texttt{\_}}\textcolor{fig1teal}{\texttt{bol}}\textcolor{fig1teal}{\texttt{ts}} \textcolor{fig1teal}{\texttt{=}} \textcolor{fig1teal}{\texttt{blue}}\textcolor{fig1teal}{\texttt{\_}}\textcolor{fig1teal}{\texttt{fiber}} \textcolor{fig1teal}{\texttt{+}} \textcolor{fig1teal}{\texttt{white}}\textcolor{fig1teal}{\texttt{\_}}\textcolor{fig1teal}{\texttt{fiber}}\\
\mbox{\texttt{~~~~}} \textcolor{fig1teal}{\texttt{result}} \textcolor{fig1teal}{\texttt{=}} \textcolor{fig1teal}{\texttt{total}}\textcolor{fig1teal}{\texttt{\_}}\textcolor{fig1teal}{\texttt{bol}}\textcolor{fig1teal}{\texttt{ts}}\\
\mbox{\texttt{~~~~}} \textcolor{fig1teal}{\texttt{return}} \textcolor{fig1teal}{\texttt{result}}\textcolor{fig1teal}{\texttt{[PAD]}}\textcolor{fig1teal}{\texttt{[PAD]}}\textcolor{fig1teal}{\texttt{[PAD]}}\textcolor{fig1teal}{\texttt{[PAD]}}\textcolor{fig1teal}{\texttt{[PAD]}}\textcolor{fig1teal}{\texttt{[PAD]}}\textcolor{fig1teal}{\texttt{[PAD]}}\textcolor{fig1teal}{\texttt{[PAD]}}\textcolor{fig1teal}{\texttt{[PAD]}}\textcolor{fig1teal}{\texttt{[PAD]}}
\end{samplebox}
\begin{samplebox}{\normalfont\textbf{$\ours$ (PR), Round: 32} \hfill \normalfont\scriptsize \textcolor{darkgray}{Committed: \textbf{483/483}}}
\scriptsize\linespread{0.9}\selectfont
\textcolor{black}{\texttt{A}} \textcolor{black}{\texttt{robe}} \textcolor{black}{\texttt{takes}} \textcolor{black}{\texttt{2}} \textcolor{black}{\texttt{bolts}} \textcolor{black}{\texttt{of}} \textcolor{black}{\texttt{blue}} \textcolor{black}{\texttt{fiber}} \textcolor{black}{\texttt{and}} \textcolor{black}{\texttt{half}} \textcolor{black}{\texttt{that}} \textcolor{black}{\texttt{much}} \textcolor{black}{\texttt{white}} \textcolor{black}{\texttt{fiber}}\textcolor{black}{\texttt{.}}  \textcolor{black}{\texttt{How}} \textcolor{black}{\texttt{many}} \textcolor{black}{\texttt{bolts}} \textcolor{black}{\texttt{in}} \textcolor{black}{\texttt{total}} \textcolor{black}{\texttt{does}} \textcolor{black}{\texttt{it}} \textcolor{black}{\texttt{take}}\textcolor{black}{\texttt{?}}\\
\textcolor{fig1teal}{\texttt{def}} \textcolor{fig1teal}{\texttt{simple}}\textcolor{fig1teal}{\texttt{\_}}\textcolor{fig1teal}{\texttt{math}}\textcolor{fig1teal}{\texttt{\_}}\textcolor{fig1teal}{\texttt{problem}}\textcolor{fig1teal}{\texttt{()}} \textcolor{fig1teal}{\texttt{->}} \textcolor{fig1teal}{\texttt{int}}\textcolor{fig1teal}{\texttt{:}}\\
\mbox{\texttt{~~~~}} \textcolor{fig1teal}{\texttt{'''}}\\
\mbox{\texttt{~~~~}} \textcolor{fig1teal}{\texttt{A}} \textcolor{fig1teal}{\texttt{robe}} \textcolor{fig1teal}{\texttt{takes}} \textcolor{fig1teal}{\texttt{2}} \textcolor{fig1teal}{\texttt{bolts}} \textcolor{fig1teal}{\texttt{of}} \textcolor{fig1teal}{\texttt{blue}} \textcolor{fig1teal}{\texttt{fiber}} \textcolor{fig1teal}{\texttt{and}} \textcolor{fig1teal}{\texttt{half}} \textcolor{fig1teal}{\texttt{that}} \textcolor{fig1teal}{\texttt{much}} \textcolor{fig1teal}{\texttt{white}} \textcolor{fig1teal}{\texttt{fiber}}\textcolor{fig1teal}{\texttt{.}}\\
\mbox{\texttt{~~~~}}\\
\mbox{\texttt{~~~~}} \textcolor{fig1teal}{\texttt{How}} \textcolor{fig1teal}{\texttt{many}} \textcolor{fig1teal}{\texttt{bolts}} \textcolor{fig1teal}{\texttt{in}} \textcolor{fig1teal}{\texttt{total}} \textcolor{fig1teal}{\texttt{does}} \textcolor{fig1teal}{\texttt{it}} \textcolor{fig1teal}{\texttt{take}}\textcolor{fig1teal}{\texttt{?}}\\
\mbox{\texttt{~~~~}} \textcolor{fig1teal}{\texttt{'''}}\\
\mbox{\texttt{~~~~}} \textcolor{fig1teal}{\texttt{blue}}\textcolor{fig1teal}{\texttt{\_}}\textcolor{fig1teal}{\texttt{fiber}} \textcolor{fig1teal}{\texttt{=}} \textcolor{fig1teal}{\texttt{2}}\\
\mbox{\texttt{~~~~}} \textcolor{fig1teal}{\texttt{white}}\textcolor{fig1teal}{\texttt{\_}}\textcolor{fig1teal}{\texttt{fiber}} \textcolor{fig1teal}{\texttt{=}} \textcolor{fig1teal}{\texttt{blue}}\textcolor{fig1teal}{\texttt{\_}}\textcolor{fig1teal}{\texttt{fiber}}\textcolor{fig1teal}{\texttt{*}}\textcolor{fig1teal}{\texttt{0}}\textcolor{fig1teal}{\texttt{.}}\textcolor{fig1teal}{\texttt{5}}\\
\mbox{\texttt{~~~~}} \textcolor{fig1teal}{\texttt{total}}\textcolor{fig1teal}{\texttt{\_}}\textcolor{fig1teal}{\texttt{bol}}\textcolor{fig1teal}{\texttt{ts}} \textcolor{fig1teal}{\texttt{=}} \textcolor{fig1teal}{\texttt{blue}}\textcolor{fig1teal}{\texttt{\_}}\textcolor{fig1teal}{\texttt{fiber}} \textcolor{fig1teal}{\texttt{+}} \textcolor{fig1teal}{\texttt{white}}\textcolor{fig1teal}{\texttt{\_}}\textcolor{fig1teal}{\texttt{fiber}}\\
\mbox{\texttt{~~~~}} \textcolor{fig1teal}{\texttt{result}} \textcolor{fig1teal}{\texttt{=}} \textcolor{fig1teal}{\texttt{total}}\textcolor{fig1teal}{\texttt{\_}}\textcolor{fig1teal}{\texttt{bol}}\textcolor{fig1teal}{\texttt{ts}}\\
\mbox{\texttt{~~~~}} \textcolor{fig1teal}{\texttt{return}} \textcolor{fig1teal}{\texttt{result}}\textcolor{fig1teal}{\texttt{[PAD]}}\textcolor{fig1teal}{\texttt{[PAD]}}\textcolor{fig1teal}{\texttt{[PAD]}}\textcolor{fig1teal}{\texttt{[PAD]}}\textcolor{fig1teal}{\texttt{[PAD]}}\textcolor{fig1teal}{\texttt{[PAD]}}\textcolor{fig1teal}{\texttt{[PAD]}}\textcolor{fig1teal}{\texttt{[PAD]}}\textcolor{fig1teal}{\texttt{[PAD]}}\textcolor{fig1teal}{\texttt{[PAD]}}
\end{samplebox}
\caption{Posterior Refinement trajectory on TinyGSM, Sample-1.}
\label{fig:refine_sample_tinygsm_1}
\end{figure}

\begin{figure}[H]
\centering
\begin{samplebox}{\normalfont\textbf{$\ours$ (PR), Round: 1} \hfill \normalfont\scriptsize \textcolor{darkgray}{Committed: \textbf{348/455}}}
\scriptsize\linespread{0.9}\selectfont
\textcolor{black}{\texttt{In}} \textcolor{black}{\texttt{a}} \textcolor{black}{\texttt{dance}} \textcolor{black}{\texttt{class}} \textcolor{black}{\texttt{of}} \textcolor{black}{\texttt{2}}\textcolor{black}{\texttt{0}} \textcolor{black}{\texttt{students}}\textcolor{black}{\texttt{,}} \textcolor{black}{\texttt{2}}\textcolor{black}{\texttt{0}}\textcolor{black}{\texttt{\%}} \textcolor{black}{\texttt{enrolled}} \textcolor{black}{\texttt{in}} \textcolor{black}{\texttt{contemporary}} \textcolor{black}{\texttt{dance}}\textcolor{black}{\texttt{,}} \textcolor{black}{\texttt{2}}\textcolor{black}{\texttt{5}}\textcolor{black}{\texttt{\%}} \textcolor{black}{\texttt{of}} \textcolor{black}{\texttt{the}} \textcolor{black}{\texttt{remaining}} \textcolor{black}{\texttt{enrolled}} \textcolor{black}{\texttt{in}} \textcolor{black}{\texttt{jazz}} \textcolor{black}{\texttt{dance}}\textcolor{black}{\texttt{,}} \textcolor{black}{\texttt{and}} \textcolor{black}{\texttt{the}} \textcolor{black}{\texttt{rest}} \textcolor{black}{\texttt{enrolled}} \textcolor{black}{\texttt{in}} \textcolor{black}{\texttt{hip}}\textcolor{black}{\texttt{-}}\textcolor{black}{\texttt{hop}} \textcolor{black}{\texttt{dance}}\textcolor{black}{\texttt{.}} \textcolor{black}{\texttt{What}} \textcolor{black}{\texttt{percentage}} \textcolor{black}{\texttt{of}} \textcolor{black}{\texttt{the}} \textcolor{black}{\texttt{entire}} \textcolor{black}{\texttt{students}} \textcolor{black}{\texttt{enrolled}} \textcolor{black}{\texttt{in}} \textcolor{black}{\texttt{hip}}\textcolor{black}{\texttt{-}}\textcolor{black}{\texttt{hop}} \textcolor{black}{\texttt{dance}}\textcolor{black}{\texttt{?}}\\
\textcolor{fig1teal}{\texttt{def}} \textcolor{fig1teal}{\texttt{simple}}\textcolor{fig1teal}{\texttt{\_}}\textcolor{fig1teal}{\texttt{math}}\textcolor{fig1teal}{\texttt{\_}}\textcolor{fig1teal}{\texttt{problem}}\textcolor{fig1teal}{\texttt{()}} \textcolor{fig1teal}{\texttt{->}} \textcolor{fig1teal}{\texttt{int}}\textcolor{fig1teal}{\texttt{:}}\\
\mbox{\texttt{~~~~}} \textcolor{fig1ochre}{\texttt{'''}}\\
\mbox{\texttt{~~~~}} \textcolor{fig1teal}{\texttt{In}} \textcolor{fig1teal}{\texttt{a}} \textcolor{fig1teal}{\texttt{dance}} \textcolor{fig1teal}{\texttt{class}} \textcolor{fig1teal}{\texttt{of}} \textcolor{fig1teal}{\texttt{2}}\textcolor{fig1teal}{\texttt{0}} \textcolor{fig1teal}{\texttt{students}}\textcolor{fig1teal}{\texttt{,}} \textcolor{fig1teal}{\texttt{2}}\textcolor{fig1teal}{\texttt{0}}\textcolor{fig1teal}{\texttt{\%}} \textcolor{fig1teal}{\texttt{enrolled}} \textcolor{fig1teal}{\texttt{in}} \textcolor{fig1teal}{\texttt{contemporary}} \textcolor{fig1teal}{\texttt{dance}}\textcolor{fig1teal}{\texttt{,}} \textcolor{fig1teal}{\texttt{2}}\textcolor{fig1teal}{\texttt{5}}\textcolor{fig1teal}{\texttt{\%}} \textcolor{fig1teal}{\texttt{of}} \textcolor{fig1teal}{\texttt{the}} \textcolor{fig1teal}{\texttt{remaining}} \textcolor{fig1teal}{\texttt{enrolled}} \textcolor{fig1teal}{\texttt{in}} \textcolor{fig1teal}{\texttt{jazz}} \textcolor{fig1teal}{\texttt{dance}}\textcolor{fig1teal}{\texttt{,}} \textcolor{fig1teal}{\texttt{and}} \textcolor{fig1teal}{\texttt{the}} \textcolor{fig1teal}{\texttt{rest}} \textcolor{fig1teal}{\texttt{enrolled}} \textcolor{fig1teal}{\texttt{in}} \textcolor{fig1teal}{\texttt{hip}}\textcolor{fig1teal}{\texttt{-}}\textcolor{fig1teal}{\texttt{hop}} \textcolor{fig1teal}{\texttt{dance}}\textcolor{fig1teal}{\texttt{.}}\\
\mbox{\texttt{~~~~}} \textcolor{fig1teal}{\texttt{What}} \textcolor{fig1teal}{\texttt{percentage}} \textcolor{fig1teal}{\texttt{of}} \textcolor{fig1teal}{\texttt{the}} \textcolor{fig1teal}{\texttt{entire}} \textcolor{fig1teal}{\texttt{students}} \textcolor{fig1teal}{\texttt{enrolled}} \textcolor{fig1teal}{\texttt{in}} \textcolor{fig1teal}{\texttt{hip}}\textcolor{fig1teal}{\texttt{-}}\textcolor{fig1teal}{\texttt{hop}} \textcolor{fig1teal}{\texttt{dance}}\textcolor{fig1teal}{\texttt{?}}\\
\mbox{\texttt{~~~~}} \textcolor{fig1ochre}{\texttt{'''}}\\
\mbox{\texttt{~~~~}} \textcolor{fig1ochre}{\texttt{total}}\textcolor{fig1teal}{\texttt{\_}}\textcolor{fig1ochre}{\texttt{students}} \textcolor{fig1ochre}{\texttt{=}} \textcolor{fig1ochre}{\texttt{2}}\textcolor{fig1ochre}{\texttt{0}}\\
\mbox{\texttt{~~~~}} \textcolor{fig1ochre}{\texttt{contemporary}}\textcolor{fig1ochre}{\texttt{\_}}\textcolor{fig1ochre}{\texttt{students}}\textcolor{fig1ochre}{\texttt{ance}} \textcolor{fig1ochre}{\texttt{=}}\textcolor{fig1ochre}{\texttt{(}}\textcolor{fig1ochre}{\texttt{\_}}\textcolor{fig1ochre}{\texttt{students}} \textcolor{fig1ochre}{\texttt{*}} \textcolor{fig1ochre}{\texttt{0}}\textcolor{fig1ochre}{\texttt{.}}\textcolor{fig1ochre}{\texttt{2}}\\
\mbox{\texttt{~~~~}} \textcolor{fig1ochre}{\texttt{remaining}} \textcolor{fig1ochre}{\texttt{remaining}}\textcolor{fig1ochre}{\texttt{\_}}\textcolor{fig1ochre}{\texttt{students}} \textcolor{fig1ochre}{\texttt{=}} \textcolor{fig1ochre}{\texttt{total}}\textcolor{fig1ochre}{\texttt{\_}}\textcolor{fig1ochre}{\texttt{students}} \textcolor{fig1ochre}{\texttt{-}} \textcolor{fig1ochre}{\texttt{total}}\textcolor{fig1ochre}{\texttt{\_}}\textcolor{fig1ochre}{\texttt{students}} \textcolor{fig1ochre}{\texttt{-}} \textcolor{fig1ochre}{\texttt{contemporary}}\textcolor{fig1ochre}{\texttt{\_}}\textcolor{fig1ochre}{\texttt{students}}\textcolor{fig1ochre}{\texttt{ance}} \textcolor{fig1ochre}{\texttt{jazz}}\textcolor{fig1ochre}{\texttt{\_}}\textcolor{fig1ochre}{\texttt{d}}\textcolor{fig1ochre}{\texttt{ance}}\textcolor{fig1ochre}{\texttt{\_}}\textcolor{fig1ochre}{\texttt{students}}\textcolor{fig1ochre}{\texttt{students}} \textcolor{fig1ochre}{\texttt{0}}\textcolor{fig1ochre}{\texttt{.}}\textcolor{fig1ochre}{\texttt{2}}\textcolor{fig1ochre}{\texttt{2}}\textcolor{fig1ochre}{\texttt{5}}\\
\mbox{\texttt{~~~~}} \textcolor{fig1ochre}{\texttt{hip}}\textcolor{fig1ochre}{\texttt{\_}}\textcolor{fig1ochre}{\texttt{hop}}\textcolor{fig1ochre}{\texttt{\_}}\textcolor{fig1ochre}{\texttt{hop}} \textcolor{fig1ochre}{\texttt{=}} \textcolor{fig1ochre}{\texttt{(}}\textcolor{fig1ochre}{\texttt{hop}}\textcolor{fig1ochre}{\texttt{\_}}\textcolor{fig1ochre}{\texttt{\_}}\textcolor{fig1ochre}{\texttt{students}} \textcolor{fig1ochre}{\texttt{(}} \textcolor{fig1ochre}{\texttt{jazz}}\textcolor{fig1ochre}{\texttt{\_}}\textcolor{fig1ochre}{\texttt{d}}\textcolor{fig1ochre}{\texttt{ance}}\textcolor{fig1ochre}{\texttt{\_}}\textcolor{fig1ochre}{\texttt{ance}}\textcolor{fig1ochre}{\texttt{\_}}\textcolor{fig1ochre}{\texttt{students}}\textcolor{fig1ochre}{\texttt{1}}\textcolor{fig1ochre}{\texttt{0}}\textcolor{fig1ochre}{\texttt{0}}\textcolor{fig1ochre}{\texttt{/}}\textcolor{fig1ochre}{\texttt{total}}\textcolor{fig1ochre}{\texttt{\_}}\textcolor{fig1ochre}{\texttt{students}}\\
\mbox{\texttt{~~~~}}\textcolor{fig1ochre}{\texttt{)}} \textcolor{fig1ochre}{\texttt{hip}}\textcolor{fig1ochre}{\texttt{hop}}\textcolor{fig1ochre}{\texttt{1}}\textcolor{fig1ochre}{\texttt{0}}\textcolor{fig1ochre}{\texttt{ance}}\\
\mbox{\texttt{~~~~}} \textcolor{fig1ochre}{\texttt{return}} \textcolor{fig1ochre}{\texttt{result}}\textcolor{fig1ochre}{\texttt{[PAD]}}\textcolor{fig1ochre}{\texttt{[PAD]}}\textcolor{fig1ochre}{\texttt{[PAD]}}\textcolor{fig1ochre}{\texttt{[PAD]}}\textcolor{fig1ochre}{\texttt{[PAD]}}\textcolor{fig1ochre}{\texttt{[PAD]}}\textcolor{fig1ochre}{\texttt{[PAD]}}\textcolor{fig1ochre}{\texttt{[PAD]}}\textcolor{fig1ochre}{\texttt{[PAD]}}\textcolor{fig1ochre}{\texttt{[PAD]}}
\end{samplebox}
\begin{samplebox}{\normalfont\textbf{$\ours$ (PR), Round: 10} \hfill \normalfont\scriptsize \textcolor{darkgray}{Committed: \textbf{410/455}}}
\scriptsize\linespread{0.9}\selectfont
\textcolor{black}{\texttt{In}} \textcolor{black}{\texttt{a}} \textcolor{black}{\texttt{dance}} \textcolor{black}{\texttt{class}} \textcolor{black}{\texttt{of}} \textcolor{black}{\texttt{2}}\textcolor{black}{\texttt{0}} \textcolor{black}{\texttt{students}}\textcolor{black}{\texttt{,}} \textcolor{black}{\texttt{2}}\textcolor{black}{\texttt{0}}\textcolor{black}{\texttt{\%}} \textcolor{black}{\texttt{enrolled}} \textcolor{black}{\texttt{in}} \textcolor{black}{\texttt{contemporary}} \textcolor{black}{\texttt{dance}}\textcolor{black}{\texttt{,}} \textcolor{black}{\texttt{2}}\textcolor{black}{\texttt{5}}\textcolor{black}{\texttt{\%}} \textcolor{black}{\texttt{of}} \textcolor{black}{\texttt{the}} \textcolor{black}{\texttt{remaining}} \textcolor{black}{\texttt{enrolled}} \textcolor{black}{\texttt{in}} \textcolor{black}{\texttt{jazz}} \textcolor{black}{\texttt{dance}}\textcolor{black}{\texttt{,}} \textcolor{black}{\texttt{and}} \textcolor{black}{\texttt{the}} \textcolor{black}{\texttt{rest}} \textcolor{black}{\texttt{enrolled}} \textcolor{black}{\texttt{in}} \textcolor{black}{\texttt{hip}}\textcolor{black}{\texttt{-}}\textcolor{black}{\texttt{hop}} \textcolor{black}{\texttt{dance}}\textcolor{black}{\texttt{.}} \textcolor{black}{\texttt{What}} \textcolor{black}{\texttt{percentage}} \textcolor{black}{\texttt{of}} \textcolor{black}{\texttt{the}} \textcolor{black}{\texttt{entire}} \textcolor{black}{\texttt{students}} \textcolor{black}{\texttt{enrolled}} \textcolor{black}{\texttt{in}} \textcolor{black}{\texttt{hip}}\textcolor{black}{\texttt{-}}\textcolor{black}{\texttt{hop}} \textcolor{black}{\texttt{dance}}\textcolor{black}{\texttt{?}}\\
\textcolor{fig1teal}{\texttt{def}} \textcolor{fig1teal}{\texttt{simple}}\textcolor{fig1teal}{\texttt{\_}}\textcolor{fig1teal}{\texttt{math}}\textcolor{fig1teal}{\texttt{\_}}\textcolor{fig1teal}{\texttt{problem}}\textcolor{fig1teal}{\texttt{()}} \textcolor{fig1teal}{\texttt{->}} \textcolor{fig1teal}{\texttt{int}}\textcolor{fig1teal}{\texttt{:}}\\
\mbox{\texttt{~~~~}} \textcolor{fig1teal}{\texttt{'''}}\\
\mbox{\texttt{~~~~}} \textcolor{fig1teal}{\texttt{In}} \textcolor{fig1teal}{\texttt{a}} \textcolor{fig1teal}{\texttt{dance}} \textcolor{fig1teal}{\texttt{class}} \textcolor{fig1teal}{\texttt{of}} \textcolor{fig1teal}{\texttt{2}}\textcolor{fig1teal}{\texttt{0}} \textcolor{fig1teal}{\texttt{students}}\textcolor{fig1teal}{\texttt{,}} \textcolor{fig1teal}{\texttt{2}}\textcolor{fig1teal}{\texttt{0}}\textcolor{fig1teal}{\texttt{\%}} \textcolor{fig1teal}{\texttt{enrolled}} \textcolor{fig1teal}{\texttt{in}} \textcolor{fig1teal}{\texttt{contemporary}} \textcolor{fig1teal}{\texttt{dance}}\textcolor{fig1teal}{\texttt{,}} \textcolor{fig1teal}{\texttt{2}}\textcolor{fig1teal}{\texttt{5}}\textcolor{fig1teal}{\texttt{\%}} \textcolor{fig1teal}{\texttt{of}} \textcolor{fig1teal}{\texttt{the}} \textcolor{fig1teal}{\texttt{remaining}} \textcolor{fig1teal}{\texttt{enrolled}} \textcolor{fig1teal}{\texttt{in}} \textcolor{fig1teal}{\texttt{jazz}} \textcolor{fig1teal}{\texttt{dance}}\textcolor{fig1teal}{\texttt{,}} \textcolor{fig1teal}{\texttt{and}} \textcolor{fig1teal}{\texttt{the}} \textcolor{fig1teal}{\texttt{rest}} \textcolor{fig1teal}{\texttt{enrolled}} \textcolor{fig1teal}{\texttt{in}} \textcolor{fig1teal}{\texttt{hip}}\textcolor{fig1teal}{\texttt{-}}\textcolor{fig1teal}{\texttt{hop}} \textcolor{fig1teal}{\texttt{dance}}\textcolor{fig1teal}{\texttt{.}}\\
\mbox{\texttt{~~~~}} \textcolor{fig1teal}{\texttt{What}} \textcolor{fig1teal}{\texttt{percentage}} \textcolor{fig1teal}{\texttt{of}} \textcolor{fig1teal}{\texttt{the}} \textcolor{fig1teal}{\texttt{entire}} \textcolor{fig1teal}{\texttt{students}} \textcolor{fig1teal}{\texttt{enrolled}} \textcolor{fig1teal}{\texttt{in}} \textcolor{fig1teal}{\texttt{hip}}\textcolor{fig1teal}{\texttt{-}}\textcolor{fig1teal}{\texttt{hop}} \textcolor{fig1teal}{\texttt{dance}}\textcolor{fig1teal}{\texttt{?}}\\
\mbox{\texttt{~~~~}} \textcolor{fig1teal}{\texttt{'''}}\\
\mbox{\texttt{~~~~}} \textcolor{fig1teal}{\texttt{total}}\textcolor{fig1teal}{\texttt{\_}}\textcolor{fig1teal}{\texttt{students}} \textcolor{fig1teal}{\texttt{=}} \textcolor{fig1teal}{\texttt{2}}\textcolor{fig1teal}{\texttt{0}}\\
\mbox{\texttt{~~~~}} \textcolor{fig1ochre}{\texttt{contemporary}}\textcolor{fig1teal}{\texttt{\_}}\textcolor{fig1ochre}{\texttt{students}} \textcolor{fig1teal}{\texttt{=}} \textcolor{fig1teal}{\texttt{total}}\textcolor{fig1teal}{\texttt{\_}}\textcolor{fig1teal}{\texttt{students}} \textcolor{fig1teal}{\texttt{*}} \textcolor{fig1teal}{\texttt{0}}\textcolor{fig1teal}{\texttt{.}}\textcolor{fig1teal}{\texttt{2}}\\
\mbox{\texttt{~~~~}} \textcolor{fig1ochre}{\texttt{remaining}}\textcolor{fig1teal}{\texttt{\_}}\textcolor{fig1teal}{\texttt{students}} \textcolor{fig1teal}{\texttt{=}} \textcolor{fig1teal}{\texttt{total}}\textcolor{fig1teal}{\texttt{\_}}\textcolor{fig1teal}{\texttt{students}} \textcolor{fig1teal}{\texttt{-}} \textcolor{fig1ochre}{\texttt{contemporary}}\textcolor{fig1teal}{\texttt{\_}}\textcolor{fig1ochre}{\texttt{students}}\\
\mbox{\texttt{~~~~}} \textcolor{fig1ochre}{\texttt{jazz}}\textcolor{fig1ochre}{\texttt{\_}}\textcolor{fig1ochre}{\texttt{\_}}\textcolor{fig1ochre}{\texttt{ance}} \textcolor{fig1teal}{\texttt{=}} \textcolor{fig1ochre}{\texttt{remaining}}\textcolor{fig1teal}{\texttt{\_}}\textcolor{fig1teal}{\texttt{students}} \textcolor{fig1teal}{\texttt{*}} \textcolor{fig1teal}{\texttt{0}}\textcolor{fig1teal}{\texttt{.}}\textcolor{fig1teal}{\texttt{2}}\textcolor{fig1teal}{\texttt{5}}\\
\mbox{\texttt{~~~~}} \textcolor{fig1ochre}{\texttt{hip}}\textcolor{fig1ochre}{\texttt{hop}}\textcolor{fig1ochre}{\texttt{\_}}\textcolor{fig1ochre}{\texttt{students}} \textcolor{fig1ochre}{\texttt{=}} \textcolor{fig1ochre}{\texttt{remaining}}\textcolor{fig1ochre}{\texttt{students}}\textcolor{fig1ochre}{\texttt{students}}\textcolor{fig1ochre}{\texttt{students}} \textcolor{fig1ochre}{\texttt{-}} \textcolor{fig1ochre}{\texttt{jazz}}\textcolor{fig1ochre}{\texttt{\_}}\textcolor{fig1ochre}{\texttt{students}}\textcolor{fig1ochre}{\texttt{ance}}\\
\mbox{\texttt{~~~~}} \textcolor{fig1ochre}{\texttt{hip}}\textcolor{fig1ochre}{\texttt{\_}}\textcolor{fig1ochre}{\texttt{hop}}\textcolor{fig1teal}{\texttt{\_}}\textcolor{fig1ochre}{\texttt{students}} \textcolor{fig1teal}{\texttt{=}} \textcolor{fig1teal}{\texttt{(}}\textcolor{fig1ochre}{\texttt{hip}}\textcolor{fig1ochre}{\texttt{hop}}\textcolor{fig1ochre}{\texttt{\_}}\textcolor{fig1ochre}{\texttt{students}} \textcolor{fig1ochre}{\texttt{/}} \textcolor{fig1ochre}{\texttt{total}}\textcolor{fig1teal}{\texttt{\_}}\textcolor{fig1teal}{\texttt{students}}\textcolor{fig1teal}{\texttt{)}} \textcolor{fig1teal}{\texttt{*}} \textcolor{fig1teal}{\texttt{1}}\textcolor{fig1teal}{\texttt{0}}\textcolor{fig1teal}{\texttt{0}}\\
\mbox{\texttt{~~~~}} \textcolor{fig1teal}{\texttt{result}} \textcolor{fig1teal}{\texttt{=}} \textcolor{fig1ochre}{\texttt{hip}}\textcolor{fig1ochre}{\texttt{\_}}\textcolor{fig1ochre}{\texttt{hip}}\textcolor{fig1ochre}{\texttt{hop}}\textcolor{fig1ochre}{\texttt{\_}}\textcolor{fig1ochre}{\texttt{percentage}}\\
\mbox{\texttt{~~~~}} \textcolor{fig1ochre}{\texttt{result}}\textcolor{fig1ochre}{\texttt{[PAD]}}\textcolor{fig1teal}{\texttt{[PAD]}}\textcolor{fig1teal}{\texttt{[PAD]}}\textcolor{fig1teal}{\texttt{[PAD]}}\textcolor{fig1teal}{\texttt{[PAD]}}\textcolor{fig1teal}{\texttt{[PAD]}}\textcolor{fig1teal}{\texttt{[PAD]}}\textcolor{fig1teal}{\texttt{[PAD]}}\textcolor{fig1teal}{\texttt{[PAD]}}\textcolor{fig1teal}{\texttt{[PAD]}}
\end{samplebox}
\begin{samplebox}{\normalfont\textbf{$\ours$ (PR), Round: 20} \hfill \normalfont\scriptsize \textcolor{darkgray}{Committed: \textbf{437/455}}}
\scriptsize\linespread{0.9}\selectfont
\textcolor{black}{\texttt{In}} \textcolor{black}{\texttt{a}} \textcolor{black}{\texttt{dance}} \textcolor{black}{\texttt{class}} \textcolor{black}{\texttt{of}} \textcolor{black}{\texttt{2}}\textcolor{black}{\texttt{0}} \textcolor{black}{\texttt{students}}\textcolor{black}{\texttt{,}} \textcolor{black}{\texttt{2}}\textcolor{black}{\texttt{0}}\textcolor{black}{\texttt{\%}} \textcolor{black}{\texttt{enrolled}} \textcolor{black}{\texttt{in}} \textcolor{black}{\texttt{contemporary}} \textcolor{black}{\texttt{dance}}\textcolor{black}{\texttt{,}} \textcolor{black}{\texttt{2}}\textcolor{black}{\texttt{5}}\textcolor{black}{\texttt{\%}} \textcolor{black}{\texttt{of}} \textcolor{black}{\texttt{the}} \textcolor{black}{\texttt{remaining}} \textcolor{black}{\texttt{enrolled}} \textcolor{black}{\texttt{in}} \textcolor{black}{\texttt{jazz}} \textcolor{black}{\texttt{dance}}\textcolor{black}{\texttt{,}} \textcolor{black}{\texttt{and}} \textcolor{black}{\texttt{the}} \textcolor{black}{\texttt{rest}} \textcolor{black}{\texttt{enrolled}} \textcolor{black}{\texttt{in}} \textcolor{black}{\texttt{hip}}\textcolor{black}{\texttt{-}}\textcolor{black}{\texttt{hop}} \textcolor{black}{\texttt{dance}}\textcolor{black}{\texttt{.}} \textcolor{black}{\texttt{What}} \textcolor{black}{\texttt{percentage}} \textcolor{black}{\texttt{of}} \textcolor{black}{\texttt{the}} \textcolor{black}{\texttt{entire}} \textcolor{black}{\texttt{students}} \textcolor{black}{\texttt{enrolled}} \textcolor{black}{\texttt{in}} \textcolor{black}{\texttt{hip}}\textcolor{black}{\texttt{-}}\textcolor{black}{\texttt{hop}} \textcolor{black}{\texttt{dance}}\textcolor{black}{\texttt{?}}\\
\textcolor{fig1teal}{\texttt{def}} \textcolor{fig1teal}{\texttt{simple}}\textcolor{fig1teal}{\texttt{\_}}\textcolor{fig1teal}{\texttt{math}}\textcolor{fig1teal}{\texttt{\_}}\textcolor{fig1teal}{\texttt{problem}}\textcolor{fig1teal}{\texttt{()}} \textcolor{fig1teal}{\texttt{->}} \textcolor{fig1teal}{\texttt{int}}\textcolor{fig1teal}{\texttt{:}}\\
\mbox{\texttt{~~~~}} \textcolor{fig1teal}{\texttt{'''}}\\
\mbox{\texttt{~~~~}} \textcolor{fig1teal}{\texttt{In}} \textcolor{fig1teal}{\texttt{a}} \textcolor{fig1teal}{\texttt{dance}} \textcolor{fig1teal}{\texttt{class}} \textcolor{fig1teal}{\texttt{of}} \textcolor{fig1teal}{\texttt{2}}\textcolor{fig1teal}{\texttt{0}} \textcolor{fig1teal}{\texttt{students}}\textcolor{fig1teal}{\texttt{,}} \textcolor{fig1teal}{\texttt{2}}\textcolor{fig1teal}{\texttt{0}}\textcolor{fig1teal}{\texttt{\%}} \textcolor{fig1teal}{\texttt{enrolled}} \textcolor{fig1teal}{\texttt{in}} \textcolor{fig1teal}{\texttt{contemporary}} \textcolor{fig1teal}{\texttt{dance}}\textcolor{fig1teal}{\texttt{,}} \textcolor{fig1teal}{\texttt{2}}\textcolor{fig1teal}{\texttt{5}}\textcolor{fig1teal}{\texttt{\%}} \textcolor{fig1teal}{\texttt{of}} \textcolor{fig1teal}{\texttt{the}} \textcolor{fig1teal}{\texttt{remaining}} \textcolor{fig1teal}{\texttt{enrolled}} \textcolor{fig1teal}{\texttt{in}} \textcolor{fig1teal}{\texttt{jazz}} \textcolor{fig1teal}{\texttt{dance}}\textcolor{fig1teal}{\texttt{,}} \textcolor{fig1teal}{\texttt{and}} \textcolor{fig1teal}{\texttt{the}} \textcolor{fig1teal}{\texttt{rest}} \textcolor{fig1teal}{\texttt{enrolled}} \textcolor{fig1teal}{\texttt{in}} \textcolor{fig1teal}{\texttt{hip}}\textcolor{fig1teal}{\texttt{-}}\textcolor{fig1teal}{\texttt{hop}} \textcolor{fig1teal}{\texttt{dance}}\textcolor{fig1teal}{\texttt{.}}\\
\mbox{\texttt{~~~~}} \textcolor{fig1teal}{\texttt{What}} \textcolor{fig1teal}{\texttt{percentage}} \textcolor{fig1teal}{\texttt{of}} \textcolor{fig1teal}{\texttt{the}} \textcolor{fig1teal}{\texttt{entire}} \textcolor{fig1teal}{\texttt{students}} \textcolor{fig1teal}{\texttt{enrolled}} \textcolor{fig1teal}{\texttt{in}} \textcolor{fig1teal}{\texttt{hip}}\textcolor{fig1teal}{\texttt{-}}\textcolor{fig1teal}{\texttt{hop}} \textcolor{fig1teal}{\texttt{dance}}\textcolor{fig1teal}{\texttt{?}}\\
\mbox{\texttt{~~~~}} \textcolor{fig1teal}{\texttt{'''}}\\
\mbox{\texttt{~~~~}} \textcolor{fig1teal}{\texttt{total}}\textcolor{fig1teal}{\texttt{\_}}\textcolor{fig1teal}{\texttt{students}} \textcolor{fig1teal}{\texttt{=}} \textcolor{fig1teal}{\texttt{2}}\textcolor{fig1teal}{\texttt{0}}\\
\mbox{\texttt{~~~~}} \textcolor{fig1ochre}{\texttt{contemporary}}\textcolor{fig1teal}{\texttt{\_}}\textcolor{fig1teal}{\texttt{students}} \textcolor{fig1teal}{\texttt{=}} \textcolor{fig1teal}{\texttt{total}}\textcolor{fig1teal}{\texttt{\_}}\textcolor{fig1teal}{\texttt{students}} \textcolor{fig1teal}{\texttt{*}} \textcolor{fig1teal}{\texttt{0}}\textcolor{fig1teal}{\texttt{.}}\textcolor{fig1teal}{\texttt{2}}\\
\mbox{\texttt{~~~~}} \textcolor{fig1teal}{\texttt{remaining}}\textcolor{fig1teal}{\texttt{\_}}\textcolor{fig1teal}{\texttt{students}} \textcolor{fig1teal}{\texttt{=}} \textcolor{fig1teal}{\texttt{total}}\textcolor{fig1teal}{\texttt{\_}}\textcolor{fig1teal}{\texttt{students}} \textcolor{fig1teal}{\texttt{-}} \textcolor{fig1ochre}{\texttt{contemporary}}\textcolor{fig1teal}{\texttt{\_}}\textcolor{fig1teal}{\texttt{students}}\\
\mbox{\texttt{~~~~}} \textcolor{fig1ochre}{\texttt{jazz}}\textcolor{fig1ochre}{\texttt{\_}}\textcolor{fig1ochre}{\texttt{\_}}\textcolor{fig1ochre}{\texttt{students}} \textcolor{fig1teal}{\texttt{=}} \textcolor{fig1teal}{\texttt{remaining}}\textcolor{fig1teal}{\texttt{\_}}\textcolor{fig1teal}{\texttt{students}} \textcolor{fig1teal}{\texttt{*}} \textcolor{fig1teal}{\texttt{0}}\textcolor{fig1teal}{\texttt{.}}\textcolor{fig1teal}{\texttt{2}}\textcolor{fig1teal}{\texttt{5}}\\
\mbox{\texttt{~~~~}} \textcolor{fig1teal}{\texttt{hip}}\textcolor{fig1teal}{\texttt{hop}}\textcolor{fig1teal}{\texttt{\_}}\textcolor{fig1teal}{\texttt{students}} \textcolor{fig1teal}{\texttt{=}} \textcolor{fig1teal}{\texttt{remaining}}\textcolor{fig1teal}{\texttt{\_}}\textcolor{fig1teal}{\texttt{students}} \textcolor{fig1teal}{\texttt{-}} \textcolor{fig1ochre}{\texttt{jazz}}\textcolor{fig1ochre}{\texttt{\_}}\textcolor{fig1ochre}{\texttt{d}}\textcolor{fig1ochre}{\texttt{ance}}\\
\mbox{\texttt{~~~~}} \textcolor{fig1ochre}{\texttt{percentage}}\textcolor{fig1teal}{\texttt{\_}}\textcolor{fig1ochre}{\texttt{hip}}\textcolor{fig1ochre}{\texttt{hop}}\textcolor{fig1teal}{\texttt{\_}}\textcolor{fig1ochre}{\texttt{students}} \textcolor{fig1teal}{\texttt{=}} \textcolor{fig1teal}{\texttt{(}}\textcolor{fig1teal}{\texttt{hip}}\textcolor{fig1teal}{\texttt{hop}}\textcolor{fig1teal}{\texttt{\_}}\textcolor{fig1teal}{\texttt{students}} \textcolor{fig1teal}{\texttt{/}} \textcolor{fig1teal}{\texttt{total}}\textcolor{fig1teal}{\texttt{\_}}\textcolor{fig1teal}{\texttt{students}}\textcolor{fig1teal}{\texttt{)}} \textcolor{fig1teal}{\texttt{*}} \textcolor{fig1teal}{\texttt{1}}\textcolor{fig1teal}{\texttt{0}}\textcolor{fig1teal}{\texttt{0}}\\
\mbox{\texttt{~~~~}} \textcolor{fig1teal}{\texttt{result}} \textcolor{fig1teal}{\texttt{=}} \textcolor{fig1ochre}{\texttt{percentage}}\textcolor{fig1teal}{\texttt{\_}}\textcolor{fig1ochre}{\texttt{hip}}\textcolor{fig1ochre}{\texttt{hop}}\textcolor{fig1teal}{\texttt{\_}}\textcolor{fig1ochre}{\texttt{students}}\\
\mbox{\texttt{~~~~}} \textcolor{fig1teal}{\texttt{return}} \textcolor{fig1teal}{\texttt{result}}\textcolor{fig1teal}{\texttt{[PAD]}}\textcolor{fig1teal}{\texttt{[PAD]}}\textcolor{fig1teal}{\texttt{[PAD]}}\textcolor{fig1teal}{\texttt{[PAD]}}\textcolor{fig1teal}{\texttt{[PAD]}}\textcolor{fig1teal}{\texttt{[PAD]}}\textcolor{fig1teal}{\texttt{[PAD]}}\textcolor{fig1teal}{\texttt{[PAD]}}\textcolor{fig1teal}{\texttt{[PAD]}}\textcolor{fig1teal}{\texttt{[PAD]}}
\end{samplebox}
\begin{samplebox}{\normalfont\textbf{$\ours$ (PR), Round: 32} \hfill \normalfont\scriptsize \textcolor{darkgray}{Committed: \textbf{455/455}}}
\scriptsize\linespread{0.9}\selectfont
\textcolor{black}{\texttt{In}} \textcolor{black}{\texttt{a}} \textcolor{black}{\texttt{dance}} \textcolor{black}{\texttt{class}} \textcolor{black}{\texttt{of}} \textcolor{black}{\texttt{2}}\textcolor{black}{\texttt{0}} \textcolor{black}{\texttt{students}}\textcolor{black}{\texttt{,}} \textcolor{black}{\texttt{2}}\textcolor{black}{\texttt{0}}\textcolor{black}{\texttt{\%}} \textcolor{black}{\texttt{enrolled}} \textcolor{black}{\texttt{in}} \textcolor{black}{\texttt{contemporary}} \textcolor{black}{\texttt{dance}}\textcolor{black}{\texttt{,}} \textcolor{black}{\texttt{2}}\textcolor{black}{\texttt{5}}\textcolor{black}{\texttt{\%}} \textcolor{black}{\texttt{of}} \textcolor{black}{\texttt{the}} \textcolor{black}{\texttt{remaining}} \textcolor{black}{\texttt{enrolled}} \textcolor{black}{\texttt{in}} \textcolor{black}{\texttt{jazz}} \textcolor{black}{\texttt{dance}}\textcolor{black}{\texttt{,}} \textcolor{black}{\texttt{and}} \textcolor{black}{\texttt{the}} \textcolor{black}{\texttt{rest}} \textcolor{black}{\texttt{enrolled}} \textcolor{black}{\texttt{in}} \textcolor{black}{\texttt{hip}}\textcolor{black}{\texttt{-}}\textcolor{black}{\texttt{hop}} \textcolor{black}{\texttt{dance}}\textcolor{black}{\texttt{.}} \textcolor{black}{\texttt{What}} \textcolor{black}{\texttt{percentage}} \textcolor{black}{\texttt{of}} \textcolor{black}{\texttt{the}} \textcolor{black}{\texttt{entire}} \textcolor{black}{\texttt{students}} \textcolor{black}{\texttt{enrolled}} \textcolor{black}{\texttt{in}} \textcolor{black}{\texttt{hip}}\textcolor{black}{\texttt{-}}\textcolor{black}{\texttt{hop}} \textcolor{black}{\texttt{dance}}\textcolor{black}{\texttt{?}}\\
\textcolor{fig1teal}{\texttt{def}} \textcolor{fig1teal}{\texttt{simple}}\textcolor{fig1teal}{\texttt{\_}}\textcolor{fig1teal}{\texttt{math}}\textcolor{fig1teal}{\texttt{\_}}\textcolor{fig1teal}{\texttt{problem}}\textcolor{fig1teal}{\texttt{()}} \textcolor{fig1teal}{\texttt{->}} \textcolor{fig1teal}{\texttt{int}}\textcolor{fig1teal}{\texttt{:}}\\
\mbox{\texttt{~~~~}} \textcolor{fig1teal}{\texttt{'''}}\\
\mbox{\texttt{~~~~}} \textcolor{fig1teal}{\texttt{In}} \textcolor{fig1teal}{\texttt{a}} \textcolor{fig1teal}{\texttt{dance}} \textcolor{fig1teal}{\texttt{class}} \textcolor{fig1teal}{\texttt{of}} \textcolor{fig1teal}{\texttt{2}}\textcolor{fig1teal}{\texttt{0}} \textcolor{fig1teal}{\texttt{students}}\textcolor{fig1teal}{\texttt{,}} \textcolor{fig1teal}{\texttt{2}}\textcolor{fig1teal}{\texttt{0}}\textcolor{fig1teal}{\texttt{\%}} \textcolor{fig1teal}{\texttt{enrolled}} \textcolor{fig1teal}{\texttt{in}} \textcolor{fig1teal}{\texttt{contemporary}} \textcolor{fig1teal}{\texttt{dance}}\textcolor{fig1teal}{\texttt{,}} \textcolor{fig1teal}{\texttt{2}}\textcolor{fig1teal}{\texttt{5}}\textcolor{fig1teal}{\texttt{\%}} \textcolor{fig1teal}{\texttt{of}} \textcolor{fig1teal}{\texttt{the}} \textcolor{fig1teal}{\texttt{remaining}} \textcolor{fig1teal}{\texttt{enrolled}} \textcolor{fig1teal}{\texttt{in}} \textcolor{fig1teal}{\texttt{jazz}} \textcolor{fig1teal}{\texttt{dance}}\textcolor{fig1teal}{\texttt{,}} \textcolor{fig1teal}{\texttt{and}} \textcolor{fig1teal}{\texttt{the}} \textcolor{fig1teal}{\texttt{rest}} \textcolor{fig1teal}{\texttt{enrolled}} \textcolor{fig1teal}{\texttt{in}} \textcolor{fig1teal}{\texttt{hip}}\textcolor{fig1teal}{\texttt{-}}\textcolor{fig1teal}{\texttt{hop}} \textcolor{fig1teal}{\texttt{dance}}\textcolor{fig1teal}{\texttt{.}}\\
\mbox{\texttt{~~~~}} \textcolor{fig1teal}{\texttt{What}} \textcolor{fig1teal}{\texttt{percentage}} \textcolor{fig1teal}{\texttt{of}} \textcolor{fig1teal}{\texttt{the}} \textcolor{fig1teal}{\texttt{entire}} \textcolor{fig1teal}{\texttt{students}} \textcolor{fig1teal}{\texttt{enrolled}} \textcolor{fig1teal}{\texttt{in}} \textcolor{fig1teal}{\texttt{hip}}\textcolor{fig1teal}{\texttt{-}}\textcolor{fig1teal}{\texttt{hop}} \textcolor{fig1teal}{\texttt{dance}}\textcolor{fig1teal}{\texttt{?}}\\
\mbox{\texttt{~~~~}} \textcolor{fig1teal}{\texttt{'''}}\\
\mbox{\texttt{~~~~}} \textcolor{fig1teal}{\texttt{total}}\textcolor{fig1teal}{\texttt{\_}}\textcolor{fig1teal}{\texttt{students}} \textcolor{fig1teal}{\texttt{=}} \textcolor{fig1teal}{\texttt{2}}\textcolor{fig1teal}{\texttt{0}}\\
\mbox{\texttt{~~~~}} \textcolor{fig1teal}{\texttt{contemporary}}\textcolor{fig1teal}{\texttt{\_}}\textcolor{fig1teal}{\texttt{students}} \textcolor{fig1teal}{\texttt{=}} \textcolor{fig1teal}{\texttt{total}}\textcolor{fig1teal}{\texttt{\_}}\textcolor{fig1teal}{\texttt{students}} \textcolor{fig1teal}{\texttt{*}} \textcolor{fig1teal}{\texttt{0}}\textcolor{fig1teal}{\texttt{.}}\textcolor{fig1teal}{\texttt{2}}\\
\mbox{\texttt{~~~~}} \textcolor{fig1teal}{\texttt{remaining}}\textcolor{fig1teal}{\texttt{\_}}\textcolor{fig1teal}{\texttt{students}} \textcolor{fig1teal}{\texttt{=}} \textcolor{fig1teal}{\texttt{total}}\textcolor{fig1teal}{\texttt{\_}}\textcolor{fig1teal}{\texttt{students}} \textcolor{fig1teal}{\texttt{-}} \textcolor{fig1teal}{\texttt{contemporary}}\textcolor{fig1teal}{\texttt{\_}}\textcolor{fig1teal}{\texttt{students}}\\
\mbox{\texttt{~~~~}} \textcolor{fig1teal}{\texttt{jazz}}\textcolor{fig1teal}{\texttt{\_}}\textcolor{fig1teal}{\texttt{d}}\textcolor{fig1teal}{\texttt{ance}} \textcolor{fig1teal}{\texttt{=}} \textcolor{fig1teal}{\texttt{remaining}}\textcolor{fig1teal}{\texttt{\_}}\textcolor{fig1teal}{\texttt{students}} \textcolor{fig1teal}{\texttt{*}} \textcolor{fig1teal}{\texttt{0}}\textcolor{fig1teal}{\texttt{.}}\textcolor{fig1teal}{\texttt{2}}\textcolor{fig1teal}{\texttt{5}}\\
\mbox{\texttt{~~~~}} \textcolor{fig1teal}{\texttt{hip}}\textcolor{fig1teal}{\texttt{hop}}\textcolor{fig1teal}{\texttt{\_}}\textcolor{fig1teal}{\texttt{students}} \textcolor{fig1teal}{\texttt{=}} \textcolor{fig1teal}{\texttt{remaining}}\textcolor{fig1teal}{\texttt{\_}}\textcolor{fig1teal}{\texttt{students}} \textcolor{fig1teal}{\texttt{-}} \textcolor{fig1teal}{\texttt{jazz}}\textcolor{fig1teal}{\texttt{\_}}\textcolor{fig1teal}{\texttt{d}}\textcolor{fig1teal}{\texttt{ance}}\\
\mbox{\texttt{~~~~}} \textcolor{fig1teal}{\texttt{percentage}}\textcolor{fig1teal}{\texttt{\_}}\textcolor{fig1teal}{\texttt{hip}}\textcolor{fig1teal}{\texttt{hop}}\textcolor{fig1teal}{\texttt{\_}}\textcolor{fig1teal}{\texttt{students}} \textcolor{fig1teal}{\texttt{=}} \textcolor{fig1teal}{\texttt{(}}\textcolor{fig1teal}{\texttt{hip}}\textcolor{fig1teal}{\texttt{hop}}\textcolor{fig1teal}{\texttt{\_}}\textcolor{fig1teal}{\texttt{students}} \textcolor{fig1teal}{\texttt{/}} \textcolor{fig1teal}{\texttt{total}}\textcolor{fig1teal}{\texttt{\_}}\textcolor{fig1teal}{\texttt{students}}\textcolor{fig1teal}{\texttt{)}} \textcolor{fig1teal}{\texttt{*}} \textcolor{fig1teal}{\texttt{1}}\textcolor{fig1teal}{\texttt{0}}\textcolor{fig1teal}{\texttt{0}}\\
\mbox{\texttt{~~~~}} \textcolor{fig1teal}{\texttt{result}} \textcolor{fig1teal}{\texttt{=}} \textcolor{fig1teal}{\texttt{percentage}}\textcolor{fig1teal}{\texttt{\_}}\textcolor{fig1teal}{\texttt{hip}}\textcolor{fig1teal}{\texttt{hop}}\textcolor{fig1teal}{\texttt{\_}}\textcolor{fig1teal}{\texttt{students}}\\
\mbox{\texttt{~~~~}} \textcolor{fig1teal}{\texttt{return}} \textcolor{fig1teal}{\texttt{result}}\textcolor{fig1teal}{\texttt{[PAD]}}\textcolor{fig1teal}{\texttt{[PAD]}}\textcolor{fig1teal}{\texttt{[PAD]}}\textcolor{fig1teal}{\texttt{[PAD]}}\textcolor{fig1teal}{\texttt{[PAD]}}\textcolor{fig1teal}{\texttt{[PAD]}}\textcolor{fig1teal}{\texttt{[PAD]}}\textcolor{fig1teal}{\texttt{[PAD]}}\textcolor{fig1teal}{\texttt{[PAD]}}\textcolor{fig1teal}{\texttt{[PAD]}}
\end{samplebox}
\caption{Posterior Refinement trajectory on TinyGSM, Sample-2.}
\label{fig:refine_sample_tinygsm_2}
\end{figure}

% ===================== OpenWebText =====================
\begin{figure}[H]
\centering
\begin{samplebox}{\normalfont\textbf{$\ours$ (PR), Round: 1} \hfill \normalfont\scriptsize \textcolor{darkgray}{Committed: \textbf{700/1024}}}
\scriptsize\linespread{0.9}\selectfont
 \textcolor{fig1teal}{say} \textcolor{fig1teal}{a} \textcolor{fig1ochre}{quarter} \textcolor{fig1teal}{of} \textcolor{fig1teal}{a} \textcolor{fig1ochre}{million}\textcolor{fig1teal}{,} \textcolor{fig1teal}{\$}\textcolor{fig1teal}{5}\textcolor{fig1teal}{,}\textcolor{fig1teal}{000} \textcolor{fig1ochre}{sh}\textcolor{fig1ochre}{aked} \textcolor{fig1teal}{on} \textcolor{fig1teal}{the} \textcolor{fig1teal}{network} \textcolor{fig1teal}{by} \textcolor{fig1teal}{three} \textcolor{fig1teal}{women}\textcolor{fig1teal}{—}\textcolor{fig1ochre}{even} \textcolor{fig1teal}{though} \textcolor{fig1teal}{it}\textcolor{fig1teal}{�}\textcolor{fig1teal}{�}\textcolor{fig1teal}{s} \textcolor{fig1teal}{the} \textcolor{fig1ochre}{largest} \textcolor{fig1ochre}{club} \textcolor{fig1teal}{in} \textcolor{fig1teal}{the} \textcolor{fig1teal}{U}\textcolor{fig1teal}{.}\textcolor{fig1teal}{S}\textcolor{fig1teal}{.} \textcolor{fig1ochre}{dollars} \textcolor{fig1teal}{today}\textcolor{fig1teal}{,} \textcolor{fig1teal}{that} \textcolor{fig1teal}{is} \textcolor{fig1teal}{a} \textcolor{fig1teal}{national} \textcolor{fig1teal}{good}\textcolor{fig1teal}{.}\textcolor{fig1teal}{�}\textcolor{fig1teal}{�}  \textcolor{fig1ochre}{While} \textcolor{fig1ochre}{recognize} \textcolor{fig1teal}{that} \textcolor{fig1ochre}{garner}\textcolor{fig1teal}{ing} \textcolor{fig1teal}{money} \textcolor{fig1teal}{from} \textcolor{fig1teal}{�}\textcolor{fig1teal}{�}\textcolor{fig1teal}{m}\textcolor{fig1ochre}{oms}\textcolor{fig1teal}{�}\textcolor{fig1teal}{�} \textcolor{fig1teal}{a} \textcolor{fig1teal}{�}\textcolor{fig1teal}{�}\textcolor{fig1ochre}{talk}\textcolor{fig1teal}{,}\textcolor{fig1teal}{�}\textcolor{fig1teal}{�} \textcolor{fig1teal}{B}\textcolor{fig1ochre}{ohn} \textcolor{fig1teal}{said} \textcolor{fig1teal}{it} \textcolor{fig1ochre}{seemed}\textcolor{fig1ochre}{omed} \textcolor{fig1teal}{that} \textcolor{fig1teal}{that} \textcolor{fig1teal}{is} \textcolor{fig1teal}{a}\textcolor{fig1ochre}{foot} \textcolor{fig1teal}{later}\textcolor{fig1teal}{,} \textcolor{fig1ochre}{given} \textcolor{fig1teal}{the} \textcolor{fig1ochre}{Top} \textcolor{fig1ochre}{Brother} \textcolor{fig1teal}{race} \textcolor{fig1teal}{in} \textcolor{fig1teal}{the} \textcolor{fig1teal}{F} \textcolor{fig1teal}{few} \textcolor{fig1ochre}{Republicans} \textcolor{fig1teal}{a} \textcolor{fig1teal}{good} \textcolor{fig1ochre}{analogy} \textcolor{fig1teal}{and} \textcolor{fig1teal}{could} \textcolor{fig1ochre}{beat} \textcolor{fig1ochre}{Thor}\textcolor{fig1teal}{n}\textcolor{fig1teal}{�}\textcolor{fig1teal}{�}\textcolor{fig1teal}{s} \textcolor{fig1ochre}{Ar}\textcolor{fig1ochre}{arry}\textcolor{fig1teal}{l} \textcolor{fig1teal}{The}\textcolor{fig1ochre}{ss} \textcolor{fig1ochre}{dropping}\textcolor{fig1teal}{—}\textcolor{fig1ochre}{Car}\textcolor{fig1ochre}{agh} \textcolor{fig1teal}{said} \textcolor{fig1teal}{he}\textcolor{fig1teal}{�}\textcolor{fig1teal}{�}\textcolor{fig1teal}{s} \textcolor{fig1teal}{happy} \textcolor{fig1teal}{to} \textcolor{fig1teal}{be} \textcolor{fig1ochre}{bothered} \textcolor{fig1teal}{that} \textcolor{fig1teal}{the} \textcolor{fig1ochre}{shift} \textcolor{fig1teal}{around} \textcolor{fig1ochre}{presenting} \textcolor{fig1teal}{of} \textcolor{fig1teal}{a} \textcolor{fig1ochre}{female} \textcolor{fig1teal}{character} \textcolor{fig1ochre}{acting} \textcolor{fig1teal}{under} \textcolor{fig1teal}{an} \textcolor{fig1ochre}{inside} \textcolor{fig1ochre}{fixed} \textcolor{fig1ochre}{view}\textcolor{fig1teal}{,} \textcolor{fig1teal}{not} \textcolor{fig1teal}{at} \textcolor{fig1teal}{least} \textcolor{fig1teal}{one} \textcolor{fig1ochre}{Republican}\textcolor{fig1teal}{,} \textcolor{fig1teal}{could} \textcolor{fig1teal}{become} \textcolor{fig1teal}{the} \textcolor{fig1ochre}{tour} \textcolor{fig1teal}{more} \textcolor{fig1ochre}{attractive}\textcolor{fig1teal}{.} \textcolor{fig1teal}{�}\textcolor{fig1teal}{�}\textcolor{fig1teal}{If} \textcolor{fig1teal}{it} \textcolor{fig1ochre}{teaches} \textcolor{fig1teal}{me} \textcolor{fig1teal}{wrong}\textcolor{fig1teal}{,} \textcolor{fig1teal}{it}\textcolor{fig1teal}{�}\textcolor{fig1teal}{�}\textcolor{fig1teal}{s} \textcolor{fig1teal}{going} \textcolor{fig1teal}{to} \textcolor{fig1teal}{stop} \textcolor{fig1teal}{working} \textcolor{fig1teal}{for} \textcolor{fig1teal}{two} \textcolor{fig1teal}{years}\textcolor{fig1teal}{,}\textcolor{fig1teal}{�}\textcolor{fig1teal}{�} \textcolor{fig1teal}{he} \textcolor{fig1teal}{said}\textcolor{fig1teal}{.} \textcolor{fig1teal}{�}\textcolor{fig1teal}{�}\textcolor{fig1teal}{You} \textcolor{fig1teal}{think}\textcolor{fig1teal}{,} \textcolor{fig1teal}{the} \textcolor{fig1ochre}{lead}\textcolor{fig1ochre}{iest} \textcolor{fig1teal}{media} \textcolor{fig1ochre}{hosts} \textcolor{fig1teal}{and} \textcolor{fig1teal}{me}\textcolor{fig1teal}{of}\textcolor{fig1teal}{us} \textcolor{fig1ochre}{look} \textcolor{fig1teal}{at} \textcolor{fig1teal}{your} \textcolor{fig1ochre}{lips}\textcolor{fig1teal}{,} \textcolor{fig1ochre}{passion}\textcolor{fig1teal}{,} \textcolor{fig1teal}{and} \textcolor{fig1teal}{y}\textcolor{fig1ochre}{side} \textcolor{fig1teal}{you} \textcolor{fig1teal}{with} \textcolor{fig1teal}{a} \textcolor{fig1ochre}{grin} \textcolor{fig1teal}{while} \textcolor{fig1ochre}{dragging} \textcolor{fig1teal}{it} \textcolor{fig1teal}{through} \textcolor{fig1teal}{along}\textcolor{fig1teal}{?}\textcolor{fig1teal}{�}\textcolor{fig1teal}{�}\textcolor{fig1teal}{�}\textcolor{fig1teal}{�}  \textcolor{fig1ochre}{Louis} \textcolor{fig1ochre}{How}\textcolor{fig1ochre}{dain} \textcolor{fig1teal}{is} \textcolor{fig1teal}{also} \textcolor{fig1ochre}{pond}\textcolor{fig1ochre}{illing} \textcolor{fig1teal}{her} \textcolor{fig1ochre}{talk} \textcolor{fig1teal}{show}\textcolor{fig1teal}{�}\textcolor{fig1teal}{�}\textcolor{fig1teal}{s} \textcolor{fig1teal}{lack} \textcolor{fig1teal}{of} \textcolor{fig1ochre}{directing} \textcolor{fig1teal}{around} \textcolor{fig1teal}{more} \textcolor{fig1ochre}{outrageous} \textcolor{fig1ochre}{satire} \textcolor{fig1ochre}{aimed} \textcolor{fig1teal}{at} \textcolor{fig1teal}{children}\textcolor{fig1teal}{.} \textcolor{fig1teal}{�}\textcolor{fig1teal}{�}\textcolor{fig1teal}{And} \textcolor{fig1ochre}{frankly}\textcolor{fig1teal}{,} \textcolor{fig1teal}{they}\textcolor{fig1teal}{�}\textcolor{fig1teal}{�}\textcolor{fig1teal}{re} \textcolor{fig1teal}{quite} \textcolor{fig1ochre}{obsessed}\textcolor{fig1ochre}{hing} \textcolor{fig1teal}{why} \textcolor{fig1ochre}{possibly} \textcolor{fig1teal}{the} \textcolor{fig1teal}{most} \textcolor{fig1ochre}{sensational}\textcolor{fig1teal}{,} \textcolor{fig1teal}{from} \textcolor{fig1teal}{the} \textcolor{fig1ochre}{source}\textcolor{fig1teal}{,} \textcolor{fig1ochre}{sort} \textcolor{fig1teal}{of} \textcolor{fig1ochre}{average} \textcolor{fig1ochre}{creator}\textcolor{fig1teal}{,} \textcolor{fig1teal}{I}\textcolor{fig1ochre}{'ve} \textcolor{fig1teal}{seen} \textcolor{fig1teal}{with}\textcolor{fig1teal}{,} \textcolor{fig1teal}{can} \textcolor{fig1teal}{one} \textcolor{fig1teal}{create} \textcolor{fig1teal}{that}\textcolor{fig1teal}{.}\textcolor{fig1teal}{�}\textcolor{fig1teal}{�} \textcolor{fig1teal}{He} \textcolor{fig1teal}{was} \textcolor{fig1ochre}{concerned} \textcolor{fig1teal}{about} \textcolor{fig1teal}{that} \textcolor{fig1teal}{when} \textcolor{fig1teal}{she} \textcolor{fig1teal}{said} \textcolor{fig1teal}{he} \textcolor{fig1ochre}{expects} \textcolor{fig1ochre}{Warner} \textcolor{fig1ochre}{Bros}\textcolor{fig1teal}{.} \textcolor{fig1teal}{to} \textcolor{fig1ochre}{embro}\textcolor{fig1ochre}{ep} \textcolor{fig1teal}{the} \textcolor{fig1teal}{show}\textcolor{fig1teal}{.} \textcolor{fig1ochre}{Of}\textcolor{fig1ochre}{umed}\textcolor{fig1teal}{,} \textcolor{fig1teal}{these} \textcolor{fig1ochre}{debates} \textcolor{fig1ochre}{fit} \textcolor{fig1teal}{into} \textcolor{fig1teal}{the} \textcolor{fig1ochre}{kind} \textcolor{fig1teal}{of} \textcolor{fig1ochre}{consistency} \textcolor{fig1teal}{between} \textcolor{fig1teal}{a} \textcolor{fig1ochre}{mention} \textcolor{fig1teal}{of} \textcolor{fig1ochre}{nit}\textcolor{fig1ochre}{w}\textcolor{fig1ochre}{zy} \textcolor{fig1ochre}{humor} \textcolor{fig1teal}{and} \textcolor{fig1teal}{the} \textcolor{fig1ochre}{company} \textcolor{fig1teal}{she}\textcolor{fig1teal}{�}\textcolor{fig1teal}{�}\textcolor{fig1teal}{s} \textcolor{fig1teal}{trying} \textcolor{fig1teal}{to} \textcolor{fig1ochre}{plug} \textcolor{fig1teal}{on}\textcolor{fig1teal}{.}\textcolor{fig1teal}{�}\textcolor{fig1teal}{�}\textcolor{fig1teal}{I} \textcolor{fig1ochre}{rel} \textcolor{fig1ochre}{literally} \textcolor{fig1teal}{been} \textcolor{fig1ochre}{outraged} \textcolor{fig1teal}{by} \textcolor{fig1teal}{your} \textcolor{fig1teal}{response} \textcolor{fig1teal}{to} \textcolor{fig1teal}{had} \textcolor{fig1teal}{so} \textcolor{fig1teal}{much} \textcolor{fig1teal}{needs} \textcolor{fig1teal}{to} \textcolor{fig1ochre}{budget}\textcolor{fig1teal}{.} \textcolor{fig1teal}{I} \textcolor{fig1teal}{say}\textcolor{fig1teal}{,} \textcolor{fig1ochre}{Grant}\textcolor{fig1teal}{,} \textcolor{fig1teal}{that}\textcolor{fig1teal}{�}\textcolor{fig1teal}{�}\textcolor{fig1teal}{s} \textcolor{fig1teal}{me} \textcolor{fig1teal}{why} \textcolor{fig1teal}{we}\textcolor{fig1teal}{�}\textcolor{fig1teal}{�}\textcolor{fig1teal}{re} \textcolor{fig1teal}{just} \textcolor{fig1ochre}{publishing} \textcolor{fig1ochre}{Links} \textcolor{fig1teal}{for} \textcolor{fig1teal}{our} \textcolor{fig1ochre}{servers}\textcolor{fig1teal}{.} \textcolor{fig1teal}{In} \textcolor{fig1teal}{my} \textcolor{fig1ochre}{opinion}\textcolor{fig1teal}{,} \textcolor{fig1teal}{we} \textcolor{fig1teal}{have} \textcolor{fig1ochre}{modeled} \textcolor{fig1teal}{it} \textcolor{fig1teal}{on} \textcolor{fig1teal}{their} \textcolor{fig1teal}{own}\textcolor{fig1teal}{,} \textcolor{fig1teal}{and} \textcolor{fig1ochre}{final} \textcolor{fig1ochre}{dollars} \textcolor{fig1teal}{have} \textcolor{fig1teal}{been} \textcolor{fig1ochre}{recorded}\textcolor{fig1teal}{.} \textcolor{fig1teal}{This} \textcolor{fig1teal}{post} \textcolor{fig1ochre}{debate} \textcolor{fig1ochre}{call}\textcolor{fig1ochre}{key} \textcolor{fig1ochre}{Guardian}\textcolor{fig1teal}{.} \textcolor{fig1teal}{I} \textcolor{fig1ochre}{remember} \textcolor{fig1ochre}{EL} \textcolor{fig1ochre}{Margaret} \textcolor{fig1teal}{[}\textcolor{fig1ochre}{sic}\textcolor{fig1teal}{]} \textcolor{fig1teal}{�}\textcolor{fig1teal}{�}\textcolor{fig1ochre}{100} \textcolor{fig1teal}{years} \textcolor{fig1teal}{later}\textcolor{fig1teal}{.}\textcolor{fig1teal}{�}\textcolor{fig1teal}{�} \textcolor{fig1ochre}{When} \textcolor{fig1teal}{she} \textcolor{fig1teal}{needed} \textcolor{fig1teal}{�}\textcolor{fig1teal}{�}\textcolor{fig1ochre}{new} \textcolor{fig1ochre}{decision} \textcolor{fig1teal}{to} \textcolor{fig1teal}{make} \textcolor{fig1ochre}{higher} \textcolor{fig1teal}{news}\textcolor{fig1teal}{,}\textcolor{fig1teal}{�}\textcolor{fig1teal}{�} \textcolor{fig1teal}{we} \textcolor{fig1teal}{gave} \textcolor{fig1teal}{the} \textcolor{fig1teal}{U}\textcolor{fig1teal}{.}\textcolor{fig1teal}{S}\textcolor{fig1teal}{.} \textcolor{fig1ochre}{Telecommunications} \textcolor{fig1teal}{and} \textcolor{fig1ochre}{App} \textcolor{fig1ochre}{Administration}\textcolor{fig1ochre}{ÃÂ} \textcolor{fig1teal}{when} \textcolor{fig1ochre}{covering} \textcolor{fig1teal}{10}\textcolor{fig1teal}{-}\textcolor{fig1teal}{20} \textcolor{fig1teal}{percent} \textcolor{fig1teal}{of} \textcolor{fig1teal}{the} \textcolor{fig1teal}{gas}\textcolor{fig1teal}{-}\textcolor{fig1ochre}{lot}\textcolor{fig1ochre}{iling} \textcolor{fig1ochre}{royalty}\textcolor{fig1teal}{.} \textcolor{fig1teal}{The} \textcolor{fig1teal}{news} \textcolor{fig1ochre}{analytics} \textcolor{fig1ochre}{television} \textcolor{fig1ochre}{base} \textcolor{fig1teal}{of} \textcolor{fig1ochre}{methane} \textcolor{fig1ochre}{requests} \textcolor{fig1teal}{can} \textcolor{fig1ochre}{exceed} \textcolor{fig1teal}{any} \textcolor{fig1ochre}{bank}\textcolor{fig1teal}{.} \textcolor{fig1teal}{My} \textcolor{fig1teal}{organization} \textcolor{fig1teal}{used} \textcolor{fig1teal}{these} \textcolor{fig1ochre}{rates} \textcolor{fig1teal}{and} \textcolor{fig1ochre}{benchmarks} \textcolor{fig1teal}{to} \textcolor{fig1teal}{the} \textcolor{fig1teal}{problem} \textcolor{fig1teal}{during} \textcolor{fig1teal}{the} \textcolor{fig1teal}{days} \textcolor{fig1teal}{of} \textcolor{fig1ochre}{removing} \textcolor{fig1ochre}{BET} \textcolor{fig1teal}{which} \textcolor{fig1teal}{has} \textcolor{fig1teal}{a} \textcolor{fig1ochre}{widespread} \textcolor{fig1teal}{impact} \textcolor{fig1teal}{on} \textcolor{fig1teal}{the} \textcolor{fig1teal}{50}\textcolor{fig1teal}{/}\textcolor{fig1ochre}{share} \textcolor{fig1ochre}{Disney} \textcolor{fig1teal}{industry}\textcolor{fig1teal}{.} \textcolor{fig1teal}{I} \textcolor{fig1teal}{am} \textcolor{fig1ochre}{sorry} \textcolor{fig1teal}{for} \textcolor{fig1teal}{the} \textcolor{fig1ochre}{results} \textcolor{fig1teal}{of} \textcolor{fig1teal}{taking} \textcolor{fig1teal}{this} \textcolor{fig1teal}{responsibility} \textcolor{fig1teal}{on}\textcolor{fig1teal}{.} \textcolor{fig1teal}{Americans} \textcolor{fig1teal}{have} \textcolor{fig1teal}{the} \textcolor{fig1teal}{right} \textcolor{fig1teal}{to} \textcolor{fig1teal}{know} \textcolor{fig1teal}{with} \textcolor{fig1teal}{you} \textcolor{fig1teal}{top} \textcolor{fig1ochre}{league} \textcolor{fig1teal}{where} \textcolor{fig1teal}{you} \textcolor{fig1teal}{stand} \textcolor{fig1teal}{or} \textcolor{fig1teal}{we} \textcolor{fig1teal}{per}\textcolor{fig1ochre}{vert} \textcolor{fig1teal}{with} \textcolor{fig1teal}{that} \textcolor{fig1ochre}{version} \textcolor{fig1teal}{of} \textcolor{fig1ochre}{Thank} \textcolor{fig1teal}{You}\textcolor{fig1teal}{,} \textcolor{fig1teal}{the} \textcolor{fig1ochre}{Globe} \textcolor{fig1teal}{and} \textcolor{fig1ochre}{Public} \textcolor{fig1ochre}{Media} \textcolor{fig1ochre}{Rich}\textcolor{fig1teal}{,} \textcolor{fig1teal}{and} \textcolor{fig1teal}{get} \textcolor{fig1ochre}{specifics} \textcolor{fig1teal}{of} \textcolor{fig1teal}{what} \textcolor{fig1teal}{you} \textcolor{fig1teal}{are} \textcolor{fig1ochre}{listening} \textcolor{fig1teal}{for}\textcolor{fig1teal}{.} \textcolor{fig1teal}{We} \textcolor{fig1teal}{are} \textcolor{fig1teal}{not} \textcolor{fig1ochre}{prepared} \textcolor{fig1teal}{for} \textcolor{fig1ochre}{plagiar}\textcolor{fig1teal}{ization} \textcolor{fig1teal}{of} \textcolor{fig1teal}{our} \textcolor{fig1ochre}{residential} \textcolor{fig1ochre}{property}\textcolor{fig1teal}{,} \textcolor{fig1teal}{and} \textcolor{fig1teal}{the} \textcolor{fig1teal}{strength} \textcolor{fig1teal}{of} \textcolor{fig1teal}{our} \textcolor{fig1ochre}{sales} \textcolor{fig1teal}{is} \textcolor{fig1ochre}{truly} \textcolor{fig1teal}{great}\textcolor{fig1teal}{.}  \textcolor{fig1teal}{I} \textcolor{fig1teal}{told} \textcolor{fig1teal}{anyone} \textcolor{fig1ochre}{talking} \textcolor{fig1teal}{to} \textcolor{fig1teal}{me} \textcolor{fig1teal}{about} \textcolor{fig1ochre}{extras} \textcolor{fig1teal}{as} \textcolor{fig1ochre}{consumers}\textcolor{fig1teal}{.}  \textcolor{fig1ochre}{why} \textcolor{fig1ochre}{shouldn}\textcolor{fig1teal}{'t} \textcolor{fig1teal}{it} \textcolor{fig1teal}{be} \textcolor{fig1ochre}{okay}\textcolor{fig1teal}{?}  \textcolor{fig1teal}{(} \textcolor{fig1teal}{here}\textcolor{fig1teal}{'s} \textcolor{fig1teal}{my} \textcolor{fig1teal}{response}\textcolor{fig1teal}{,} \textcolor{fig1teal}{I}\textcolor{fig1ochre}{'m}\textcolor{fig1ochre}{acebook} \textcolor{fig1teal}{that} \textcolor{fig1ochre}{explains} \textcolor{fig1teal}{some} \textcolor{fig1ochre}{form} \textcolor{fig1teal}{in}\textcolor{fig1ochre}{art} \textcolor{fig1teal}{from} \textcolor{fig1teal}{the} \textcolor{fig1teal}{person} \textcolor{fig1ochre}{responsible}\textcolor{fig1teal}{.} \textcolor{fig1ochre}{Further}\textcolor{fig1teal}{,} \textcolor{fig1teal}{I} \textcolor{fig1ochre}{appreciated} \textcolor{fig1teal}{our} \textcolor{fig1teal}{national} \textcolor{fig1ochre}{headquarters} \textcolor{fig1teal}{and} \textcolor{fig1teal}{companies} \textcolor{fig1ochre}{available} \textcolor{fig1teal}{for} \textcolor{fig1teal}{a} \textcolor{fig1ochre}{copy} \textcolor{fig1teal}{or} \textcolor{fig1teal}{down} \textcolor{fig1teal}{when} \textcolor{fig1ochre}{serving} \textcolor{fig1teal}{an} \textcolor{fig1ochre}{American} \textcolor{fig1ochre}{Broadcasting} \textcolor{fig1teal}{We} \textcolor{fig1ochre}{Bid} \textcolor{fig1ochre}{network} \textcolor{fig1teal}{or} \textcolor{fig1teal}{the} \textcolor{fig1teal}{"}\textcolor{fig1ochre}{only} \textcolor{fig1teal}{one} \textcolor{fig1teal}{of} \textcolor{fig1teal}{which} \textcolor{fig1ochre}{contained} \textcolor{fig1teal}{no} \textcolor{fig1ochre}{concern} \textcolor{fig1teal}{I} \textcolor{fig1teal}{was} \textcolor{fig1ochre}{disconnected} \textcolor{fig1teal}{from}\textcolor{fig1teal}{.}  \textcolor{fig1teal}{I} \textcolor{fig1ochre}{guess} \textcolor{fig1teal}{it} \textcolor{fig1ochre}{leaves} \textcolor{fig1ochre}{upon} \textcolor{fig1teal}{fine} \textcolor{fig1teal}{in} \textcolor{fig1teal}{every} \textcolor{fig1teal}{way}\textcolor{fig1teal}{.} \textcolor{fig1ochre}{Please} \textcolor{fig1ochre}{please} \textcolor{fig1teal}{continue} \textcolor{fig1teal}{if} \textcolor{fig1teal}{you} \textcolor{fig1teal}{at} \textcolor{fig1teal}{I} \textcolor{fig1teal}{feel} \textcolor{fig1teal}{like}\textcolor{fig1teal}{.}  \textcolor{fig1ochre}{Just} \textcolor{fig1teal}{know} \textcolor{fig1teal}{you} \textcolor{fig1teal}{would} \textcolor{fig1teal}{like} \textcolor{fig1teal}{us} \textcolor{fig1teal}{to} \textcolor{fig1ochre}{advertise} \textcolor{fig1teal}{as} \textcolor{fig1teal}{a} \textcolor{fig1ochre}{subsidiary} \textcolor{fig1teal}{of} \textcolor{fig1teal}{this} \textcolor{fig1teal}{business}\textcolor{fig1teal}{,} \textcolor{fig1ochre}{CLICK} \textcolor{fig1teal}{to} \textcolor{fig1ochre}{g}\textcolor{fig1ochre}{mark} \textcolor{fig1teal}{the} \textcolor{fig1ochre}{comments}\textcolor{fig1teal}{.}\textcolor{fig1ochre}{Russia} \textcolor{fig1teal}{is} \textcolor{fig1teal}{not} \textcolor{fig1teal}{�}\textcolor{fig1teal}{�}\textcolor{fig1ochre}{even} \textcolor{fig1ochre}{gossip}\textcolor{fig1teal}{,}\textcolor{fig1teal}{�}\textcolor{fig1teal}{�} \textcolor{fig1teal}{by} \textcolor{fig1ochre}{Walter} \textcolor{fig1ochre}{Go}\textcolor{fig1teal}{is} \textcolor{fig1ochre}{plainly} \textcolor{fig1teal}{put} \textcolor{fig1ochre}{straight} \textcolor{fig1teal}{for} \textcolor{fig1teal}{office} \textcolor{fig1teal}{at} \textcolor{fig1teal}{the} \textcolor{fig1ochre}{apex} \textcolor{fig1teal}{of} \textcolor{fig1teal}{the} \textcolor{fig1ochre}{spreading}\textcolor{fig1teal}{-}\textcolor{fig1teal}{the}\textcolor{fig1teal}{-} \textcolor{fig1teal}{news}\textcolor{fig1teal}{.} \textcolor{fig1teal}{In} \textcolor{fig1teal}{Russia}\textcolor{fig1teal}{,} \textcolor{fig1ochre}{recalling} \textcolor{fig1teal}{a} \textcolor{fig1ochre}{Russian} \textcolor{fig1ochre}{accent}\textcolor{fig1teal}{,} \textcolor{fig1teal}{his} \textcolor{fig1teal}{five} \textcolor{fig1teal}{state} \textcolor{fig1ochre}{ministers}\textcolor{fig1teal}{,} \textcolor{fig1ochre}{Spain} \textcolor{fig1teal}{and} \textcolor{fig1ochre}{Spain} \textcolor{fig1teal}{said} \textcolor{fig1teal}{they} \textcolor{fig1ochre}{spotted} \textcolor{fig1teal}{and} \textcolor{fig1ochre}{talked} \textcolor{fig1teal}{about} \textcolor{fig1teal}{through} \textcolor{fig1teal}{the} \textcolor{fig1ochre}{Justice} \textcolor{fig1teal}{R}\textcolor{fig1ochre}{uan} \textcolor{fig1ochre}{Center}\textcolor{fig1teal}{.} \textcolor{fig1ochre}{Go}\textcolor{fig1teal}{is}\textcolor{fig1ochre}{oid} \textcolor{fig1ochre}{phrases} \textcolor{fig1teal}{such} \textcolor{fig1teal}{as} \textcolor{fig1teal}{"}\textcolor{fig1ochre}{Bay}\textcolor{fig1teal}{and}\textcolor{fig1ochre}{iana}\textcolor{fig1teal}{"} \textcolor{fig1teal}{to} \textcolor{fig1teal}{the} \textcolor{fig1teal}{US}\textcolor{fig1teal}{.}  \textcolor{fig1ochre}{At} \textcolor{fig1ochre}{overall}\textcolor{fig1teal}{,} \textcolor{fig1teal}{the} \textcolor{fig1ochre}{HBO} \textcolor{fig1ochre}{Weekly} \textcolor{fig1teal}{were} \textcolor{fig1ochre}{dance}\textcolor{fig1teal}{,} \textcolor{fig1ochre}{stamped} \textcolor{fig1teal}{on} \textcolor{fig1ochre}{themselves}\textcolor{fig1teal}{,} \textcolor{fig1ochre}{shut}\textcolor{fig1ochre}{ting}\textcolor{fig1teal}{led} \textcolor{fig1teal}{and} \textcolor{fig1ochre}{slight}\textcolor{fig1teal}{ed} \textcolor{fig1teal}{by} \textcolor{fig1teal}{a} \textcolor{fig1ochre}{Dom}\textcolor{fig1teal}{-}\textcolor{fig1ochre}{Christian}\textcolor{fig1teal}{-}\textcolor{fig1ochre}{old} \textcolor{fig1teal}{media} \textcolor{fig1ochre}{man} \textcolor{fig1teal}{from} \textcolor{fig1ochre}{Britain}\textcolor{fig1teal}{,"} \textcolor{fig1ochre}{declares} \textcolor{fig1teal}{his} \textcolor{fig1ochre}{tweets}\textcolor{fig1teal}{.} \textcolor{fig1teal}{"}\textcolor{fig1ochre}{They} \textcolor{fig1teal}{did} \textcolor{fig1teal}{then} \textcolor{fig1ochre}{come} \textcolor{fig1teal}{less} \textcolor{fig1teal}{from} \textcolor{fig1ochre}{NBC}\textcolor{fig1teal}{,} \textcolor{fig1teal}{who} \textcolor{fig1ochre}{purported} \textcolor{fig1teal}{to} \textcolor{fig1teal}{understand} \textcolor{fig1teal}{the} \textcolor{fig1ochre}{specific} \textcolor{fig1ochre}{needs} \textcolor{fig1teal}{and} \textcolor{fig1ochre}{talked} \textcolor{fig1teal}{up} \textcolor{fig1ochre}{endlessly} \textcolor{fig1teal}{about} \textcolor{fig1teal}{both} \textcolor{fig1ochre}{milk} \textcolor{fig1teal}{and} \textcolor{fig1ochre}{population}\textcolor{fig1teal}{."} \textcolor{fig1teal}{No}\textcolor{fig1teal}{,} \textcolor{fig1teal}{we} \textcolor{fig1ochre}{barely} \textcolor{fig1ochre}{hear} \textcolor{fig1teal}{a} \textcolor{fig1ochre}{damn} \textcolor{fig1teal}{much} \textcolor{fig1teal}{about} \textcolor{fig1ochre}{radioactive}\textcolor{fig1ochre}{omics}\textcolor{fig1teal}{,} \textcolor{fig1teal}{from} \textcolor{fig1ochre}{1985}\textcolor{fig1teal}{-}\textcolor{fig1ochre}{Japan} \textcolor{fig1teal}{the} \textcolor{fig1teal}{only} \textcolor{fig1ochre}{device} \textcolor{fig1teal}{since} \textcolor{fig1teal}{the} \textcolor{fig1ochre}{release} \textcolor{fig1teal}{of} \textcolor{fig1ochre}{MH} \textcolor{fig1ochre}{figures}\textcolor{fig1ochre}{--}\textcolor{fig1ochre}{before} \textcolor{fig1ochre}{latest} \textcolor{fig1ochre}{discoveries} \textcolor{fig1teal}{have} \textcolor{fig1ochre}{become} \textcolor{fig1teal}{part} \textcolor{fig1teal}{of} \textcolor{fig1teal}{the} \textcolor{fig1teal}{data}\textcolor{fig1teal}{.}  \textcolor{fig1teal}{The} \textcolor{fig1teal}{in} \textcolor{fig1ochre}{order} \textcolor{fig1teal}{range} \textcolor{fig1ochre}{?} \textcolor{fig1teal}{the} \textcolor{fig1ochre}{numbers} \textcolor{fig1teal}{on} \textcolor{fig1ochre}{Earth} \textcolor{fig1teal}{by} \textcolor{fig1teal}{New} \textcolor{fig1ochre}{episode} \textcolor{fig1teal}{were} \textcolor{fig1teal}{the} \textcolor{fig1ochre}{majority} \textcolor{fig1teal}{of} \textcolor{fig1teal}{the} \textcolor{fig1teal}{Q}\textcolor{fig1teal}{1} \textcolor{fig1ochre}{Six}\textcolor{fig1teal}{.} \textcolor{fig1ochre}{U}\textcolor{fig1ochre}{MD}\textcolor{fig1teal}{-}\textcolor{fig1ochre}{36} \textcolor{fig1ochre}{never} \textcolor{fig1ochre}{+}\textcolor{fig1ochre}{leave} \textcolor{fig1ochre}{No}\textcolor{fig1teal}{.} \textcolor{fig1teal}{1}\textcolor{fig1teal}{,} \textcolor{fig1ochre}{Channel} \textcolor{fig1ochre}{East} \textcolor{fig1teal}{and} \textcolor{fig1ochre}{drinking}\textcolor{fig1teal}{,} \textcolor{fig1teal}{while} \textcolor{fig1ochre}{losing} \textcolor{fig1teal}{11}\textcolor{fig1teal}{.}\textcolor{fig1ochre}{0} \textcolor{fig1ochre}{million}\textcolor{fig1teal}{.}  \textcolor{fig1ochre}{Japan} \textcolor{fig1ochre}{represented} \textcolor{fig1teal}{"}\textcolor{fig1ochre}{No}\textcolor{fig1teal}{.} \textcolor{fig1teal}{26}\textcolor{fig1teal}{"} \textcolor{fig1teal}{that} \textcolor{fig1teal}{it} \textcolor{fig1ochre}{dominated} \textcolor{fig1teal}{the} \textcolor{fig1teal}{last} \textcolor{fig1teal}{air} \textcolor{fig1teal}{shows} \textcolor{fig1teal}{in} \textcolor{fig1teal}{that} \textcolor{fig1ochre}{range}\textcolor{fig1teal}{,} \textcolor{fig1ochre}{fro} \textcolor{fig1teal}{the} \textcolor{fig1ochre}{December} \textcolor{fig1ochre}{2017} \textcolor{fig1teal}{report}\textcolor{fig1teal}{.} \textcolor{fig1teal}{In} \textcolor{fig1ochre}{2000} \textcolor{fig1teal}{period} \textcolor{fig1teal}{officials}\textcolor{fig1ochre}{leg}\textcolor{fig1teal}{ally} \textcolor{fig1teal}{had} \textcolor{fig1teal}{2}\textcolor{fig1teal}{,}\textcolor{fig1ochre}{600} \textcolor{fig1ochre}{fewer} \textcolor{fig1teal}{TV} \textcolor{fig1ochre}{channels}\textcolor{fig1teal}{,} \textcolor{fig1teal}{and} \textcolor{fig1ochre}{Singapore} \textcolor{fig1teal}{made} \textcolor{fig1ochre}{roughly} \textcolor{fig1teal}{12}\textcolor{fig1teal}{.}\textcolor{fig1teal}{2} \textcolor{fig1ochre}{million} \textcolor{fig1teal}{per} \textcolor{fig1teal}{100}\textcolor{fig1ochre}{minute} \textcolor{fig1ochre}{citiz} \textcolor{fig1ochre}{units}\textcolor{fig1teal}{.} \textcolor{fig1teal}{However}\textcolor{fig1teal}{,} \textcolor{fig1ochre}{Jan} \textcolor{fig1ochre}{18}\textcolor{fig1teal}{.}\textcolor{fig1teal}{25} \textcolor{fig1ochre}{million} \textcolor{fig1ochre}{developed} \textcolor{fig1ochre}{watching} \textcolor{fig1ochre}{Eyes} \textcolor{fig1ochre}{commercials}\textcolor{fig1teal}{,} \textcolor{fig1teal}{the} \textcolor{fig1teal}{same} \textcolor{fig1ochre}{metric} \textcolor{fig1teal}{in} \textcolor{fig1ochre}{UK}\textcolor{fig1teal}{,} \textcolor{fig1teal}{and} \textcolor{fig1teal}{a} \textcolor{fig1teal}{few} \textcolor{fig1teal}{other} \textcolor{fig1ochre}{fluct} \textcolor{fig1ochre}{much} \textcolor{fig1ochre}{times} \textcolor{fig1teal}{in} \textcolor{fig1teal}{that} \textcolor{fig1ochre}{period}\textcolor{fig1teal}{.}  \textcolor{fig1teal}{But} \textcolor{fig1ochre}{ISIS} \textcolor{fig1ochre}{wat} \textcolor{fig1ochre}{wrapped} \textcolor{fig1ochre}{wound} \textcolor{fig1teal}{up} \textcolor{fig1teal}{more} \textcolor{fig1teal}{like} \textcolor{fig1ochre}{Germans} \textcolor{fig1teal}{and} \textcolor{fig1ochre}{Chinese}\textcolor{fig1teal}{:}  \textcolor{fig1ochre}{Note}\textcolor{fig1teal}{:} \textcolor{fig1ochre}{Watch} \textcolor{fig1teal}{today}\textcolor{fig1teal}{'s} \textcolor{fig1ochre}{transcript} \textcolor{fig1teal}{live} \textcolor{fig1teal}{on} \textcolor{fig1ochre}{>}
\end{samplebox}
\begin{samplebox}{\normalfont\textbf{$\ours$ (PR), Round: 2} \hfill \normalfont\scriptsize \textcolor{darkgray}{Committed: \textbf{864/1024}}}
\scriptsize\linespread{0.9}\selectfont
 \textcolor{fig1teal}{say} \textcolor{fig1teal}{a} \textcolor{fig1ochre}{quarter} \textcolor{fig1teal}{of} \textcolor{fig1teal}{a} \textcolor{fig1ochre}{degree}\textcolor{fig1teal}{,} \textcolor{fig1teal}{\$}\textcolor{fig1teal}{5}\textcolor{fig1teal}{,}\textcolor{fig1teal}{000} \textcolor{fig1teal}{a} \textcolor{fig1teal}{year} \textcolor{fig1teal}{on} \textcolor{fig1teal}{the} \textcolor{fig1teal}{network} \textcolor{fig1teal}{by} \textcolor{fig1teal}{three} \textcolor{fig1teal}{women}\textcolor{fig1teal}{—}\textcolor{fig1ochre}{even} \textcolor{fig1teal}{though} \textcolor{fig1teal}{it}\textcolor{fig1teal}{�}\textcolor{fig1teal}{�}\textcolor{fig1teal}{s} \textcolor{fig1teal}{the} \textcolor{fig1teal}{only} \textcolor{fig1teal}{thing} \textcolor{fig1teal}{in} \textcolor{fig1teal}{the} \textcolor{fig1teal}{U}\textcolor{fig1teal}{.}\textcolor{fig1teal}{S}\textcolor{fig1teal}{.} \textcolor{fig1teal}{network} \textcolor{fig1teal}{today}\textcolor{fig1teal}{,} \textcolor{fig1teal}{that} \textcolor{fig1teal}{is} \textcolor{fig1teal}{a} \textcolor{fig1teal}{national} \textcolor{fig1teal}{good}\textcolor{fig1teal}{.}\textcolor{fig1teal}{�}\textcolor{fig1teal}{�}  \textcolor{fig1teal}{But} \textcolor{fig1teal}{saying} \textcolor{fig1teal}{that} \textcolor{fig1ochre}{garner}\textcolor{fig1teal}{ing} \textcolor{fig1teal}{money} \textcolor{fig1teal}{from} \textcolor{fig1teal}{�}\textcolor{fig1teal}{�}\textcolor{fig1teal}{m}\textcolor{fig1ochre}{ovies}\textcolor{fig1teal}{�}\textcolor{fig1teal}{�} \textcolor{fig1teal}{a} \textcolor{fig1teal}{�}\textcolor{fig1teal}{�}\textcolor{fig1ochre}{law}\textcolor{fig1teal}{,}\textcolor{fig1teal}{�}\textcolor{fig1teal}{�} \textcolor{fig1teal}{B}\textcolor{fig1ochre}{add} \textcolor{fig1teal}{said} \textcolor{fig1teal}{it} \textcolor{fig1teal}{would} \textcolor{fig1teal}{seem} \textcolor{fig1teal}{that} \textcolor{fig1teal}{that} \textcolor{fig1teal}{is} \textcolor{fig1teal}{a} \textcolor{fig1ochre}{step} \textcolor{fig1teal}{later}\textcolor{fig1teal}{,} \textcolor{fig1teal}{but} \textcolor{fig1teal}{the} \textcolor{fig1ochre}{invention} \textcolor{fig1teal}{of} \textcolor{fig1teal}{race} \textcolor{fig1teal}{in} \textcolor{fig1teal}{the} \textcolor{fig1teal}{F} \textcolor{fig1teal}{few} \textcolor{fig1teal}{is} \textcolor{fig1teal}{a} \textcolor{fig1teal}{good} \textcolor{fig1ochre}{beauty} \textcolor{fig1teal}{and} \textcolor{fig1teal}{could} \textcolor{fig1teal}{make} \textcolor{fig1ochre}{Rain}\textcolor{fig1teal}{n}\textcolor{fig1teal}{�}\textcolor{fig1teal}{�}\textcolor{fig1teal}{s} \textcolor{fig1ochre}{Yellow} \textcolor{fig1ochre}{air}\textcolor{fig1teal}{l} \textcolor{fig1teal}{The} \textcolor{fig1ochre}{2007} \textcolor{fig1ochre}{uncomfortable}\textcolor{fig1teal}{—}\textcolor{fig1ochre}{but} \textcolor{fig1teal}{he} \textcolor{fig1teal}{said} \textcolor{fig1teal}{he}\textcolor{fig1teal}{�}\textcolor{fig1teal}{�}\textcolor{fig1teal}{s} \textcolor{fig1teal}{happy} \textcolor{fig1teal}{to} \textcolor{fig1teal}{be} \textcolor{fig1ochre}{convinced} \textcolor{fig1teal}{that} \textcolor{fig1teal}{the} \textcolor{fig1ochre}{building} \textcolor{fig1teal}{around} \textcolor{fig1teal}{that} \textcolor{fig1teal}{of} \textcolor{fig1teal}{a} \textcolor{fig1ochre}{black} \textcolor{fig1teal}{character} \textcolor{fig1teal}{together} \textcolor{fig1teal}{under} \textcolor{fig1teal}{an} \textcolor{fig1ochre}{arc} \textcolor{fig1teal}{of} \textcolor{fig1ochre}{creator}\textcolor{fig1teal}{,} \textcolor{fig1teal}{not} \textcolor{fig1teal}{at} \textcolor{fig1teal}{least} \textcolor{fig1teal}{one} \textcolor{fig1teal}{woman}\textcolor{fig1teal}{,} \textcolor{fig1teal}{could} \textcolor{fig1teal}{become} \textcolor{fig1teal}{the} \textcolor{fig1teal}{idea} \textcolor{fig1teal}{more} \textcolor{fig1ochre}{credible}\textcolor{fig1teal}{.} \textcolor{fig1teal}{�}\textcolor{fig1teal}{�}\textcolor{fig1teal}{If} \textcolor{fig1teal}{it} \textcolor{fig1teal}{gets} \textcolor{fig1teal}{me} \textcolor{fig1teal}{wrong}\textcolor{fig1teal}{,} \textcolor{fig1teal}{it}\textcolor{fig1teal}{�}\textcolor{fig1teal}{�}\textcolor{fig1teal}{s} \textcolor{fig1teal}{going} \textcolor{fig1teal}{to} \textcolor{fig1teal}{stop} \textcolor{fig1teal}{working} \textcolor{fig1teal}{for} \textcolor{fig1teal}{two} \textcolor{fig1teal}{years}\textcolor{fig1teal}{,}\textcolor{fig1teal}{�}\textcolor{fig1teal}{�} \textcolor{fig1teal}{he} \textcolor{fig1teal}{said}\textcolor{fig1teal}{.} \textcolor{fig1teal}{�}\textcolor{fig1teal}{�}\textcolor{fig1teal}{You} \textcolor{fig1teal}{think}\textcolor{fig1teal}{,} \textcolor{fig1teal}{the} \textcolor{fig1ochre}{Channel} \textcolor{fig1teal}{of} \textcolor{fig1teal}{media} \textcolor{fig1ochre}{producers} \textcolor{fig1teal}{and} \textcolor{fig1teal}{me}\textcolor{fig1teal}{of}\textcolor{fig1teal}{us} \textcolor{fig1teal}{men} \textcolor{fig1teal}{at} \textcolor{fig1teal}{your} \textcolor{fig1ochre}{nine}\textcolor{fig1teal}{,} \textcolor{fig1ochre}{crude}\textcolor{fig1teal}{,} \textcolor{fig1teal}{and} \textcolor{fig1teal}{y}\textcolor{fig1ochre}{ab} \textcolor{fig1teal}{you} \textcolor{fig1teal}{with} \textcolor{fig1teal}{a} \textcolor{fig1ochre}{bottle} \textcolor{fig1teal}{while} \textcolor{fig1ochre}{dragging} \textcolor{fig1teal}{it} \textcolor{fig1teal}{through} \textcolor{fig1teal}{along}\textcolor{fig1teal}{?}\textcolor{fig1teal}{�}\textcolor{fig1teal}{�}\textcolor{fig1teal}{�}\textcolor{fig1teal}{�}  \textcolor{fig1teal}{But} \textcolor{fig1ochre}{Her}\textcolor{fig1ochre}{dan} \textcolor{fig1teal}{is} \textcolor{fig1teal}{also} \textcolor{fig1ochre}{annoyed} \textcolor{fig1teal}{that} \textcolor{fig1teal}{her} \textcolor{fig1teal}{own} \textcolor{fig1teal}{show}\textcolor{fig1teal}{�}\textcolor{fig1teal}{�}\textcolor{fig1teal}{s} \textcolor{fig1teal}{lack} \textcolor{fig1teal}{of} \textcolor{fig1ochre}{diversity} \textcolor{fig1teal}{around} \textcolor{fig1teal}{more} \textcolor{fig1ochre}{powerful} \textcolor{fig1ochre}{taught} \textcolor{fig1ochre}{humor} \textcolor{fig1teal}{at} \textcolor{fig1teal}{children}\textcolor{fig1teal}{.} \textcolor{fig1teal}{�}\textcolor{fig1teal}{�}\textcolor{fig1teal}{And} \textcolor{fig1ochre}{look}\textcolor{fig1teal}{,} \textcolor{fig1teal}{they}\textcolor{fig1teal}{�}\textcolor{fig1teal}{�}\textcolor{fig1teal}{re} \textcolor{fig1teal}{quite} \textcolor{fig1teal}{certainly}\textcolor{fig1teal}{,} \textcolor{fig1teal}{why} \textcolor{fig1teal}{not} \textcolor{fig1teal}{the} \textcolor{fig1teal}{most} \textcolor{fig1teal}{part}\textcolor{fig1teal}{,} \textcolor{fig1teal}{from} \textcolor{fig1teal}{the} \textcolor{fig1ochre}{backgrounds}\textcolor{fig1teal}{,} \textcolor{fig1ochre}{middle} \textcolor{fig1teal}{of} \textcolor{fig1teal}{the} \textcolor{fig1teal}{world}\textcolor{fig1teal}{,} \textcolor{fig1teal}{I} \textcolor{fig1teal}{have} \textcolor{fig1teal}{seen} \textcolor{fig1teal}{with}\textcolor{fig1teal}{,} \textcolor{fig1teal}{can} \textcolor{fig1teal}{one} \textcolor{fig1teal}{create} \textcolor{fig1teal}{that}\textcolor{fig1teal}{.}\textcolor{fig1teal}{�}\textcolor{fig1teal}{�} \textcolor{fig1teal}{He} \textcolor{fig1teal}{was} \textcolor{fig1teal}{asked} \textcolor{fig1teal}{about} \textcolor{fig1teal}{that} \textcolor{fig1teal}{when} \textcolor{fig1teal}{she} \textcolor{fig1teal}{said} \textcolor{fig1teal}{he} \textcolor{fig1teal}{asked} \textcolor{fig1ochre}{Shirley} \textcolor{fig1teal}{K}\textcolor{fig1teal}{.} \textcolor{fig1teal}{to} \textcolor{fig1ochre}{cut} \textcolor{fig1teal}{on} \textcolor{fig1teal}{the} \textcolor{fig1teal}{show}\textcolor{fig1teal}{.} \textcolor{fig1teal}{In} \textcolor{fig1ochre}{turn}\textcolor{fig1teal}{,} \textcolor{fig1teal}{these} \textcolor{fig1teal}{threats} \textcolor{fig1teal}{get} \textcolor{fig1teal}{into} \textcolor{fig1teal}{the} \textcolor{fig1ochre}{air} \textcolor{fig1teal}{of} \textcolor{fig1ochre}{disconnect} \textcolor{fig1teal}{between} \textcolor{fig1teal}{a} \textcolor{fig1teal}{number} \textcolor{fig1teal}{of} \textcolor{fig1ochre}{high}\textcolor{fig1teal}{-}\textcolor{fig1ochre}{lif}\textcolor{fig1teal}{ers} \textcolor{fig1teal}{and} \textcolor{fig1teal}{the} \textcolor{fig1teal}{people} \textcolor{fig1teal}{she}\textcolor{fig1teal}{�}\textcolor{fig1teal}{�}\textcolor{fig1teal}{s} \textcolor{fig1teal}{trying} \textcolor{fig1teal}{to} \textcolor{fig1ochre}{fire} \textcolor{fig1teal}{on}\textcolor{fig1teal}{.}\textcolor{fig1teal}{�}\textcolor{fig1teal}{�}\textcolor{fig1teal}{I} \textcolor{fig1teal}{have} \textcolor{fig1teal}{just} \textcolor{fig1teal}{been} \textcolor{fig1ochre}{defeated} \textcolor{fig1teal}{by} \textcolor{fig1teal}{your} \textcolor{fig1teal}{response} \textcolor{fig1teal}{to} \textcolor{fig1teal}{had} \textcolor{fig1teal}{so} \textcolor{fig1teal}{much} \textcolor{fig1teal}{needs} \textcolor{fig1teal}{to} \textcolor{fig1teal}{clear}\textcolor{fig1teal}{.} \textcolor{fig1teal}{I} \textcolor{fig1teal}{say}\textcolor{fig1teal}{,} \textcolor{fig1teal}{well}\textcolor{fig1teal}{,} \textcolor{fig1teal}{that}\textcolor{fig1teal}{�}\textcolor{fig1teal}{�}\textcolor{fig1teal}{s} \textcolor{fig1teal}{me} \textcolor{fig1teal}{why} \textcolor{fig1teal}{we}\textcolor{fig1teal}{�}\textcolor{fig1teal}{�}\textcolor{fig1teal}{re} \textcolor{fig1teal}{just} \textcolor{fig1teal}{as} \textcolor{fig1ochre}{passionate} \textcolor{fig1teal}{for} \textcolor{fig1teal}{our} \textcolor{fig1teal}{work}\textcolor{fig1teal}{.} \textcolor{fig1teal}{In} \textcolor{fig1teal}{my} \textcolor{fig1teal}{campaign}\textcolor{fig1teal}{,} \textcolor{fig1teal}{we} \textcolor{fig1teal}{have} \textcolor{fig1teal}{done} \textcolor{fig1teal}{it} \textcolor{fig1teal}{on} \textcolor{fig1teal}{their} \textcolor{fig1teal}{own}\textcolor{fig1teal}{,} \textcolor{fig1teal}{and} \textcolor{fig1teal}{the} \textcolor{fig1ochre}{benefits} \textcolor{fig1teal}{have} \textcolor{fig1teal}{been} \textcolor{fig1teal}{positive}\textcolor{fig1teal}{.} \textcolor{fig1teal}{This} \textcolor{fig1teal}{post} \textcolor{fig1teal}{was} \textcolor{fig1teal}{a} \textcolor{fig1ochre}{generation} \textcolor{fig1teal}{thing}\textcolor{fig1teal}{.} \textcolor{fig1teal}{I} \textcolor{fig1ochre}{ought} \textcolor{fig1teal}{to} \textcolor{fig1teal}{say} \textcolor{fig1teal}{[}\textcolor{fig1teal}{1}\textcolor{fig1teal}{]} \textcolor{fig1teal}{�}\textcolor{fig1teal}{�}\textcolor{fig1ochre}{Ten} \textcolor{fig1teal}{years} \textcolor{fig1teal}{later}\textcolor{fig1teal}{.}\textcolor{fig1teal}{�}\textcolor{fig1teal}{�} \textcolor{fig1ochre}{Because} \textcolor{fig1teal}{she} \textcolor{fig1teal}{needed} \textcolor{fig1teal}{�}\textcolor{fig1teal}{�}\textcolor{fig1ochre}{98} \textcolor{fig1teal}{years} \textcolor{fig1teal}{to} \textcolor{fig1teal}{make} \textcolor{fig1teal}{the} \textcolor{fig1teal}{news}\textcolor{fig1teal}{,}\textcolor{fig1teal}{�}\textcolor{fig1teal}{�} \textcolor{fig1teal}{we} \textcolor{fig1teal}{gave} \textcolor{fig1teal}{the} \textcolor{fig1teal}{U}\textcolor{fig1teal}{.}\textcolor{fig1teal}{S}\textcolor{fig1teal}{.} \textcolor{fig1ochre}{Oil} \textcolor{fig1teal}{and} \textcolor{fig1ochre}{certainty} \textcolor{fig1ochre}{prices} \textcolor{fig1ochre}{growth} \textcolor{fig1teal}{when} \textcolor{fig1teal}{only} \textcolor{fig1teal}{10}\textcolor{fig1teal}{-}\textcolor{fig1teal}{20} \textcolor{fig1teal}{percent} \textcolor{fig1teal}{of} \textcolor{fig1teal}{the} \textcolor{fig1teal}{gas}\textcolor{fig1teal}{-}\textcolor{fig1ochre}{fuel} \textcolor{fig1ochre}{emissions} \textcolor{fig1ochre}{exists}\textcolor{fig1teal}{.} \textcolor{fig1teal}{The} \textcolor{fig1teal}{news} \textcolor{fig1teal}{to} \textcolor{fig1teal}{the} \textcolor{fig1ochre}{front} \textcolor{fig1teal}{of} \textcolor{fig1ochre}{American} \textcolor{fig1ochre}{consumers} \textcolor{fig1teal}{can} \textcolor{fig1ochre}{improve} \textcolor{fig1teal}{any} \textcolor{fig1teal}{time}\textcolor{fig1teal}{.} \textcolor{fig1teal}{My} \textcolor{fig1teal}{organization} \textcolor{fig1teal}{used} \textcolor{fig1teal}{these} \textcolor{fig1ochre}{examples} \textcolor{fig1teal}{and} \textcolor{fig1ochre}{expanded} \textcolor{fig1teal}{to} \textcolor{fig1teal}{the} \textcolor{fig1teal}{problem} \textcolor{fig1teal}{during} \textcolor{fig1teal}{the} \textcolor{fig1teal}{days} \textcolor{fig1teal}{of} \textcolor{fig1ochre}{fracking}\textcolor{fig1teal}{,} \textcolor{fig1teal}{which} \textcolor{fig1teal}{has} \textcolor{fig1teal}{a} \textcolor{fig1ochre}{horrible} \textcolor{fig1teal}{impact} \textcolor{fig1teal}{on} \textcolor{fig1teal}{the} \textcolor{fig1teal}{50}\textcolor{fig1teal}{/}\textcolor{fig1teal}{Co}\textcolor{fig1teal}{al} \textcolor{fig1teal}{industry}\textcolor{fig1teal}{.} \textcolor{fig1teal}{I} \textcolor{fig1teal}{am} \textcolor{fig1ochre}{grateful} \textcolor{fig1teal}{for} \textcolor{fig1teal}{the} \textcolor{fig1ochre}{image} \textcolor{fig1teal}{of} \textcolor{fig1teal}{taking} \textcolor{fig1teal}{this} \textcolor{fig1teal}{responsibility} \textcolor{fig1teal}{on}\textcolor{fig1teal}{.} \textcolor{fig1teal}{Americans} \textcolor{fig1teal}{have} \textcolor{fig1teal}{the} \textcolor{fig1teal}{right} \textcolor{fig1teal}{to} \textcolor{fig1teal}{know} \textcolor{fig1teal}{with} \textcolor{fig1teal}{you} \textcolor{fig1teal}{top} \textcolor{fig1teal}{of} \textcolor{fig1teal}{where} \textcolor{fig1teal}{you} \textcolor{fig1teal}{stand} \textcolor{fig1teal}{or} \textcolor{fig1teal}{we} \textcolor{fig1teal}{per}\textcolor{fig1ochre}{severe} \textcolor{fig1teal}{with} \textcolor{fig1teal}{that} \textcolor{fig1ochre}{task} \textcolor{fig1teal}{of} \textcolor{fig1ochre}{calling} \textcolor{fig1teal}{You}\textcolor{fig1teal}{,} \textcolor{fig1teal}{the} \textcolor{fig1ochre}{bill} \textcolor{fig1teal}{and} \textcolor{fig1ochre}{throwing} \textcolor{fig1teal}{people} \textcolor{fig1ochre}{out}\textcolor{fig1teal}{,} \textcolor{fig1teal}{and} \textcolor{fig1teal}{get} \textcolor{fig1teal}{ahead} \textcolor{fig1teal}{of} \textcolor{fig1teal}{what} \textcolor{fig1teal}{you} \textcolor{fig1teal}{are} \textcolor{fig1ochre}{fighting} \textcolor{fig1teal}{for}\textcolor{fig1teal}{.} \textcolor{fig1teal}{We} \textcolor{fig1teal}{are} \textcolor{fig1teal}{not} \textcolor{fig1ochre}{planned} \textcolor{fig1teal}{for} \textcolor{fig1ochre}{digit}\textcolor{fig1teal}{ization} \textcolor{fig1teal}{of} \textcolor{fig1teal}{our} \textcolor{fig1ochre}{Network} \textcolor{fig1teal}{now}\textcolor{fig1teal}{,} \textcolor{fig1teal}{and} \textcolor{fig1teal}{the} \textcolor{fig1teal}{strength} \textcolor{fig1teal}{of} \textcolor{fig1teal}{our} \textcolor{fig1ochre}{businesses} \textcolor{fig1teal}{is} \textcolor{fig1teal}{very} \textcolor{fig1teal}{great}\textcolor{fig1teal}{.}  \textcolor{fig1teal}{I} \textcolor{fig1teal}{told} \textcolor{fig1teal}{anyone} \textcolor{fig1ochre}{talking} \textcolor{fig1teal}{to} \textcolor{fig1teal}{me} \textcolor{fig1teal}{about} \textcolor{fig1teal}{that} \textcolor{fig1teal}{as} \textcolor{fig1teal}{well}\textcolor{fig1teal}{.}  \textcolor{fig1teal}{F}\textcolor{fig1ochre}{oul}\textcolor{fig1teal}{'t} \textcolor{fig1teal}{it} \textcolor{fig1teal}{be} \textcolor{fig1teal}{fine}\textcolor{fig1teal}{?}  \textcolor{fig1teal}{(} \textcolor{fig1teal}{here}\textcolor{fig1teal}{'s} \textcolor{fig1teal}{my} \textcolor{fig1teal}{response}\textcolor{fig1teal}{,} \textcolor{fig1teal}{I} \textcolor{fig1teal}{said} \textcolor{fig1teal}{yes} \textcolor{fig1teal}{that} \textcolor{fig1ochre}{showed} \textcolor{fig1teal}{some} \textcolor{fig1ochre}{compliance} \textcolor{fig1teal}{in} \textcolor{fig1teal}{action} \textcolor{fig1teal}{from} \textcolor{fig1teal}{the} \textcolor{fig1teal}{person} \textcolor{fig1ochre}{interviewed}\textcolor{fig1teal}{.} \textcolor{fig1teal}{However}\textcolor{fig1teal}{,} \textcolor{fig1teal}{I} \textcolor{fig1teal}{think} \textcolor{fig1teal}{our} \textcolor{fig1teal}{national} \textcolor{fig1ochre}{reps} \textcolor{fig1teal}{and} \textcolor{fig1teal}{companies} \textcolor{fig1teal}{called} \textcolor{fig1teal}{for} \textcolor{fig1teal}{a} \textcolor{fig1teal}{up} \textcolor{fig1teal}{or} \textcolor{fig1teal}{down} \textcolor{fig1teal}{when} \textcolor{fig1ochre}{raising} \textcolor{fig1teal}{an} \textcolor{fig1teal}{issue}\textcolor{fig1teal}{.} \textcolor{fig1teal}{We} \textcolor{fig1teal}{and} \textcolor{fig1ochre}{Times} \textcolor{fig1teal}{or} \textcolor{fig1teal}{the} \textcolor{fig1teal}{"}\textcolor{fig1ochre}{-"} \textcolor{fig1teal}{one} \textcolor{fig1teal}{of} \textcolor{fig1teal}{which}\textcolor{fig1ochre}{"-} \textcolor{fig1teal}{no} \textcolor{fig1teal}{one} \textcolor{fig1teal}{I} \textcolor{fig1teal}{was} \textcolor{fig1teal}{hearing} \textcolor{fig1teal}{from}\textcolor{fig1teal}{.}  \textcolor{fig1teal}{I} \textcolor{fig1teal}{think} \textcolor{fig1teal}{it} \textcolor{fig1teal}{is} \textcolor{fig1teal}{just} \textcolor{fig1teal}{fine} \textcolor{fig1teal}{in} \textcolor{fig1teal}{every} \textcolor{fig1teal}{way}\textcolor{fig1teal}{.} \textcolor{fig1teal}{I} \textcolor{fig1teal}{will} \textcolor{fig1teal}{continue} \textcolor{fig1teal}{if} \textcolor{fig1teal}{you} \textcolor{fig1teal}{at} \textcolor{fig1teal}{I} \textcolor{fig1teal}{feel} \textcolor{fig1teal}{like}\textcolor{fig1teal}{.}  \textcolor{fig1teal}{I} \textcolor{fig1teal}{know} \textcolor{fig1teal}{you} \textcolor{fig1teal}{would} \textcolor{fig1teal}{like} \textcolor{fig1teal}{us} \textcolor{fig1teal}{to}\textcolor{fig1teal}{,} \textcolor{fig1teal}{as} \textcolor{fig1teal}{a} \textcolor{fig1teal}{part} \textcolor{fig1teal}{of} \textcolor{fig1teal}{this} \textcolor{fig1teal}{business}\textcolor{fig1teal}{,} \textcolor{fig1ochre}{respond} \textcolor{fig1teal}{to} \textcolor{fig1teal}{us} \textcolor{fig1teal}{without} \textcolor{fig1teal}{the}\textcolor{fig1ochre}{ism}\textcolor{fig1teal}{.}\textcolor{fig1teal}{This} \textcolor{fig1teal}{is} \textcolor{fig1teal}{not} \textcolor{fig1teal}{�}\textcolor{fig1teal}{�}\textcolor{fig1teal}{The} \textcolor{fig1ochre}{Price}\textcolor{fig1teal}{,}\textcolor{fig1teal}{�}\textcolor{fig1teal}{�} \textcolor{fig1teal}{by} \textcolor{fig1ochre}{Richard} \textcolor{fig1ochre}{Brand}\textcolor{fig1teal}{is}\textcolor{fig1teal}{,} \textcolor{fig1teal}{put} \textcolor{fig1teal}{off} \textcolor{fig1teal}{for} \textcolor{fig1teal}{office} \textcolor{fig1teal}{at} \textcolor{fig1teal}{the} \textcolor{fig1teal}{top} \textcolor{fig1teal}{of} \textcolor{fig1teal}{the} \textcolor{fig1ochre}{hit}\textcolor{fig1teal}{-}\textcolor{fig1teal}{the}\textcolor{fig1teal}{-} \textcolor{fig1teal}{news}\textcolor{fig1teal}{.} \textcolor{fig1teal}{In} \textcolor{fig1teal}{Russia}\textcolor{fig1teal}{,} \textcolor{fig1teal}{it} \textcolor{fig1teal}{a} \textcolor{fig1teal}{lot} \textcolor{fig1ochre}{tighter}\textcolor{fig1teal}{,} \textcolor{fig1teal}{his} \textcolor{fig1teal}{five} \textcolor{fig1teal}{state} \textcolor{fig1ochre}{secrets}\textcolor{fig1teal}{,} \textcolor{fig1ochre}{papers} \textcolor{fig1teal}{and} \textcolor{fig1ochre}{papers} \textcolor{fig1teal}{said} \textcolor{fig1teal}{they} \textcolor{fig1ochre}{announced} \textcolor{fig1teal}{and} \textcolor{fig1ochre}{talked} \textcolor{fig1teal}{about} \textcolor{fig1teal}{through} \textcolor{fig1teal}{the} \textcolor{fig1teal}{daily} \textcolor{fig1teal}{R}\textcolor{fig1ochre}{US} \textcolor{fig1ochre}{Person}\textcolor{fig1teal}{.} \textcolor{fig1ochre}{Brand}\textcolor{fig1teal}{is} \textcolor{fig1ochre}{compared} \textcolor{fig1ochre}{terms} \textcolor{fig1teal}{such} \textcolor{fig1teal}{as} \textcolor{fig1teal}{"}\textcolor{fig1ochre}{men}\textcolor{fig1teal}{and} \textcolor{fig1ochre}{Development}\textcolor{fig1teal}{"} \textcolor{fig1teal}{to} \textcolor{fig1teal}{the} \textcolor{fig1teal}{US}\textcolor{fig1teal}{.}  \textcolor{fig1teal}{"}\textcolor{fig1ochre}{Initially}\textcolor{fig1teal}{,} \textcolor{fig1teal}{the} \textcolor{fig1teal}{government} \textcolor{fig1ochre}{heads} \textcolor{fig1teal}{were} \textcolor{fig1ochre}{Russian}\textcolor{fig1teal}{,} \textcolor{fig1ochre}{briefly} \textcolor{fig1teal}{on} \textcolor{fig1ochre}{stage}\textcolor{fig1teal}{,} \textcolor{fig1teal}{but} \textcolor{fig1ochre}{dazz}\textcolor{fig1teal}{led} \textcolor{fig1teal}{and} \textcolor{fig1ochre}{eclips}\textcolor{fig1teal}{ed} \textcolor{fig1teal}{by} \textcolor{fig1teal}{a} \textcolor{fig1teal}{right}\textcolor{fig1teal}{-}\textcolor{fig1teal}{of}\textcolor{fig1teal}{-}\textcolor{fig1ochre}{aligned} \textcolor{fig1teal}{media} \textcolor{fig1ochre}{blitz} \textcolor{fig1teal}{from} \textcolor{fig1teal}{America}\textcolor{fig1teal}{,"} \textcolor{fig1teal}{wrote} \textcolor{fig1teal}{his} \textcolor{fig1ochre}{thesis}\textcolor{fig1teal}{.} \textcolor{fig1teal}{"}\textcolor{fig1teal}{They} \textcolor{fig1teal}{did} \textcolor{fig1teal}{then} \textcolor{fig1ochre}{learn} \textcolor{fig1teal}{less} \textcolor{fig1teal}{from} \textcolor{fig1ochre}{Alm}\textcolor{fig1teal}{,} \textcolor{fig1teal}{who} \textcolor{fig1ochre}{came} \textcolor{fig1teal}{to} \textcolor{fig1teal}{understand} \textcolor{fig1teal}{the} \textcolor{fig1teal}{country}\textcolor{fig1teal}{,} \textcolor{fig1teal}{and} \textcolor{fig1ochre}{picked} \textcolor{fig1teal}{up} \textcolor{fig1teal}{its} \textcolor{fig1teal}{about} \textcolor{fig1teal}{both} \textcolor{fig1ochre}{midnight} \textcolor{fig1teal}{and} \textcolor{fig1ochre}{digital}\textcolor{fig1teal}{."} \textcolor{fig1teal}{No}\textcolor{fig1teal}{,} \textcolor{fig1teal}{we}\textcolor{fig1teal}{'re} \textcolor{fig1ochre}{talking} \textcolor{fig1teal}{a} \textcolor{fig1teal}{little} \textcolor{fig1teal}{much} \textcolor{fig1teal}{about} \textcolor{fig1teal}{the} \textcolor{fig1teal}{data}\textcolor{fig1teal}{,} \textcolor{fig1teal}{from} \textcolor{fig1ochre}{1999}\textcolor{fig1teal}{-}\textcolor{fig1teal}{on} \textcolor{fig1teal}{the} \textcolor{fig1teal}{only} \textcolor{fig1teal}{point} \textcolor{fig1teal}{since} \textcolor{fig1teal}{the} \textcolor{fig1ochre}{collapse} \textcolor{fig1teal}{of} \textcolor{fig1teal}{the} \textcolor{fig1teal}{C}\textcolor{fig1ochre}{DS}\textcolor{fig1teal}{,} \textcolor{fig1teal}{we} \textcolor{fig1teal}{may} \textcolor{fig1teal}{have} \textcolor{fig1ochre}{forgotten} \textcolor{fig1teal}{part} \textcolor{fig1teal}{of} \textcolor{fig1teal}{the} \textcolor{fig1teal}{data}\textcolor{fig1teal}{.}  \textcolor{fig1teal}{The} \textcolor{fig1teal}{in} \textcolor{fig1teal}{particular} \textcolor{fig1teal}{range} \textcolor{fig1teal}{of} \textcolor{fig1teal}{the} \textcolor{fig1teal}{shows} \textcolor{fig1teal}{on} \textcolor{fig1ochre}{Radio} \textcolor{fig1teal}{by} \textcolor{fig1teal}{New} \textcolor{fig1ochre}{Russians} \textcolor{fig1teal}{were} \textcolor{fig1teal}{the} \textcolor{fig1ochre}{majority} \textcolor{fig1teal}{of} \textcolor{fig1teal}{the} \textcolor{fig1teal}{Q}\textcolor{fig1teal}{1} \textcolor{fig1teal}{period}\textcolor{fig1teal}{.} \textcolor{fig1teal}{The} \textcolor{fig1teal}{non}\textcolor{fig1teal}{-}\textcolor{fig1ochre}{script} \textcolor{fig1ochre}{period} \textcolor{fig1ochre}{landed} \textcolor{fig1teal}{on} \textcolor{fig1ochre}{Oct}\textcolor{fig1teal}{.} \textcolor{fig1teal}{1}\textcolor{fig1teal}{,} \textcolor{fig1ochre}{2012}\textcolor{fig1teal}{,} \textcolor{fig1teal}{and} \textcolor{fig1ochre}{aired}\textcolor{fig1teal}{,} \textcolor{fig1teal}{while} \textcolor{fig1teal}{its} \textcolor{fig1teal}{11}\textcolor{fig1teal}{.}\textcolor{fig1ochre}{25} \textcolor{fig1ochre}{ended}\textcolor{fig1teal}{.}  \textcolor{fig1ochre}{These} \textcolor{fig1ochre}{discussing} \textcolor{fig1teal}{"}\textcolor{fig1ochre}{Sept}\textcolor{fig1teal}{.} \textcolor{fig1teal}{26}\textcolor{fig1teal}{"} \textcolor{fig1teal}{that} \textcolor{fig1teal}{it} \textcolor{fig1teal}{were} \textcolor{fig1teal}{the} \textcolor{fig1teal}{last} \textcolor{fig1teal}{air} \textcolor{fig1teal}{shows} \textcolor{fig1teal}{in} \textcolor{fig1teal}{that} \textcolor{fig1teal}{period}\textcolor{fig1teal}{,} \textcolor{fig1teal}{not} \textcolor{fig1teal}{the} \textcolor{fig1ochre}{whole} \textcolor{fig1ochre}{GDP} \textcolor{fig1teal}{report}\textcolor{fig1teal}{.} \textcolor{fig1teal}{In} \textcolor{fig1teal}{that} \textcolor{fig1teal}{period} \textcolor{fig1teal}{officials} \textcolor{fig1ochre}{proportion}\textcolor{fig1teal}{ally} \textcolor{fig1teal}{had} \textcolor{fig1teal}{2}\textcolor{fig1teal}{,}\textcolor{fig1ochre}{440} \textcolor{fig1teal}{online} \textcolor{fig1teal}{TV} \textcolor{fig1ochre}{ads}\textcolor{fig1teal}{,} \textcolor{fig1teal}{and} \textcolor{fig1teal}{they} \textcolor{fig1teal}{made} \textcolor{fig1teal}{up} \textcolor{fig1teal}{12}\textcolor{fig1teal}{.}\textcolor{fig1teal}{2} \textcolor{fig1ochre}{million} \textcolor{fig1teal}{per} \textcolor{fig1teal}{100}\textcolor{fig1teal}{,}\textcolor{fig1teal}{000} \textcolor{fig1ochre}{viewers}\textcolor{fig1teal}{.} \textcolor{fig1teal}{However}\textcolor{fig1teal}{,} \textcolor{fig1teal}{the} \textcolor{fig1teal}{19}\textcolor{fig1teal}{.}\textcolor{fig1teal}{25} \textcolor{fig1ochre}{million} \textcolor{fig1teal}{didn}\textcolor{fig1teal}{'t} \textcolor{fig1ochre}{grow} \textcolor{fig1teal}{further}\textcolor{fig1teal}{,} \textcolor{fig1teal}{the} \textcolor{fig1teal}{same} \textcolor{fig1teal}{year} \textcolor{fig1teal}{in} \textcolor{fig1ochre}{2011}\textcolor{fig1teal}{,} \textcolor{fig1teal}{and} \textcolor{fig1teal}{a} \textcolor{fig1teal}{few} \textcolor{fig1teal}{other} \textcolor{fig1teal}{shows} \textcolor{fig1teal}{were} \textcolor{fig1teal}{added} \textcolor{fig1teal}{in} \textcolor{fig1teal}{that} \textcolor{fig1teal}{period}\textcolor{fig1teal}{.}  \textcolor{fig1teal}{But} \textcolor{fig1teal}{this} \textcolor{fig1teal}{season} \textcolor{fig1teal}{back} \textcolor{fig1ochre}{gives} \textcolor{fig1teal}{up} \textcolor{fig1teal}{more} \textcolor{fig1teal}{like} \textcolor{fig1teal}{America} \textcolor{fig1teal}{and} \textcolor{fig1teal}{Italy}\textcolor{fig1teal}{:}  \textcolor{fig1ochre}{Note}\textcolor{fig1teal}{:} \textcolor{fig1teal}{In} \textcolor{fig1teal}{today}\textcolor{fig1teal}{'s} \textcolor{fig1ochre}{lineup} \textcolor{fig1teal}{live} \textcolor{fig1teal}{on} \textcolor{fig1teal}{He}
\end{samplebox}
\caption{Posterior Refinement trajectory on OpenWebText, Sample-1.}
\label{fig:refine_sample_owt_1}
\end{figure}
\clearpage

\begin{figure}[H]
\centering
\begin{samplebox}{\normalfont\textbf{$\ours$ (PR), Round: 3} \hfill \normalfont\scriptsize \textcolor{darkgray}{Committed: \textbf{944/1024}}}
\scriptsize\linespread{0.9}\selectfont
 \textcolor{fig1teal}{say} \textcolor{fig1teal}{a} \textcolor{fig1ochre}{maximum} \textcolor{fig1teal}{of} \textcolor{fig1teal}{a} \textcolor{fig1teal}{contract}\textcolor{fig1teal}{,} \textcolor{fig1teal}{\$}\textcolor{fig1teal}{5}\textcolor{fig1teal}{,}\textcolor{fig1teal}{000} \textcolor{fig1teal}{a} \textcolor{fig1teal}{year} \textcolor{fig1teal}{on} \textcolor{fig1teal}{the} \textcolor{fig1teal}{network} \textcolor{fig1teal}{by} \textcolor{fig1teal}{three} \textcolor{fig1teal}{women}\textcolor{fig1teal}{—}\textcolor{fig1teal}{even} \textcolor{fig1teal}{though} \textcolor{fig1teal}{it}\textcolor{fig1teal}{�}\textcolor{fig1teal}{�}\textcolor{fig1teal}{s} \textcolor{fig1teal}{the} \textcolor{fig1teal}{only} \textcolor{fig1teal}{thing} \textcolor{fig1teal}{in} \textcolor{fig1teal}{the} \textcolor{fig1teal}{U}\textcolor{fig1teal}{.}\textcolor{fig1teal}{S}\textcolor{fig1teal}{.} \textcolor{fig1teal}{network} \textcolor{fig1teal}{today}\textcolor{fig1teal}{,} \textcolor{fig1teal}{that} \textcolor{fig1teal}{is} \textcolor{fig1teal}{a} \textcolor{fig1teal}{national} \textcolor{fig1teal}{good}\textcolor{fig1teal}{.}\textcolor{fig1teal}{�}\textcolor{fig1teal}{�}  \textcolor{fig1teal}{But} \textcolor{fig1teal}{saying} \textcolor{fig1teal}{that} \textcolor{fig1ochre}{deduct}\textcolor{fig1teal}{ing} \textcolor{fig1teal}{money} \textcolor{fig1teal}{from} \textcolor{fig1teal}{�}\textcolor{fig1teal}{�}\textcolor{fig1teal}{m}\textcolor{fig1ochre}{iles}\textcolor{fig1teal}{�}\textcolor{fig1teal}{�} \textcolor{fig1teal}{a} \textcolor{fig1teal}{�}\textcolor{fig1teal}{�}\textcolor{fig1ochre}{race}\textcolor{fig1teal}{,}\textcolor{fig1teal}{�}\textcolor{fig1teal}{�} \textcolor{fig1teal}{B}\textcolor{fig1ochre}{ally} \textcolor{fig1teal}{said} \textcolor{fig1teal}{it} \textcolor{fig1teal}{would} \textcolor{fig1teal}{seem} \textcolor{fig1teal}{that} \textcolor{fig1teal}{that} \textcolor{fig1teal}{is} \textcolor{fig1teal}{a} \textcolor{fig1teal}{step} \textcolor{fig1teal}{later}\textcolor{fig1teal}{,} \textcolor{fig1teal}{but} \textcolor{fig1teal}{the} \textcolor{fig1ochre}{stress} \textcolor{fig1teal}{of} \textcolor{fig1teal}{race} \textcolor{fig1teal}{in} \textcolor{fig1teal}{the} \textcolor{fig1teal}{F} \textcolor{fig1teal}{few} \textcolor{fig1teal}{is} \textcolor{fig1teal}{a} \textcolor{fig1teal}{good} \textcolor{fig1teal}{thing} \textcolor{fig1teal}{and} \textcolor{fig1teal}{could} \textcolor{fig1teal}{make} \textcolor{fig1ochre}{Wyn}\textcolor{fig1teal}{n}\textcolor{fig1teal}{�}\textcolor{fig1teal}{�}\textcolor{fig1teal}{s} \textcolor{fig1ochre}{Basic} \textcolor{fig1ochre}{é}\textcolor{fig1teal}{l} \textcolor{fig1teal}{The} \textcolor{fig1teal}{World} \textcolor{fig1ochre}{reconsider}\textcolor{fig1teal}{—}\textcolor{fig1teal}{and} \textcolor{fig1teal}{he} \textcolor{fig1teal}{said} \textcolor{fig1teal}{he}\textcolor{fig1teal}{�}\textcolor{fig1teal}{�}\textcolor{fig1teal}{s} \textcolor{fig1teal}{happy} \textcolor{fig1teal}{to} \textcolor{fig1teal}{be} \textcolor{fig1ochre}{hopeful} \textcolor{fig1teal}{that} \textcolor{fig1teal}{the} \textcolor{fig1ochre}{moving} \textcolor{fig1teal}{around} \textcolor{fig1teal}{that} \textcolor{fig1teal}{of} \textcolor{fig1teal}{a} \textcolor{fig1ochre}{fictional} \textcolor{fig1teal}{character} \textcolor{fig1teal}{together} \textcolor{fig1teal}{under} \textcolor{fig1teal}{an} \textcolor{fig1teal}{array} \textcolor{fig1teal}{of} \textcolor{fig1teal}{men}\textcolor{fig1teal}{,} \textcolor{fig1teal}{not} \textcolor{fig1teal}{at} \textcolor{fig1teal}{least} \textcolor{fig1teal}{one} \textcolor{fig1teal}{woman}\textcolor{fig1teal}{,} \textcolor{fig1teal}{could} \textcolor{fig1teal}{become} \textcolor{fig1teal}{the} \textcolor{fig1teal}{idea} \textcolor{fig1teal}{more} \textcolor{fig1ochre}{mainstream}\textcolor{fig1teal}{.} \textcolor{fig1teal}{�}\textcolor{fig1teal}{�}\textcolor{fig1teal}{If} \textcolor{fig1teal}{it} \textcolor{fig1teal}{gets} \textcolor{fig1teal}{me} \textcolor{fig1teal}{wrong}\textcolor{fig1teal}{,} \textcolor{fig1teal}{it}\textcolor{fig1teal}{�}\textcolor{fig1teal}{�}\textcolor{fig1teal}{s} \textcolor{fig1teal}{going} \textcolor{fig1teal}{to} \textcolor{fig1teal}{stop} \textcolor{fig1teal}{working} \textcolor{fig1teal}{for} \textcolor{fig1teal}{two} \textcolor{fig1teal}{years}\textcolor{fig1teal}{,}\textcolor{fig1teal}{�}\textcolor{fig1teal}{�} \textcolor{fig1teal}{he} \textcolor{fig1teal}{said}\textcolor{fig1teal}{.} \textcolor{fig1teal}{�}\textcolor{fig1teal}{�}\textcolor{fig1teal}{You} \textcolor{fig1teal}{think}\textcolor{fig1teal}{,} \textcolor{fig1teal}{the} \textcolor{fig1teal}{amount} \textcolor{fig1teal}{of} \textcolor{fig1teal}{media} \textcolor{fig1ochre}{sells} \textcolor{fig1teal}{and} \textcolor{fig1teal}{me}\textcolor{fig1teal}{of}\textcolor{fig1teal}{us} \textcolor{fig1teal}{men} \textcolor{fig1teal}{at} \textcolor{fig1teal}{your} \textcolor{fig1ochre}{heels}\textcolor{fig1teal}{,}\textcolor{fig1ochre}{ÃÂÃÂ}\textcolor{fig1teal}{,} \textcolor{fig1teal}{and} \textcolor{fig1teal}{y}\textcolor{fig1ochre}{anking} \textcolor{fig1teal}{you} \textcolor{fig1teal}{with} \textcolor{fig1teal}{a} \textcolor{fig1teal}{sub} \textcolor{fig1teal}{while} \textcolor{fig1ochre}{dragging} \textcolor{fig1teal}{it} \textcolor{fig1teal}{through} \textcolor{fig1teal}{along}\textcolor{fig1teal}{?}\textcolor{fig1teal}{�}\textcolor{fig1teal}{�}\textcolor{fig1teal}{�}\textcolor{fig1teal}{�}  \textcolor{fig1teal}{But} \textcolor{fig1teal}{B}\textcolor{fig1ochre}{ally} \textcolor{fig1teal}{is} \textcolor{fig1teal}{also} \textcolor{fig1ochre}{aware} \textcolor{fig1teal}{that} \textcolor{fig1teal}{her} \textcolor{fig1teal}{own} \textcolor{fig1teal}{show}\textcolor{fig1teal}{�}\textcolor{fig1teal}{�}\textcolor{fig1teal}{s} \textcolor{fig1teal}{lack} \textcolor{fig1teal}{of} \textcolor{fig1ochre}{campaigns} \textcolor{fig1teal}{around} \textcolor{fig1teal}{more} \textcolor{fig1teal}{is} \textcolor{fig1ochre}{directed} \textcolor{fig1teal}{directed} \textcolor{fig1teal}{at} \textcolor{fig1teal}{children}\textcolor{fig1teal}{.} \textcolor{fig1teal}{�}\textcolor{fig1teal}{�}\textcolor{fig1teal}{And} \textcolor{fig1teal}{yes}\textcolor{fig1teal}{,} \textcolor{fig1teal}{they}\textcolor{fig1teal}{�}\textcolor{fig1teal}{�}\textcolor{fig1teal}{re} \textcolor{fig1teal}{quite} \textcolor{fig1teal}{certainly}\textcolor{fig1teal}{,} \textcolor{fig1teal}{why} \textcolor{fig1teal}{not} \textcolor{fig1teal}{the} \textcolor{fig1teal}{most} \textcolor{fig1teal}{part}\textcolor{fig1teal}{,} \textcolor{fig1teal}{from} \textcolor{fig1teal}{the} \textcolor{fig1teal}{discussion}\textcolor{fig1teal}{,} \textcolor{fig1teal}{most} \textcolor{fig1teal}{of} \textcolor{fig1teal}{the} \textcolor{fig1teal}{world}\textcolor{fig1teal}{,} \textcolor{fig1teal}{I} \textcolor{fig1teal}{have} \textcolor{fig1teal}{seen} \textcolor{fig1teal}{with}\textcolor{fig1teal}{,} \textcolor{fig1teal}{can} \textcolor{fig1teal}{one} \textcolor{fig1teal}{create} \textcolor{fig1teal}{that}\textcolor{fig1teal}{.}\textcolor{fig1teal}{�}\textcolor{fig1teal}{�} \textcolor{fig1teal}{He} \textcolor{fig1teal}{was} \textcolor{fig1teal}{asked} \textcolor{fig1teal}{about} \textcolor{fig1teal}{that} \textcolor{fig1teal}{when} \textcolor{fig1teal}{she} \textcolor{fig1teal}{said} \textcolor{fig1teal}{he} \textcolor{fig1teal}{asked} \textcolor{fig1ochre}{Louis} \textcolor{fig1teal}{K}\textcolor{fig1teal}{.} \textcolor{fig1teal}{to} \textcolor{fig1teal}{star} \textcolor{fig1teal}{on} \textcolor{fig1teal}{the} \textcolor{fig1teal}{show}\textcolor{fig1teal}{.} \textcolor{fig1teal}{In} \textcolor{fig1teal}{reality}\textcolor{fig1teal}{,} \textcolor{fig1teal}{these} \textcolor{fig1teal}{threats} \textcolor{fig1teal}{get} \textcolor{fig1teal}{into} \textcolor{fig1teal}{the} \textcolor{fig1ochre}{gaps} \textcolor{fig1teal}{of} \textcolor{fig1ochre}{contention} \textcolor{fig1teal}{between} \textcolor{fig1teal}{a} \textcolor{fig1teal}{number} \textcolor{fig1teal}{of} \textcolor{fig1ochre}{sea}\textcolor{fig1teal}{-}\textcolor{fig1ochre}{tain}\textcolor{fig1teal}{ers} \textcolor{fig1teal}{and} \textcolor{fig1teal}{the} \textcolor{fig1teal}{people} \textcolor{fig1teal}{she}\textcolor{fig1teal}{�}\textcolor{fig1teal}{�}\textcolor{fig1teal}{s} \textcolor{fig1teal}{trying} \textcolor{fig1teal}{to} \textcolor{fig1teal}{keep} \textcolor{fig1teal}{on}\textcolor{fig1teal}{.}\textcolor{fig1teal}{�}\textcolor{fig1teal}{�}\textcolor{fig1teal}{I} \textcolor{fig1teal}{have} \textcolor{fig1teal}{just} \textcolor{fig1teal}{been} \textcolor{fig1ochre}{surprised} \textcolor{fig1teal}{by} \textcolor{fig1teal}{your} \textcolor{fig1teal}{response} \textcolor{fig1teal}{to} \textcolor{fig1teal}{had} \textcolor{fig1teal}{so} \textcolor{fig1teal}{much} \textcolor{fig1teal}{needs} \textcolor{fig1teal}{to} \textcolor{fig1teal}{clear}\textcolor{fig1teal}{.} \textcolor{fig1teal}{I} \textcolor{fig1teal}{say}\textcolor{fig1teal}{,} \textcolor{fig1teal}{well}\textcolor{fig1teal}{,} \textcolor{fig1teal}{that}\textcolor{fig1teal}{�}\textcolor{fig1teal}{�}\textcolor{fig1teal}{s} \textcolor{fig1teal}{me} \textcolor{fig1teal}{why} \textcolor{fig1teal}{we}\textcolor{fig1teal}{�}\textcolor{fig1teal}{�}\textcolor{fig1teal}{re} \textcolor{fig1teal}{just} \textcolor{fig1teal}{as} \textcolor{fig1ochre}{grateful} \textcolor{fig1teal}{for} \textcolor{fig1teal}{our} \textcolor{fig1teal}{work}\textcolor{fig1teal}{.} \textcolor{fig1teal}{In} \textcolor{fig1teal}{my} \textcolor{fig1teal}{campaign}\textcolor{fig1teal}{,} \textcolor{fig1teal}{we} \textcolor{fig1teal}{have} \textcolor{fig1teal}{done} \textcolor{fig1teal}{it} \textcolor{fig1teal}{on} \textcolor{fig1teal}{their} \textcolor{fig1teal}{own}\textcolor{fig1teal}{,} \textcolor{fig1teal}{and} \textcolor{fig1teal}{the} \textcolor{fig1teal}{results} \textcolor{fig1teal}{have} \textcolor{fig1teal}{been} \textcolor{fig1teal}{positive}\textcolor{fig1teal}{.} \textcolor{fig1teal}{This} \textcolor{fig1teal}{post} \textcolor{fig1teal}{was} \textcolor{fig1teal}{a} \textcolor{fig1teal}{last} \textcolor{fig1teal}{thing}\textcolor{fig1teal}{.} \textcolor{fig1teal}{I} \textcolor{fig1teal}{want} \textcolor{fig1teal}{to} \textcolor{fig1teal}{say} \textcolor{fig1teal}{[}\textcolor{fig1teal}{1}\textcolor{fig1teal}{]} \textcolor{fig1teal}{�}\textcolor{fig1teal}{�}\textcolor{fig1teal}{30} \textcolor{fig1teal}{years} \textcolor{fig1teal}{later}\textcolor{fig1teal}{.}\textcolor{fig1teal}{�}\textcolor{fig1teal}{�} \textcolor{fig1teal}{When} \textcolor{fig1teal}{she} \textcolor{fig1teal}{needed} \textcolor{fig1teal}{�}\textcolor{fig1teal}{�}\textcolor{fig1teal}{30} \textcolor{fig1teal}{years} \textcolor{fig1teal}{to} \textcolor{fig1teal}{make} \textcolor{fig1teal}{the} \textcolor{fig1teal}{news}\textcolor{fig1teal}{,}\textcolor{fig1teal}{�}\textcolor{fig1teal}{�} \textcolor{fig1teal}{we} \textcolor{fig1teal}{gave} \textcolor{fig1teal}{the} \textcolor{fig1teal}{U}\textcolor{fig1teal}{.}\textcolor{fig1teal}{S}\textcolor{fig1teal}{.} \textcolor{fig1teal}{news} \textcolor{fig1teal}{and} \textcolor{fig1ochre}{gave} \textcolor{fig1teal}{this} \textcolor{fig1teal}{year} \textcolor{fig1teal}{when} \textcolor{fig1teal}{only} \textcolor{fig1teal}{10}\textcolor{fig1teal}{-}\textcolor{fig1teal}{20} \textcolor{fig1teal}{percent} \textcolor{fig1teal}{of} \textcolor{fig1teal}{the} \textcolor{fig1teal}{gas}\textcolor{fig1teal}{-}\textcolor{fig1ochre}{wing}\textcolor{fig1ochre}{aded} \textcolor{fig1ochre}{covered}\textcolor{fig1teal}{.} \textcolor{fig1teal}{The} \textcolor{fig1teal}{news} \textcolor{fig1teal}{to} \textcolor{fig1teal}{the} \textcolor{fig1teal}{people} \textcolor{fig1teal}{of} \textcolor{fig1teal}{this} \textcolor{fig1teal}{region} \textcolor{fig1teal}{can} \textcolor{fig1ochre}{change} \textcolor{fig1teal}{any} \textcolor{fig1teal}{time}\textcolor{fig1teal}{.} \textcolor{fig1teal}{My} \textcolor{fig1teal}{organization} \textcolor{fig1teal}{used} \textcolor{fig1teal}{these} \textcolor{fig1teal}{weapons} \textcolor{fig1teal}{and} \textcolor{fig1ochre}{light} \textcolor{fig1teal}{to} \textcolor{fig1teal}{the} \textcolor{fig1teal}{problem} \textcolor{fig1teal}{during} \textcolor{fig1teal}{the} \textcolor{fig1teal}{days} \textcolor{fig1teal}{of} \textcolor{fig1ochre}{fracking}\textcolor{fig1teal}{,} \textcolor{fig1teal}{which} \textcolor{fig1teal}{has} \textcolor{fig1teal}{a} \textcolor{fig1teal}{huge} \textcolor{fig1teal}{impact} \textcolor{fig1teal}{on} \textcolor{fig1teal}{the} \textcolor{fig1teal}{50}\textcolor{fig1teal}{/}\textcolor{fig1teal}{Co}\textcolor{fig1teal}{al} \textcolor{fig1teal}{industry}\textcolor{fig1teal}{.} \textcolor{fig1teal}{I} \textcolor{fig1teal}{am} \textcolor{fig1ochre}{sorry} \textcolor{fig1teal}{for} \textcolor{fig1teal}{the} \textcolor{fig1teal}{process} \textcolor{fig1teal}{of} \textcolor{fig1teal}{taking} \textcolor{fig1teal}{this} \textcolor{fig1teal}{responsibility} \textcolor{fig1teal}{on}\textcolor{fig1teal}{.} \textcolor{fig1teal}{Americans} \textcolor{fig1teal}{have} \textcolor{fig1teal}{the} \textcolor{fig1teal}{right} \textcolor{fig1teal}{to} \textcolor{fig1teal}{know} \textcolor{fig1teal}{with} \textcolor{fig1teal}{you} \textcolor{fig1teal}{top} \textcolor{fig1teal}{of} \textcolor{fig1teal}{where} \textcolor{fig1teal}{you} \textcolor{fig1teal}{stand} \textcolor{fig1teal}{or} \textcolor{fig1teal}{we} \textcolor{fig1teal}{per} \textcolor{fig1teal}{side} \textcolor{fig1teal}{with} \textcolor{fig1teal}{that} \textcolor{fig1teal}{first} \textcolor{fig1teal}{of}\textcolor{fig1teal}{.} \textcolor{fig1teal}{You}\textcolor{fig1teal}{,} \textcolor{fig1teal}{the} \textcolor{fig1ochre}{reader} \textcolor{fig1teal}{and} \textcolor{fig1teal}{other} \textcolor{fig1teal}{people} \textcolor{fig1teal}{do}\textcolor{fig1teal}{,} \textcolor{fig1teal}{and} \textcolor{fig1teal}{get} \textcolor{fig1teal}{ahead} \textcolor{fig1teal}{of} \textcolor{fig1teal}{what} \textcolor{fig1teal}{you} \textcolor{fig1teal}{are} \textcolor{fig1teal}{asked} \textcolor{fig1teal}{for}\textcolor{fig1teal}{.} \textcolor{fig1teal}{We} \textcolor{fig1teal}{are} \textcolor{fig1teal}{not} \textcolor{fig1teal}{out} \textcolor{fig1teal}{for} \textcolor{fig1teal}{modern}\textcolor{fig1teal}{ization} \textcolor{fig1teal}{of} \textcolor{fig1teal}{our} \textcolor{fig1teal}{industry} \textcolor{fig1teal}{now}\textcolor{fig1teal}{,} \textcolor{fig1teal}{and} \textcolor{fig1teal}{the} \textcolor{fig1teal}{strength} \textcolor{fig1teal}{of} \textcolor{fig1teal}{our} \textcolor{fig1teal}{industry} \textcolor{fig1teal}{is} \textcolor{fig1teal}{very} \textcolor{fig1teal}{great}\textcolor{fig1teal}{.}  \textcolor{fig1teal}{I} \textcolor{fig1teal}{told} \textcolor{fig1teal}{anyone} \textcolor{fig1ochre}{write} \textcolor{fig1teal}{to} \textcolor{fig1teal}{me} \textcolor{fig1teal}{about} \textcolor{fig1teal}{that} \textcolor{fig1teal}{as} \textcolor{fig1teal}{well}\textcolor{fig1teal}{.}  \textcolor{fig1teal}{F}\textcolor{fig1teal}{n}\textcolor{fig1teal}{'t} \textcolor{fig1teal}{it} \textcolor{fig1teal}{be} \textcolor{fig1teal}{fine}\textcolor{fig1teal}{?}  \textcolor{fig1teal}{(} \textcolor{fig1teal}{here}\textcolor{fig1teal}{'s} \textcolor{fig1teal}{my} \textcolor{fig1teal}{response}\textcolor{fig1teal}{,} \textcolor{fig1teal}{I} \textcolor{fig1teal}{said} \textcolor{fig1teal}{yes} \textcolor{fig1teal}{that} \textcolor{fig1ochre}{provoked} \textcolor{fig1teal}{some} \textcolor{fig1teal}{policy} \textcolor{fig1teal}{in} \textcolor{fig1teal}{action} \textcolor{fig1teal}{from} \textcolor{fig1teal}{the} \textcolor{fig1teal}{person} \textcolor{fig1teal}{I}\textcolor{fig1teal}{.} \textcolor{fig1teal}{However}\textcolor{fig1teal}{,} \textcolor{fig1teal}{I} \textcolor{fig1teal}{think} \textcolor{fig1teal}{our} \textcolor{fig1teal}{national} \textcolor{fig1teal}{leaders} \textcolor{fig1teal}{and} \textcolor{fig1teal}{companies} \textcolor{fig1teal}{called} \textcolor{fig1teal}{for} \textcolor{fig1teal}{a} \textcolor{fig1teal}{up} \textcolor{fig1teal}{or} \textcolor{fig1teal}{down} \textcolor{fig1teal}{when} \textcolor{fig1teal}{said} \textcolor{fig1teal}{an} \textcolor{fig1teal}{issue}\textcolor{fig1teal}{.} \textcolor{fig1teal}{We} \textcolor{fig1teal}{and} \textcolor{fig1teal}{out} \textcolor{fig1teal}{or} \textcolor{fig1teal}{the} \textcolor{fig1teal}{"}\textcolor{fig1teal}{the} \textcolor{fig1teal}{one} \textcolor{fig1teal}{of} \textcolor{fig1teal}{which} \textcolor{fig1teal}{was} \textcolor{fig1teal}{no} \textcolor{fig1teal}{one} \textcolor{fig1teal}{I} \textcolor{fig1teal}{was} \textcolor{fig1teal}{hearing} \textcolor{fig1teal}{from}\textcolor{fig1teal}{.}  \textcolor{fig1teal}{I} \textcolor{fig1teal}{think} \textcolor{fig1teal}{it} \textcolor{fig1teal}{is} \textcolor{fig1teal}{just} \textcolor{fig1teal}{fine} \textcolor{fig1teal}{in} \textcolor{fig1teal}{every} \textcolor{fig1teal}{way}\textcolor{fig1teal}{.} \textcolor{fig1teal}{I} \textcolor{fig1teal}{will} \textcolor{fig1teal}{continue} \textcolor{fig1teal}{if} \textcolor{fig1teal}{you} \textcolor{fig1teal}{at} \textcolor{fig1teal}{I} \textcolor{fig1teal}{feel} \textcolor{fig1teal}{like}\textcolor{fig1teal}{.}  \textcolor{fig1teal}{I} \textcolor{fig1teal}{know} \textcolor{fig1teal}{you} \textcolor{fig1teal}{would} \textcolor{fig1teal}{like} \textcolor{fig1teal}{us} \textcolor{fig1teal}{to}\textcolor{fig1teal}{,} \textcolor{fig1teal}{as} \textcolor{fig1teal}{a} \textcolor{fig1teal}{part} \textcolor{fig1teal}{of} \textcolor{fig1teal}{this} \textcolor{fig1teal}{business}\textcolor{fig1teal}{,} \textcolor{fig1teal}{listen} \textcolor{fig1teal}{to} \textcolor{fig1teal}{us} \textcolor{fig1teal}{without} \textcolor{fig1teal}{the} \textcolor{fig1teal}{news}\textcolor{fig1teal}{.}\textcolor{fig1teal}{This} \textcolor{fig1teal}{is} \textcolor{fig1teal}{not} \textcolor{fig1teal}{�}\textcolor{fig1teal}{�}\textcolor{fig1teal}{The} \textcolor{fig1ochre}{Unknown}\textcolor{fig1teal}{,}\textcolor{fig1teal}{�}\textcolor{fig1teal}{�} \textcolor{fig1teal}{by} \textcolor{fig1ochre}{Sol} \textcolor{fig1ochre}{Kot}\textcolor{fig1teal}{is}\textcolor{fig1teal}{,} \textcolor{fig1teal}{put} \textcolor{fig1teal}{off} \textcolor{fig1teal}{for} \textcolor{fig1teal}{office} \textcolor{fig1teal}{at} \textcolor{fig1teal}{the} \textcolor{fig1teal}{top} \textcolor{fig1teal}{of} \textcolor{fig1teal}{the} \textcolor{fig1teal}{year}\textcolor{fig1teal}{-}\textcolor{fig1teal}{the}\textcolor{fig1teal}{-} \textcolor{fig1teal}{news}\textcolor{fig1teal}{.} \textcolor{fig1teal}{In} \textcolor{fig1teal}{Russia}\textcolor{fig1teal}{,} \textcolor{fig1teal}{it} \textcolor{fig1teal}{a} \textcolor{fig1teal}{lot} \textcolor{fig1teal}{like}\textcolor{fig1teal}{,} \textcolor{fig1teal}{his} \textcolor{fig1teal}{five} \textcolor{fig1teal}{state}\textcolor{fig1ochre}{mates}\textcolor{fig1teal}{,} \textcolor{fig1ochre}{bloggers} \textcolor{fig1teal}{and} \textcolor{fig1ochre}{politicians} \textcolor{fig1teal}{said} \textcolor{fig1teal}{they} \textcolor{fig1teal}{did} \textcolor{fig1teal}{and} \textcolor{fig1teal}{wrote} \textcolor{fig1teal}{about} \textcolor{fig1teal}{through} \textcolor{fig1teal}{the} \textcolor{fig1teal}{daily} \textcolor{fig1teal}{R}\textcolor{fig1ochre}{SK}\textcolor{fig1teal}{T}\textcolor{fig1teal}{.} \textcolor{fig1ochre}{Kom}\textcolor{fig1teal}{is} \textcolor{fig1teal}{already} \textcolor{fig1ochre}{claimed} \textcolor{fig1teal}{such} \textcolor{fig1teal}{as} \textcolor{fig1teal}{"}\textcolor{fig1teal}{c}\textcolor{fig1teal}{and}\textcolor{fig1ochre}{ie}\textcolor{fig1teal}{"} \textcolor{fig1teal}{to} \textcolor{fig1teal}{the} \textcolor{fig1teal}{US}\textcolor{fig1teal}{.}  \textcolor{fig1teal}{"}\textcolor{fig1ochre}{Apparently}\textcolor{fig1teal}{,} \textcolor{fig1teal}{the} \textcolor{fig1teal}{government} \textcolor{fig1ochre}{cards} \textcolor{fig1teal}{were} \textcolor{fig1ochre}{flying}\textcolor{fig1teal}{,} \textcolor{fig1teal}{not} \textcolor{fig1teal}{on} \textcolor{fig1teal}{them}\textcolor{fig1teal}{,} \textcolor{fig1teal}{but} \textcolor{fig1ochre}{pedd}\textcolor{fig1teal}{led} \textcolor{fig1teal}{and} \textcolor{fig1ochre}{stomp}\textcolor{fig1teal}{ed} \textcolor{fig1teal}{by} \textcolor{fig1teal}{a} \textcolor{fig1teal}{right}\textcolor{fig1teal}{-}\textcolor{fig1teal}{of}\textcolor{fig1teal}{-}\textcolor{fig1ochre}{center} \textcolor{fig1teal}{media} \textcolor{fig1teal}{cover} \textcolor{fig1teal}{from} \textcolor{fig1teal}{America}\textcolor{fig1teal}{,"} \textcolor{fig1teal}{wrote} \textcolor{fig1teal}{his} \textcolor{fig1ochre}{predecessor}\textcolor{fig1teal}{.} \textcolor{fig1teal}{"}\textcolor{fig1teal}{They} \textcolor{fig1teal}{did} \textcolor{fig1teal}{then} \textcolor{fig1ochre}{steal} \textcolor{fig1teal}{less} \textcolor{fig1teal}{from} \textcolor{fig1teal}{I}\textcolor{fig1teal}{,} \textcolor{fig1teal}{who} \textcolor{fig1ochre}{tends} \textcolor{fig1teal}{to} \textcolor{fig1teal}{understand} \textcolor{fig1teal}{the} \textcolor{fig1teal}{country}\textcolor{fig1teal}{,} \textcolor{fig1teal}{and} \textcolor{fig1teal}{made} \textcolor{fig1teal}{up} \textcolor{fig1teal}{its} \textcolor{fig1teal}{about} \textcolor{fig1teal}{both} \textcolor{fig1ochre}{television} \textcolor{fig1teal}{and} \textcolor{fig1teal}{present}\textcolor{fig1teal}{."} \textcolor{fig1teal}{No}\textcolor{fig1teal}{,} \textcolor{fig1teal}{we}\textcolor{fig1teal}{'re} \textcolor{fig1ochre}{learning} \textcolor{fig1teal}{a} \textcolor{fig1teal}{little} \textcolor{fig1teal}{much} \textcolor{fig1teal}{about} \textcolor{fig1teal}{the} \textcolor{fig1teal}{data}\textcolor{fig1teal}{,} \textcolor{fig1teal}{from} \textcolor{fig1teal}{then}\textcolor{fig1teal}{-}\textcolor{fig1teal}{on} \textcolor{fig1teal}{the} \textcolor{fig1teal}{only} \textcolor{fig1teal}{point} \textcolor{fig1teal}{since} \textcolor{fig1teal}{the} \textcolor{fig1ochre}{advent} \textcolor{fig1teal}{of} \textcolor{fig1teal}{the} \textcolor{fig1teal}{C}\textcolor{fig1ochre}{ata}\textcolor{fig1teal}{,} \textcolor{fig1teal}{we} \textcolor{fig1teal}{may} \textcolor{fig1teal}{have} \textcolor{fig1teal}{lost} \textcolor{fig1teal}{part} \textcolor{fig1teal}{of} \textcolor{fig1teal}{the} \textcolor{fig1teal}{data}\textcolor{fig1teal}{.}  \textcolor{fig1teal}{The} \textcolor{fig1teal}{in} \textcolor{fig1teal}{particular} \textcolor{fig1teal}{range} \textcolor{fig1teal}{of} \textcolor{fig1teal}{the} \textcolor{fig1teal}{shows} \textcolor{fig1teal}{on} \textcolor{fig1teal}{Wednesday} \textcolor{fig1teal}{by} \textcolor{fig1teal}{New} \textcolor{fig1ochre}{Moscow} \textcolor{fig1teal}{were} \textcolor{fig1teal}{the} \textcolor{fig1ochre}{result} \textcolor{fig1teal}{of} \textcolor{fig1teal}{the} \textcolor{fig1teal}{Q}\textcolor{fig1teal}{1} \textcolor{fig1teal}{period}\textcolor{fig1teal}{.} \textcolor{fig1teal}{The} \textcolor{fig1teal}{non}\textcolor{fig1teal}{-}\textcolor{fig1ochre}{target} \textcolor{fig1teal}{shows} \textcolor{fig1ochre}{aired} \textcolor{fig1teal}{on} \textcolor{fig1ochre}{Sept}\textcolor{fig1teal}{.} \textcolor{fig1teal}{1}\textcolor{fig1teal}{,} \textcolor{fig1teal}{11}\textcolor{fig1teal}{,} \textcolor{fig1teal}{and} \textcolor{fig1teal}{14}\textcolor{fig1teal}{,} \textcolor{fig1teal}{while} \textcolor{fig1teal}{its} \textcolor{fig1teal}{11}\textcolor{fig1teal}{.}\textcolor{fig1teal}{6} \textcolor{fig1ochre}{aired}\textcolor{fig1teal}{.}  \textcolor{fig1ochre}{They}\textcolor{fig1ochre}{'re} \textcolor{fig1teal}{"}\textcolor{fig1teal}{17}\textcolor{fig1teal}{.} \textcolor{fig1teal}{26}\textcolor{fig1teal}{"} \textcolor{fig1teal}{that} \textcolor{fig1teal}{it} \textcolor{fig1teal}{were} \textcolor{fig1teal}{the} \textcolor{fig1teal}{last} \textcolor{fig1teal}{air} \textcolor{fig1teal}{shows} \textcolor{fig1teal}{in} \textcolor{fig1teal}{that} \textcolor{fig1teal}{period}\textcolor{fig1teal}{,} \textcolor{fig1teal}{not} \textcolor{fig1teal}{the} \textcolor{fig1ochre}{Grim}\textcolor{fig1ochre}{IT} \textcolor{fig1teal}{report}\textcolor{fig1teal}{.} \textcolor{fig1teal}{In} \textcolor{fig1teal}{that} \textcolor{fig1teal}{period} \textcolor{fig1teal}{officials} \textcolor{fig1ochre}{inform}\textcolor{fig1teal}{ally} \textcolor{fig1teal}{had} \textcolor{fig1teal}{2}\textcolor{fig1teal}{,}\textcolor{fig1teal}{000} \textcolor{fig1teal}{online} \textcolor{fig1teal}{TV} \textcolor{fig1ochre}{stations}\textcolor{fig1teal}{,} \textcolor{fig1teal}{and} \textcolor{fig1teal}{they} \textcolor{fig1teal}{made} \textcolor{fig1teal}{up} \textcolor{fig1teal}{12}\textcolor{fig1teal}{.}\textcolor{fig1teal}{2} \textcolor{fig1teal}{shows} \textcolor{fig1teal}{per} \textcolor{fig1teal}{100}\textcolor{fig1teal}{,}\textcolor{fig1teal}{000} \textcolor{fig1teal}{viewers}\textcolor{fig1teal}{.} \textcolor{fig1teal}{However}\textcolor{fig1teal}{,} \textcolor{fig1teal}{the} \textcolor{fig1teal}{19}\textcolor{fig1teal}{.}\textcolor{fig1teal}{25} \textcolor{fig1ochre}{rating} \textcolor{fig1teal}{didn}\textcolor{fig1teal}{'t} \textcolor{fig1teal}{change} \textcolor{fig1teal}{further}\textcolor{fig1teal}{,} \textcolor{fig1teal}{the} \textcolor{fig1teal}{same} \textcolor{fig1teal}{year} \textcolor{fig1teal}{in} \textcolor{fig1ochre}{ON}\textcolor{fig1teal}{,} \textcolor{fig1teal}{and} \textcolor{fig1teal}{a} \textcolor{fig1teal}{few} \textcolor{fig1teal}{other} \textcolor{fig1teal}{shows} \textcolor{fig1teal}{were} \textcolor{fig1teal}{added} \textcolor{fig1teal}{in} \textcolor{fig1teal}{that} \textcolor{fig1teal}{period}\textcolor{fig1teal}{.}  \textcolor{fig1teal}{But} \textcolor{fig1teal}{this} \textcolor{fig1teal}{season} \textcolor{fig1teal}{back}\textcolor{fig1ochre}{ends} \textcolor{fig1teal}{up} \textcolor{fig1teal}{more} \textcolor{fig1teal}{like} \textcolor{fig1teal}{America} \textcolor{fig1teal}{and} \textcolor{fig1teal}{Italy}\textcolor{fig1teal}{:}  \textcolor{fig1ochre}{rot}\textcolor{fig1teal}{:} \textcolor{fig1teal}{In} \textcolor{fig1teal}{today}\textcolor{fig1teal}{'s} \textcolor{fig1teal}{interview} \textcolor{fig1teal}{live} \textcolor{fig1teal}{on} \textcolor{fig1teal}{He}
\end{samplebox}
\begin{samplebox}{\normalfont\textbf{$\ours$ (PR), Round: 4} \hfill \normalfont\scriptsize \textcolor{darkgray}{Committed: \textbf{1024/1024}}}
\scriptsize\linespread{0.9}\selectfont
 \textcolor{fig1teal}{say} \textcolor{fig1teal}{a} \textcolor{fig1teal}{year} \textcolor{fig1teal}{of} \textcolor{fig1teal}{a} \textcolor{fig1teal}{contract}\textcolor{fig1teal}{,} \textcolor{fig1teal}{\$}\textcolor{fig1teal}{5}\textcolor{fig1teal}{,}\textcolor{fig1teal}{000} \textcolor{fig1teal}{a} \textcolor{fig1teal}{year} \textcolor{fig1teal}{on} \textcolor{fig1teal}{the} \textcolor{fig1teal}{network} \textcolor{fig1teal}{by} \textcolor{fig1teal}{three} \textcolor{fig1teal}{women}\textcolor{fig1teal}{—}\textcolor{fig1teal}{even} \textcolor{fig1teal}{though} \textcolor{fig1teal}{it}\textcolor{fig1teal}{�}\textcolor{fig1teal}{�}\textcolor{fig1teal}{s} \textcolor{fig1teal}{the} \textcolor{fig1teal}{only} \textcolor{fig1teal}{thing} \textcolor{fig1teal}{in} \textcolor{fig1teal}{the} \textcolor{fig1teal}{U}\textcolor{fig1teal}{.}\textcolor{fig1teal}{S}\textcolor{fig1teal}{.} \textcolor{fig1teal}{network} \textcolor{fig1teal}{today}\textcolor{fig1teal}{,} \textcolor{fig1teal}{that} \textcolor{fig1teal}{is} \textcolor{fig1teal}{a} \textcolor{fig1teal}{national} \textcolor{fig1teal}{good}\textcolor{fig1teal}{.}\textcolor{fig1teal}{�}\textcolor{fig1teal}{�}  \textcolor{fig1teal}{But} \textcolor{fig1teal}{saying} \textcolor{fig1teal}{that} \textcolor{fig1teal}{deduct}\textcolor{fig1teal}{ing} \textcolor{fig1teal}{money} \textcolor{fig1teal}{from} \textcolor{fig1teal}{�}\textcolor{fig1teal}{�}\textcolor{fig1teal}{m}\textcolor{fig1teal}{iles}\textcolor{fig1teal}{�}\textcolor{fig1teal}{�} \textcolor{fig1teal}{a} \textcolor{fig1teal}{�}\textcolor{fig1teal}{�}\textcolor{fig1teal}{no}\textcolor{fig1teal}{,}\textcolor{fig1teal}{�}\textcolor{fig1teal}{�} \textcolor{fig1teal}{B}\textcolor{fig1teal}{aca} \textcolor{fig1teal}{said} \textcolor{fig1teal}{it} \textcolor{fig1teal}{would} \textcolor{fig1teal}{seem} \textcolor{fig1teal}{that} \textcolor{fig1teal}{that} \textcolor{fig1teal}{is} \textcolor{fig1teal}{a} \textcolor{fig1teal}{step} \textcolor{fig1teal}{later}\textcolor{fig1teal}{,} \textcolor{fig1teal}{but} \textcolor{fig1teal}{the} \textcolor{fig1teal}{issue} \textcolor{fig1teal}{of} \textcolor{fig1teal}{race} \textcolor{fig1teal}{in} \textcolor{fig1teal}{the} \textcolor{fig1teal}{F} \textcolor{fig1teal}{few} \textcolor{fig1teal}{is} \textcolor{fig1teal}{a} \textcolor{fig1teal}{good} \textcolor{fig1teal}{thing} \textcolor{fig1teal}{and} \textcolor{fig1teal}{could} \textcolor{fig1teal}{make} \textcolor{fig1teal}{Cour}\textcolor{fig1teal}{n}\textcolor{fig1teal}{�}\textcolor{fig1teal}{�}\textcolor{fig1teal}{s} \textcolor{fig1teal}{Killing}\textcolor{fig1teal}{aj}\textcolor{fig1teal}{l} \textcolor{fig1teal}{The} \textcolor{fig1teal}{World} \textcolor{fig1teal}{expert}\textcolor{fig1teal}{—}\textcolor{fig1teal}{and} \textcolor{fig1teal}{he} \textcolor{fig1teal}{said} \textcolor{fig1teal}{he}\textcolor{fig1teal}{�}\textcolor{fig1teal}{�}\textcolor{fig1teal}{s} \textcolor{fig1teal}{happy} \textcolor{fig1teal}{to} \textcolor{fig1teal}{be} \textcolor{fig1teal}{concerned} \textcolor{fig1teal}{that} \textcolor{fig1teal}{the} \textcolor{fig1teal}{issues} \textcolor{fig1teal}{around} \textcolor{fig1teal}{that} \textcolor{fig1teal}{of} \textcolor{fig1teal}{a} \textcolor{fig1teal}{racial} \textcolor{fig1teal}{character} \textcolor{fig1teal}{together} \textcolor{fig1teal}{under} \textcolor{fig1teal}{an} \textcolor{fig1teal}{array} \textcolor{fig1teal}{of} \textcolor{fig1teal}{men}\textcolor{fig1teal}{,} \textcolor{fig1teal}{not} \textcolor{fig1teal}{at} \textcolor{fig1teal}{least} \textcolor{fig1teal}{one} \textcolor{fig1teal}{woman}\textcolor{fig1teal}{,} \textcolor{fig1teal}{could} \textcolor{fig1teal}{become} \textcolor{fig1teal}{the} \textcolor{fig1teal}{idea} \textcolor{fig1teal}{more} \textcolor{fig1teal}{popular}\textcolor{fig1teal}{.} \textcolor{fig1teal}{�}\textcolor{fig1teal}{�}\textcolor{fig1teal}{If} \textcolor{fig1teal}{it} \textcolor{fig1teal}{gets} \textcolor{fig1teal}{me} \textcolor{fig1teal}{wrong}\textcolor{fig1teal}{,} \textcolor{fig1teal}{it}\textcolor{fig1teal}{�}\textcolor{fig1teal}{�}\textcolor{fig1teal}{s} \textcolor{fig1teal}{going} \textcolor{fig1teal}{to} \textcolor{fig1teal}{stop} \textcolor{fig1teal}{working} \textcolor{fig1teal}{for} \textcolor{fig1teal}{two} \textcolor{fig1teal}{years}\textcolor{fig1teal}{,}\textcolor{fig1teal}{�}\textcolor{fig1teal}{�} \textcolor{fig1teal}{he} \textcolor{fig1teal}{said}\textcolor{fig1teal}{.} \textcolor{fig1teal}{�}\textcolor{fig1teal}{�}\textcolor{fig1teal}{You} \textcolor{fig1teal}{think}\textcolor{fig1teal}{,} \textcolor{fig1teal}{the} \textcolor{fig1teal}{amount} \textcolor{fig1teal}{of} \textcolor{fig1teal}{media} \textcolor{fig1teal}{company} \textcolor{fig1teal}{and} \textcolor{fig1teal}{me}\textcolor{fig1teal}{of}\textcolor{fig1teal}{us} \textcolor{fig1teal}{men} \textcolor{fig1teal}{at} \textcolor{fig1teal}{your} \textcolor{fig1teal}{side}\textcolor{fig1teal}{,} \textcolor{fig1teal}{all}\textcolor{fig1teal}{,} \textcolor{fig1teal}{and} \textcolor{fig1teal}{y}\textcolor{fig1teal}{anking} \textcolor{fig1teal}{you} \textcolor{fig1teal}{with} \textcolor{fig1teal}{a} \textcolor{fig1teal}{sub} \textcolor{fig1teal}{while} \textcolor{fig1teal}{halfway} \textcolor{fig1teal}{it} \textcolor{fig1teal}{through} \textcolor{fig1teal}{along}\textcolor{fig1teal}{?}\textcolor{fig1teal}{�}\textcolor{fig1teal}{�}\textcolor{fig1teal}{�}\textcolor{fig1teal}{�}  \textcolor{fig1teal}{But} \textcolor{fig1teal}{B}\textcolor{fig1teal}{ale} \textcolor{fig1teal}{is} \textcolor{fig1teal}{also} \textcolor{fig1teal}{convinced} \textcolor{fig1teal}{that} \textcolor{fig1teal}{her} \textcolor{fig1teal}{own} \textcolor{fig1teal}{show}\textcolor{fig1teal}{�}\textcolor{fig1teal}{�}\textcolor{fig1teal}{s} \textcolor{fig1teal}{lack} \textcolor{fig1teal}{of} \textcolor{fig1teal}{control} \textcolor{fig1teal}{around} \textcolor{fig1teal}{more} \textcolor{fig1teal}{is} \textcolor{fig1teal}{more} \textcolor{fig1teal}{directed} \textcolor{fig1teal}{at} \textcolor{fig1teal}{children}\textcolor{fig1teal}{.} \textcolor{fig1teal}{�}\textcolor{fig1teal}{�}\textcolor{fig1teal}{And} \textcolor{fig1teal}{yes}\textcolor{fig1teal}{,} \textcolor{fig1teal}{they}\textcolor{fig1teal}{�}\textcolor{fig1teal}{�}\textcolor{fig1teal}{re} \textcolor{fig1teal}{quite} \textcolor{fig1teal}{certainly}\textcolor{fig1teal}{,} \textcolor{fig1teal}{why} \textcolor{fig1teal}{not} \textcolor{fig1teal}{the} \textcolor{fig1teal}{most} \textcolor{fig1teal}{part}\textcolor{fig1teal}{,} \textcolor{fig1teal}{from} \textcolor{fig1teal}{the} \textcolor{fig1teal}{discussion}\textcolor{fig1teal}{,} \textcolor{fig1teal}{most} \textcolor{fig1teal}{of} \textcolor{fig1teal}{the} \textcolor{fig1teal}{world}\textcolor{fig1teal}{,} \textcolor{fig1teal}{I} \textcolor{fig1teal}{have} \textcolor{fig1teal}{seen} \textcolor{fig1teal}{with}\textcolor{fig1teal}{,} \textcolor{fig1teal}{can} \textcolor{fig1teal}{one} \textcolor{fig1teal}{create} \textcolor{fig1teal}{that}\textcolor{fig1teal}{.}\textcolor{fig1teal}{�}\textcolor{fig1teal}{�} \textcolor{fig1teal}{He} \textcolor{fig1teal}{was} \textcolor{fig1teal}{asked} \textcolor{fig1teal}{about} \textcolor{fig1teal}{that} \textcolor{fig1teal}{when} \textcolor{fig1teal}{she} \textcolor{fig1teal}{said} \textcolor{fig1teal}{he} \textcolor{fig1teal}{asked} \textcolor{fig1teal}{Larry} \textcolor{fig1teal}{K}\textcolor{fig1teal}{.} \textcolor{fig1teal}{to} \textcolor{fig1teal}{star} \textcolor{fig1teal}{on} \textcolor{fig1teal}{the} \textcolor{fig1teal}{show}\textcolor{fig1teal}{.} \textcolor{fig1teal}{In} \textcolor{fig1teal}{reality}\textcolor{fig1teal}{,} \textcolor{fig1teal}{these} \textcolor{fig1teal}{threats} \textcolor{fig1teal}{get} \textcolor{fig1teal}{into} \textcolor{fig1teal}{the} \textcolor{fig1teal}{line} \textcolor{fig1teal}{of} \textcolor{fig1teal}{questioning} \textcolor{fig1teal}{between} \textcolor{fig1teal}{a} \textcolor{fig1teal}{number} \textcolor{fig1teal}{of} \textcolor{fig1teal}{circuit}\textcolor{fig1teal}{-}\textcolor{fig1teal}{attack}\textcolor{fig1teal}{ers} \textcolor{fig1teal}{and} \textcolor{fig1teal}{the} \textcolor{fig1teal}{people} \textcolor{fig1teal}{she}\textcolor{fig1teal}{�}\textcolor{fig1teal}{�}\textcolor{fig1teal}{s} \textcolor{fig1teal}{trying} \textcolor{fig1teal}{to} \textcolor{fig1teal}{keep} \textcolor{fig1teal}{on}\textcolor{fig1teal}{.}\textcolor{fig1teal}{�}\textcolor{fig1teal}{�}\textcolor{fig1teal}{I} \textcolor{fig1teal}{have} \textcolor{fig1teal}{just} \textcolor{fig1teal}{been} \textcolor{fig1teal}{surprised} \textcolor{fig1teal}{by} \textcolor{fig1teal}{your} \textcolor{fig1teal}{response} \textcolor{fig1teal}{to} \textcolor{fig1teal}{had} \textcolor{fig1teal}{so} \textcolor{fig1teal}{much} \textcolor{fig1teal}{needs} \textcolor{fig1teal}{to} \textcolor{fig1teal}{clear}\textcolor{fig1teal}{.} \textcolor{fig1teal}{I} \textcolor{fig1teal}{say}\textcolor{fig1teal}{,} \textcolor{fig1teal}{well}\textcolor{fig1teal}{,} \textcolor{fig1teal}{that}\textcolor{fig1teal}{�}\textcolor{fig1teal}{�}\textcolor{fig1teal}{s} \textcolor{fig1teal}{me} \textcolor{fig1teal}{why} \textcolor{fig1teal}{we}\textcolor{fig1teal}{�}\textcolor{fig1teal}{�}\textcolor{fig1teal}{re} \textcolor{fig1teal}{just} \textcolor{fig1teal}{as} \textcolor{fig1teal}{ready} \textcolor{fig1teal}{for} \textcolor{fig1teal}{our} \textcolor{fig1teal}{work}\textcolor{fig1teal}{.} \textcolor{fig1teal}{In} \textcolor{fig1teal}{my} \textcolor{fig1teal}{campaign}\textcolor{fig1teal}{,} \textcolor{fig1teal}{we} \textcolor{fig1teal}{have} \textcolor{fig1teal}{done} \textcolor{fig1teal}{it} \textcolor{fig1teal}{on} \textcolor{fig1teal}{their} \textcolor{fig1teal}{own}\textcolor{fig1teal}{,} \textcolor{fig1teal}{and} \textcolor{fig1teal}{the} \textcolor{fig1teal}{results} \textcolor{fig1teal}{have} \textcolor{fig1teal}{been} \textcolor{fig1teal}{positive}\textcolor{fig1teal}{.} \textcolor{fig1teal}{This} \textcolor{fig1teal}{post} \textcolor{fig1teal}{was} \textcolor{fig1teal}{a} \textcolor{fig1teal}{last} \textcolor{fig1teal}{thing}\textcolor{fig1teal}{.} \textcolor{fig1teal}{I} \textcolor{fig1teal}{want} \textcolor{fig1teal}{to} \textcolor{fig1teal}{say} \textcolor{fig1teal}{[}\textcolor{fig1teal}{1}\textcolor{fig1teal}{]} \textcolor{fig1teal}{�}\textcolor{fig1teal}{�}\textcolor{fig1teal}{30} \textcolor{fig1teal}{years} \textcolor{fig1teal}{later}\textcolor{fig1teal}{.}\textcolor{fig1teal}{�}\textcolor{fig1teal}{�} \textcolor{fig1teal}{When} \textcolor{fig1teal}{she} \textcolor{fig1teal}{needed} \textcolor{fig1teal}{�}\textcolor{fig1teal}{�}\textcolor{fig1teal}{30} \textcolor{fig1teal}{years} \textcolor{fig1teal}{to} \textcolor{fig1teal}{make} \textcolor{fig1teal}{the} \textcolor{fig1teal}{news}\textcolor{fig1teal}{,}\textcolor{fig1teal}{�}\textcolor{fig1teal}{�} \textcolor{fig1teal}{we} \textcolor{fig1teal}{gave} \textcolor{fig1teal}{the} \textcolor{fig1teal}{U}\textcolor{fig1teal}{.}\textcolor{fig1teal}{S}\textcolor{fig1teal}{.} \textcolor{fig1teal}{news} \textcolor{fig1teal}{and} \textcolor{fig1teal}{done} \textcolor{fig1teal}{this} \textcolor{fig1teal}{year} \textcolor{fig1teal}{when} \textcolor{fig1teal}{only} \textcolor{fig1teal}{10}\textcolor{fig1teal}{-}\textcolor{fig1teal}{20} \textcolor{fig1teal}{percent} \textcolor{fig1teal}{of} \textcolor{fig1teal}{the} \textcolor{fig1teal}{gas}\textcolor{fig1teal}{-}\textcolor{fig1teal}{use} \textcolor{fig1teal}{industry} \textcolor{fig1teal}{is}\textcolor{fig1teal}{.} \textcolor{fig1teal}{The} \textcolor{fig1teal}{news} \textcolor{fig1teal}{to} \textcolor{fig1teal}{the} \textcolor{fig1teal}{people} \textcolor{fig1teal}{of} \textcolor{fig1teal}{this} \textcolor{fig1teal}{region} \textcolor{fig1teal}{can} \textcolor{fig1teal}{come} \textcolor{fig1teal}{any} \textcolor{fig1teal}{time}\textcolor{fig1teal}{.} \textcolor{fig1teal}{My} \textcolor{fig1teal}{organization} \textcolor{fig1teal}{used} \textcolor{fig1teal}{these} \textcolor{fig1teal}{weapons} \textcolor{fig1teal}{and} \textcolor{fig1teal}{lied} \textcolor{fig1teal}{to} \textcolor{fig1teal}{the} \textcolor{fig1teal}{problem} \textcolor{fig1teal}{during} \textcolor{fig1teal}{the} \textcolor{fig1teal}{days} \textcolor{fig1teal}{of} \textcolor{fig1teal}{Vietnam}\textcolor{fig1teal}{,} \textcolor{fig1teal}{which} \textcolor{fig1teal}{has} \textcolor{fig1teal}{a} \textcolor{fig1teal}{huge} \textcolor{fig1teal}{impact} \textcolor{fig1teal}{on} \textcolor{fig1teal}{the} \textcolor{fig1teal}{50}\textcolor{fig1teal}{/}\textcolor{fig1teal}{Co}\textcolor{fig1teal}{al} \textcolor{fig1teal}{industry}\textcolor{fig1teal}{.} \textcolor{fig1teal}{I} \textcolor{fig1teal}{am} \textcolor{fig1teal}{glad} \textcolor{fig1teal}{for} \textcolor{fig1teal}{the} \textcolor{fig1teal}{process} \textcolor{fig1teal}{of} \textcolor{fig1teal}{taking} \textcolor{fig1teal}{this} \textcolor{fig1teal}{responsibility} \textcolor{fig1teal}{on}\textcolor{fig1teal}{.} \textcolor{fig1teal}{Americans} \textcolor{fig1teal}{have} \textcolor{fig1teal}{the} \textcolor{fig1teal}{right} \textcolor{fig1teal}{to} \textcolor{fig1teal}{know} \textcolor{fig1teal}{with} \textcolor{fig1teal}{you} \textcolor{fig1teal}{top} \textcolor{fig1teal}{of} \textcolor{fig1teal}{where} \textcolor{fig1teal}{you} \textcolor{fig1teal}{stand} \textcolor{fig1teal}{or} \textcolor{fig1teal}{we} \textcolor{fig1teal}{per} \textcolor{fig1teal}{side} \textcolor{fig1teal}{with} \textcolor{fig1teal}{that} \textcolor{fig1teal}{first} \textcolor{fig1teal}{of}\textcolor{fig1teal}{.} \textcolor{fig1teal}{You}\textcolor{fig1teal}{,} \textcolor{fig1teal}{the} \textcolor{fig1teal}{people} \textcolor{fig1teal}{and} \textcolor{fig1teal}{other} \textcolor{fig1teal}{people} \textcolor{fig1teal}{do}\textcolor{fig1teal}{,} \textcolor{fig1teal}{and} \textcolor{fig1teal}{get} \textcolor{fig1teal}{ahead} \textcolor{fig1teal}{of} \textcolor{fig1teal}{what} \textcolor{fig1teal}{you} \textcolor{fig1teal}{are} \textcolor{fig1teal}{asked} \textcolor{fig1teal}{for}\textcolor{fig1teal}{.} \textcolor{fig1teal}{We} \textcolor{fig1teal}{are} \textcolor{fig1teal}{not} \textcolor{fig1teal}{out} \textcolor{fig1teal}{for} \textcolor{fig1teal}{modern}\textcolor{fig1teal}{ization} \textcolor{fig1teal}{of} \textcolor{fig1teal}{our} \textcolor{fig1teal}{industry} \textcolor{fig1teal}{now}\textcolor{fig1teal}{,} \textcolor{fig1teal}{and} \textcolor{fig1teal}{the} \textcolor{fig1teal}{strength} \textcolor{fig1teal}{of} \textcolor{fig1teal}{our} \textcolor{fig1teal}{industry} \textcolor{fig1teal}{is} \textcolor{fig1teal}{very} \textcolor{fig1teal}{great}\textcolor{fig1teal}{.}  \textcolor{fig1teal}{I} \textcolor{fig1teal}{told} \textcolor{fig1teal}{anyone} \textcolor{fig1teal}{listening} \textcolor{fig1teal}{to} \textcolor{fig1teal}{me} \textcolor{fig1teal}{about} \textcolor{fig1teal}{that} \textcolor{fig1teal}{as} \textcolor{fig1teal}{well}\textcolor{fig1teal}{.}  \textcolor{fig1teal}{F}\textcolor{fig1teal}{n}\textcolor{fig1teal}{'t} \textcolor{fig1teal}{it} \textcolor{fig1teal}{be} \textcolor{fig1teal}{fine}\textcolor{fig1teal}{?}  \textcolor{fig1teal}{(} \textcolor{fig1teal}{here}\textcolor{fig1teal}{'s} \textcolor{fig1teal}{my} \textcolor{fig1teal}{response}\textcolor{fig1teal}{,} \textcolor{fig1teal}{I} \textcolor{fig1teal}{said} \textcolor{fig1teal}{yes} \textcolor{fig1teal}{that}\textcolor{fig1teal}{'s} \textcolor{fig1teal}{some} \textcolor{fig1teal}{policy} \textcolor{fig1teal}{in} \textcolor{fig1teal}{action} \textcolor{fig1teal}{from} \textcolor{fig1teal}{the} \textcolor{fig1teal}{person} \textcolor{fig1teal}{I}\textcolor{fig1teal}{.} \textcolor{fig1teal}{However}\textcolor{fig1teal}{,} \textcolor{fig1teal}{I} \textcolor{fig1teal}{think} \textcolor{fig1teal}{our} \textcolor{fig1teal}{national} \textcolor{fig1teal}{leaders} \textcolor{fig1teal}{and} \textcolor{fig1teal}{companies} \textcolor{fig1teal}{called} \textcolor{fig1teal}{for} \textcolor{fig1teal}{a} \textcolor{fig1teal}{up} \textcolor{fig1teal}{or} \textcolor{fig1teal}{down} \textcolor{fig1teal}{when} \textcolor{fig1teal}{said} \textcolor{fig1teal}{an} \textcolor{fig1teal}{issue}\textcolor{fig1teal}{.} \textcolor{fig1teal}{We} \textcolor{fig1teal}{and} \textcolor{fig1teal}{out} \textcolor{fig1teal}{or} \textcolor{fig1teal}{the} \textcolor{fig1teal}{"}\textcolor{fig1teal}{the} \textcolor{fig1teal}{one} \textcolor{fig1teal}{of} \textcolor{fig1teal}{which} \textcolor{fig1teal}{was} \textcolor{fig1teal}{no} \textcolor{fig1teal}{one} \textcolor{fig1teal}{I} \textcolor{fig1teal}{was} \textcolor{fig1teal}{hearing} \textcolor{fig1teal}{from}\textcolor{fig1teal}{.}  \textcolor{fig1teal}{I} \textcolor{fig1teal}{think} \textcolor{fig1teal}{it} \textcolor{fig1teal}{is} \textcolor{fig1teal}{just} \textcolor{fig1teal}{fine} \textcolor{fig1teal}{in} \textcolor{fig1teal}{every} \textcolor{fig1teal}{way}\textcolor{fig1teal}{.} \textcolor{fig1teal}{I} \textcolor{fig1teal}{will} \textcolor{fig1teal}{continue} \textcolor{fig1teal}{if} \textcolor{fig1teal}{you} \textcolor{fig1teal}{at} \textcolor{fig1teal}{I} \textcolor{fig1teal}{feel} \textcolor{fig1teal}{like}\textcolor{fig1teal}{.}  \textcolor{fig1teal}{I} \textcolor{fig1teal}{know} \textcolor{fig1teal}{you} \textcolor{fig1teal}{would} \textcolor{fig1teal}{like} \textcolor{fig1teal}{us} \textcolor{fig1teal}{to}\textcolor{fig1teal}{,} \textcolor{fig1teal}{as} \textcolor{fig1teal}{a} \textcolor{fig1teal}{part} \textcolor{fig1teal}{of} \textcolor{fig1teal}{this} \textcolor{fig1teal}{business}\textcolor{fig1teal}{,} \textcolor{fig1teal}{listen} \textcolor{fig1teal}{to} \textcolor{fig1teal}{us} \textcolor{fig1teal}{without} \textcolor{fig1teal}{the} \textcolor{fig1teal}{news}\textcolor{fig1teal}{.}\textcolor{fig1teal}{This} \textcolor{fig1teal}{is} \textcolor{fig1teal}{not} \textcolor{fig1teal}{�}\textcolor{fig1teal}{�}\textcolor{fig1teal}{The} \textcolor{fig1teal}{Coup}\textcolor{fig1teal}{,}\textcolor{fig1teal}{�}\textcolor{fig1teal}{�} \textcolor{fig1teal}{by} \textcolor{fig1teal}{Robert} \textcolor{fig1teal}{Ar}\textcolor{fig1teal}{is}\textcolor{fig1teal}{,} \textcolor{fig1teal}{put} \textcolor{fig1teal}{off} \textcolor{fig1teal}{for} \textcolor{fig1teal}{office} \textcolor{fig1teal}{at} \textcolor{fig1teal}{the} \textcolor{fig1teal}{top} \textcolor{fig1teal}{of} \textcolor{fig1teal}{the} \textcolor{fig1teal}{year}\textcolor{fig1teal}{-}\textcolor{fig1teal}{the}\textcolor{fig1teal}{-} \textcolor{fig1teal}{news}\textcolor{fig1teal}{.} \textcolor{fig1teal}{In} \textcolor{fig1teal}{Russia}\textcolor{fig1teal}{,} \textcolor{fig1teal}{it} \textcolor{fig1teal}{a} \textcolor{fig1teal}{lot} \textcolor{fig1teal}{like}\textcolor{fig1teal}{,} \textcolor{fig1teal}{his} \textcolor{fig1teal}{five} \textcolor{fig1teal}{state} \textcolor{fig1teal}{secretaries}\textcolor{fig1teal}{,} \textcolor{fig1teal}{officials} \textcolor{fig1teal}{and} \textcolor{fig1teal}{even} \textcolor{fig1teal}{said} \textcolor{fig1teal}{they} \textcolor{fig1teal}{did} \textcolor{fig1teal}{and} \textcolor{fig1teal}{wrote} \textcolor{fig1teal}{about} \textcolor{fig1teal}{through} \textcolor{fig1teal}{the} \textcolor{fig1teal}{daily} \textcolor{fig1teal}{R}\textcolor{fig1teal}{.}\textcolor{fig1teal}{T}\textcolor{fig1teal}{.} \textcolor{fig1teal}{Bas}\textcolor{fig1teal}{is} \textcolor{fig1teal}{already}\textcolor{fig1teal}{,} \textcolor{fig1teal}{such} \textcolor{fig1teal}{as} \textcolor{fig1teal}{"}\textcolor{fig1teal}{c}\textcolor{fig1teal}{and}\textcolor{fig1teal}{aries}\textcolor{fig1teal}{"} \textcolor{fig1teal}{to} \textcolor{fig1teal}{the} \textcolor{fig1teal}{US}\textcolor{fig1teal}{.}  \textcolor{fig1teal}{"}\textcolor{fig1teal}{They}\textcolor{fig1teal}{,} \textcolor{fig1teal}{the} \textcolor{fig1teal}{government}\textcolor{fig1teal}{,} \textcolor{fig1teal}{were} \textcolor{fig1teal}{dependent}\textcolor{fig1teal}{,} \textcolor{fig1teal}{not} \textcolor{fig1teal}{on} \textcolor{fig1teal}{them}\textcolor{fig1teal}{,} \textcolor{fig1teal}{but} \textcolor{fig1teal}{muff}\textcolor{fig1teal}{led} \textcolor{fig1teal}{and} \textcolor{fig1teal}{disinfect}\textcolor{fig1teal}{ed} \textcolor{fig1teal}{by} \textcolor{fig1teal}{a} \textcolor{fig1teal}{right}\textcolor{fig1teal}{-}\textcolor{fig1teal}{of}\textcolor{fig1teal}{-}\textcolor{fig1teal}{center} \textcolor{fig1teal}{media} \textcolor{fig1teal}{cover} \textcolor{fig1teal}{from} \textcolor{fig1teal}{America}\textcolor{fig1teal}{,"} \textcolor{fig1teal}{wrote} \textcolor{fig1teal}{his} \textcolor{fig1teal}{friend}\textcolor{fig1teal}{.} \textcolor{fig1teal}{"}\textcolor{fig1teal}{They} \textcolor{fig1teal}{did} \textcolor{fig1teal}{then} \textcolor{fig1teal}{much} \textcolor{fig1teal}{less} \textcolor{fig1teal}{from} \textcolor{fig1teal}{I}\textcolor{fig1teal}{,} \textcolor{fig1teal}{who} \textcolor{fig1teal}{wanted} \textcolor{fig1teal}{to} \textcolor{fig1teal}{understand} \textcolor{fig1teal}{the} \textcolor{fig1teal}{country}\textcolor{fig1teal}{,} \textcolor{fig1teal}{and} \textcolor{fig1teal}{made} \textcolor{fig1teal}{up} \textcolor{fig1teal}{its} \textcolor{fig1teal}{about} \textcolor{fig1teal}{both} \textcolor{fig1teal}{past} \textcolor{fig1teal}{and} \textcolor{fig1teal}{present}\textcolor{fig1teal}{."} \textcolor{fig1teal}{No}\textcolor{fig1teal}{,} \textcolor{fig1teal}{we}\textcolor{fig1teal}{'re} \textcolor{fig1teal}{upset} \textcolor{fig1teal}{a} \textcolor{fig1teal}{little} \textcolor{fig1teal}{much} \textcolor{fig1teal}{about} \textcolor{fig1teal}{the} \textcolor{fig1teal}{data}\textcolor{fig1teal}{,} \textcolor{fig1teal}{from} \textcolor{fig1teal}{then}\textcolor{fig1teal}{-}\textcolor{fig1teal}{on} \textcolor{fig1teal}{the} \textcolor{fig1teal}{only} \textcolor{fig1teal}{point} \textcolor{fig1teal}{since} \textcolor{fig1teal}{the} \textcolor{fig1teal}{end} \textcolor{fig1teal}{of} \textcolor{fig1teal}{the} \textcolor{fig1teal}{C}\textcolor{fig1teal}{ATE}\textcolor{fig1teal}{,} \textcolor{fig1teal}{we} \textcolor{fig1teal}{may} \textcolor{fig1teal}{have} \textcolor{fig1teal}{lost} \textcolor{fig1teal}{part} \textcolor{fig1teal}{of} \textcolor{fig1teal}{the} \textcolor{fig1teal}{data}\textcolor{fig1teal}{.}  \textcolor{fig1teal}{The} \textcolor{fig1teal}{in} \textcolor{fig1teal}{particular} \textcolor{fig1teal}{range} \textcolor{fig1teal}{of} \textcolor{fig1teal}{the} \textcolor{fig1teal}{shows} \textcolor{fig1teal}{on} \textcolor{fig1teal}{Wednesday} \textcolor{fig1teal}{by} \textcolor{fig1teal}{New}\textcolor{fig1teal}{Times} \textcolor{fig1teal}{were} \textcolor{fig1teal}{the} \textcolor{fig1teal}{first} \textcolor{fig1teal}{of} \textcolor{fig1teal}{the} \textcolor{fig1teal}{Q}\textcolor{fig1teal}{1} \textcolor{fig1teal}{period}\textcolor{fig1teal}{.} \textcolor{fig1teal}{The} \textcolor{fig1teal}{non}\textcolor{fig1teal}{-}\textcolor{fig1teal}{State} \textcolor{fig1teal}{shows} \textcolor{fig1teal}{aired} \textcolor{fig1teal}{on} \textcolor{fig1teal}{pp}\textcolor{fig1teal}{.} \textcolor{fig1teal}{1}\textcolor{fig1teal}{,} \textcolor{fig1teal}{11}\textcolor{fig1teal}{,} \textcolor{fig1teal}{and} \textcolor{fig1teal}{14}\textcolor{fig1teal}{,} \textcolor{fig1teal}{while} \textcolor{fig1teal}{its} \textcolor{fig1teal}{11}\textcolor{fig1teal}{.}\textcolor{fig1teal}{6}\textcolor{fig1teal}{16}\textcolor{fig1teal}{.}  \textcolor{fig1teal}{It}\textcolor{fig1teal}{'s} \textcolor{fig1teal}{"}\textcolor{fig1teal}{17}\textcolor{fig1teal}{.} \textcolor{fig1teal}{26}\textcolor{fig1teal}{"} \textcolor{fig1teal}{that} \textcolor{fig1teal}{it} \textcolor{fig1teal}{were} \textcolor{fig1teal}{the} \textcolor{fig1teal}{last} \textcolor{fig1teal}{air} \textcolor{fig1teal}{shows} \textcolor{fig1teal}{in} \textcolor{fig1teal}{that} \textcolor{fig1teal}{period}\textcolor{fig1teal}{,} \textcolor{fig1teal}{not} \textcolor{fig1teal}{the} \textcolor{fig1teal}{first} \textcolor{fig1teal}{per} \textcolor{fig1teal}{report}\textcolor{fig1teal}{.} \textcolor{fig1teal}{In} \textcolor{fig1teal}{that} \textcolor{fig1teal}{period} \textcolor{fig1teal}{officials} \textcolor{fig1teal}{provision}\textcolor{fig1teal}{ally} \textcolor{fig1teal}{had} \textcolor{fig1teal}{2}\textcolor{fig1teal}{,}\textcolor{fig1teal}{000} \textcolor{fig1teal}{online} \textcolor{fig1teal}{TV} \textcolor{fig1teal}{viewers}\textcolor{fig1teal}{,} \textcolor{fig1teal}{and} \textcolor{fig1teal}{they} \textcolor{fig1teal}{made} \textcolor{fig1teal}{up} \textcolor{fig1teal}{12}\textcolor{fig1teal}{.}\textcolor{fig1teal}{2} \textcolor{fig1teal}{shows} \textcolor{fig1teal}{per} \textcolor{fig1teal}{100}\textcolor{fig1teal}{,}\textcolor{fig1teal}{000} \textcolor{fig1teal}{viewers}\textcolor{fig1teal}{.} \textcolor{fig1teal}{However}\textcolor{fig1teal}{,} \textcolor{fig1teal}{the} \textcolor{fig1teal}{19}\textcolor{fig1teal}{.}\textcolor{fig1teal}{25} \textcolor{fig1teal}{range} \textcolor{fig1teal}{didn}\textcolor{fig1teal}{'t} \textcolor{fig1teal}{change} \textcolor{fig1teal}{further}\textcolor{fig1teal}{,} \textcolor{fig1teal}{the} \textcolor{fig1teal}{same} \textcolor{fig1teal}{year} \textcolor{fig1teal}{in} \textcolor{fig1teal}{2012}\textcolor{fig1teal}{,} \textcolor{fig1teal}{and} \textcolor{fig1teal}{a} \textcolor{fig1teal}{few} \textcolor{fig1teal}{other} \textcolor{fig1teal}{shows} \textcolor{fig1teal}{were} \textcolor{fig1teal}{added} \textcolor{fig1teal}{in} \textcolor{fig1teal}{that} \textcolor{fig1teal}{period}\textcolor{fig1teal}{.}  \textcolor{fig1teal}{But} \textcolor{fig1teal}{this} \textcolor{fig1teal}{season} \textcolor{fig1teal}{back} \textcolor{fig1teal}{it} \textcolor{fig1teal}{up} \textcolor{fig1teal}{more} \textcolor{fig1teal}{like} \textcolor{fig1teal}{America} \textcolor{fig1teal}{and} \textcolor{fig1teal}{Italy}\textcolor{fig1teal}{:}  \textcolor{fig1teal}{Steve}\textcolor{fig1teal}{:} \textcolor{fig1teal}{In} \textcolor{fig1teal}{today}\textcolor{fig1teal}{'s} \textcolor{fig1teal}{interview} \textcolor{fig1teal}{live} \textcolor{fig1teal}{on} \textcolor{fig1teal}{He}
\end{samplebox}
\caption{Posterior Refinement trajectory on OpenWebText, Sample-1. (cont.)}
\label{fig:refine_sample_owt_1_r3}
\end{figure}
\clearpage

\begin{figure}[H]
\centering
\begin{samplebox}{\normalfont\textbf{$\ours$ (PR), Round: 1} \hfill \normalfont\scriptsize \textcolor{darkgray}{Committed: \textbf{719/1024}}}
\scriptsize\linespread{0.9}\selectfont
\textcolor{fig1teal}{.}  \textcolor{fig1ochre}{Ped}\textcolor{fig1ochre}{than} \textcolor{fig1teal}{is} \textcolor{fig1teal}{another} \textcolor{fig1ochre}{man} \textcolor{fig1ochre}{named} \textcolor{fig1ochre}{Zack}\textcolor{fig1teal}{.} \textcolor{fig1teal}{He} \textcolor{fig1teal}{has} \textcolor{fig1teal}{been} \textcolor{fig1ochre}{practicing} \textcolor{fig1teal}{on} \textcolor{fig1teal}{a} \textcolor{fig1ochre}{presidential} \textcolor{fig1ochre}{trap} \textcolor{fig1teal}{for} \textcolor{fig1teal}{9} \textcolor{fig1teal}{years}\textcolor{fig1teal}{,} \textcolor{fig1teal}{shooting}\textcolor{fig1ochre}{wn} \textcolor{fig1ochre}{nights} \textcolor{fig1teal}{in} \textcolor{fig1ochre}{opposed} \textcolor{fig1teal}{parts} \textcolor{fig1teal}{of} \textcolor{fig1ochre}{Bain}\textcolor{fig1ochre}{bush}\textcolor{fig1teal}{.} \textcolor{fig1teal}{I} \textcolor{fig1ochre}{loved} \textcolor{fig1ochre}{Owen}\textcolor{fig1teal}{'s} \textcolor{fig1ochre}{perspective} \textcolor{fig1teal}{that} \textcolor{fig1teal}{after} \textcolor{fig1teal}{the} \textcolor{fig1teal}{recent} \textcolor{fig1ochre}{robbery} \textcolor{fig1teal}{in} \textcolor{fig1teal}{front} \textcolor{fig1teal}{of} \textcolor{fig1teal}{school}\textcolor{fig1teal}{,} \textcolor{fig1teal}{many} \textcolor{fig1ochre}{reading} \textcolor{fig1teal}{of} \textcolor{fig1teal}{them} \textcolor{fig1ochre}{transformed} \textcolor{fig1teal}{it} \textcolor{fig1teal}{into} \textcolor{fig1teal}{a} \textcolor{fig1ochre}{fountain} \textcolor{fig1teal}{of} \textcolor{fig1ochre}{grass}\textcolor{fig1teal}{.}  \textcolor{fig1ochre}{Merc} \textcolor{fig1ochre}{afterward}\textcolor{fig1teal}{,} \textcolor{fig1ochre}{el}\textcolor{fig1ochre}{MW} \textcolor{fig1ochre}{David} \textcolor{fig1ochre}{Larry} \textcolor{fig1ochre}{Mac}\textcolor{fig1teal}{p}\textcolor{fig1ochre}{onder} \textcolor{fig1ochre}{directed} \textcolor{fig1teal}{no} \textcolor{fig1teal}{return}\textcolor{fig1teal}{.} \textcolor{fig1ochre}{When} \textcolor{fig1ochre}{talking} \textcolor{fig1teal}{to} \textcolor{fig1teal}{the} \textcolor{fig1teal}{press} \textcolor{fig1ochre}{shortly} \textcolor{fig1teal}{after} \textcolor{fig1teal}{the} \textcolor{fig1ochre}{district} \textcolor{fig1ochre}{raid}\textcolor{fig1teal}{,} \textcolor{fig1ochre}{Elk} \textcolor{fig1ochre}{Cross} \textcolor{fig1ochre}{national} \textcolor{fig1ochre}{SL}\textcolor{fig1teal}{R} \textcolor{fig1ochre}{Rep} \textcolor{fig1ochre}{Dave} \textcolor{fig1ochre}{Shore} \textcolor{fig1ochre}{tells} \textcolor{fig1teal}{me} \textcolor{fig1teal}{the} \textcolor{fig1teal}{news}\textcolor{fig1teal}{:}  \textcolor{fig1teal}{�}\textcolor{fig1teal}{�}\textcolor{fig1teal}{To} \textcolor{fig1ochre}{tell} \textcolor{fig1teal}{it} \textcolor{fig1teal}{in} \textcolor{fig1teal}{an} \textcolor{fig1teal}{office} \textcolor{fig1teal}{that} \textcolor{fig1teal}{is} \textcolor{fig1teal}{like} \textcolor{fig1teal}{a} \textcolor{fig1ochre}{lake} \textcolor{fig1teal}{—} \textcolor{fig1teal}{we} \textcolor{fig1ochre}{remind} \textcolor{fig1ochre}{conversations} \textcolor{fig1teal}{of} \textcolor{fig1teal}{our} \textcolor{fig1ochre}{ties} \textcolor{fig1teal}{with} \textcolor{fig1teal}{the} \textcolor{fig1ochre}{establishment}\textcolor{fig1teal}{.}\textcolor{fig1teal}{�}\textcolor{fig1teal}{�}  \textcolor{fig1teal}{It} \textcolor{fig1teal}{was} \textcolor{fig1teal}{the} \textcolor{fig1teal}{first} \textcolor{fig1teal}{police} \textcolor{fig1ochre}{raid} \textcolor{fig1teal}{for} \textcolor{fig1teal}{an} \textcolor{fig1ochre}{owner} \textcolor{fig1teal}{in} \textcolor{fig1ochre}{strange} \textcolor{fig1ochre}{systems} \textcolor{fig1teal}{of} \textcolor{fig1ochre}{Pokemon}\textcolor{fig1teal}{.} \textcolor{fig1teal}{It} \textcolor{fig1teal}{was} \textcolor{fig1teal}{also} \textcolor{fig1teal}{that} \textcolor{fig1ochre}{Arpaio} \textcolor{fig1ochre}{chose} \textcolor{fig1teal}{to} \textcolor{fig1teal}{open} \textcolor{fig1ochre}{dried} \textcolor{fig1ochre}{gas} \textcolor{fig1teal}{and} \textcolor{fig1ochre}{EP} \textcolor{fig1ochre}{Broadway}\textcolor{fig1teal}{,} \textcolor{fig1teal}{and} \textcolor{fig1teal}{then} \textcolor{fig1teal}{was}\textcolor{fig1teal}{,} \textcolor{fig1teal}{to} \textcolor{fig1ochre}{mute} \textcolor{fig1teal}{short} \textcolor{fig1teal}{his} \textcolor{fig1ochre}{successor}\textcolor{fig1teal}{,} \textcolor{fig1ochre}{Waters}\textcolor{fig1teal}{,} \textcolor{fig1teal}{who} \textcolor{fig1teal}{has} \textcolor{fig1teal}{served} \textcolor{fig1teal}{as} \textcolor{fig1teal}{software} \textcolor{fig1ochre}{analyst} \textcolor{fig1teal}{for} \textcolor{fig1ochre}{Wireless}\textcolor{fig1teal}{,} \textcolor{fig1ochre}{Cisco}\textcolor{fig1teal}{,} \textcolor{fig1teal}{and} \textcolor{fig1ochre}{Home}\textcolor{fig1ochre}{ry} \textcolor{fig1ochre}{Brands}\textcolor{fig1teal}{.}  \textcolor{fig1teal}{�}\textcolor{fig1teal}{�}\textcolor{fig1teal}{There}\textcolor{fig1teal}{�}\textcolor{fig1teal}{�}\textcolor{fig1teal}{s} \textcolor{fig1teal}{a} \textcolor{fig1teal}{lot} \textcolor{fig1teal}{for} \textcolor{fig1teal}{me} \textcolor{fig1teal}{to} \textcolor{fig1ochre}{raise} \textcolor{fig1teal}{with} \textcolor{fig1teal}{myself} \textcolor{fig1teal}{that} \textcolor{fig1teal}{would} \textcolor{fig1ochre}{affect} \textcolor{fig1teal}{me}\textcolor{fig1teal}{.} \textcolor{fig1teal}{I} \textcolor{fig1teal}{think} \textcolor{fig1teal}{this} \textcolor{fig1teal}{is} \textcolor{fig1teal}{my} \textcolor{fig1teal}{first} \textcolor{fig1ochre}{things}\textcolor{fig1teal}{,}\textcolor{fig1teal}{�}\textcolor{fig1teal}{�} \textcolor{fig1teal}{he} \textcolor{fig1teal}{says}\textcolor{fig1teal}{.} \textcolor{fig1teal}{�}\textcolor{fig1teal}{�}\textcolor{fig1teal}{These} \textcolor{fig1teal}{were} \textcolor{fig1teal}{the} \textcolor{fig1ochre}{customer} \textcolor{fig1ochre}{addresses}\textcolor{fig1teal}{.} \textcolor{fig1teal}{If} \textcolor{fig1ochre}{nobody} \textcolor{fig1ochre}{refuses} \textcolor{fig1teal}{to} \textcolor{fig1teal}{see} \textcolor{fig1teal}{the} \textcolor{fig1teal}{new} \textcolor{fig1ochre}{steel} \textcolor{fig1ochre}{Trust} \textcolor{fig1teal}{was} \textcolor{fig1teal}{made} \textcolor{fig1teal}{to} \textcolor{fig1ochre}{disable} \textcolor{fig1teal}{all} \textcolor{fig1teal}{of} \textcolor{fig1teal}{him}\textcolor{fig1teal}{,} \textcolor{fig1teal}{ask} \textcolor{fig1teal}{one} \textcolor{fig1teal}{of} \textcolor{fig1teal}{you} \textcolor{fig1teal}{to} \textcolor{fig1teal}{think} \textcolor{fig1teal}{about} \textcolor{fig1teal}{what} \textcolor{fig1teal}{each} \textcolor{fig1teal}{other} \textcolor{fig1ochre}{tomorrow} \textcolor{fig1teal}{could} \textcolor{fig1teal}{show}\textcolor{fig1teal}{.}  \textcolor{fig1ochre}{When} \textcolor{fig1teal}{New} \textcolor{fig1ochre}{England}\textcolor{fig1teal}{�}\textcolor{fig1teal}{�}\textcolor{fig1teal}{s} \textcolor{fig1teal}{State} \textcolor{fig1teal}{House} \textcolor{fig1teal}{and} \textcolor{fig1ochre}{V}\textcolor{fig1teal}{-}\textcolor{fig1ochre}{ins} \textcolor{fig1ochre}{operate}\textcolor{fig1teal}{,} \textcolor{fig1teal}{his} \textcolor{fig1ochre}{specialty} \textcolor{fig1teal}{is} \textcolor{fig1teal}{focused} \textcolor{fig1teal}{on} \textcolor{fig1ochre}{capac}\textcolor{fig1ochre}{ging} \textcolor{fig1teal}{growth}\textcolor{fig1teal}{.} \textcolor{fig1ochre}{Amid} \textcolor{fig1teal}{the} \textcolor{fig1ochre}{neighborhood} \textcolor{fig1teal}{health} \textcolor{fig1teal}{system}\textcolor{fig1teal}{,} \textcolor{fig1ochre}{Tesla} \textcolor{fig1ochre}{refuses} \textcolor{fig1teal}{to} \textcolor{fig1ochre}{pass} \textcolor{fig1ochre}{feed}\textcolor{fig1teal}{-}\textcolor{fig1ochre}{only} \textcolor{fig1teal}{take} \textcolor{fig1teal}{of} \textcolor{fig1ochre}{AT}\textcolor{fig1ochre}{fuel} \textcolor{fig1teal}{and} \textcolor{fig1ochre}{uranium} \textcolor{fig1teal}{from} \textcolor{fig1teal}{a} \textcolor{fig1teal}{friend} \textcolor{fig1ochre}{destructive} \textcolor{fig1ochre}{power}\textcolor{fig1teal}{,} \textcolor{fig1teal}{there} \textcolor{fig1teal}{are} \textcolor{fig1ochre}{checks} \textcolor{fig1teal}{in} \textcolor{fig1teal}{on} \textcolor{fig1teal}{life}\textcolor{fig1teal}{s} \textcolor{fig1teal}{of} \textcolor{fig1ochre}{terrorists}\textcolor{fig1teal}{,} \textcolor{fig1teal}{and} \textcolor{fig1ochre}{activists} \textcolor{fig1ochre}{modifying}\textcolor{fig1teal}{,} \textcolor{fig1teal}{but} \textcolor{fig1teal}{not} \textcolor{fig1ochre}{countering}\textcolor{fig1teal}{,} \textcolor{fig1teal}{these} \textcolor{fig1ochre}{weapons}\textcolor{fig1teal}{.}  \textcolor{fig1ochre}{That} \textcolor{fig1teal}{was} \textcolor{fig1teal}{because} \textcolor{fig1teal}{of} \textcolor{fig1teal}{the} \textcolor{fig1teal}{amount} \textcolor{fig1teal}{of} \textcolor{fig1teal}{real} \textcolor{fig1ochre}{arrests} \textcolor{fig1teal}{and} \textcolor{fig1teal}{the} \textcolor{fig1teal}{day} \textcolor{fig1ochre}{preceding} \textcolor{fig1teal}{the} \textcolor{fig1ochre}{NRA}\textcolor{fig1teal}{�}\textcolor{fig1teal}{�}\textcolor{fig1teal}{s} \textcolor{fig1ochre}{Feb}\textcolor{fig1teal}{.} \textcolor{fig1ochre}{13}\textcolor{fig1teal}{,} \textcolor{fig1ochre}{2016} \textcolor{fig1ochre}{targeting} \textcolor{fig1ochre}{Tam}\textcolor{fig1ochre}{ois}\textcolor{fig1teal}{in} \textcolor{fig1teal}{was} \textcolor{fig1ochre}{convicted}\textcolor{fig1teal}{,} \textcolor{fig1teal}{because} \textcolor{fig1teal}{many} \textcolor{fig1ochre}{Terror}\textcolor{fig1teal}{ist}\textcolor{fig1ochre}{ees} \textcolor{fig1teal}{were} \textcolor{fig1ochre}{excluded} \textcolor{fig1teal}{from} \textcolor{fig1ochre}{collapse} \textcolor{fig1ochre}{victims}\textcolor{fig1teal}{.} \textcolor{fig1teal}{40} \textcolor{fig1teal}{percent} \textcolor{fig1teal}{of} \textcolor{fig1teal}{the} \textcolor{fig1teal}{number} \textcolor{fig1teal}{of} \textcolor{fig1ochre}{dead} \textcolor{fig1teal}{drug} \textcolor{fig1ochre}{circulation}\textcolor{fig1teal}{,} \textcolor{fig1teal}{as} \textcolor{fig1teal}{well} \textcolor{fig1teal}{as} \textcolor{fig1ochre}{se}\textcolor{fig1ochre}{anger}\textcolor{fig1ochre}{ak} \textcolor{fig1teal}{and} \textcolor{fig1ochre}{Floor}\textcolor{fig1teal}{ing} \textcolor{fig1teal}{were} \textcolor{fig1teal}{all} \textcolor{fig1teal}{21}\textcolor{fig1teal}{-}\textcolor{fig1teal}{year}\textcolor{fig1teal}{-}\textcolor{fig1ochre}{old} \textcolor{fig1teal}{students} \textcolor{fig1teal}{on} \textcolor{fig1teal}{their} \textcolor{fig1teal}{own}\textcolor{fig1teal}{.}  \textcolor{fig1teal}{�}\textcolor{fig1teal}{�}\textcolor{fig1teal}{I} \textcolor{fig1teal}{think} \textcolor{fig1teal}{with} \textcolor{fig1teal}{all} \textcolor{fig1teal}{the} \textcolor{fig1ochre}{classified} \textcolor{fig1teal}{data} \textcolor{fig1teal}{since} \textcolor{fig1teal}{I}\textcolor{fig1teal}{�}\textcolor{fig1teal}{�}\textcolor{fig1teal}{ve} \textcolor{fig1teal}{worked} \textcolor{fig1teal}{with} \textcolor{fig1teal}{David} \textcolor{fig1ochre}{Are}\textcolor{fig1ochre}{hoff}\textcolor{fig1teal}{,} \textcolor{fig1teal}{the} \textcolor{fig1teal}{amount} \textcolor{fig1teal}{of} \textcolor{fig1ochre}{false} \textcolor{fig1teal}{sites} \textcolor{fig1teal}{or} \textcolor{fig1teal}{what} \textcolor{fig1teal}{I} \textcolor{fig1ochre}{heard} \textcolor{fig1teal}{there} \textcolor{fig1teal}{was}\textcolor{fig1teal}{;} \textcolor{fig1teal}{a} \textcolor{fig1teal}{lot} \textcolor{fig1teal}{in} \textcolor{fig1teal}{which} \textcolor{fig1teal}{a} \textcolor{fig1teal}{lot} \textcolor{fig1ochre}{aimed} \textcolor{fig1teal}{more} \textcolor{fig1teal}{at} \textcolor{fig1teal}{people} \textcolor{fig1teal}{other} \textcolor{fig1teal}{than} \textcolor{fig1ochre}{Brian}\textcolor{fig1teal}{�}\textcolor{fig1teal}{�}\textcolor{fig1teal}{s}\textcolor{fig1teal}{,} \textcolor{fig1teal}{but} \textcolor{fig1teal}{I} \textcolor{fig1teal}{know}\textcolor{fig1teal}{,} \textcolor{fig1teal}{my} \textcolor{fig1ochre}{cost} \textcolor{fig1ochre}{allows} \textcolor{fig1teal}{us} \textcolor{fig1teal}{to} \textcolor{fig1teal}{say} \textcolor{fig1ochre}{yes}\textcolor{fig1teal}{,}\textcolor{fig1teal}{�}\textcolor{fig1teal}{�} \textcolor{fig1ochre}{Pere}\textcolor{fig1ochre}{anan} \textcolor{fig1teal}{said}\textcolor{fig1teal}{.} \textcolor{fig1teal}{�}\textcolor{fig1teal}{�}\textcolor{fig1teal}{In} \textcolor{fig1ochre}{reality}\textcolor{fig1teal}{,} \textcolor{fig1teal}{I}\textcolor{fig1teal}{�}\textcolor{fig1teal}{�}\textcolor{fig1teal}{m} \textcolor{fig1teal}{not} \textcolor{fig1teal}{sure} \textcolor{fig1teal}{I} \textcolor{fig1teal}{really}\textcolor{fig1teal}{�}\textcolor{fig1teal}{�}\textcolor{fig1teal}{m} \textcolor{fig1ochre}{willing} \textcolor{fig1teal}{to} \textcolor{fig1teal}{do} \textcolor{fig1teal}{that} \textcolor{fig1teal}{money} \textcolor{fig1teal}{in} \textcolor{fig1ochre}{Chicago}\textcolor{fig1teal}{.} \textcolor{fig1teal}{I} \textcolor{fig1teal}{was} \textcolor{fig1ochre}{seen} \textcolor{fig1teal}{to} \textcolor{fig1teal}{do} \textcolor{fig1teal}{this} \textcolor{fig1teal}{to} \textcolor{fig1ochre}{himself} \textcolor{fig1teal}{and} \textcolor{fig1teal}{so} \textcolor{fig1teal}{other} \textcolor{fig1teal}{areas} \textcolor{fig1teal}{don}\textcolor{fig1teal}{�}\textcolor{fig1teal}{�}\textcolor{fig1teal}{t} \textcolor{fig1ochre}{credit} \textcolor{fig1teal}{them} \textcolor{fig1teal}{the} \textcolor{fig1teal}{data}\textcolor{fig1teal}{.}\textcolor{fig1teal}{�}\textcolor{fig1teal}{�}  \textcolor{fig1teal}{The} \textcolor{fig1ochre}{father} \textcolor{fig1teal}{didn}\textcolor{fig1teal}{�}\textcolor{fig1teal}{�}\textcolor{fig1teal}{t} \textcolor{fig1teal}{fall} \textcolor{fig1teal}{to} \textcolor{fig1ochre}{fruition}\textcolor{fig1teal}{,} \textcolor{fig1teal}{and} \textcolor{fig1ochre}{Powers} \textcolor{fig1teal}{used} \textcolor{fig1teal}{to} \textcolor{fig1ochre}{imply} \textcolor{fig1teal}{that} \textcolor{fig1teal}{he} \textcolor{fig1teal}{could} \textcolor{fig1teal}{make} \textcolor{fig1teal}{people} \textcolor{fig1ochre}{free} \textcolor{fig1teal}{to} \textcolor{fig1ochre}{anyone} \textcolor{fig1teal}{with} \textcolor{fig1teal}{the} \textcolor{fig1ochre}{yard} \textcolor{fig1teal}{crisis}\textcolor{fig1teal}{.}  \textcolor{fig1teal}{�}\textcolor{fig1teal}{�}\textcolor{fig1teal}{One} \textcolor{fig1teal}{of} \textcolor{fig1teal}{the} \textcolor{fig1ochre}{guys} \textcolor{fig1teal}{had} \textcolor{fig1ochre}{teeth} \textcolor{fig1teal}{on} \textcolor{fig1teal}{the} \textcolor{fig1ochre}{attack} \textcolor{fig1teal}{and} \textcolor{fig1teal}{I} \textcolor{fig1teal}{was} \textcolor{fig1ochre}{shocked} \textcolor{fig1teal}{at} \textcolor{fig1teal}{what} \textcolor{fig1teal}{was} \textcolor{fig1teal}{no} \textcolor{fig1ochre}{threat}\textcolor{fig1teal}{,}\textcolor{fig1teal}{�}\textcolor{fig1teal}{�} \textcolor{fig1ochre}{Eddie}\textcolor{fig1ochre}{yer} \textcolor{fig1teal}{told} \textcolor{fig1ochre}{Virginia} \textcolor{fig1ochre}{Bhar}\textcolor{fig1teal}{at} \textcolor{fig1teal}{looking} \textcolor{fig1teal}{at} \textcolor{fig1teal}{the} \textcolor{fig1ochre}{virus}\textcolor{fig1teal}{�}\textcolor{fig1teal}{�}\textcolor{fig1teal}{s} \textcolor{fig1ochre}{recent}\textcolor{fig1teal}{-}\textcolor{fig1ochre}{media}\textcolor{fig1teal}{-}\textcolor{fig1ochre}{view}\textcolor{fig1teal}{,} \textcolor{fig1ochre}{spike} \textcolor{fig1teal}{against} \textcolor{fig1ochre}{attention}\textcolor{fig1teal}{.} \textcolor{fig1ochre}{When} \textcolor{fig1teal}{I} \textcolor{fig1teal}{asked} \textcolor{fig1ochre}{Owen} \textcolor{fig1teal}{into} \textcolor{fig1teal}{a} \textcolor{fig1teal}{New} \textcolor{fig1ochre}{York} \textcolor{fig1ochre}{bands} \textcolor{fig1ochre}{member} \textcolor{fig1teal}{at} \textcolor{fig1ochre}{Harvard}\textcolor{fig1teal}{,} \textcolor{fig1teal}{at} \textcolor{fig1ochre}{Election} \textcolor{fig1teal}{Center} \textcolor{fig1teal}{—} \textcolor{fig1teal}{something} \textcolor{fig1teal}{she} \textcolor{fig1teal}{said} \textcolor{fig1teal}{had} \textcolor{fig1teal}{been} \textcolor{fig1teal}{very} \textcolor{fig1teal}{effective} \textcolor{fig1teal}{—} \textcolor{fig1ochre}{dam} \textcolor{fig1ochre}{cited} \textcolor{fig1teal}{was} \textcolor{fig1teal}{some} \textcolor{fig1ochre}{encouragement}\textcolor{fig1teal}{.} \textcolor{fig1teal}{For} \textcolor{fig1teal}{the} \textcolor{fig1teal}{first} \textcolor{fig1teal}{time}\textcolor{fig1teal}{,} \textcolor{fig1teal}{I} \textcolor{fig1teal}{had} \textcolor{fig1teal}{to} \textcolor{fig1teal}{take} \textcolor{fig1teal}{off} \textcolor{fig1teal}{the} \textcolor{fig1teal}{same} \textcolor{fig1ochre}{dyed}\textcolor{fig1ochre}{boy} \textcolor{fig1teal}{ch}\textcolor{fig1ochre}{ava}\textcolor{fig1teal}{.} \textcolor{fig1teal}{This} \textcolor{fig1ochre}{wasn}\textcolor{fig1teal}{�}\textcolor{fig1teal}{�}\textcolor{fig1teal}{t} \textcolor{fig1ochre}{welcome} \textcolor{fig1teal}{as} \textcolor{fig1teal}{we} \textcolor{fig1ochre}{talked} \textcolor{fig1teal}{on} \textcolor{fig1teal}{with} \textcolor{fig1ochre}{Bhar}\textcolor{fig1teal}{at}\textcolor{fig1teal}{,} \textcolor{fig1ochre}{Student} \textcolor{fig1ochre}{invasion}\textcolor{fig1teal}{,} \textcolor{fig1ochre}{everything} \textcolor{fig1ochre}{tick}\textcolor{fig1teal}{er} \textcolor{fig1teal}{and} \textcolor{fig1ochre}{especially} \textcolor{fig1teal}{any} \textcolor{fig1teal}{potential} \textcolor{fig1ochre}{routes}\textcolor{fig1teal}{.}  \textcolor{fig1teal}{�}\textcolor{fig1teal}{�}\textcolor{fig1teal}{That}\textcolor{fig1teal}{�}\textcolor{fig1teal}{�}\textcolor{fig1teal}{s} \textcolor{fig1ochre}{OK}\textcolor{fig1teal}{.} \textcolor{fig1teal}{But} \textcolor{fig1teal}{I} \textcolor{fig1teal}{know} \textcolor{fig1ochre}{Chuck}\textcolor{fig1ochre}{car}\textcolor{fig1ochre}{illa} \textcolor{fig1ochre}{stepped} \textcolor{fig1teal}{in} \textcolor{fig1teal}{and} \textcolor{fig1teal}{got} \textcolor{fig1teal}{a} \textcolor{fig1teal}{control} \textcolor{fig1teal}{and} \textcolor{fig1teal}{the} \textcolor{fig1ochre}{incident} \textcolor{fig1teal}{was} \textcolor{fig1teal}{over} \textcolor{fig1teal}{and} \textcolor{fig1teal}{it}\textcolor{fig1teal}{�}\textcolor{fig1teal}{�}\textcolor{fig1teal}{s} \textcolor{fig1teal}{always} \textcolor{fig1teal}{questions} \textcolor{fig1teal}{about} \textcolor{fig1teal}{the} \textcolor{fig1ochre}{nit}\textcolor{fig1ochre}{hing} \textcolor{fig1teal}{per} \textcolor{fig1ochre}{gallon}\textcolor{fig1teal}{.}\textcolor{fig1teal}{�}\textcolor{fig1teal}{�}  \textcolor{fig1ochre}{Before} \textcolor{fig1teal}{he} \textcolor{fig1ochre}{targeted}  \textcolor{fig1teal}{The} \textcolor{fig1teal}{official} \textcolor{fig1ochre}{suspect} \textcolor{fig1teal}{among} \textcolor{fig1teal}{the} \textcolor{fig1ochre}{constituents} \textcolor{fig1teal}{was} \textcolor{fig1ochre}{filled} \textcolor{fig1teal}{but}\textcolor{fig1ochre}{¯¯¯¯¯¯¯¯} \textcolor{fig1teal}{(}\textcolor{fig1teal}{D}\textcolor{fig1teal}{-}\textcolor{fig1ochre}{Conn}\textcolor{fig1ochre}{.)} \textcolor{fig1ochre}{appeared} \textcolor{fig1teal}{a} \textcolor{fig1ochre}{little} \textcolor{fig1ochre}{exasper}\textcolor{fig1ochre}{atory} \textcolor{fig1teal}{of} \textcolor{fig1teal}{him}\textcolor{fig1teal}{.} \textcolor{fig1ochre}{Since} \textcolor{fig1teal}{such} \textcolor{fig1teal}{never} \textcolor{fig1teal}{the} \textcolor{fig1ochre}{CEO} \textcolor{fig1teal}{he} \textcolor{fig1teal}{is}\textcolor{fig1teal}{,} \textcolor{fig1teal}{Or}\textcolor{fig1ochre}{ning} \textcolor{fig1teal}{said} \textcolor{fig1teal}{it} \textcolor{fig1teal}{was} \textcolor{fig1teal}{not} \textcolor{fig1ochre}{unusual} \textcolor{fig1teal}{to} \textcolor{fig1ochre}{speculate} \textcolor{fig1teal}{that} \textcolor{fig1ochre}{Mass} \textcolor{fig1ochre}{Corp} \textcolor{fig1teal}{could} \textcolor{fig1teal}{get} \textcolor{fig1teal}{home} \textcolor{fig1teal}{to} \textcolor{fig1teal}{his} \textcolor{fig1ochre}{potent} \textcolor{fig1ochre}{ship}\textcolor{fig1ochre}{iles}\textcolor{fig1teal}{.} \textcolor{fig1teal}{The} \textcolor{fig1teal}{National} \textcolor{fig1ochre}{Security} \textcolor{fig1ochre}{Relations} \textcolor{fig1teal}{official}\textcolor{fig1teal}{,} \textcolor{fig1ochre}{Cook}\textcolor{fig1teal}{,} \textcolor{fig1teal}{reported} \textcolor{fig1teal}{after} \textcolor{fig1teal}{a} \textcolor{fig1teal}{running} \textcolor{fig1teal}{for} \textcolor{fig1ochre}{supervisor} \textcolor{fig1teal}{that} \textcolor{fig1ochre}{newly} \textcolor{fig1ochre}{Creek} \textcolor{fig1ochre}{Fif}\textcolor{fig1ochre}{ese} \textcolor{fig1teal}{was} \textcolor{fig1teal}{looking} \textcolor{fig1teal}{to} \textcolor{fig1ochre}{collect} \textcolor{fig1teal}{energy} \textcolor{fig1ochre}{exams} \textcolor{fig1teal}{for} \textcolor{fig1ochre}{residency} \textcolor{fig1teal}{rules} \textcolor{fig1teal}{in} \textcolor{fig1ochre}{Medical} \textcolor{fig1ochre}{Heights} \textcolor{fig1teal}{and} \textcolor{fig1ochre}{Hampton} \textcolor{fig1ochre}{environments}\textcolor{fig1teal}{,} \textcolor{fig1teal}{changing} \textcolor{fig1teal}{them} \textcolor{fig1teal}{to} \textcolor{fig1ochre}{spur} \textcolor{fig1teal}{future} \textcolor{fig1ochre}{elections}\textcolor{fig1teal}{,} \textcolor{fig1teal}{and} \textcolor{fig1teal}{to} \textcolor{fig1ochre}{assess} \textcolor{fig1teal}{the} \textcolor{fig1ochre}{validity} \textcolor{fig1teal}{of} \textcolor{fig1teal}{the} \textcolor{fig1ochre}{savings}\textcolor{fig1teal}{.} \textcolor{fig1ochre}{Asked} \textcolor{fig1teal}{what} \textcolor{fig1ochre}{transpired} \textcolor{fig1teal}{in} \textcolor{fig1teal}{a} \textcolor{fig1ochre}{press} \textcolor{fig1ochre}{shop} \textcolor{fig1teal}{hours} \textcolor{fig1teal}{early} \textcolor{fig1teal}{by} \textcolor{fig1teal}{my} \textcolor{fig1ochre}{hosts}\textcolor{fig1teal}{,} \textcolor{fig1ochre}{Adams} \textcolor{fig1teal}{said} \textcolor{fig1teal}{she} \textcolor{fig1teal}{had} \textcolor{fig1teal}{no} \textcolor{fig1teal}{access} \textcolor{fig1teal}{to} \textcolor{fig1teal}{the} \textcolor{fig1teal}{data}\textcolor{fig1teal}{.} \textcolor{fig1teal}{What} \textcolor{fig1ochre}{questioning} \textcolor{fig1teal}{what} \textcolor{fig1teal}{the} \textcolor{fig1ochre}{EPA}\textcolor{fig1teal}{,} \textcolor{fig1teal}{Se}\textcolor{fig1ochre}{yer} \textcolor{fig1teal}{told} \textcolor{fig1teal}{me} \textcolor{fig1teal}{—} \textcolor{fig1ochre}{yes} \textcolor{fig1teal}{�}\textcolor{fig1teal}{�}\textcolor{fig1teal}{I} \textcolor{fig1teal}{did} \textcolor{fig1teal}{not} \textcolor{fig1teal}{do} \textcolor{fig1teal}{that} \textcolor{fig1teal}{before} \textcolor{fig1teal}{I} \textcolor{fig1teal}{opened} \textcolor{fig1teal}{that} \textcolor{fig1ochre}{nonsense}\textcolor{fig1teal}{.}\textcolor{fig1teal}{�}\textcolor{fig1teal}{�}  \textcolor{fig1ochre}{Since} \textcolor{fig1ochre}{Nat}\textcolor{fig1ochre}{bie} \textcolor{fig1ochre}{spent} \textcolor{fig1teal}{these} \textcolor{fig1ochre}{70}\textcolor{fig1teal}{-}\textcolor{fig1ochre}{seven} \textcolor{fig1teal}{years} \textcolor{fig1teal}{in} \textcolor{fig1ochre}{prison} \textcolor{fig1teal}{for} \textcolor{fig1ochre}{felony}\textcolor{fig1ochre}{core}\textcolor{fig1teal}{,} \textcolor{fig1teal}{he} \textcolor{fig1ochre}{sticks} \textcolor{fig1teal}{to} \textcolor{fig1ochre}{Weed} \textcolor{fig1ochre}{terrorists} \textcolor{fig1teal}{every} \textcolor{fig1ochre}{imaginable} \textcolor{fig1teal}{in} \textcolor{fig1ochre}{filing} \textcolor{fig1ochre}{ly}\textcolor{fig1ochre}{ings} \textcolor{fig1teal}{of} \textcolor{fig1ochre}{time} \textcolor{fig1teal}{with} \textcolor{fig1teal}{a} \textcolor{fig1teal}{few} \textcolor{fig1ochre}{chunks} \textcolor{fig1teal}{of} \textcolor{fig1ochre}{scapego}\textcolor{fig1teal}{at} \textcolor{fig1teal}{here} \textcolor{fig1teal}{that}\textcolor{fig1teal}{�}\textcolor{fig1teal}{�}\textcolor{fig1teal}{s} \textcolor{fig1teal}{been} \textcolor{fig1teal}{tax} \textcolor{fig1teal}{for} \textcolor{fig1teal}{every} \textcolor{fig1teal}{type} \textcolor{fig1teal}{of} \textcolor{fig1ochre}{oversight} \textcolor{fig1teal}{by} \textcolor{fig1ochre}{countless}\textcolor{fig1ochre}{gun}\textcolor{fig1teal}{,} \textcolor{fig1ochre}{unm}\textcolor{fig1ochre}{ob}\textcolor{fig1ochre}{ashed} \textcolor{fig1teal}{in} \textcolor{fig1teal}{an} \textcolor{fig1ochre}{airline} \textcolor{fig1teal}{and} \textcolor{fig1teal}{by} \textcolor{fig1teal}{a} \textcolor{fig1ochre}{colleague} \textcolor{fig1teal}{and} \textcolor{fig1ochre}{lump}\textcolor{fig1teal}{ed} \textcolor{fig1teal}{in} \textcolor{fig1teal}{with} \textcolor{fig1teal}{criminal} \textcolor{fig1ochre}{interrogation}\textcolor{fig1ochre}{bles} \textcolor{fig1teal}{for} \textcolor{fig1ochre}{whom} \textcolor{fig1teal}{he} \textcolor{fig1ochre}{testified} \textcolor{fig1teal}{as} \textcolor{fig1teal}{a} \textcolor{fig1ochre}{Holocaust} \textcolor{fig1ochre}{survivor}\textcolor{fig1teal}{.}  \textcolor{fig1ochre}{Special} \textcolor{fig1ochre}{Vegas} \textcolor{fig1teal}{and} \textcolor{fig1teal}{the} \textcolor{fig1ochre}{environment} \textcolor{fig1teal}{of} \textcolor{fig1ochre}{Oxy}\textcolor{fig1ochre}{ware} \textcolor{fig1teal}{has} \textcolor{fig1teal}{is} \textcolor{fig1teal}{one} \textcolor{fig1teal}{of} \textcolor{fig1teal}{the} \textcolor{fig1teal}{company}\textcolor{fig1teal}{�}\textcolor{fig1teal}{�}\textcolor{fig1teal}{s} \textcolor{fig1teal}{most} \textcolor{fig1teal}{important} \textcolor{fig1teal}{part} \textcolor{fig1teal}{of} \textcolor{fig1teal}{business} \textcolor{fig1teal}{now}\textcolor{fig1teal}{.}  \textcolor{fig1ochre}{Supporters} \textcolor{fig1teal}{of} \textcolor{fig1ochre}{Mr} \textcolor{fig1ochre}{Zin}\textcolor{fig1ochre}{yer} \textcolor{fig1teal}{told} \textcolor{fig1ochre}{Arnold} \textcolor{fig1teal}{that} \textcolor{fig1teal}{a} \textcolor{fig1teal}{possible} \textcolor{fig1teal}{run} \textcolor{fig1teal}{for} \textcolor{fig1teal}{the} \textcolor{fig1teal}{oil} \textcolor{fig1ochre}{mogul} \textcolor{fig1teal}{was} \textcolor{fig1teal}{not} \textcolor{fig1teal}{too} \textcolor{fig1teal}{possible}\textcolor{fig1teal}{.} \textcolor{fig1ochre}{Dev}\textcolor{fig1ochre}{ine} \textcolor{fig1teal}{told} \textcolor{fig1ochre}{Park} \textcolor{fig1ochre}{Radio}\textcolor{fig1teal}{�}\textcolor{fig1teal}{�}\textcolor{fig1teal}{s} \textcolor{fig1ochre}{Jill} \textcolor{fig1teal}{O}\textcolor{fig1ochre}{aken}\textcolor{fig1teal}{,} \textcolor{fig1teal}{�}\textcolor{fig1teal}{�}\textcolor{fig1teal}{You}\textcolor{fig1teal}{�}\textcolor{fig1teal}{�}\textcolor{fig1teal}{re}
\end{samplebox}
\begin{samplebox}{\normalfont\textbf{$\ours$ (PR), Round: 2} \hfill \normalfont\scriptsize \textcolor{darkgray}{Committed: \textbf{871/1024}}}
\scriptsize\linespread{0.9}\selectfont
\textcolor{fig1teal}{.}  \textcolor{fig1ochre}{Or}\textcolor{fig1ochre}{noon} \textcolor{fig1teal}{is} \textcolor{fig1teal}{another} \textcolor{fig1ochre}{prick}\textcolor{fig1teal}{ly} \textcolor{fig1ochre}{photographer}\textcolor{fig1teal}{.} \textcolor{fig1teal}{He} \textcolor{fig1teal}{has} \textcolor{fig1teal}{been} \textcolor{fig1teal}{working} \textcolor{fig1teal}{on} \textcolor{fig1teal}{a} \textcolor{fig1ochre}{firearms} \textcolor{fig1ochre}{secretary} \textcolor{fig1teal}{for} \textcolor{fig1teal}{9} \textcolor{fig1teal}{years}\textcolor{fig1teal}{,} \textcolor{fig1teal}{shooting} \textcolor{fig1teal}{for}\textcolor{fig1ochre}{EMOTE} \textcolor{fig1teal}{in} \textcolor{fig1teal}{different} \textcolor{fig1teal}{parts} \textcolor{fig1teal}{of} \textcolor{fig1teal}{New} \textcolor{fig1ochre}{Hampshire}\textcolor{fig1teal}{.} \textcolor{fig1teal}{I} \textcolor{fig1teal}{think} \textcolor{fig1teal}{it}\textcolor{fig1teal}{'s} \textcolor{fig1teal}{just} \textcolor{fig1teal}{that} \textcolor{fig1teal}{after} \textcolor{fig1teal}{the} \textcolor{fig1teal}{recent} \textcolor{fig1teal}{project} \textcolor{fig1teal}{in} \textcolor{fig1teal}{front} \textcolor{fig1teal}{of} \textcolor{fig1teal}{school}\textcolor{fig1teal}{,} \textcolor{fig1teal}{many} \textcolor{fig1teal}{more} \textcolor{fig1teal}{of} \textcolor{fig1teal}{them} \textcolor{fig1ochre}{got} \textcolor{fig1teal}{it} \textcolor{fig1teal}{into} \textcolor{fig1teal}{a} \textcolor{fig1ochre}{ton} \textcolor{fig1teal}{of} \textcolor{fig1ochre}{fashion}\textcolor{fig1teal}{.}  \textcolor{fig1ochre}{Sunday} \textcolor{fig1teal}{night}\textcolor{fig1teal}{,} \textcolor{fig1teal}{Or}\textcolor{fig1ochre}{gal} \textcolor{fig1teal}{came} \textcolor{fig1teal}{to} \textcolor{fig1ochre}{Pul}\textcolor{fig1teal}{p}\textcolor{fig1teal}{,} \textcolor{fig1teal}{with} \textcolor{fig1teal}{no} \textcolor{fig1teal}{return}\textcolor{fig1teal}{.} \textcolor{fig1teal}{In} \textcolor{fig1teal}{response} \textcolor{fig1teal}{to} \textcolor{fig1teal}{the} \textcolor{fig1teal}{press} \textcolor{fig1ochre}{conference} \textcolor{fig1teal}{after} \textcolor{fig1teal}{the} \textcolor{fig1ochre}{combat} \textcolor{fig1teal}{day}\textcolor{fig1teal}{,} \textcolor{fig1ochre}{Representative} \textcolor{fig1ochre}{East}\textcolor{fig1ochre}{etts} \textcolor{fig1teal}{(}\textcolor{fig1teal}{R}\textcolor{fig1teal}{-}\textcolor{fig1ochre}{NY}\textcolor{fig1teal}{)} \textcolor{fig1teal}{told} \textcolor{fig1teal}{me} \textcolor{fig1teal}{the} \textcolor{fig1teal}{news}\textcolor{fig1teal}{:}  \textcolor{fig1teal}{�}\textcolor{fig1teal}{�}\textcolor{fig1teal}{To} \textcolor{fig1teal}{find} \textcolor{fig1teal}{it} \textcolor{fig1teal}{in} \textcolor{fig1teal}{an} \textcolor{fig1teal}{office} \textcolor{fig1teal}{that} \textcolor{fig1teal}{is} \textcolor{fig1teal}{like} \textcolor{fig1teal}{a} \textcolor{fig1ochre}{office} \textcolor{fig1teal}{—} \textcolor{fig1teal}{we} \textcolor{fig1teal}{still} \textcolor{fig1teal}{think} \textcolor{fig1teal}{of} \textcolor{fig1teal}{our} \textcolor{fig1teal}{business} \textcolor{fig1teal}{with} \textcolor{fig1teal}{the} \textcolor{fig1ochre}{gun}\textcolor{fig1teal}{.}\textcolor{fig1teal}{�}\textcolor{fig1teal}{�}  \textcolor{fig1teal}{It} \textcolor{fig1teal}{was} \textcolor{fig1teal}{the} \textcolor{fig1teal}{first} \textcolor{fig1teal}{police} \textcolor{fig1ochre}{sting} \textcolor{fig1teal}{for} \textcolor{fig1teal}{an} \textcolor{fig1ochre}{NRA} \textcolor{fig1teal}{in} \textcolor{fig1teal}{this} \textcolor{fig1teal}{time} \textcolor{fig1teal}{of} \textcolor{fig1teal}{war}\textcolor{fig1teal}{.} \textcolor{fig1teal}{It} \textcolor{fig1teal}{was} \textcolor{fig1teal}{also} \textcolor{fig1teal}{that} \textcolor{fig1teal}{he} \textcolor{fig1teal}{came} \textcolor{fig1teal}{to} \textcolor{fig1teal}{open} \textcolor{fig1teal}{the} \textcolor{fig1teal}{business} \textcolor{fig1teal}{and} \textcolor{fig1teal}{its} \textcolor{fig1ochre}{presence}\textcolor{fig1teal}{,} \textcolor{fig1teal}{and} \textcolor{fig1teal}{then} \textcolor{fig1teal}{was}\textcolor{fig1teal}{,} \textcolor{fig1teal}{to} \textcolor{fig1teal}{cut} \textcolor{fig1teal}{short} \textcolor{fig1teal}{his} \textcolor{fig1ochre}{friend}\textcolor{fig1teal}{,} \textcolor{fig1ochre}{having}\textcolor{fig1teal}{,} \textcolor{fig1teal}{who} \textcolor{fig1teal}{has} \textcolor{fig1teal}{served} \textcolor{fig1teal}{as} \textcolor{fig1teal}{software} \textcolor{fig1ochre}{analyst} \textcolor{fig1teal}{for} \textcolor{fig1ochre}{Africa}\textcolor{fig1teal}{,} \textcolor{fig1ochre}{better}\textcolor{fig1teal}{,} \textcolor{fig1teal}{and} \textcolor{fig1ochre}{HP} \textcolor{fig1teal}{for} \textcolor{fig1teal}{decades}\textcolor{fig1teal}{.}  \textcolor{fig1teal}{�}\textcolor{fig1teal}{�}\textcolor{fig1teal}{There}\textcolor{fig1teal}{�}\textcolor{fig1teal}{�}\textcolor{fig1teal}{s} \textcolor{fig1teal}{a} \textcolor{fig1teal}{lot} \textcolor{fig1teal}{for} \textcolor{fig1teal}{me} \textcolor{fig1teal}{to} \textcolor{fig1teal}{do} \textcolor{fig1teal}{with} \textcolor{fig1teal}{myself} \textcolor{fig1teal}{that} \textcolor{fig1teal}{would} \textcolor{fig1ochre}{help} \textcolor{fig1teal}{me}\textcolor{fig1teal}{.} \textcolor{fig1teal}{I} \textcolor{fig1teal}{think} \textcolor{fig1teal}{this} \textcolor{fig1teal}{is} \textcolor{fig1teal}{my} \textcolor{fig1teal}{first} \textcolor{fig1ochre}{job}\textcolor{fig1teal}{,}\textcolor{fig1teal}{�}\textcolor{fig1teal}{�} \textcolor{fig1teal}{he} \textcolor{fig1teal}{says}\textcolor{fig1teal}{.} \textcolor{fig1teal}{�}\textcolor{fig1teal}{�}\textcolor{fig1teal}{These} \textcolor{fig1teal}{were} \textcolor{fig1teal}{the} \textcolor{fig1teal}{real} \textcolor{fig1ochre}{Deal}\textcolor{fig1teal}{.} \textcolor{fig1teal}{If} \textcolor{fig1teal}{you} \textcolor{fig1teal}{have} \textcolor{fig1teal}{to} \textcolor{fig1teal}{see} \textcolor{fig1teal}{the} \textcolor{fig1teal}{new} \textcolor{fig1ochre}{structures} \textcolor{fig1teal}{which} \textcolor{fig1teal}{was} \textcolor{fig1teal}{made} \textcolor{fig1teal}{to} \textcolor{fig1ochre}{accommodate} \textcolor{fig1teal}{all} \textcolor{fig1teal}{of} \textcolor{fig1teal}{him}\textcolor{fig1teal}{,} \textcolor{fig1teal}{ask} \textcolor{fig1teal}{one} \textcolor{fig1teal}{of} \textcolor{fig1teal}{you} \textcolor{fig1teal}{to} \textcolor{fig1teal}{think} \textcolor{fig1teal}{about} \textcolor{fig1teal}{what} \textcolor{fig1teal}{each} \textcolor{fig1teal}{other} \textcolor{fig1ochre}{applicant} \textcolor{fig1teal}{could} \textcolor{fig1teal}{show}\textcolor{fig1teal}{.}  \textcolor{fig1teal}{On} \textcolor{fig1teal}{New} \textcolor{fig1teal}{York}\textcolor{fig1teal}{�}\textcolor{fig1teal}{�}\textcolor{fig1teal}{s} \textcolor{fig1teal}{State} \textcolor{fig1teal}{House} \textcolor{fig1teal}{and} \textcolor{fig1ochre}{pass}\textcolor{fig1teal}{-}\textcolor{fig1teal}{A} \textcolor{fig1ochre}{Hill}\textcolor{fig1teal}{,} \textcolor{fig1teal}{his} \textcolor{fig1ochre}{policy} \textcolor{fig1teal}{is} \textcolor{fig1teal}{focused} \textcolor{fig1teal}{on} \textcolor{fig1ochre}{investment} \textcolor{fig1teal}{and} \textcolor{fig1teal}{growth}\textcolor{fig1teal}{.} \textcolor{fig1teal}{In} \textcolor{fig1teal}{the} \textcolor{fig1teal}{public} \textcolor{fig1teal}{health} \textcolor{fig1teal}{system}\textcolor{fig1teal}{,} \textcolor{fig1teal}{for} \textcolor{fig1teal}{him} \textcolor{fig1teal}{to} \textcolor{fig1teal}{get} \textcolor{fig1teal}{world}\textcolor{fig1teal}{-}\textcolor{fig1teal}{class} \textcolor{fig1teal}{take} \textcolor{fig1teal}{of} \textcolor{fig1teal}{a} \textcolor{fig1ochre}{city} \textcolor{fig1teal}{and} \textcolor{fig1teal}{money} \textcolor{fig1teal}{from} \textcolor{fig1teal}{a} \textcolor{fig1teal}{friend} \textcolor{fig1teal}{to} \textcolor{fig1ochre}{fans}\textcolor{fig1teal}{,} \textcolor{fig1teal}{there} \textcolor{fig1teal}{are} \textcolor{fig1teal}{people} \textcolor{fig1teal}{in} \textcolor{fig1teal}{on} \textcolor{fig1teal}{life}\textcolor{fig1teal}{s} \textcolor{fig1teal}{of} \textcolor{fig1teal}{days}\textcolor{fig1teal}{,} \textcolor{fig1teal}{and} \textcolor{fig1ochre}{selling} \textcolor{fig1ochre}{drug}\textcolor{fig1teal}{,} \textcolor{fig1teal}{but} \textcolor{fig1teal}{not} \textcolor{fig1ochre}{heroin}\textcolor{fig1teal}{,} \textcolor{fig1teal}{these} \textcolor{fig1teal}{days}\textcolor{fig1teal}{.}  \textcolor{fig1teal}{It} \textcolor{fig1teal}{was} \textcolor{fig1teal}{because} \textcolor{fig1teal}{of} \textcolor{fig1teal}{the} \textcolor{fig1teal}{amount} \textcolor{fig1teal}{of} \textcolor{fig1teal}{real} \textcolor{fig1ochre}{pills} \textcolor{fig1teal}{and} \textcolor{fig1teal}{the} \textcolor{fig1teal}{day} \textcolor{fig1teal}{off} \textcolor{fig1teal}{the} \textcolor{fig1teal}{company}\textcolor{fig1teal}{�}\textcolor{fig1teal}{�}\textcolor{fig1teal}{s} \textcolor{fig1ochre}{eye}\textcolor{fig1teal}{.} \textcolor{fig1ochre}{Plus}\textcolor{fig1teal}{,} \textcolor{fig1teal}{the} \textcolor{fig1ochre}{paying} \textcolor{fig1teal}{of} \textcolor{fig1ochre}{un}\textcolor{fig1teal}{in} \textcolor{fig1teal}{was} \textcolor{fig1teal}{personal}\textcolor{fig1teal}{,} \textcolor{fig1teal}{because} \textcolor{fig1teal}{many} \textcolor{fig1ochre}{narc}\textcolor{fig1teal}{ist} \textcolor{fig1ochre}{drugs} \textcolor{fig1teal}{were} \textcolor{fig1teal}{typical} \textcolor{fig1teal}{from} \textcolor{fig1teal}{last} \textcolor{fig1ochre}{dose}\textcolor{fig1teal}{.} \textcolor{fig1teal}{40} \textcolor{fig1teal}{percent} \textcolor{fig1teal}{of} \textcolor{fig1teal}{the} \textcolor{fig1teal}{number} \textcolor{fig1teal}{of} \textcolor{fig1ochre}{recreational} \textcolor{fig1teal}{drug} \textcolor{fig1teal}{use}\textcolor{fig1teal}{,} \textcolor{fig1teal}{as} \textcolor{fig1teal}{well} \textcolor{fig1teal}{as} \textcolor{fig1ochre}{Bi} \textcolor{fig1ochre}{etc}\textcolor{fig1teal}{.} \textcolor{fig1teal}{and} \textcolor{fig1ochre}{diet}\textcolor{fig1teal}{ing} \textcolor{fig1teal}{were} \textcolor{fig1teal}{all} \textcolor{fig1teal}{21}\textcolor{fig1teal}{-}\textcolor{fig1teal}{year}\textcolor{fig1teal}{-}\textcolor{fig1teal}{old} \textcolor{fig1teal}{students} \textcolor{fig1teal}{on} \textcolor{fig1teal}{their} \textcolor{fig1teal}{own}\textcolor{fig1teal}{.}  \textcolor{fig1teal}{�}\textcolor{fig1teal}{�}\textcolor{fig1teal}{I} \textcolor{fig1teal}{think} \textcolor{fig1teal}{with} \textcolor{fig1teal}{all} \textcolor{fig1teal}{the} \textcolor{fig1ochre}{academic} \textcolor{fig1teal}{data} \textcolor{fig1teal}{since} \textcolor{fig1teal}{I}\textcolor{fig1teal}{�}\textcolor{fig1teal}{�}\textcolor{fig1teal}{ve} \textcolor{fig1teal}{worked} \textcolor{fig1teal}{with} \textcolor{fig1teal}{David} \textcolor{fig1ochre}{Mu}\textcolor{fig1teal}{um}\textcolor{fig1teal}{,} \textcolor{fig1teal}{the} \textcolor{fig1teal}{amount} \textcolor{fig1teal}{of} \textcolor{fig1ochre}{truth} \textcolor{fig1teal}{sites} \textcolor{fig1teal}{or} \textcolor{fig1teal}{what} \textcolor{fig1teal}{I} \textcolor{fig1teal}{found} \textcolor{fig1teal}{there} \textcolor{fig1teal}{was}\textcolor{fig1teal}{;} \textcolor{fig1teal}{a} \textcolor{fig1teal}{lot} \textcolor{fig1teal}{in} \textcolor{fig1teal}{which} \textcolor{fig1teal}{a} \textcolor{fig1teal}{lot} \textcolor{fig1ochre}{looked} \textcolor{fig1teal}{more} \textcolor{fig1teal}{at} \textcolor{fig1teal}{people} \textcolor{fig1teal}{other} \textcolor{fig1teal}{than} \textcolor{fig1teal}{David}\textcolor{fig1teal}{�}\textcolor{fig1teal}{�}\textcolor{fig1teal}{s}\textcolor{fig1teal}{,} \textcolor{fig1teal}{but} \textcolor{fig1teal}{I} \textcolor{fig1teal}{know}\textcolor{fig1teal}{,} \textcolor{fig1teal}{my} \textcolor{fig1teal}{brain} \textcolor{fig1teal}{want} \textcolor{fig1teal}{us} \textcolor{fig1teal}{to} \textcolor{fig1teal}{say} \textcolor{fig1teal}{that}\textcolor{fig1teal}{,}\textcolor{fig1teal}{�}\textcolor{fig1teal}{�} \textcolor{fig1teal}{Se}\textcolor{fig1ochre}{gal} \textcolor{fig1teal}{said}\textcolor{fig1teal}{.} \textcolor{fig1teal}{�}\textcolor{fig1teal}{�}\textcolor{fig1teal}{In} \textcolor{fig1ochre}{Mississippi}\textcolor{fig1teal}{,} \textcolor{fig1teal}{I}\textcolor{fig1teal}{�}\textcolor{fig1teal}{�}\textcolor{fig1teal}{m} \textcolor{fig1teal}{not} \textcolor{fig1teal}{sure} \textcolor{fig1teal}{I} \textcolor{fig1teal}{really}\textcolor{fig1teal}{�}\textcolor{fig1teal}{�}\textcolor{fig1teal}{m} \textcolor{fig1teal}{going} \textcolor{fig1teal}{to} \textcolor{fig1teal}{do} \textcolor{fig1teal}{that} \textcolor{fig1teal}{money} \textcolor{fig1teal}{in} \textcolor{fig1ochre}{Manhattan}\textcolor{fig1teal}{.} \textcolor{fig1teal}{I} \textcolor{fig1teal}{was} \textcolor{fig1teal}{going} \textcolor{fig1teal}{to} \textcolor{fig1teal}{do} \textcolor{fig1teal}{this} \textcolor{fig1teal}{to} \textcolor{fig1teal}{schools} \textcolor{fig1teal}{and} \textcolor{fig1teal}{so} \textcolor{fig1teal}{other} \textcolor{fig1teal}{areas} \textcolor{fig1teal}{don}\textcolor{fig1teal}{�}\textcolor{fig1teal}{�}\textcolor{fig1teal}{t} \textcolor{fig1teal}{give} \textcolor{fig1teal}{them} \textcolor{fig1teal}{the} \textcolor{fig1teal}{data}\textcolor{fig1teal}{.}\textcolor{fig1teal}{�}\textcolor{fig1teal}{�}  \textcolor{fig1teal}{The} \textcolor{fig1ochre}{answers} \textcolor{fig1teal}{didn}\textcolor{fig1teal}{�}\textcolor{fig1teal}{�}\textcolor{fig1teal}{t} \textcolor{fig1teal}{fall} \textcolor{fig1teal}{to} \textcolor{fig1ochre}{Les}\textcolor{fig1teal}{,} \textcolor{fig1teal}{and} \textcolor{fig1teal}{he} \textcolor{fig1teal}{used} \textcolor{fig1teal}{to} \textcolor{fig1ochre}{worry} \textcolor{fig1teal}{that} \textcolor{fig1teal}{he} \textcolor{fig1teal}{could} \textcolor{fig1teal}{make} \textcolor{fig1teal}{people} \textcolor{fig1ochre}{arrangements} \textcolor{fig1teal}{to} \textcolor{fig1ochre}{cope} \textcolor{fig1teal}{with} \textcolor{fig1teal}{the} \textcolor{fig1ochre}{2008} \textcolor{fig1teal}{crisis}\textcolor{fig1teal}{.}  \textcolor{fig1teal}{�}\textcolor{fig1teal}{�}\textcolor{fig1teal}{One} \textcolor{fig1teal}{of} \textcolor{fig1teal}{the} \textcolor{fig1ochre}{managers} \textcolor{fig1teal}{had} \textcolor{fig1teal}{me} \textcolor{fig1teal}{on} \textcolor{fig1teal}{the} \textcolor{fig1teal}{back} \textcolor{fig1teal}{and} \textcolor{fig1teal}{I} \textcolor{fig1teal}{was} \textcolor{fig1teal}{looking} \textcolor{fig1teal}{at} \textcolor{fig1teal}{what} \textcolor{fig1teal}{was} \textcolor{fig1teal}{no} \textcolor{fig1ochre}{floor}\textcolor{fig1teal}{,}\textcolor{fig1teal}{�}\textcolor{fig1teal}{�} \textcolor{fig1ochre}{Or}\textcolor{fig1ochre}{lor} \textcolor{fig1teal}{told}\textcolor{fig1teal}{,} \textcolor{fig1ochre}{br}\textcolor{fig1teal}{at} \textcolor{fig1teal}{looking} \textcolor{fig1teal}{at} \textcolor{fig1teal}{the} \textcolor{fig1ochre}{auditor}\textcolor{fig1teal}{�}\textcolor{fig1teal}{�}\textcolor{fig1teal}{s} \textcolor{fig1ochre}{lady}\textcolor{fig1teal}{-}\textcolor{fig1teal}{in}\textcolor{fig1teal}{-}\textcolor{fig1ochre}{law}\textcolor{fig1teal}{,} \textcolor{fig1teal}{looking} \textcolor{fig1teal}{against} \textcolor{fig1teal}{him}\textcolor{fig1teal}{.} \textcolor{fig1teal}{And} \textcolor{fig1teal}{I} \textcolor{fig1teal}{asked} \textcolor{fig1teal}{him} \textcolor{fig1teal}{into} \textcolor{fig1teal}{a} \textcolor{fig1teal}{New} \textcolor{fig1teal}{York} \textcolor{fig1ochre}{Food} \textcolor{fig1ochre}{suite} \textcolor{fig1teal}{at} \textcolor{fig1teal}{home}\textcolor{fig1teal}{,} \textcolor{fig1teal}{at} \textcolor{fig1ochre}{Roosevelt} \textcolor{fig1teal}{Center} \textcolor{fig1teal}{—} \textcolor{fig1teal}{something} \textcolor{fig1teal}{she} \textcolor{fig1teal}{said} \textcolor{fig1teal}{had} \textcolor{fig1teal}{been} \textcolor{fig1teal}{very} \textcolor{fig1teal}{effective} \textcolor{fig1teal}{—} \textcolor{fig1teal}{and} \textcolor{fig1teal}{there} \textcolor{fig1teal}{was} \textcolor{fig1teal}{some} \textcolor{fig1ochre}{chance}\textcolor{fig1teal}{.} \textcolor{fig1teal}{For} \textcolor{fig1teal}{the} \textcolor{fig1teal}{first} \textcolor{fig1teal}{time}\textcolor{fig1teal}{,} \textcolor{fig1teal}{I} \textcolor{fig1teal}{had} \textcolor{fig1teal}{to} \textcolor{fig1teal}{take} \textcolor{fig1teal}{off} \textcolor{fig1teal}{the} \textcolor{fig1teal}{same} \textcolor{fig1ochre}{shirt} \textcolor{fig1teal}{and} \textcolor{fig1teal}{ch}\textcolor{fig1ochre}{imes}\textcolor{fig1teal}{.} \textcolor{fig1teal}{This} \textcolor{fig1teal}{didn}\textcolor{fig1teal}{�}\textcolor{fig1teal}{�}\textcolor{fig1teal}{t} \textcolor{fig1ochre}{happen} \textcolor{fig1teal}{as} \textcolor{fig1teal}{we} \textcolor{fig1teal}{continued} \textcolor{fig1teal}{on} \textcolor{fig1teal}{with} \textcolor{fig1ochre}{Tab}\textcolor{fig1teal}{at}\textcolor{fig1teal}{,} \textcolor{fig1ochre}{Dr} \textcolor{fig1ochre}{continued}\textcolor{fig1teal}{,} \textcolor{fig1ochre}{ch}\textcolor{fig1ochre}{aun}\textcolor{fig1teal}{er} \textcolor{fig1teal}{and} \textcolor{fig1ochre}{counting} \textcolor{fig1teal}{any} \textcolor{fig1teal}{potential} \textcolor{fig1ochre}{logistics}\textcolor{fig1teal}{.}  \textcolor{fig1teal}{�}\textcolor{fig1teal}{�}\textcolor{fig1teal}{That}\textcolor{fig1teal}{�}\textcolor{fig1teal}{�}\textcolor{fig1teal}{s} \textcolor{fig1ochre}{normal}\textcolor{fig1teal}{.} \textcolor{fig1teal}{But} \textcolor{fig1teal}{I} \textcolor{fig1teal}{know} \textcolor{fig1teal}{that} \textcolor{fig1teal}{eventually} \textcolor{fig1teal}{you} \textcolor{fig1ochre}{moved} \textcolor{fig1teal}{in} \textcolor{fig1teal}{and} \textcolor{fig1teal}{got} \textcolor{fig1teal}{a} \textcolor{fig1teal}{control} \textcolor{fig1teal}{and} \textcolor{fig1teal}{the} \textcolor{fig1teal}{work} \textcolor{fig1teal}{was} \textcolor{fig1teal}{over} \textcolor{fig1teal}{and} \textcolor{fig1teal}{it}\textcolor{fig1teal}{�}\textcolor{fig1teal}{�}\textcolor{fig1teal}{s} \textcolor{fig1teal}{always} \textcolor{fig1teal}{questions} \textcolor{fig1teal}{about} \textcolor{fig1teal}{the} \textcolor{fig1ochre}{cost} \textcolor{fig1teal}{from} \textcolor{fig1teal}{per} \textcolor{fig1teal}{se}\textcolor{fig1teal}{.}\textcolor{fig1teal}{�}\textcolor{fig1teal}{�}  \textcolor{fig1ochre}{Where} \textcolor{fig1teal}{he} \textcolor{fig1ochre}{agreed}  \textcolor{fig1teal}{The} \textcolor{fig1teal}{official} \textcolor{fig1ochre}{reputation} \textcolor{fig1teal}{among} \textcolor{fig1teal}{the} \textcolor{fig1ochre}{researchers} \textcolor{fig1teal}{was} \textcolor{fig1ochre}{controversial} \textcolor{fig1teal}{but} \textcolor{fig1teal}{some} \textcolor{fig1teal}{(}\textcolor{fig1teal}{D}\textcolor{fig1teal}{-}\textcolor{fig1ochre}{state} \textcolor{fig1teal}{not} \textcolor{fig1ochre}{necessarily} \textcolor{fig1teal}{a} \textcolor{fig1teal}{good} \textcolor{fig1teal}{one}\textcolor{fig1teal}{)} \textcolor{fig1teal}{of} \textcolor{fig1teal}{him}\textcolor{fig1teal}{.} \textcolor{fig1ochre}{As} \textcolor{fig1teal}{such} \textcolor{fig1teal}{never} \textcolor{fig1teal}{the} \textcolor{fig1ochre}{candidate} \textcolor{fig1teal}{he} \textcolor{fig1teal}{is}\textcolor{fig1teal}{,} \textcolor{fig1teal}{Or}\textcolor{fig1ochre}{bee} \textcolor{fig1teal}{said} \textcolor{fig1teal}{it} \textcolor{fig1teal}{was} \textcolor{fig1teal}{not} \textcolor{fig1ochre}{clear} \textcolor{fig1teal}{to} \textcolor{fig1teal}{understand} \textcolor{fig1teal}{that} \textcolor{fig1teal}{an} \textcolor{fig1teal}{operation} \textcolor{fig1teal}{could} \textcolor{fig1teal}{get} \textcolor{fig1teal}{home} \textcolor{fig1teal}{to} \textcolor{fig1teal}{his} \textcolor{fig1teal}{first} \textcolor{fig1ochre}{master} \textcolor{fig1ochre}{cause}\textcolor{fig1teal}{.} \textcolor{fig1teal}{The} \textcolor{fig1teal}{National} \textcolor{fig1ochre}{Energy} \textcolor{fig1ochre}{Council} \textcolor{fig1teal}{official}\textcolor{fig1teal}{,} \textcolor{fig1ochre}{Per}\textcolor{fig1teal}{,} \textcolor{fig1teal}{reported} \textcolor{fig1teal}{after} \textcolor{fig1teal}{a} \textcolor{fig1teal}{running} \textcolor{fig1teal}{for} \textcolor{fig1ochre}{mayor} \textcolor{fig1teal}{that} \textcolor{fig1ochre}{Page}\textcolor{fig1ochre}{day} \textcolor{fig1teal}{said} \textcolor{fig1teal}{he} \textcolor{fig1teal}{was} \textcolor{fig1teal}{looking} \textcolor{fig1teal}{to} \textcolor{fig1teal}{the} \textcolor{fig1teal}{energy} \textcolor{fig1ochre}{company} \textcolor{fig1teal}{for} \textcolor{fig1ochre}{industry} \textcolor{fig1teal}{rules} \textcolor{fig1teal}{in} \textcolor{fig1teal}{the} \textcolor{fig1ochre}{cameras} \textcolor{fig1teal}{and} \textcolor{fig1teal}{the} \textcolor{fig1teal}{business}\textcolor{fig1teal}{,} \textcolor{fig1teal}{changing} \textcolor{fig1teal}{them} \textcolor{fig1teal}{to} \textcolor{fig1teal}{fit} \textcolor{fig1teal}{future} \textcolor{fig1ochre}{contracting}\textcolor{fig1teal}{,} \textcolor{fig1teal}{and} \textcolor{fig1teal}{to} \textcolor{fig1ochre}{skew} \textcolor{fig1teal}{the} \textcolor{fig1teal}{size} \textcolor{fig1teal}{of} \textcolor{fig1teal}{the} \textcolor{fig1ochre}{document}\textcolor{fig1teal}{.} \textcolor{fig1teal}{At} \textcolor{fig1teal}{what} \textcolor{fig1teal}{was} \textcolor{fig1teal}{in} \textcolor{fig1teal}{a} \textcolor{fig1ochre}{course} \textcolor{fig1teal}{three} \textcolor{fig1teal}{hours} \textcolor{fig1teal}{early} \textcolor{fig1teal}{by} \textcolor{fig1teal}{my} \textcolor{fig1ochre}{reporter}\textcolor{fig1teal}{,} \textcolor{fig1ochre}{Larson} \textcolor{fig1teal}{said} \textcolor{fig1teal}{she} \textcolor{fig1teal}{had} \textcolor{fig1teal}{no} \textcolor{fig1teal}{access} \textcolor{fig1teal}{to} \textcolor{fig1teal}{the} \textcolor{fig1teal}{data}\textcolor{fig1teal}{.} \textcolor{fig1teal}{What} \textcolor{fig1ochre}{triggered} \textcolor{fig1teal}{what} \textcolor{fig1teal}{the} \textcolor{fig1ochre}{director}\textcolor{fig1teal}{,} \textcolor{fig1teal}{Se}\textcolor{fig1ochre}{gal} \textcolor{fig1teal}{told} \textcolor{fig1teal}{me} \textcolor{fig1teal}{—} \textcolor{fig1teal}{was} \textcolor{fig1teal}{�}\textcolor{fig1teal}{�}\textcolor{fig1teal}{I} \textcolor{fig1teal}{did} \textcolor{fig1teal}{not} \textcolor{fig1teal}{do} \textcolor{fig1teal}{that} \textcolor{fig1teal}{before} \textcolor{fig1teal}{I} \textcolor{fig1teal}{opened} \textcolor{fig1teal}{that} \textcolor{fig1ochre}{tank}\textcolor{fig1teal}{.}\textcolor{fig1teal}{�}\textcolor{fig1teal}{�}  \textcolor{fig1teal}{To} \textcolor{fig1teal}{have} \textcolor{fig1teal}{one} \textcolor{fig1teal}{of} \textcolor{fig1teal}{these} \textcolor{fig1teal}{know}\textcolor{fig1teal}{-} \textcolor{fig1ochre}{ya} \textcolor{fig1teal}{years} \textcolor{fig1teal}{in} \textcolor{fig1ochre}{regulation} \textcolor{fig1teal}{for} \textcolor{fig1ochre}{energy} \textcolor{fig1teal}{life}\textcolor{fig1teal}{,} \textcolor{fig1teal}{he} \textcolor{fig1teal}{had} \textcolor{fig1teal}{to} \textcolor{fig1ochre}{start} \textcolor{fig1teal}{getting} \textcolor{fig1teal}{every} \textcolor{fig1ochre}{answer} \textcolor{fig1teal}{in} \textcolor{fig1ochre}{exchange} \textcolor{fig1teal}{for} \textcolor{fig1ochre}{respect} \textcolor{fig1teal}{of} \textcolor{fig1ochre}{background} \textcolor{fig1teal}{with} \textcolor{fig1teal}{a} \textcolor{fig1teal}{few} \textcolor{fig1ochre}{pounds} \textcolor{fig1teal}{of} \textcolor{fig1ochre}{bureaucr}\textcolor{fig1teal}{at} \textcolor{fig1teal}{here} \textcolor{fig1teal}{that}\textcolor{fig1teal}{�}\textcolor{fig1teal}{�}\textcolor{fig1teal}{s} \textcolor{fig1teal}{been} \textcolor{fig1teal}{tax} \textcolor{fig1teal}{for} \textcolor{fig1teal}{every} \textcolor{fig1teal}{type} \textcolor{fig1teal}{of} \textcolor{fig1teal}{activity} \textcolor{fig1teal}{by} \textcolor{fig1ochre}{multiple} \textcolor{fig1teal}{people}\textcolor{fig1teal}{,} \textcolor{fig1teal}{including} \textcolor{fig1ochre}{health}\textcolor{fig1teal}{,} \textcolor{fig1teal}{in} \textcolor{fig1teal}{an} \textcolor{fig1ochre}{individual} \textcolor{fig1teal}{and} \textcolor{fig1teal}{by} \textcolor{fig1teal}{a} \textcolor{fig1ochre}{stroke} \textcolor{fig1teal}{and} \textcolor{fig1ochre}{lump}\textcolor{fig1teal}{ed} \textcolor{fig1teal}{in} \textcolor{fig1teal}{with} \textcolor{fig1teal}{criminal} \textcolor{fig1teal}{activity}\textcolor{fig1teal}{,} \textcolor{fig1teal}{for} \textcolor{fig1teal}{which} \textcolor{fig1teal}{he} \textcolor{fig1ochre}{serves} \textcolor{fig1teal}{as} \textcolor{fig1teal}{a} \textcolor{fig1ochre}{security} \textcolor{fig1teal}{day}\textcolor{fig1teal}{.}  \textcolor{fig1ochre}{Today}\textcolor{fig1teal}{,} \textcolor{fig1teal}{and} \textcolor{fig1teal}{the} \textcolor{fig1teal}{type} \textcolor{fig1teal}{of} \textcolor{fig1ochre}{resources} \textcolor{fig1teal}{he} \textcolor{fig1teal}{has} \textcolor{fig1teal}{is} \textcolor{fig1teal}{one} \textcolor{fig1teal}{of} \textcolor{fig1teal}{the} \textcolor{fig1teal}{company}\textcolor{fig1teal}{�}\textcolor{fig1teal}{�}\textcolor{fig1teal}{s} \textcolor{fig1teal}{most} \textcolor{fig1teal}{important} \textcolor{fig1teal}{part} \textcolor{fig1teal}{of} \textcolor{fig1teal}{business} \textcolor{fig1teal}{now}\textcolor{fig1teal}{.}  \textcolor{fig1teal}{One} \textcolor{fig1teal}{of} \textcolor{fig1teal}{his} \textcolor{fig1ochre}{former} \textcolor{fig1ochre}{staffers} \textcolor{fig1teal}{told} \textcolor{fig1teal}{me} \textcolor{fig1teal}{that} \textcolor{fig1teal}{a} \textcolor{fig1teal}{possible} \textcolor{fig1teal}{run} \textcolor{fig1teal}{for} \textcolor{fig1teal}{the} \textcolor{fig1teal}{oil} \textcolor{fig1ochre}{refinery} \textcolor{fig1teal}{was} \textcolor{fig1teal}{not} \textcolor{fig1teal}{too} \textcolor{fig1teal}{possible}\textcolor{fig1teal}{.} \textcolor{fig1teal}{He} \textcolor{fig1teal}{also} \textcolor{fig1teal}{told} \textcolor{fig1teal}{The} \textcolor{fig1ochre}{Post}\textcolor{fig1teal}{�}\textcolor{fig1teal}{�}\textcolor{fig1teal}{s} \textcolor{fig1ochre}{Barbara} \textcolor{fig1teal}{O}\textcolor{fig1ochre}{ster}\textcolor{fig1teal}{,} \textcolor{fig1teal}{�}\textcolor{fig1teal}{�}\textcolor{fig1teal}{You}\textcolor{fig1teal}{�}\textcolor{fig1teal}{�}\textcolor{fig1teal}{re}
\end{samplebox}
\caption{Posterior Refinement trajectory on OpenWebText, Sample-2.}
\label{fig:refine_sample_owt_2}
\end{figure}
\clearpage

\begin{figure}[H]
\centering
\begin{samplebox}{\normalfont\textbf{$\ours$ (PR), Round: 3} \hfill \normalfont\scriptsize \textcolor{darkgray}{Committed: \textbf{948/1024}}}
\scriptsize\linespread{0.9}\selectfont
\textcolor{fig1teal}{.}  \textcolor{fig1ochre}{Corn}\textcolor{fig1ochre}{nie} \textcolor{fig1teal}{is} \textcolor{fig1teal}{another} \textcolor{fig1ochre}{prick}\textcolor{fig1teal}{ly} \textcolor{fig1teal}{guy}\textcolor{fig1teal}{.} \textcolor{fig1teal}{He} \textcolor{fig1teal}{has} \textcolor{fig1teal}{been} \textcolor{fig1teal}{working} \textcolor{fig1teal}{on} \textcolor{fig1teal}{a} \textcolor{fig1ochre}{wine} \textcolor{fig1ochre}{operation} \textcolor{fig1teal}{for} \textcolor{fig1teal}{9} \textcolor{fig1teal}{years}\textcolor{fig1teal}{,} \textcolor{fig1teal}{shooting} \textcolor{fig1teal}{for} \textcolor{fig1teal}{new} \textcolor{fig1teal}{in} \textcolor{fig1teal}{different} \textcolor{fig1teal}{parts} \textcolor{fig1teal}{of} \textcolor{fig1teal}{New} \textcolor{fig1teal}{York}\textcolor{fig1teal}{.} \textcolor{fig1teal}{I} \textcolor{fig1teal}{think} \textcolor{fig1teal}{it}\textcolor{fig1teal}{'s} \textcolor{fig1teal}{just} \textcolor{fig1teal}{that} \textcolor{fig1teal}{after} \textcolor{fig1teal}{the} \textcolor{fig1teal}{recent} \textcolor{fig1teal}{project} \textcolor{fig1teal}{in} \textcolor{fig1teal}{front} \textcolor{fig1teal}{of} \textcolor{fig1teal}{school}\textcolor{fig1teal}{,} \textcolor{fig1teal}{many} \textcolor{fig1teal}{more} \textcolor{fig1teal}{of} \textcolor{fig1teal}{them} \textcolor{fig1ochre}{turned} \textcolor{fig1teal}{it} \textcolor{fig1teal}{into} \textcolor{fig1teal}{a} \textcolor{fig1teal}{line} \textcolor{fig1teal}{of} \textcolor{fig1ochre}{driving}\textcolor{fig1teal}{.}  \textcolor{fig1teal}{That} \textcolor{fig1teal}{night}\textcolor{fig1teal}{,} \textcolor{fig1teal}{Or}\textcolor{fig1ochre}{zan} \textcolor{fig1teal}{came} \textcolor{fig1teal}{to} \textcolor{fig1ochre}{crim}\textcolor{fig1teal}{p}\textcolor{fig1teal}{,} \textcolor{fig1teal}{with} \textcolor{fig1teal}{no} \textcolor{fig1teal}{return}\textcolor{fig1teal}{.} \textcolor{fig1teal}{In} \textcolor{fig1teal}{response} \textcolor{fig1teal}{to} \textcolor{fig1teal}{the} \textcolor{fig1teal}{press} \textcolor{fig1teal}{conference} \textcolor{fig1teal}{after} \textcolor{fig1teal}{the} \textcolor{fig1teal}{opening} \textcolor{fig1teal}{day}\textcolor{fig1teal}{,} \textcolor{fig1ochre}{Senator} \textcolor{fig1teal}{Andrew} \textcolor{fig1ochre}{Holmes} \textcolor{fig1teal}{(}\textcolor{fig1teal}{R}\textcolor{fig1teal}{-}\textcolor{fig1ochre}{FL}\textcolor{fig1teal}{)} \textcolor{fig1teal}{told} \textcolor{fig1teal}{me} \textcolor{fig1teal}{the} \textcolor{fig1teal}{news}\textcolor{fig1teal}{:}  \textcolor{fig1teal}{�}\textcolor{fig1teal}{�}\textcolor{fig1teal}{To} \textcolor{fig1teal}{find} \textcolor{fig1teal}{it} \textcolor{fig1teal}{in} \textcolor{fig1teal}{an} \textcolor{fig1teal}{office} \textcolor{fig1teal}{that} \textcolor{fig1teal}{is} \textcolor{fig1teal}{like} \textcolor{fig1teal}{a} \textcolor{fig1ochre}{cart} \textcolor{fig1teal}{—} \textcolor{fig1teal}{we} \textcolor{fig1teal}{still} \textcolor{fig1teal}{think} \textcolor{fig1teal}{of} \textcolor{fig1teal}{our} \textcolor{fig1teal}{business} \textcolor{fig1teal}{with} \textcolor{fig1teal}{the} \textcolor{fig1teal}{government}\textcolor{fig1teal}{.}\textcolor{fig1teal}{�}\textcolor{fig1teal}{�}  \textcolor{fig1teal}{It} \textcolor{fig1teal}{was} \textcolor{fig1teal}{the} \textcolor{fig1teal}{first} \textcolor{fig1teal}{police} \textcolor{fig1ochre}{order} \textcolor{fig1teal}{for} \textcolor{fig1teal}{an} \textcolor{fig1teal}{organization} \textcolor{fig1teal}{in} \textcolor{fig1teal}{this} \textcolor{fig1teal}{time} \textcolor{fig1teal}{of} \textcolor{fig1teal}{war}\textcolor{fig1teal}{.} \textcolor{fig1teal}{It} \textcolor{fig1teal}{was} \textcolor{fig1teal}{also} \textcolor{fig1teal}{that} \textcolor{fig1teal}{he} \textcolor{fig1teal}{came} \textcolor{fig1teal}{to} \textcolor{fig1teal}{open} \textcolor{fig1teal}{the} \textcolor{fig1teal}{business} \textcolor{fig1teal}{and} \textcolor{fig1teal}{its} \textcolor{fig1teal}{website}\textcolor{fig1teal}{,} \textcolor{fig1teal}{and} \textcolor{fig1teal}{then} \textcolor{fig1teal}{was}\textcolor{fig1teal}{,} \textcolor{fig1teal}{to} \textcolor{fig1teal}{cut} \textcolor{fig1teal}{short} \textcolor{fig1teal}{his} \textcolor{fig1ochre}{successor}\textcolor{fig1teal}{,} \textcolor{fig1ochre}{Kyle}\textcolor{fig1teal}{,} \textcolor{fig1teal}{who} \textcolor{fig1teal}{has} \textcolor{fig1teal}{served} \textcolor{fig1teal}{as} \textcolor{fig1teal}{software} \textcolor{fig1teal}{services} \textcolor{fig1teal}{for} \textcolor{fig1ochre}{IBM}\textcolor{fig1teal}{,} \textcolor{fig1ochre}{Google}\textcolor{fig1teal}{,} \textcolor{fig1teal}{and} \textcolor{fig1ochre}{Microsoft} \textcolor{fig1teal}{for} \textcolor{fig1teal}{decades}\textcolor{fig1teal}{.}  \textcolor{fig1teal}{�}\textcolor{fig1teal}{�}\textcolor{fig1teal}{There}\textcolor{fig1teal}{�}\textcolor{fig1teal}{�}\textcolor{fig1teal}{s} \textcolor{fig1teal}{a} \textcolor{fig1teal}{lot} \textcolor{fig1teal}{for} \textcolor{fig1teal}{me} \textcolor{fig1teal}{to} \textcolor{fig1teal}{do} \textcolor{fig1teal}{with} \textcolor{fig1teal}{myself} \textcolor{fig1teal}{that} \textcolor{fig1teal}{would} \textcolor{fig1teal}{help} \textcolor{fig1teal}{me}\textcolor{fig1teal}{.} \textcolor{fig1teal}{I} \textcolor{fig1teal}{think} \textcolor{fig1teal}{this} \textcolor{fig1teal}{is} \textcolor{fig1teal}{my} \textcolor{fig1teal}{first} \textcolor{fig1teal}{work}\textcolor{fig1teal}{,}\textcolor{fig1teal}{�}\textcolor{fig1teal}{�} \textcolor{fig1teal}{he} \textcolor{fig1teal}{says}\textcolor{fig1teal}{.} \textcolor{fig1teal}{�}\textcolor{fig1teal}{�}\textcolor{fig1teal}{These} \textcolor{fig1teal}{were} \textcolor{fig1teal}{the} \textcolor{fig1teal}{real} \textcolor{fig1ochre}{guys}\textcolor{fig1teal}{.} \textcolor{fig1teal}{If} \textcolor{fig1teal}{you} \textcolor{fig1teal}{have} \textcolor{fig1teal}{to} \textcolor{fig1teal}{see} \textcolor{fig1teal}{the} \textcolor{fig1teal}{new} \textcolor{fig1ochre}{JavaScript} \textcolor{fig1teal}{which} \textcolor{fig1teal}{was} \textcolor{fig1teal}{made} \textcolor{fig1teal}{to} \textcolor{fig1ochre}{replace} \textcolor{fig1teal}{all} \textcolor{fig1teal}{of} \textcolor{fig1teal}{him}\textcolor{fig1teal}{,} \textcolor{fig1teal}{ask} \textcolor{fig1teal}{one} \textcolor{fig1teal}{of} \textcolor{fig1teal}{you} \textcolor{fig1teal}{to} \textcolor{fig1teal}{think} \textcolor{fig1teal}{about} \textcolor{fig1teal}{what} \textcolor{fig1teal}{each} \textcolor{fig1teal}{other} \textcolor{fig1teal}{work} \textcolor{fig1teal}{could} \textcolor{fig1teal}{show}\textcolor{fig1teal}{.}  \textcolor{fig1teal}{On} \textcolor{fig1teal}{New} \textcolor{fig1teal}{York}\textcolor{fig1teal}{�}\textcolor{fig1teal}{�}\textcolor{fig1teal}{s} \textcolor{fig1teal}{State} \textcolor{fig1teal}{House} \textcolor{fig1teal}{and} \textcolor{fig1ochre}{Mid}\textcolor{fig1teal}{-}\textcolor{fig1teal}{A} \textcolor{fig1ochre}{syndrome}\textcolor{fig1teal}{,} \textcolor{fig1teal}{his} \textcolor{fig1teal}{work} \textcolor{fig1teal}{is} \textcolor{fig1teal}{focused} \textcolor{fig1teal}{on} \textcolor{fig1teal}{money} \textcolor{fig1teal}{and} \textcolor{fig1teal}{growth}\textcolor{fig1teal}{.} \textcolor{fig1teal}{In} \textcolor{fig1teal}{the} \textcolor{fig1teal}{public} \textcolor{fig1teal}{health} \textcolor{fig1teal}{system}\textcolor{fig1teal}{,} \textcolor{fig1teal}{for} \textcolor{fig1teal}{him} \textcolor{fig1teal}{to} \textcolor{fig1teal}{get} \textcolor{fig1teal}{world}\textcolor{fig1teal}{-}\textcolor{fig1teal}{class} \textcolor{fig1teal}{take} \textcolor{fig1teal}{of} \textcolor{fig1teal}{a} \textcolor{fig1ochre}{charity} \textcolor{fig1teal}{and} \textcolor{fig1teal}{money} \textcolor{fig1teal}{from} \textcolor{fig1teal}{a} \textcolor{fig1teal}{friend} \textcolor{fig1teal}{to} \textcolor{fig1ochre}{invest}\textcolor{fig1teal}{,} \textcolor{fig1teal}{there} \textcolor{fig1teal}{are} \textcolor{fig1teal}{people} \textcolor{fig1teal}{in} \textcolor{fig1teal}{on} \textcolor{fig1teal}{life}\textcolor{fig1teal}{s} \textcolor{fig1teal}{of} \textcolor{fig1teal}{days}\textcolor{fig1teal}{,} \textcolor{fig1teal}{and} \textcolor{fig1ochre}{counting} \textcolor{fig1teal}{fine}\textcolor{fig1teal}{,} \textcolor{fig1teal}{but} \textcolor{fig1teal}{not} \textcolor{fig1teal}{enough}\textcolor{fig1teal}{,} \textcolor{fig1teal}{these} \textcolor{fig1teal}{days}\textcolor{fig1teal}{.}  \textcolor{fig1teal}{It} \textcolor{fig1teal}{was} \textcolor{fig1teal}{because} \textcolor{fig1teal}{of} \textcolor{fig1teal}{the} \textcolor{fig1teal}{amount} \textcolor{fig1teal}{of} \textcolor{fig1teal}{real} \textcolor{fig1teal}{money} \textcolor{fig1teal}{and} \textcolor{fig1teal}{the} \textcolor{fig1teal}{day} \textcolor{fig1teal}{off} \textcolor{fig1teal}{the} \textcolor{fig1teal}{company}\textcolor{fig1teal}{�}\textcolor{fig1teal}{�}\textcolor{fig1teal}{s} \textcolor{fig1teal}{office}\textcolor{fig1teal}{.} \textcolor{fig1ochre}{Interestingly}\textcolor{fig1teal}{,} \textcolor{fig1teal}{the} \textcolor{fig1teal}{use} \textcolor{fig1teal}{of} \textcolor{fig1ochre}{ric}\textcolor{fig1teal}{in} \textcolor{fig1teal}{was} \textcolor{fig1teal}{personal}\textcolor{fig1teal}{,} \textcolor{fig1teal}{because} \textcolor{fig1teal}{many} \textcolor{fig1teal}{reform}\textcolor{fig1teal}{ist}\textcolor{fig1ochre}{encies} \textcolor{fig1teal}{were} \textcolor{fig1teal}{typical} \textcolor{fig1teal}{from} \textcolor{fig1teal}{last} \textcolor{fig1teal}{year}\textcolor{fig1teal}{.} \textcolor{fig1teal}{40} \textcolor{fig1teal}{percent} \textcolor{fig1teal}{of} \textcolor{fig1teal}{the} \textcolor{fig1teal}{number} \textcolor{fig1teal}{of} \textcolor{fig1ochre}{illicit} \textcolor{fig1teal}{drug} \textcolor{fig1teal}{use}\textcolor{fig1teal}{,} \textcolor{fig1teal}{as} \textcolor{fig1teal}{well} \textcolor{fig1teal}{as} \textcolor{fig1teal}{money} \textcolor{fig1ochre}{laundering}\textcolor{fig1teal}{.} \textcolor{fig1teal}{and} \textcolor{fig1ochre}{prospect}\textcolor{fig1teal}{ing} \textcolor{fig1teal}{were} \textcolor{fig1teal}{all} \textcolor{fig1teal}{21}\textcolor{fig1teal}{-}\textcolor{fig1teal}{year}\textcolor{fig1teal}{-}\textcolor{fig1teal}{old} \textcolor{fig1teal}{students} \textcolor{fig1teal}{on} \textcolor{fig1teal}{their} \textcolor{fig1teal}{own}\textcolor{fig1teal}{.}  \textcolor{fig1teal}{�}\textcolor{fig1teal}{�}\textcolor{fig1teal}{I} \textcolor{fig1teal}{think} \textcolor{fig1teal}{with} \textcolor{fig1teal}{all} \textcolor{fig1teal}{the} \textcolor{fig1teal}{big} \textcolor{fig1teal}{data} \textcolor{fig1teal}{since} \textcolor{fig1teal}{I}\textcolor{fig1teal}{�}\textcolor{fig1teal}{�}\textcolor{fig1teal}{ve} \textcolor{fig1teal}{worked} \textcolor{fig1teal}{with} \textcolor{fig1teal}{David} \textcolor{fig1teal}{Sh}\textcolor{fig1teal}{um}\textcolor{fig1teal}{,} \textcolor{fig1teal}{the} \textcolor{fig1teal}{amount} \textcolor{fig1teal}{of} \textcolor{fig1teal}{energy} \textcolor{fig1teal}{sites} \textcolor{fig1teal}{or} \textcolor{fig1teal}{what} \textcolor{fig1teal}{I} \textcolor{fig1teal}{found} \textcolor{fig1teal}{there} \textcolor{fig1teal}{was}\textcolor{fig1teal}{;} \textcolor{fig1teal}{a} \textcolor{fig1teal}{lot} \textcolor{fig1teal}{in} \textcolor{fig1teal}{which} \textcolor{fig1teal}{a} \textcolor{fig1teal}{lot} \textcolor{fig1teal}{was} \textcolor{fig1teal}{more} \textcolor{fig1teal}{at} \textcolor{fig1teal}{people} \textcolor{fig1teal}{other} \textcolor{fig1teal}{than} \textcolor{fig1teal}{David}\textcolor{fig1teal}{�}\textcolor{fig1teal}{�}\textcolor{fig1teal}{s}\textcolor{fig1teal}{,} \textcolor{fig1teal}{but} \textcolor{fig1teal}{I} \textcolor{fig1teal}{know}\textcolor{fig1teal}{,} \textcolor{fig1teal}{my} \textcolor{fig1teal}{brain} \textcolor{fig1teal}{want} \textcolor{fig1teal}{us} \textcolor{fig1teal}{to} \textcolor{fig1teal}{say} \textcolor{fig1teal}{that}\textcolor{fig1teal}{,}\textcolor{fig1teal}{�}\textcolor{fig1teal}{�} \textcolor{fig1teal}{Se}\textcolor{fig1ochre}{anke} \textcolor{fig1teal}{said}\textcolor{fig1teal}{.} \textcolor{fig1teal}{�}\textcolor{fig1teal}{�}\textcolor{fig1teal}{In} \textcolor{fig1teal}{fact}\textcolor{fig1teal}{,} \textcolor{fig1teal}{I}\textcolor{fig1teal}{�}\textcolor{fig1teal}{�}\textcolor{fig1teal}{m} \textcolor{fig1teal}{not} \textcolor{fig1teal}{sure} \textcolor{fig1teal}{I} \textcolor{fig1teal}{really}\textcolor{fig1teal}{�}\textcolor{fig1teal}{�}\textcolor{fig1teal}{m} \textcolor{fig1teal}{going} \textcolor{fig1teal}{to} \textcolor{fig1teal}{do} \textcolor{fig1teal}{that} \textcolor{fig1teal}{money} \textcolor{fig1teal}{in} \textcolor{fig1ochre}{David}\textcolor{fig1teal}{.} \textcolor{fig1teal}{I} \textcolor{fig1teal}{was} \textcolor{fig1teal}{going} \textcolor{fig1teal}{to} \textcolor{fig1teal}{do} \textcolor{fig1teal}{this} \textcolor{fig1teal}{to} \textcolor{fig1teal}{schools} \textcolor{fig1teal}{and} \textcolor{fig1teal}{so} \textcolor{fig1teal}{other} \textcolor{fig1teal}{areas} \textcolor{fig1teal}{don}\textcolor{fig1teal}{�}\textcolor{fig1teal}{�}\textcolor{fig1teal}{t} \textcolor{fig1teal}{give} \textcolor{fig1teal}{them} \textcolor{fig1teal}{the} \textcolor{fig1teal}{data}\textcolor{fig1teal}{.}\textcolor{fig1teal}{�}\textcolor{fig1teal}{�}  \textcolor{fig1teal}{The} \textcolor{fig1ochre}{share} \textcolor{fig1teal}{didn}\textcolor{fig1teal}{�}\textcolor{fig1teal}{�}\textcolor{fig1teal}{t} \textcolor{fig1teal}{fall} \textcolor{fig1teal}{to} \textcolor{fig1teal}{him}\textcolor{fig1teal}{,} \textcolor{fig1teal}{and} \textcolor{fig1teal}{he} \textcolor{fig1teal}{used} \textcolor{fig1teal}{to} \textcolor{fig1teal}{show} \textcolor{fig1teal}{that} \textcolor{fig1teal}{he} \textcolor{fig1teal}{could} \textcolor{fig1teal}{make} \textcolor{fig1teal}{people} \textcolor{fig1ochre}{reached} \textcolor{fig1teal}{to} \textcolor{fig1teal}{deal} \textcolor{fig1teal}{with} \textcolor{fig1teal}{the} \textcolor{fig1teal}{financial} \textcolor{fig1teal}{crisis}\textcolor{fig1teal}{.}  \textcolor{fig1teal}{�}\textcolor{fig1teal}{�}\textcolor{fig1teal}{One} \textcolor{fig1teal}{of} \textcolor{fig1teal}{the} \textcolor{fig1teal}{individuals} \textcolor{fig1teal}{had} \textcolor{fig1teal}{me} \textcolor{fig1teal}{on} \textcolor{fig1teal}{the} \textcolor{fig1teal}{back} \textcolor{fig1teal}{and} \textcolor{fig1teal}{I} \textcolor{fig1teal}{was} \textcolor{fig1teal}{looking} \textcolor{fig1teal}{at} \textcolor{fig1teal}{what} \textcolor{fig1teal}{was} \textcolor{fig1teal}{no} \textcolor{fig1teal}{deal}\textcolor{fig1teal}{,}\textcolor{fig1teal}{�}\textcolor{fig1teal}{�} \textcolor{fig1teal}{I} \textcolor{fig1teal}{was} \textcolor{fig1teal}{told}\textcolor{fig1teal}{,} \textcolor{fig1teal}{Se}\textcolor{fig1teal}{at} \textcolor{fig1teal}{looking} \textcolor{fig1teal}{at} \textcolor{fig1teal}{the} \textcolor{fig1teal}{company}\textcolor{fig1teal}{�}\textcolor{fig1teal}{�}\textcolor{fig1teal}{s} \textcolor{fig1ochre}{father}\textcolor{fig1teal}{-}\textcolor{fig1teal}{in}\textcolor{fig1teal}{-}\textcolor{fig1ochre}{law}\textcolor{fig1teal}{,} \textcolor{fig1teal}{looking} \textcolor{fig1teal}{against} \textcolor{fig1teal}{him}\textcolor{fig1teal}{.} \textcolor{fig1teal}{And} \textcolor{fig1teal}{I} \textcolor{fig1teal}{asked} \textcolor{fig1teal}{him} \textcolor{fig1teal}{into} \textcolor{fig1teal}{a} \textcolor{fig1teal}{New} \textcolor{fig1teal}{York} \textcolor{fig1ochre}{technical} \textcolor{fig1teal}{office} \textcolor{fig1teal}{at} \textcolor{fig1teal}{home}\textcolor{fig1teal}{,} \textcolor{fig1teal}{at} \textcolor{fig1ochre}{Rockefeller} \textcolor{fig1teal}{Center} \textcolor{fig1teal}{—} \textcolor{fig1teal}{something} \textcolor{fig1teal}{she} \textcolor{fig1teal}{said} \textcolor{fig1teal}{had} \textcolor{fig1teal}{been} \textcolor{fig1teal}{very} \textcolor{fig1teal}{effective} \textcolor{fig1teal}{—} \textcolor{fig1teal}{and} \textcolor{fig1teal}{there} \textcolor{fig1teal}{was} \textcolor{fig1teal}{some} \textcolor{fig1teal}{competition}\textcolor{fig1teal}{.} \textcolor{fig1teal}{For} \textcolor{fig1teal}{the} \textcolor{fig1teal}{first} \textcolor{fig1teal}{time}\textcolor{fig1teal}{,} \textcolor{fig1teal}{I} \textcolor{fig1teal}{had} \textcolor{fig1teal}{to} \textcolor{fig1teal}{take} \textcolor{fig1teal}{off} \textcolor{fig1teal}{the} \textcolor{fig1teal}{same} \textcolor{fig1ochre}{shirt} \textcolor{fig1teal}{and} \textcolor{fig1teal}{ch}\textcolor{fig1ochre}{atted}\textcolor{fig1teal}{.} \textcolor{fig1teal}{This} \textcolor{fig1teal}{didn}\textcolor{fig1teal}{�}\textcolor{fig1teal}{�}\textcolor{fig1teal}{t} \textcolor{fig1teal}{change} \textcolor{fig1teal}{as} \textcolor{fig1teal}{we} \textcolor{fig1teal}{continued} \textcolor{fig1teal}{on} \textcolor{fig1teal}{with} \textcolor{fig1teal}{Se}\textcolor{fig1teal}{at}\textcolor{fig1teal}{,} \textcolor{fig1teal}{another} \textcolor{fig1teal}{person}\textcolor{fig1teal}{,} \textcolor{fig1teal}{home} \textcolor{fig1teal}{e}\textcolor{fig1teal}{er} \textcolor{fig1teal}{and} \textcolor{fig1ochre}{anticipating} \textcolor{fig1teal}{any} \textcolor{fig1teal}{potential} \textcolor{fig1teal}{activity}\textcolor{fig1teal}{.}  \textcolor{fig1teal}{�}\textcolor{fig1teal}{�}\textcolor{fig1teal}{That}\textcolor{fig1teal}{�}\textcolor{fig1teal}{�}\textcolor{fig1teal}{s} \textcolor{fig1teal}{great}\textcolor{fig1teal}{.} \textcolor{fig1teal}{But} \textcolor{fig1teal}{I} \textcolor{fig1teal}{know} \textcolor{fig1teal}{that} \textcolor{fig1teal}{eventually} \textcolor{fig1teal}{you} \textcolor{fig1ochre}{came} \textcolor{fig1teal}{in} \textcolor{fig1teal}{and} \textcolor{fig1teal}{got} \textcolor{fig1teal}{a} \textcolor{fig1teal}{control} \textcolor{fig1teal}{and} \textcolor{fig1teal}{the} \textcolor{fig1teal}{work} \textcolor{fig1teal}{was} \textcolor{fig1teal}{over} \textcolor{fig1teal}{and} \textcolor{fig1teal}{it}\textcolor{fig1teal}{�}\textcolor{fig1teal}{�}\textcolor{fig1teal}{s} \textcolor{fig1teal}{always} \textcolor{fig1teal}{questions} \textcolor{fig1teal}{about} \textcolor{fig1teal}{the} \textcolor{fig1ochre}{revenues} \textcolor{fig1teal}{from} \textcolor{fig1teal}{per} \textcolor{fig1teal}{se}\textcolor{fig1teal}{.}\textcolor{fig1teal}{�}\textcolor{fig1teal}{�}  \textcolor{fig1ochre}{Did} \textcolor{fig1teal}{he}\textcolor{fig1teal}{?}  \textcolor{fig1teal}{The} \textcolor{fig1teal}{official} \textcolor{fig1ochre}{perception} \textcolor{fig1teal}{among} \textcolor{fig1teal}{the} \textcolor{fig1teal}{media} \textcolor{fig1teal}{was} \textcolor{fig1teal}{times} \textcolor{fig1teal}{but} \textcolor{fig1teal}{some} \textcolor{fig1teal}{(}\textcolor{fig1teal}{D}\textcolor{fig1teal}{-}\textcolor{fig1ochre}{wa} \textcolor{fig1teal}{not} \textcolor{fig1teal}{always} \textcolor{fig1teal}{a} \textcolor{fig1teal}{good} \textcolor{fig1teal}{one}\textcolor{fig1teal}{)} \textcolor{fig1teal}{of} \textcolor{fig1teal}{him}\textcolor{fig1teal}{.} \textcolor{fig1teal}{With} \textcolor{fig1teal}{such} \textcolor{fig1teal}{never} \textcolor{fig1teal}{the} \textcolor{fig1ochre}{prosecutor} \textcolor{fig1teal}{he} \textcolor{fig1teal}{is}\textcolor{fig1teal}{,} \textcolor{fig1teal}{Or}\textcolor{fig1ochre}{enberg} \textcolor{fig1teal}{said} \textcolor{fig1teal}{it} \textcolor{fig1teal}{was} \textcolor{fig1teal}{not} \textcolor{fig1teal}{easy} \textcolor{fig1teal}{to} \textcolor{fig1teal}{understand} \textcolor{fig1teal}{that} \textcolor{fig1teal}{an} \textcolor{fig1teal}{operation} \textcolor{fig1teal}{could} \textcolor{fig1teal}{get} \textcolor{fig1teal}{home} \textcolor{fig1teal}{to} \textcolor{fig1teal}{his} \textcolor{fig1teal}{first} \textcolor{fig1ochre}{column} \textcolor{fig1teal}{ever}\textcolor{fig1teal}{.} \textcolor{fig1teal}{The} \textcolor{fig1teal}{National} \textcolor{fig1teal}{Security} \textcolor{fig1ochre}{Council} \textcolor{fig1teal}{official}\textcolor{fig1teal}{,} \textcolor{fig1teal}{though}\textcolor{fig1teal}{,} \textcolor{fig1teal}{reported} \textcolor{fig1teal}{after} \textcolor{fig1teal}{a} \textcolor{fig1teal}{running} \textcolor{fig1teal}{for} \textcolor{fig1ochre}{governor} \textcolor{fig1teal}{that} \textcolor{fig1teal}{Or}\textcolor{fig1ochre}{lov} \textcolor{fig1teal}{said} \textcolor{fig1teal}{he} \textcolor{fig1teal}{was} \textcolor{fig1teal}{looking} \textcolor{fig1teal}{to} \textcolor{fig1teal}{the} \textcolor{fig1teal}{energy} \textcolor{fig1ochre}{committees} \textcolor{fig1teal}{for} \textcolor{fig1teal}{the} \textcolor{fig1teal}{rules} \textcolor{fig1teal}{in} \textcolor{fig1teal}{the} \textcolor{fig1ochre}{boiler} \textcolor{fig1teal}{and} \textcolor{fig1teal}{the} \textcolor{fig1teal}{business}\textcolor{fig1teal}{,} \textcolor{fig1teal}{changing} \textcolor{fig1teal}{them} \textcolor{fig1teal}{to} \textcolor{fig1teal}{fit} \textcolor{fig1teal}{future} \textcolor{fig1ochre}{trends}\textcolor{fig1teal}{,} \textcolor{fig1teal}{and} \textcolor{fig1teal}{to} \textcolor{fig1ochre}{trim} \textcolor{fig1teal}{the} \textcolor{fig1teal}{size} \textcolor{fig1teal}{of} \textcolor{fig1teal}{the} \textcolor{fig1ochre}{rolls}\textcolor{fig1teal}{.} \textcolor{fig1teal}{At} \textcolor{fig1teal}{what} \textcolor{fig1teal}{was} \textcolor{fig1teal}{in} \textcolor{fig1teal}{a} \textcolor{fig1teal}{meeting} \textcolor{fig1teal}{three} \textcolor{fig1teal}{hours} \textcolor{fig1teal}{early} \textcolor{fig1teal}{by} \textcolor{fig1teal}{my} \textcolor{fig1teal}{office}\textcolor{fig1teal}{,} \textcolor{fig1teal}{she} \textcolor{fig1teal}{said} \textcolor{fig1teal}{she} \textcolor{fig1teal}{had} \textcolor{fig1teal}{no} \textcolor{fig1teal}{access} \textcolor{fig1teal}{to} \textcolor{fig1teal}{the} \textcolor{fig1teal}{data}\textcolor{fig1teal}{.} \textcolor{fig1teal}{What} \textcolor{fig1teal}{—} \textcolor{fig1teal}{what} \textcolor{fig1teal}{the}\textcolor{fig1ochre}{acebook}\textcolor{fig1teal}{,} \textcolor{fig1teal}{Se}\textcolor{fig1teal}{at} \textcolor{fig1teal}{told} \textcolor{fig1teal}{me} \textcolor{fig1teal}{—} \textcolor{fig1teal}{was} \textcolor{fig1teal}{�}\textcolor{fig1teal}{�}\textcolor{fig1teal}{I} \textcolor{fig1teal}{did} \textcolor{fig1teal}{not} \textcolor{fig1teal}{do} \textcolor{fig1teal}{that} \textcolor{fig1teal}{before} \textcolor{fig1teal}{I} \textcolor{fig1teal}{opened} \textcolor{fig1teal}{that} \textcolor{fig1ochre}{office}\textcolor{fig1teal}{.}\textcolor{fig1teal}{�}\textcolor{fig1teal}{�}  \textcolor{fig1teal}{To} \textcolor{fig1teal}{have} \textcolor{fig1teal}{one} \textcolor{fig1teal}{of} \textcolor{fig1teal}{these} \textcolor{fig1teal}{know}\textcolor{fig1teal}{-}\textcolor{fig1ochre}{how} \textcolor{fig1teal}{years} \textcolor{fig1teal}{in} \textcolor{fig1teal}{line} \textcolor{fig1teal}{for} \textcolor{fig1teal}{your} \textcolor{fig1teal}{life}\textcolor{fig1teal}{,} \textcolor{fig1teal}{he} \textcolor{fig1teal}{had} \textcolor{fig1teal}{to} \textcolor{fig1teal}{start} \textcolor{fig1teal}{getting} \textcolor{fig1teal}{every} \textcolor{fig1teal}{day} \textcolor{fig1teal}{in} \textcolor{fig1teal}{work} \textcolor{fig1teal}{for} \textcolor{fig1teal}{something} \textcolor{fig1teal}{of} \textcolor{fig1ochre}{sale} \textcolor{fig1teal}{with} \textcolor{fig1teal}{a} \textcolor{fig1teal}{few} \textcolor{fig1ochre}{ounces} \textcolor{fig1teal}{of} \textcolor{fig1ochre}{Es}\textcolor{fig1teal}{at} \textcolor{fig1teal}{here} \textcolor{fig1teal}{that}\textcolor{fig1teal}{�}\textcolor{fig1teal}{�}\textcolor{fig1teal}{s} \textcolor{fig1teal}{been} \textcolor{fig1teal}{tax} \textcolor{fig1teal}{for} \textcolor{fig1teal}{every} \textcolor{fig1teal}{type} \textcolor{fig1teal}{of} \textcolor{fig1teal}{activity} \textcolor{fig1teal}{by} \textcolor{fig1teal}{other} \textcolor{fig1teal}{people}\textcolor{fig1teal}{,} \textcolor{fig1teal}{including} \textcolor{fig1ochre}{criminal}\textcolor{fig1teal}{,} \textcolor{fig1teal}{in} \textcolor{fig1teal}{an} \textcolor{fig1ochre}{enterprise} \textcolor{fig1teal}{and} \textcolor{fig1teal}{by} \textcolor{fig1teal}{a} \textcolor{fig1teal}{company} \textcolor{fig1teal}{and} \textcolor{fig1ochre}{lump}\textcolor{fig1teal}{ed} \textcolor{fig1teal}{in} \textcolor{fig1teal}{with} \textcolor{fig1teal}{criminal} \textcolor{fig1teal}{activity}\textcolor{fig1teal}{,} \textcolor{fig1teal}{for} \textcolor{fig1teal}{which} \textcolor{fig1teal}{he} \textcolor{fig1ochre}{cited} \textcolor{fig1teal}{as} \textcolor{fig1teal}{a} \textcolor{fig1ochre}{trial} \textcolor{fig1teal}{day}\textcolor{fig1teal}{.}  \textcolor{fig1teal}{That}\textcolor{fig1teal}{,} \textcolor{fig1teal}{and} \textcolor{fig1teal}{the} \textcolor{fig1teal}{type} \textcolor{fig1teal}{of} \textcolor{fig1ochre}{remaining} \textcolor{fig1teal}{he} \textcolor{fig1teal}{has} \textcolor{fig1teal}{is} \textcolor{fig1teal}{one} \textcolor{fig1teal}{of} \textcolor{fig1teal}{the} \textcolor{fig1teal}{company}\textcolor{fig1teal}{�}\textcolor{fig1teal}{�}\textcolor{fig1teal}{s} \textcolor{fig1teal}{most} \textcolor{fig1teal}{important} \textcolor{fig1teal}{part} \textcolor{fig1teal}{of} \textcolor{fig1teal}{business} \textcolor{fig1teal}{now}\textcolor{fig1teal}{.}  \textcolor{fig1teal}{One} \textcolor{fig1teal}{of} \textcolor{fig1teal}{his} \textcolor{fig1ochre}{closest} \textcolor{fig1ochre}{clients} \textcolor{fig1teal}{told} \textcolor{fig1teal}{me} \textcolor{fig1teal}{that} \textcolor{fig1teal}{a} \textcolor{fig1teal}{possible} \textcolor{fig1teal}{run} \textcolor{fig1teal}{for} \textcolor{fig1teal}{the} \textcolor{fig1teal}{oil} \textcolor{fig1teal}{company} \textcolor{fig1teal}{was} \textcolor{fig1teal}{not} \textcolor{fig1teal}{too} \textcolor{fig1teal}{possible}\textcolor{fig1teal}{.} \textcolor{fig1teal}{He} \textcolor{fig1teal}{also} \textcolor{fig1teal}{told} \textcolor{fig1teal}{The} \textcolor{fig1teal}{Hill}\textcolor{fig1teal}{�}\textcolor{fig1teal}{�}\textcolor{fig1teal}{s} \textcolor{fig1ochre}{Peter} \textcolor{fig1teal}{O}\textcolor{fig1ochre}{plan}\textcolor{fig1teal}{,} \textcolor{fig1teal}{�}\textcolor{fig1teal}{�}\textcolor{fig1teal}{You}\textcolor{fig1teal}{�}\textcolor{fig1teal}{�}\textcolor{fig1teal}{re}
\end{samplebox}
\begin{samplebox}{\normalfont\textbf{$\ours$ (PR), Round: 4} \hfill \normalfont\scriptsize \textcolor{darkgray}{Committed: \textbf{1024/1024}}}
\scriptsize\linespread{0.9}\selectfont
\textcolor{fig1teal}{.}  \textcolor{fig1teal}{Or}\textcolor{fig1teal}{chid} \textcolor{fig1teal}{is} \textcolor{fig1teal}{another} \textcolor{fig1teal}{Death}\textcolor{fig1teal}{ly} \textcolor{fig1teal}{guy}\textcolor{fig1teal}{.} \textcolor{fig1teal}{He} \textcolor{fig1teal}{has} \textcolor{fig1teal}{been} \textcolor{fig1teal}{working} \textcolor{fig1teal}{on} \textcolor{fig1teal}{a} \textcolor{fig1teal}{small} \textcolor{fig1teal}{field} \textcolor{fig1teal}{for} \textcolor{fig1teal}{9} \textcolor{fig1teal}{years}\textcolor{fig1teal}{,} \textcolor{fig1teal}{shooting} \textcolor{fig1teal}{for} \textcolor{fig1teal}{new} \textcolor{fig1teal}{in} \textcolor{fig1teal}{different} \textcolor{fig1teal}{parts} \textcolor{fig1teal}{of} \textcolor{fig1teal}{New} \textcolor{fig1teal}{York}\textcolor{fig1teal}{.} \textcolor{fig1teal}{I} \textcolor{fig1teal}{think} \textcolor{fig1teal}{it}\textcolor{fig1teal}{'s} \textcolor{fig1teal}{just} \textcolor{fig1teal}{that} \textcolor{fig1teal}{after} \textcolor{fig1teal}{the} \textcolor{fig1teal}{recent} \textcolor{fig1teal}{project} \textcolor{fig1teal}{in} \textcolor{fig1teal}{front} \textcolor{fig1teal}{of} \textcolor{fig1teal}{school}\textcolor{fig1teal}{,} \textcolor{fig1teal}{many} \textcolor{fig1teal}{more} \textcolor{fig1teal}{of} \textcolor{fig1teal}{them} \textcolor{fig1teal}{threw} \textcolor{fig1teal}{it} \textcolor{fig1teal}{into} \textcolor{fig1teal}{a} \textcolor{fig1teal}{line} \textcolor{fig1teal}{of} \textcolor{fig1teal}{business}\textcolor{fig1teal}{.}  \textcolor{fig1teal}{That} \textcolor{fig1teal}{night}\textcolor{fig1teal}{,} \textcolor{fig1teal}{Or}\textcolor{fig1teal}{tha} \textcolor{fig1teal}{came} \textcolor{fig1teal}{to} \textcolor{fig1teal}{Dor}\textcolor{fig1teal}{p}\textcolor{fig1teal}{,} \textcolor{fig1teal}{with} \textcolor{fig1teal}{no} \textcolor{fig1teal}{return}\textcolor{fig1teal}{.} \textcolor{fig1teal}{In} \textcolor{fig1teal}{response} \textcolor{fig1teal}{to} \textcolor{fig1teal}{the} \textcolor{fig1teal}{press} \textcolor{fig1teal}{conference} \textcolor{fig1teal}{after} \textcolor{fig1teal}{the} \textcolor{fig1teal}{opening} \textcolor{fig1teal}{day}\textcolor{fig1teal}{,} \textcolor{fig1teal}{Mayor} \textcolor{fig1teal}{Andrew} \textcolor{fig1teal}{Cuomo} \textcolor{fig1teal}{(}\textcolor{fig1teal}{R}\textcolor{fig1teal}{-}\textcolor{fig1teal}{CA}\textcolor{fig1teal}{)} \textcolor{fig1teal}{told} \textcolor{fig1teal}{me} \textcolor{fig1teal}{the} \textcolor{fig1teal}{news}\textcolor{fig1teal}{:}  \textcolor{fig1teal}{�}\textcolor{fig1teal}{�}\textcolor{fig1teal}{To} \textcolor{fig1teal}{find} \textcolor{fig1teal}{it} \textcolor{fig1teal}{in} \textcolor{fig1teal}{an} \textcolor{fig1teal}{office} \textcolor{fig1teal}{that} \textcolor{fig1teal}{is} \textcolor{fig1teal}{like} \textcolor{fig1teal}{a} \textcolor{fig1teal}{university} \textcolor{fig1teal}{—} \textcolor{fig1teal}{we} \textcolor{fig1teal}{still} \textcolor{fig1teal}{think} \textcolor{fig1teal}{of} \textcolor{fig1teal}{our} \textcolor{fig1teal}{business} \textcolor{fig1teal}{with} \textcolor{fig1teal}{the} \textcolor{fig1teal}{government}\textcolor{fig1teal}{.}\textcolor{fig1teal}{�}\textcolor{fig1teal}{�}  \textcolor{fig1teal}{It} \textcolor{fig1teal}{was} \textcolor{fig1teal}{the} \textcolor{fig1teal}{first} \textcolor{fig1teal}{police} \textcolor{fig1teal}{raid} \textcolor{fig1teal}{for} \textcolor{fig1teal}{an} \textcolor{fig1teal}{organization} \textcolor{fig1teal}{in} \textcolor{fig1teal}{this} \textcolor{fig1teal}{time} \textcolor{fig1teal}{of} \textcolor{fig1teal}{war}\textcolor{fig1teal}{.} \textcolor{fig1teal}{It} \textcolor{fig1teal}{was} \textcolor{fig1teal}{also} \textcolor{fig1teal}{that} \textcolor{fig1teal}{he} \textcolor{fig1teal}{came} \textcolor{fig1teal}{to} \textcolor{fig1teal}{open} \textcolor{fig1teal}{the} \textcolor{fig1teal}{business} \textcolor{fig1teal}{and} \textcolor{fig1teal}{its} \textcolor{fig1teal}{website}\textcolor{fig1teal}{,} \textcolor{fig1teal}{and} \textcolor{fig1teal}{then} \textcolor{fig1teal}{was}\textcolor{fig1teal}{,} \textcolor{fig1teal}{to} \textcolor{fig1teal}{cut} \textcolor{fig1teal}{short} \textcolor{fig1teal}{his} \textcolor{fig1teal}{child}\textcolor{fig1teal}{,} \textcolor{fig1teal}{Charles}\textcolor{fig1teal}{,} \textcolor{fig1teal}{who} \textcolor{fig1teal}{has} \textcolor{fig1teal}{served} \textcolor{fig1teal}{as} \textcolor{fig1teal}{software} \textcolor{fig1teal}{services} \textcolor{fig1teal}{for} \textcolor{fig1teal}{Google}\textcolor{fig1teal}{,} \textcolor{fig1teal}{Shell}\textcolor{fig1teal}{,} \textcolor{fig1teal}{and} \textcolor{fig1teal}{broadcasting} \textcolor{fig1teal}{for} \textcolor{fig1teal}{decades}\textcolor{fig1teal}{.}  \textcolor{fig1teal}{�}\textcolor{fig1teal}{�}\textcolor{fig1teal}{There}\textcolor{fig1teal}{�}\textcolor{fig1teal}{�}\textcolor{fig1teal}{s} \textcolor{fig1teal}{a} \textcolor{fig1teal}{lot} \textcolor{fig1teal}{for} \textcolor{fig1teal}{me} \textcolor{fig1teal}{to} \textcolor{fig1teal}{do} \textcolor{fig1teal}{with} \textcolor{fig1teal}{myself} \textcolor{fig1teal}{that} \textcolor{fig1teal}{would} \textcolor{fig1teal}{help} \textcolor{fig1teal}{me}\textcolor{fig1teal}{.} \textcolor{fig1teal}{I} \textcolor{fig1teal}{think} \textcolor{fig1teal}{this} \textcolor{fig1teal}{is} \textcolor{fig1teal}{my} \textcolor{fig1teal}{first} \textcolor{fig1teal}{work}\textcolor{fig1teal}{,}\textcolor{fig1teal}{�}\textcolor{fig1teal}{�} \textcolor{fig1teal}{he} \textcolor{fig1teal}{says}\textcolor{fig1teal}{.} \textcolor{fig1teal}{�}\textcolor{fig1teal}{�}\textcolor{fig1teal}{These} \textcolor{fig1teal}{were} \textcolor{fig1teal}{the} \textcolor{fig1teal}{real} \textcolor{fig1teal}{ones}\textcolor{fig1teal}{.} \textcolor{fig1teal}{If} \textcolor{fig1teal}{you} \textcolor{fig1teal}{have} \textcolor{fig1teal}{to} \textcolor{fig1teal}{see} \textcolor{fig1teal}{the} \textcolor{fig1teal}{new} \textcolor{fig1teal}{culture} \textcolor{fig1teal}{which} \textcolor{fig1teal}{was} \textcolor{fig1teal}{made} \textcolor{fig1teal}{to} \textcolor{fig1teal}{help} \textcolor{fig1teal}{all} \textcolor{fig1teal}{of} \textcolor{fig1teal}{him}\textcolor{fig1teal}{,} \textcolor{fig1teal}{ask} \textcolor{fig1teal}{one} \textcolor{fig1teal}{of} \textcolor{fig1teal}{you} \textcolor{fig1teal}{to} \textcolor{fig1teal}{think} \textcolor{fig1teal}{about} \textcolor{fig1teal}{what} \textcolor{fig1teal}{each} \textcolor{fig1teal}{other} \textcolor{fig1teal}{work} \textcolor{fig1teal}{could} \textcolor{fig1teal}{show}\textcolor{fig1teal}{.}  \textcolor{fig1teal}{On} \textcolor{fig1teal}{New} \textcolor{fig1teal}{York}\textcolor{fig1teal}{�}\textcolor{fig1teal}{�}\textcolor{fig1teal}{s} \textcolor{fig1teal}{State} \textcolor{fig1teal}{House} \textcolor{fig1teal}{and} \textcolor{fig1teal}{L}\textcolor{fig1teal}{-}\textcolor{fig1teal}{A} \textcolor{fig1teal}{Major}\textcolor{fig1teal}{,} \textcolor{fig1teal}{his} \textcolor{fig1teal}{work} \textcolor{fig1teal}{is} \textcolor{fig1teal}{focused} \textcolor{fig1teal}{on} \textcolor{fig1teal}{money} \textcolor{fig1teal}{and} \textcolor{fig1teal}{growth}\textcolor{fig1teal}{.} \textcolor{fig1teal}{In} \textcolor{fig1teal}{the} \textcolor{fig1teal}{public} \textcolor{fig1teal}{health} \textcolor{fig1teal}{system}\textcolor{fig1teal}{,} \textcolor{fig1teal}{for} \textcolor{fig1teal}{him} \textcolor{fig1teal}{to} \textcolor{fig1teal}{get} \textcolor{fig1teal}{world}\textcolor{fig1teal}{-}\textcolor{fig1teal}{class} \textcolor{fig1teal}{take} \textcolor{fig1teal}{of} \textcolor{fig1teal}{a} \textcolor{fig1teal}{job} \textcolor{fig1teal}{and} \textcolor{fig1teal}{money} \textcolor{fig1teal}{from} \textcolor{fig1teal}{a} \textcolor{fig1teal}{friend} \textcolor{fig1teal}{to} \textcolor{fig1teal}{visit}\textcolor{fig1teal}{,} \textcolor{fig1teal}{there} \textcolor{fig1teal}{are} \textcolor{fig1teal}{people} \textcolor{fig1teal}{in} \textcolor{fig1teal}{on} \textcolor{fig1teal}{life}\textcolor{fig1teal}{s} \textcolor{fig1teal}{of} \textcolor{fig1teal}{days}\textcolor{fig1teal}{,} \textcolor{fig1teal}{and} \textcolor{fig1teal}{just} \textcolor{fig1teal}{fine}\textcolor{fig1teal}{,} \textcolor{fig1teal}{but} \textcolor{fig1teal}{not} \textcolor{fig1teal}{enough}\textcolor{fig1teal}{,} \textcolor{fig1teal}{these} \textcolor{fig1teal}{days}\textcolor{fig1teal}{.}  \textcolor{fig1teal}{It} \textcolor{fig1teal}{was} \textcolor{fig1teal}{because} \textcolor{fig1teal}{of} \textcolor{fig1teal}{the} \textcolor{fig1teal}{amount} \textcolor{fig1teal}{of} \textcolor{fig1teal}{real} \textcolor{fig1teal}{money} \textcolor{fig1teal}{and} \textcolor{fig1teal}{the} \textcolor{fig1teal}{day} \textcolor{fig1teal}{off} \textcolor{fig1teal}{the} \textcolor{fig1teal}{company}\textcolor{fig1teal}{�}\textcolor{fig1teal}{�}\textcolor{fig1teal}{s} \textcolor{fig1teal}{office}\textcolor{fig1teal}{.} \textcolor{fig1teal}{Then}\textcolor{fig1teal}{,} \textcolor{fig1teal}{the} \textcolor{fig1teal}{use} \textcolor{fig1teal}{of} \textcolor{fig1teal}{Lup}\textcolor{fig1teal}{in} \textcolor{fig1teal}{was} \textcolor{fig1teal}{personal}\textcolor{fig1teal}{,} \textcolor{fig1teal}{because} \textcolor{fig1teal}{many} \textcolor{fig1teal}{reform}\textcolor{fig1teal}{ist} \textcolor{fig1teal}{students} \textcolor{fig1teal}{were} \textcolor{fig1teal}{typical} \textcolor{fig1teal}{from} \textcolor{fig1teal}{last} \textcolor{fig1teal}{year}\textcolor{fig1teal}{.} \textcolor{fig1teal}{40} \textcolor{fig1teal}{percent} \textcolor{fig1teal}{of} \textcolor{fig1teal}{the} \textcolor{fig1teal}{number} \textcolor{fig1teal}{of} \textcolor{fig1teal}{real} \textcolor{fig1teal}{drug} \textcolor{fig1teal}{use}\textcolor{fig1teal}{,} \textcolor{fig1teal}{as} \textcolor{fig1teal}{well} \textcolor{fig1teal}{as} \textcolor{fig1teal}{money} \textcolor{fig1teal}{laundering}\textcolor{fig1teal}{.} \textcolor{fig1teal}{and} \textcolor{fig1teal}{mill}\textcolor{fig1teal}{ing} \textcolor{fig1teal}{were} \textcolor{fig1teal}{all} \textcolor{fig1teal}{21}\textcolor{fig1teal}{-}\textcolor{fig1teal}{year}\textcolor{fig1teal}{-}\textcolor{fig1teal}{old} \textcolor{fig1teal}{students} \textcolor{fig1teal}{on} \textcolor{fig1teal}{their} \textcolor{fig1teal}{own}\textcolor{fig1teal}{.}  \textcolor{fig1teal}{�}\textcolor{fig1teal}{�}\textcolor{fig1teal}{I} \textcolor{fig1teal}{think} \textcolor{fig1teal}{with} \textcolor{fig1teal}{all} \textcolor{fig1teal}{the} \textcolor{fig1teal}{big} \textcolor{fig1teal}{data} \textcolor{fig1teal}{since} \textcolor{fig1teal}{I}\textcolor{fig1teal}{�}\textcolor{fig1teal}{�}\textcolor{fig1teal}{ve} \textcolor{fig1teal}{worked} \textcolor{fig1teal}{with} \textcolor{fig1teal}{David} \textcolor{fig1teal}{Sh}\textcolor{fig1teal}{um}\textcolor{fig1teal}{,} \textcolor{fig1teal}{the} \textcolor{fig1teal}{amount} \textcolor{fig1teal}{of} \textcolor{fig1teal}{energy} \textcolor{fig1teal}{sites} \textcolor{fig1teal}{or} \textcolor{fig1teal}{what} \textcolor{fig1teal}{I} \textcolor{fig1teal}{found} \textcolor{fig1teal}{there} \textcolor{fig1teal}{was}\textcolor{fig1teal}{;} \textcolor{fig1teal}{a} \textcolor{fig1teal}{lot} \textcolor{fig1teal}{in} \textcolor{fig1teal}{which} \textcolor{fig1teal}{a} \textcolor{fig1teal}{lot} \textcolor{fig1teal}{was} \textcolor{fig1teal}{more} \textcolor{fig1teal}{at} \textcolor{fig1teal}{people} \textcolor{fig1teal}{other} \textcolor{fig1teal}{than} \textcolor{fig1teal}{David}\textcolor{fig1teal}{�}\textcolor{fig1teal}{�}\textcolor{fig1teal}{s}\textcolor{fig1teal}{,} \textcolor{fig1teal}{but} \textcolor{fig1teal}{I} \textcolor{fig1teal}{know}\textcolor{fig1teal}{,} \textcolor{fig1teal}{my} \textcolor{fig1teal}{brain} \textcolor{fig1teal}{want} \textcolor{fig1teal}{us} \textcolor{fig1teal}{to} \textcolor{fig1teal}{say} \textcolor{fig1teal}{that}\textcolor{fig1teal}{,}\textcolor{fig1teal}{�}\textcolor{fig1teal}{�} \textcolor{fig1teal}{Se}\textcolor{fig1teal}{at} \textcolor{fig1teal}{said}\textcolor{fig1teal}{.} \textcolor{fig1teal}{�}\textcolor{fig1teal}{�}\textcolor{fig1teal}{In} \textcolor{fig1teal}{fact}\textcolor{fig1teal}{,} \textcolor{fig1teal}{I}\textcolor{fig1teal}{�}\textcolor{fig1teal}{�}\textcolor{fig1teal}{m} \textcolor{fig1teal}{not} \textcolor{fig1teal}{sure} \textcolor{fig1teal}{I} \textcolor{fig1teal}{really}\textcolor{fig1teal}{�}\textcolor{fig1teal}{�}\textcolor{fig1teal}{m} \textcolor{fig1teal}{going} \textcolor{fig1teal}{to} \textcolor{fig1teal}{do} \textcolor{fig1teal}{that} \textcolor{fig1teal}{money} \textcolor{fig1teal}{in} \textcolor{fig1teal}{something}\textcolor{fig1teal}{.} \textcolor{fig1teal}{I} \textcolor{fig1teal}{was} \textcolor{fig1teal}{going} \textcolor{fig1teal}{to} \textcolor{fig1teal}{do} \textcolor{fig1teal}{this} \textcolor{fig1teal}{to} \textcolor{fig1teal}{schools} \textcolor{fig1teal}{and} \textcolor{fig1teal}{so} \textcolor{fig1teal}{other} \textcolor{fig1teal}{areas} \textcolor{fig1teal}{don}\textcolor{fig1teal}{�}\textcolor{fig1teal}{�}\textcolor{fig1teal}{t} \textcolor{fig1teal}{give} \textcolor{fig1teal}{them} \textcolor{fig1teal}{the} \textcolor{fig1teal}{data}\textcolor{fig1teal}{.}\textcolor{fig1teal}{�}\textcolor{fig1teal}{�}  \textcolor{fig1teal}{The} \textcolor{fig1teal}{people} \textcolor{fig1teal}{didn}\textcolor{fig1teal}{�}\textcolor{fig1teal}{�}\textcolor{fig1teal}{t} \textcolor{fig1teal}{fall} \textcolor{fig1teal}{to} \textcolor{fig1teal}{him}\textcolor{fig1teal}{,} \textcolor{fig1teal}{and} \textcolor{fig1teal}{he} \textcolor{fig1teal}{used} \textcolor{fig1teal}{to} \textcolor{fig1teal}{show} \textcolor{fig1teal}{that} \textcolor{fig1teal}{he} \textcolor{fig1teal}{could} \textcolor{fig1teal}{make} \textcolor{fig1teal}{people} \textcolor{fig1teal}{want} \textcolor{fig1teal}{to} \textcolor{fig1teal}{deal} \textcolor{fig1teal}{with} \textcolor{fig1teal}{the} \textcolor{fig1teal}{financial} \textcolor{fig1teal}{crisis}\textcolor{fig1teal}{.}  \textcolor{fig1teal}{�}\textcolor{fig1teal}{�}\textcolor{fig1teal}{One} \textcolor{fig1teal}{of} \textcolor{fig1teal}{the} \textcolor{fig1teal}{individuals} \textcolor{fig1teal}{had} \textcolor{fig1teal}{me} \textcolor{fig1teal}{on} \textcolor{fig1teal}{the} \textcolor{fig1teal}{back} \textcolor{fig1teal}{and} \textcolor{fig1teal}{I} \textcolor{fig1teal}{was} \textcolor{fig1teal}{looking} \textcolor{fig1teal}{at} \textcolor{fig1teal}{what} \textcolor{fig1teal}{was} \textcolor{fig1teal}{no} \textcolor{fig1teal}{deal}\textcolor{fig1teal}{,}\textcolor{fig1teal}{�}\textcolor{fig1teal}{�} \textcolor{fig1teal}{I} \textcolor{fig1teal}{was} \textcolor{fig1teal}{told}\textcolor{fig1teal}{,} \textcolor{fig1teal}{Se}\textcolor{fig1teal}{at} \textcolor{fig1teal}{looking} \textcolor{fig1teal}{at} \textcolor{fig1teal}{the} \textcolor{fig1teal}{company}\textcolor{fig1teal}{�}\textcolor{fig1teal}{�}\textcolor{fig1teal}{s} \textcolor{fig1teal}{partner}\textcolor{fig1teal}{-}\textcolor{fig1teal}{in}\textcolor{fig1teal}{-}\textcolor{fig1teal}{chief}\textcolor{fig1teal}{,} \textcolor{fig1teal}{looking} \textcolor{fig1teal}{against} \textcolor{fig1teal}{him}\textcolor{fig1teal}{.} \textcolor{fig1teal}{And} \textcolor{fig1teal}{I} \textcolor{fig1teal}{asked} \textcolor{fig1teal}{him} \textcolor{fig1teal}{into} \textcolor{fig1teal}{a} \textcolor{fig1teal}{New} \textcolor{fig1teal}{York} \textcolor{fig1teal}{government} \textcolor{fig1teal}{office} \textcolor{fig1teal}{at} \textcolor{fig1teal}{home}\textcolor{fig1teal}{,} \textcolor{fig1teal}{at} \textcolor{fig1teal}{Rockefeller} \textcolor{fig1teal}{Center} \textcolor{fig1teal}{—} \textcolor{fig1teal}{something} \textcolor{fig1teal}{she} \textcolor{fig1teal}{said} \textcolor{fig1teal}{had} \textcolor{fig1teal}{been} \textcolor{fig1teal}{very} \textcolor{fig1teal}{effective} \textcolor{fig1teal}{—} \textcolor{fig1teal}{and} \textcolor{fig1teal}{there} \textcolor{fig1teal}{was} \textcolor{fig1teal}{some} \textcolor{fig1teal}{competition}\textcolor{fig1teal}{.} \textcolor{fig1teal}{For} \textcolor{fig1teal}{the} \textcolor{fig1teal}{first} \textcolor{fig1teal}{time}\textcolor{fig1teal}{,} \textcolor{fig1teal}{I} \textcolor{fig1teal}{had} \textcolor{fig1teal}{to} \textcolor{fig1teal}{take} \textcolor{fig1teal}{off} \textcolor{fig1teal}{the} \textcolor{fig1teal}{same} \textcolor{fig1teal}{shoes} \textcolor{fig1teal}{and} \textcolor{fig1teal}{ch}\textcolor{fig1teal}{ucks}\textcolor{fig1teal}{.} \textcolor{fig1teal}{This} \textcolor{fig1teal}{didn}\textcolor{fig1teal}{�}\textcolor{fig1teal}{�}\textcolor{fig1teal}{t} \textcolor{fig1teal}{change} \textcolor{fig1teal}{as} \textcolor{fig1teal}{we} \textcolor{fig1teal}{continued} \textcolor{fig1teal}{on} \textcolor{fig1teal}{with} \textcolor{fig1teal}{Se}\textcolor{fig1teal}{at}\textcolor{fig1teal}{,} \textcolor{fig1teal}{another} \textcolor{fig1teal}{person}\textcolor{fig1teal}{,} \textcolor{fig1teal}{home} \textcolor{fig1teal}{e}\textcolor{fig1teal}{er} \textcolor{fig1teal}{and} \textcolor{fig1teal}{dam} \textcolor{fig1teal}{any} \textcolor{fig1teal}{potential} \textcolor{fig1teal}{activity}\textcolor{fig1teal}{.}  \textcolor{fig1teal}{�}\textcolor{fig1teal}{�}\textcolor{fig1teal}{That}\textcolor{fig1teal}{�}\textcolor{fig1teal}{�}\textcolor{fig1teal}{s} \textcolor{fig1teal}{great}\textcolor{fig1teal}{.} \textcolor{fig1teal}{But} \textcolor{fig1teal}{I} \textcolor{fig1teal}{know} \textcolor{fig1teal}{that} \textcolor{fig1teal}{eventually} \textcolor{fig1teal}{you} \textcolor{fig1teal}{got} \textcolor{fig1teal}{in} \textcolor{fig1teal}{and} \textcolor{fig1teal}{got} \textcolor{fig1teal}{a} \textcolor{fig1teal}{control} \textcolor{fig1teal}{and} \textcolor{fig1teal}{the} \textcolor{fig1teal}{work} \textcolor{fig1teal}{was} \textcolor{fig1teal}{over} \textcolor{fig1teal}{and} \textcolor{fig1teal}{it}\textcolor{fig1teal}{�}\textcolor{fig1teal}{�}\textcolor{fig1teal}{s} \textcolor{fig1teal}{always} \textcolor{fig1teal}{questions} \textcolor{fig1teal}{about} \textcolor{fig1teal}{the} \textcolor{fig1teal}{intelligence} \textcolor{fig1teal}{from} \textcolor{fig1teal}{per} \textcolor{fig1teal}{se}\textcolor{fig1teal}{.}\textcolor{fig1teal}{�}\textcolor{fig1teal}{�}  \textcolor{fig1teal}{Is} \textcolor{fig1teal}{he}\textcolor{fig1teal}{?}  \textcolor{fig1teal}{The} \textcolor{fig1teal}{official} \textcolor{fig1teal}{reaction} \textcolor{fig1teal}{among} \textcolor{fig1teal}{the} \textcolor{fig1teal}{media} \textcolor{fig1teal}{was} \textcolor{fig1teal}{times} \textcolor{fig1teal}{but} \textcolor{fig1teal}{some} \textcolor{fig1teal}{(}\textcolor{fig1teal}{D}\textcolor{fig1teal}{-}\textcolor{fig1teal}{,} \textcolor{fig1teal}{not} \textcolor{fig1teal}{always} \textcolor{fig1teal}{a} \textcolor{fig1teal}{good} \textcolor{fig1teal}{one}\textcolor{fig1teal}{)} \textcolor{fig1teal}{of} \textcolor{fig1teal}{him}\textcolor{fig1teal}{.} \textcolor{fig1teal}{With} \textcolor{fig1teal}{such} \textcolor{fig1teal}{never} \textcolor{fig1teal}{the} \textcolor{fig1teal}{way} \textcolor{fig1teal}{he} \textcolor{fig1teal}{is}\textcolor{fig1teal}{,} \textcolor{fig1teal}{Or}\textcolor{fig1teal}{chin} \textcolor{fig1teal}{said} \textcolor{fig1teal}{it} \textcolor{fig1teal}{was} \textcolor{fig1teal}{not} \textcolor{fig1teal}{easy} \textcolor{fig1teal}{to} \textcolor{fig1teal}{understand} \textcolor{fig1teal}{that} \textcolor{fig1teal}{an} \textcolor{fig1teal}{operation} \textcolor{fig1teal}{could} \textcolor{fig1teal}{get} \textcolor{fig1teal}{home} \textcolor{fig1teal}{to} \textcolor{fig1teal}{his} \textcolor{fig1teal}{first} \textcolor{fig1teal}{office} \textcolor{fig1teal}{ever}\textcolor{fig1teal}{.} \textcolor{fig1teal}{The} \textcolor{fig1teal}{National} \textcolor{fig1teal}{Security} \textcolor{fig1teal}{Relations} \textcolor{fig1teal}{official}\textcolor{fig1teal}{,} \textcolor{fig1teal}{though}\textcolor{fig1teal}{,} \textcolor{fig1teal}{reported} \textcolor{fig1teal}{after} \textcolor{fig1teal}{a} \textcolor{fig1teal}{running} \textcolor{fig1teal}{for} \textcolor{fig1teal}{office} \textcolor{fig1teal}{that} \textcolor{fig1teal}{Or}\textcolor{fig1teal}{chin} \textcolor{fig1teal}{said} \textcolor{fig1teal}{he} \textcolor{fig1teal}{was} \textcolor{fig1teal}{looking} \textcolor{fig1teal}{to} \textcolor{fig1teal}{the} \textcolor{fig1teal}{energy} \textcolor{fig1teal}{department} \textcolor{fig1teal}{for} \textcolor{fig1teal}{the} \textcolor{fig1teal}{rules} \textcolor{fig1teal}{in} \textcolor{fig1teal}{the} \textcolor{fig1teal}{business} \textcolor{fig1teal}{and} \textcolor{fig1teal}{the} \textcolor{fig1teal}{business}\textcolor{fig1teal}{,} \textcolor{fig1teal}{changing} \textcolor{fig1teal}{them} \textcolor{fig1teal}{to} \textcolor{fig1teal}{fit} \textcolor{fig1teal}{future} \textcolor{fig1teal}{threats}\textcolor{fig1teal}{,} \textcolor{fig1teal}{and} \textcolor{fig1teal}{to} \textcolor{fig1teal}{limit} \textcolor{fig1teal}{the} \textcolor{fig1teal}{size} \textcolor{fig1teal}{of} \textcolor{fig1teal}{the} \textcolor{fig1teal}{website}\textcolor{fig1teal}{.} \textcolor{fig1teal}{At} \textcolor{fig1teal}{what} \textcolor{fig1teal}{was} \textcolor{fig1teal}{in} \textcolor{fig1teal}{a} \textcolor{fig1teal}{meeting} \textcolor{fig1teal}{three} \textcolor{fig1teal}{hours} \textcolor{fig1teal}{early} \textcolor{fig1teal}{by} \textcolor{fig1teal}{my} \textcolor{fig1teal}{office}\textcolor{fig1teal}{,} \textcolor{fig1teal}{she} \textcolor{fig1teal}{said} \textcolor{fig1teal}{she} \textcolor{fig1teal}{had} \textcolor{fig1teal}{no} \textcolor{fig1teal}{access} \textcolor{fig1teal}{to} \textcolor{fig1teal}{the} \textcolor{fig1teal}{data}\textcolor{fig1teal}{.} \textcolor{fig1teal}{What} \textcolor{fig1teal}{—} \textcolor{fig1teal}{what} \textcolor{fig1teal}{the} \textcolor{fig1teal}{truth}\textcolor{fig1teal}{,} \textcolor{fig1teal}{Se}\textcolor{fig1teal}{at} \textcolor{fig1teal}{told} \textcolor{fig1teal}{me} \textcolor{fig1teal}{—} \textcolor{fig1teal}{was} \textcolor{fig1teal}{�}\textcolor{fig1teal}{�}\textcolor{fig1teal}{I} \textcolor{fig1teal}{did} \textcolor{fig1teal}{not} \textcolor{fig1teal}{do} \textcolor{fig1teal}{that} \textcolor{fig1teal}{before} \textcolor{fig1teal}{I} \textcolor{fig1teal}{opened} \textcolor{fig1teal}{that} \textcolor{fig1teal}{office}\textcolor{fig1teal}{.}\textcolor{fig1teal}{�}\textcolor{fig1teal}{�}  \textcolor{fig1teal}{To} \textcolor{fig1teal}{have} \textcolor{fig1teal}{one} \textcolor{fig1teal}{of} \textcolor{fig1teal}{these} \textcolor{fig1teal}{know}\textcolor{fig1teal}{-}\textcolor{fig1teal}{nothing} \textcolor{fig1teal}{years} \textcolor{fig1teal}{in} \textcolor{fig1teal}{line} \textcolor{fig1teal}{for} \textcolor{fig1teal}{your} \textcolor{fig1teal}{life}\textcolor{fig1teal}{,} \textcolor{fig1teal}{he} \textcolor{fig1teal}{had} \textcolor{fig1teal}{to} \textcolor{fig1teal}{start} \textcolor{fig1teal}{getting} \textcolor{fig1teal}{every} \textcolor{fig1teal}{day} \textcolor{fig1teal}{in} \textcolor{fig1teal}{work} \textcolor{fig1teal}{for} \textcolor{fig1teal}{something} \textcolor{fig1teal}{of} \textcolor{fig1teal}{himself} \textcolor{fig1teal}{with} \textcolor{fig1teal}{a} \textcolor{fig1teal}{few} \textcolor{fig1teal}{strokes} \textcolor{fig1teal}{of} \textcolor{fig1teal}{deb}\textcolor{fig1teal}{at} \textcolor{fig1teal}{here} \textcolor{fig1teal}{that}\textcolor{fig1teal}{�}\textcolor{fig1teal}{�}\textcolor{fig1teal}{s} \textcolor{fig1teal}{been} \textcolor{fig1teal}{tax} \textcolor{fig1teal}{for} \textcolor{fig1teal}{every} \textcolor{fig1teal}{type} \textcolor{fig1teal}{of} \textcolor{fig1teal}{activity} \textcolor{fig1teal}{by} \textcolor{fig1teal}{other} \textcolor{fig1teal}{people}\textcolor{fig1teal}{,} \textcolor{fig1teal}{including} \textcolor{fig1teal}{mine}\textcolor{fig1teal}{,} \textcolor{fig1teal}{in} \textcolor{fig1teal}{an} \textcolor{fig1teal}{office} \textcolor{fig1teal}{and} \textcolor{fig1teal}{by} \textcolor{fig1teal}{a} \textcolor{fig1teal}{company} \textcolor{fig1teal}{and} \textcolor{fig1teal}{lump}\textcolor{fig1teal}{ed} \textcolor{fig1teal}{in} \textcolor{fig1teal}{with} \textcolor{fig1teal}{criminal} \textcolor{fig1teal}{activity}\textcolor{fig1teal}{,} \textcolor{fig1teal}{for} \textcolor{fig1teal}{which} \textcolor{fig1teal}{he} \textcolor{fig1teal}{qualifies} \textcolor{fig1teal}{as} \textcolor{fig1teal}{a} \textcolor{fig1teal}{present} \textcolor{fig1teal}{day}\textcolor{fig1teal}{.}  \textcolor{fig1teal}{That}\textcolor{fig1teal}{,} \textcolor{fig1teal}{and} \textcolor{fig1teal}{the} \textcolor{fig1teal}{type} \textcolor{fig1teal}{of} \textcolor{fig1teal}{control} \textcolor{fig1teal}{he} \textcolor{fig1teal}{has} \textcolor{fig1teal}{is} \textcolor{fig1teal}{one} \textcolor{fig1teal}{of} \textcolor{fig1teal}{the} \textcolor{fig1teal}{company}\textcolor{fig1teal}{�}\textcolor{fig1teal}{�}\textcolor{fig1teal}{s} \textcolor{fig1teal}{most} \textcolor{fig1teal}{important} \textcolor{fig1teal}{part} \textcolor{fig1teal}{of} \textcolor{fig1teal}{business} \textcolor{fig1teal}{now}\textcolor{fig1teal}{.}  \textcolor{fig1teal}{One} \textcolor{fig1teal}{of} \textcolor{fig1teal}{his} \textcolor{fig1teal}{reporter} \textcolor{fig1teal}{friends} \textcolor{fig1teal}{told} \textcolor{fig1teal}{me} \textcolor{fig1teal}{that} \textcolor{fig1teal}{a} \textcolor{fig1teal}{possible} \textcolor{fig1teal}{run} \textcolor{fig1teal}{for} \textcolor{fig1teal}{the} \textcolor{fig1teal}{oil} \textcolor{fig1teal}{company} \textcolor{fig1teal}{was} \textcolor{fig1teal}{not} \textcolor{fig1teal}{too} \textcolor{fig1teal}{possible}\textcolor{fig1teal}{.} \textcolor{fig1teal}{He} \textcolor{fig1teal}{also} \textcolor{fig1teal}{told} \textcolor{fig1teal}{The} \textcolor{fig1teal}{Hill}\textcolor{fig1teal}{�}\textcolor{fig1teal}{�}\textcolor{fig1teal}{s} \textcolor{fig1teal}{David} \textcolor{fig1teal}{O}\textcolor{fig1teal}{ster}\textcolor{fig1teal}{,} \textcolor{fig1teal}{�}\textcolor{fig1teal}{�}\textcolor{fig1teal}{You}\textcolor{fig1teal}{�}\textcolor{fig1teal}{�}\textcolor{fig1teal}{re}
\end{samplebox}
\caption{Posterior Refinement trajectory on OpenWebText, Sample-2. (cont.)}
\label{fig:refine_sample_owt_2_r3}
\end{figure}

% ===================== TinyStories =====================
\begin{figure}[H]
\centering
\begin{samplebox}{\normalfont\textbf{$\ours$ (PR), Round: 1} \hfill \normalfont\scriptsize \textcolor{darkgray}{Committed: \textbf{72/128}}}
\scriptsize\linespread{0.9}\selectfont
\textcolor{fig1teal}{Once} \textcolor{fig1teal}{upon} \textcolor{fig1teal}{a} \textcolor{fig1teal}{time}\textcolor{fig1teal}{,} \textcolor{fig1teal}{there} \textcolor{fig1teal}{was} \textcolor{fig1teal}{a} \textcolor{fig1teal}{little} \textcolor{fig1teal}{girl} \textcolor{fig1teal}{named} \textcolor{fig1teal}{Lily}\textcolor{fig1teal}{.} \textcolor{fig1teal}{She} \textcolor{fig1teal}{loved} \textcolor{fig1teal}{to} \textcolor{fig1teal}{play} \textcolor{fig1teal}{with} \textcolor{fig1teal}{her} \textcolor{fig1ochre}{rubber}\textcolor{fig1ochre}{ade} \textcolor{fig1teal}{in} \textcolor{fig1teal}{her} \textcolor{fig1ochre}{backyard}\textcolor{fig1teal}{.} \textcolor{fig1teal}{One} \textcolor{fig1teal}{day}\textcolor{fig1teal}{,} \textcolor{fig1teal}{she} \textcolor{fig1teal}{saw} \textcolor{fig1teal}{a} \textcolor{fig1ochre}{giant} \textcolor{fig1ochre}{monster} \textcolor{fig1teal}{in} \textcolor{fig1teal}{her} \textcolor{fig1ochre}{backyard}\textcolor{fig1teal}{.} \textcolor{fig1teal}{The} \textcolor{fig1ochre}{bug} \textcolor{fig1teal}{was} \textcolor{fig1ochre}{friendly} \textcolor{fig1teal}{and} \textcolor{fig1teal}{wanted} \textcolor{fig1teal}{to} \textcolor{fig1ochre}{make} \textcolor{fig1teal}{it} \textcolor{fig1ochre}{its} \textcolor{fig1ochre}{friends}\textcolor{fig1teal}{.} \textcolor{fig1teal}{But} \textcolor{fig1ochre}{Lily} \textcolor{fig1ochre}{doll} \textcolor{fig1ochre}{Sarah} \textcolor{fig1teal}{saw} \textcolor{fig1teal}{the} \textcolor{fig1ochre}{bug} \textcolor{fig1teal}{and} \textcolor{fig1ochre}{pretended} \textcolor{fig1teal}{to} \textcolor{fig1ochre}{be} \textcolor{fig1ochre}{kind}\textcolor{fig1teal}{.} \textcolor{fig1teal}{"}\textcolor{fig1ochre}{Be} \textcolor{fig1ochre}{no}\textcolor{fig1teal}{,} \textcolor{fig1teal}{Lily}\textcolor{fig1teal}{.} \textcolor{fig1ochre}{I} \textcolor{fig1ochre}{always} \textcolor{fig1teal}{said} \textcolor{fig1teal}{it} \textcolor{fig1ochre}{doesn}\textcolor{fig1teal}{'t} \textcolor{fig1ochre}{want} \textcolor{fig1teal}{to} \textcolor{fig1ochre}{scare} \textcolor{fig1teal}{a} \textcolor{fig1ochre}{bird}\textcolor{fig1ochre}{.} \textcolor{fig1ochre}{Leave} \textcolor{fig1teal}{them}\textcolor{fig1ochre}{."}  \textcolor{fig1ochre}{Suddenly}\textcolor{fig1teal}{,} \textcolor{fig1teal}{the} \textcolor{fig1ochre}{big} \textcolor{fig1ochre}{lemon} \textcolor{fig1ochre}{crawled}\textcolor{fig1ochre}{ounced}\textcolor{fig1ochre}{.} \textcolor{fig1teal}{Lily} \textcolor{fig1ochre}{and} \textcolor{fig1teal}{could} \textcolor{fig1ochre}{hurt} \textcolor{fig1teal}{her} \textcolor{fig1ochre}{friend} \textcolor{fig1ochre}{because} \textcolor{fig1teal}{it} \textcolor{fig1ochre}{to} \textcolor{fig1ochre}{something} \textcolor{fig1ochre}{dark}\textcolor{fig1teal}{.} \textcolor{fig1teal}{The} \textcolor{fig1ochre}{toy} \textcolor{fig1ochre}{fell} \textcolor{fig1ochre}{from}\textcolor{fig1ochre}{led} \textcolor{fig1teal}{in} \textcolor{fig1teal}{the} \textcolor{fig1ochre}{wind}\textcolor{fig1ochre}{,} \textcolor{fig1ochre}{didn}\textcolor{fig1ochre}{'t} \textcolor{fig1ochre}{wanting} \textcolor{fig1teal}{to} \textcolor{fig1teal}{be} \textcolor{fig1ochre}{friends} \textcolor{fig1ochre}{by} \textcolor{fig1teal}{the} \textcolor{fig1ochre}{fake} \textcolor{fig1ochre}{might}\textcolor{fig1ochre}{'} \textcolor{fig1teal}{so}
\end{samplebox}
\begin{samplebox}{\normalfont\textbf{$\ours$ (PR), Round: 2} \hfill \normalfont\scriptsize \textcolor{darkgray}{Committed: \textbf{95/128}}}
\scriptsize\linespread{0.9}\selectfont
\textcolor{fig1teal}{Once} \textcolor{fig1teal}{upon} \textcolor{fig1teal}{a} \textcolor{fig1teal}{time}\textcolor{fig1teal}{,} \textcolor{fig1teal}{there} \textcolor{fig1teal}{was} \textcolor{fig1teal}{a} \textcolor{fig1teal}{little} \textcolor{fig1teal}{girl} \textcolor{fig1teal}{named} \textcolor{fig1teal}{Lily}\textcolor{fig1teal}{.} \textcolor{fig1teal}{She} \textcolor{fig1teal}{loved} \textcolor{fig1teal}{to} \textcolor{fig1teal}{play} \textcolor{fig1teal}{with} \textcolor{fig1teal}{her} \textcolor{fig1ochre}{red} \textcolor{fig1ochre}{plant} \textcolor{fig1teal}{in} \textcolor{fig1teal}{her} \textcolor{fig1ochre}{backyard}\textcolor{fig1teal}{.} \textcolor{fig1teal}{One} \textcolor{fig1teal}{day}\textcolor{fig1teal}{,} \textcolor{fig1teal}{she} \textcolor{fig1teal}{saw} \textcolor{fig1teal}{a} \textcolor{fig1teal}{big} \textcolor{fig1teal}{bird} \textcolor{fig1teal}{in} \textcolor{fig1teal}{her} \textcolor{fig1ochre}{backyard}\textcolor{fig1teal}{.} \textcolor{fig1teal}{The} \textcolor{fig1ochre}{flower} \textcolor{fig1teal}{was} \textcolor{fig1ochre}{enthusiastic} \textcolor{fig1teal}{and} \textcolor{fig1teal}{wanted} \textcolor{fig1teal}{to} \textcolor{fig1ochre}{give} \textcolor{fig1teal}{it} \textcolor{fig1teal}{a} \textcolor{fig1ochre}{kiss}\textcolor{fig1teal}{.} \textcolor{fig1teal}{But} \textcolor{fig1teal}{the} \textcolor{fig1ochre}{be}\textcolor{fig1ochre}{ener} \textcolor{fig1teal}{saw} \textcolor{fig1teal}{the} \textcolor{fig1ochre}{pumpkin} \textcolor{fig1teal}{and} \textcolor{fig1teal}{started} \textcolor{fig1teal}{to} \textcolor{fig1ochre}{tease} \textcolor{fig1ochre}{Lily}\textcolor{fig1teal}{.} \textcolor{fig1teal}{"}\textcolor{fig1ochre}{Be} \textcolor{fig1ochre}{careful}\textcolor{fig1teal}{,} \textcolor{fig1teal}{Lily}\textcolor{fig1teal}{.} \textcolor{fig1ochre}{The} \textcolor{fig1teal}{bird} \textcolor{fig1teal}{said} \textcolor{fig1teal}{it} \textcolor{fig1ochre}{wasn}\textcolor{fig1teal}{'t} \textcolor{fig1ochre}{nice} \textcolor{fig1teal}{to} \textcolor{fig1teal}{be} \textcolor{fig1teal}{a} \textcolor{fig1teal}{big} \textcolor{fig1ochre}{rock} \textcolor{fig1teal}{for} \textcolor{fig1teal}{them}\textcolor{fig1teal}{."}  \textcolor{fig1ochre}{Instead}\textcolor{fig1teal}{,} \textcolor{fig1teal}{the} \textcolor{fig1teal}{bird} \textcolor{fig1ochre}{flew} \textcolor{fig1ochre}{stopped} \textcolor{fig1teal}{to} \textcolor{fig1ochre}{show} \textcolor{fig1teal}{Lily} \textcolor{fig1teal}{it} \textcolor{fig1teal}{could} \textcolor{fig1ochre}{eat} \textcolor{fig1teal}{her} \textcolor{fig1ochre}{so} \textcolor{fig1ochre}{if} \textcolor{fig1teal}{it} \textcolor{fig1ochre}{talked} \textcolor{fig1teal}{and} \textcolor{fig1ochre}{talked}\textcolor{fig1teal}{.} \textcolor{fig1teal}{The} \textcolor{fig1teal}{bird} \textcolor{fig1ochre}{flew} \textcolor{fig1ochre}{high} \textcolor{fig1teal}{up} \textcolor{fig1teal}{in} \textcolor{fig1teal}{the} \textcolor{fig1ochre}{sky}\textcolor{fig1teal}{.} \textcolor{fig1teal}{Lily} \textcolor{fig1teal}{was} \textcolor{fig1teal}{happy} \textcolor{fig1teal}{to} \textcolor{fig1teal}{be} \textcolor{fig1ochre}{friends} \textcolor{fig1teal}{with} \textcolor{fig1teal}{the} \textcolor{fig1teal}{big} \textcolor{fig1teal}{bird}\textcolor{fig1ochre}{,} \textcolor{fig1teal}{so}
\end{samplebox}
\begin{samplebox}{\normalfont\textbf{$\ours$ (PR), Round: 3} \hfill \normalfont\scriptsize \textcolor{darkgray}{Committed: \textbf{112/128}}}
\scriptsize\linespread{0.9}\selectfont
\textcolor{fig1teal}{Once} \textcolor{fig1teal}{upon} \textcolor{fig1teal}{a} \textcolor{fig1teal}{time}\textcolor{fig1teal}{,} \textcolor{fig1teal}{there} \textcolor{fig1teal}{was} \textcolor{fig1teal}{a} \textcolor{fig1teal}{little} \textcolor{fig1teal}{girl} \textcolor{fig1teal}{named} \textcolor{fig1teal}{Lily}\textcolor{fig1teal}{.} \textcolor{fig1teal}{She} \textcolor{fig1teal}{loved} \textcolor{fig1teal}{to} \textcolor{fig1teal}{play} \textcolor{fig1teal}{with} \textcolor{fig1teal}{her} \textcolor{fig1teal}{toy} \textcolor{fig1ochre}{hoop} \textcolor{fig1teal}{in} \textcolor{fig1teal}{her} \textcolor{fig1ochre}{garden}\textcolor{fig1teal}{.} \textcolor{fig1teal}{One} \textcolor{fig1teal}{day}\textcolor{fig1teal}{,} \textcolor{fig1teal}{she} \textcolor{fig1teal}{saw} \textcolor{fig1teal}{a} \textcolor{fig1teal}{big} \textcolor{fig1teal}{bird} \textcolor{fig1teal}{in} \textcolor{fig1teal}{her} \textcolor{fig1teal}{garden}\textcolor{fig1teal}{.} \textcolor{fig1teal}{The} \textcolor{fig1teal}{bird} \textcolor{fig1teal}{was} \textcolor{fig1ochre}{hungry} \textcolor{fig1teal}{and} \textcolor{fig1teal}{wanted} \textcolor{fig1teal}{to} \textcolor{fig1teal}{give} \textcolor{fig1teal}{it} \textcolor{fig1teal}{a} \textcolor{fig1ochre}{ride}\textcolor{fig1teal}{.} \textcolor{fig1teal}{But} \textcolor{fig1teal}{the} \textcolor{fig1teal}{big} \textcolor{fig1teal}{bird} \textcolor{fig1teal}{saw} \textcolor{fig1teal}{the} \textcolor{fig1ochre}{plant} \textcolor{fig1teal}{and} \textcolor{fig1teal}{started} \textcolor{fig1teal}{to} \textcolor{fig1teal}{tease} \textcolor{fig1teal}{Lily}\textcolor{fig1teal}{.} \textcolor{fig1teal}{"}\textcolor{fig1ochre}{Come} \textcolor{fig1ochre}{in}\textcolor{fig1teal}{,} \textcolor{fig1teal}{Lily}\textcolor{fig1teal}{.} \textcolor{fig1teal}{The} \textcolor{fig1teal}{bird} \textcolor{fig1teal}{said} \textcolor{fig1teal}{it} \textcolor{fig1ochre}{doesn}\textcolor{fig1teal}{'t} \textcolor{fig1teal}{want} \textcolor{fig1teal}{to} \textcolor{fig1teal}{be} \textcolor{fig1teal}{a} \textcolor{fig1teal}{big} \textcolor{fig1teal}{friend} \textcolor{fig1teal}{for} \textcolor{fig1teal}{them}\textcolor{fig1teal}{."}  \textcolor{fig1ochre}{Instead}\textcolor{fig1teal}{,} \textcolor{fig1teal}{the} \textcolor{fig1teal}{bird} \textcolor{fig1ochre}{sang} \textcolor{fig1ochre}{closer} \textcolor{fig1teal}{to} \textcolor{fig1teal}{show} \textcolor{fig1teal}{Lily} \textcolor{fig1teal}{it} \textcolor{fig1teal}{could} \textcolor{fig1teal}{make} \textcolor{fig1teal}{her} \textcolor{fig1ochre}{feel} \textcolor{fig1teal}{before} \textcolor{fig1teal}{it} \textcolor{fig1ochre}{smiled} \textcolor{fig1teal}{and} \textcolor{fig1ochre}{jumped}\textcolor{fig1teal}{.} \textcolor{fig1teal}{The} \textcolor{fig1teal}{bird} \textcolor{fig1teal}{flew} \textcolor{fig1teal}{high} \textcolor{fig1teal}{up} \textcolor{fig1teal}{in} \textcolor{fig1teal}{the} \textcolor{fig1ochre}{air}\textcolor{fig1teal}{.} \textcolor{fig1teal}{Lily} \textcolor{fig1teal}{was} \textcolor{fig1teal}{happy} \textcolor{fig1teal}{to} \textcolor{fig1teal}{be} \textcolor{fig1ochre}{honest} \textcolor{fig1teal}{with} \textcolor{fig1teal}{the} \textcolor{fig1teal}{big} \textcolor{fig1teal}{bird}\textcolor{fig1teal}{,} \textcolor{fig1teal}{so}
\end{samplebox}
\begin{samplebox}{\normalfont\textbf{$\ours$ (PR), Round: 4} \hfill \normalfont\scriptsize \textcolor{darkgray}{Committed: \textbf{128/128}}}
\scriptsize\linespread{0.9}\selectfont
\textcolor{fig1teal}{Once} \textcolor{fig1teal}{upon} \textcolor{fig1teal}{a} \textcolor{fig1teal}{time}\textcolor{fig1teal}{,} \textcolor{fig1teal}{there} \textcolor{fig1teal}{was} \textcolor{fig1teal}{a} \textcolor{fig1teal}{little} \textcolor{fig1teal}{girl} \textcolor{fig1teal}{named} \textcolor{fig1teal}{Lily}\textcolor{fig1teal}{.} \textcolor{fig1teal}{She} \textcolor{fig1teal}{loved} \textcolor{fig1teal}{to} \textcolor{fig1teal}{play} \textcolor{fig1teal}{with} \textcolor{fig1teal}{her} \textcolor{fig1teal}{toy} \textcolor{fig1teal}{helicopter} \textcolor{fig1teal}{in} \textcolor{fig1teal}{her} \textcolor{fig1teal}{garden}\textcolor{fig1teal}{.} \textcolor{fig1teal}{One} \textcolor{fig1teal}{day}\textcolor{fig1teal}{,} \textcolor{fig1teal}{she} \textcolor{fig1teal}{saw} \textcolor{fig1teal}{a} \textcolor{fig1teal}{big} \textcolor{fig1teal}{bird} \textcolor{fig1teal}{in} \textcolor{fig1teal}{her} \textcolor{fig1teal}{garden}\textcolor{fig1teal}{.} \textcolor{fig1teal}{The} \textcolor{fig1teal}{bird} \textcolor{fig1teal}{was} \textcolor{fig1teal}{hungry} \textcolor{fig1teal}{and} \textcolor{fig1teal}{wanted} \textcolor{fig1teal}{to} \textcolor{fig1teal}{give} \textcolor{fig1teal}{it} \textcolor{fig1teal}{a} \textcolor{fig1teal}{ride}\textcolor{fig1teal}{.} \textcolor{fig1teal}{But} \textcolor{fig1teal}{the} \textcolor{fig1teal}{big} \textcolor{fig1teal}{bird} \textcolor{fig1teal}{saw} \textcolor{fig1teal}{the} \textcolor{fig1teal}{bird} \textcolor{fig1teal}{and} \textcolor{fig1teal}{started} \textcolor{fig1teal}{to} \textcolor{fig1teal}{tease} \textcolor{fig1teal}{Lily}\textcolor{fig1teal}{.} \textcolor{fig1teal}{"}\textcolor{fig1teal}{Come} \textcolor{fig1teal}{down}\textcolor{fig1teal}{,} \textcolor{fig1teal}{Lily}\textcolor{fig1teal}{.} \textcolor{fig1teal}{The} \textcolor{fig1teal}{bird} \textcolor{fig1teal}{said} \textcolor{fig1teal}{it} \textcolor{fig1teal}{didn}\textcolor{fig1teal}{'t} \textcolor{fig1teal}{want} \textcolor{fig1teal}{to} \textcolor{fig1teal}{be} \textcolor{fig1teal}{a} \textcolor{fig1teal}{big} \textcolor{fig1teal}{friend} \textcolor{fig1teal}{for} \textcolor{fig1teal}{them}\textcolor{fig1teal}{."}  \textcolor{fig1teal}{Suddenly}\textcolor{fig1teal}{,} \textcolor{fig1teal}{the} \textcolor{fig1teal}{bird} \textcolor{fig1teal}{flew} \textcolor{fig1teal}{back} \textcolor{fig1teal}{to} \textcolor{fig1teal}{show} \textcolor{fig1teal}{Lily} \textcolor{fig1teal}{it} \textcolor{fig1teal}{could} \textcolor{fig1teal}{make} \textcolor{fig1teal}{her} \textcolor{fig1teal}{wish} \textcolor{fig1teal}{before} \textcolor{fig1teal}{it} \textcolor{fig1teal}{sang} \textcolor{fig1teal}{and} \textcolor{fig1teal}{laughed}\textcolor{fig1teal}{.} \textcolor{fig1teal}{The} \textcolor{fig1teal}{bird} \textcolor{fig1teal}{flew} \textcolor{fig1teal}{high} \textcolor{fig1teal}{up} \textcolor{fig1teal}{in} \textcolor{fig1teal}{the} \textcolor{fig1teal}{sky}\textcolor{fig1teal}{.} \textcolor{fig1teal}{Lily} \textcolor{fig1teal}{was} \textcolor{fig1teal}{happy} \textcolor{fig1teal}{to} \textcolor{fig1teal}{be} \textcolor{fig1teal}{friends} \textcolor{fig1teal}{with} \textcolor{fig1teal}{the} \textcolor{fig1teal}{big} \textcolor{fig1teal}{bird}\textcolor{fig1teal}{,} \textcolor{fig1teal}{so}
\end{samplebox}
\caption{Posterior Refinement trajectory on TinyStories, Sample-1.}
\label{fig:refine_sample_tinystories_1}
\end{figure}

\begin{figure}[H]
\centering
\begin{samplebox}{\normalfont\textbf{$\ours$ (PR), Round: 1} \hfill \normalfont\scriptsize \textcolor{darkgray}{Committed: \textbf{71/128}}}
\scriptsize\linespread{0.9}\selectfont
\textcolor{fig1ochre}{Once} \textcolor{fig1teal}{upon} \textcolor{fig1teal}{a} \textcolor{fig1teal}{time}\textcolor{fig1teal}{,} \textcolor{fig1teal}{there} \textcolor{fig1teal}{was} \textcolor{fig1teal}{a} \textcolor{fig1teal}{little} \textcolor{fig1ochre}{girl} \textcolor{fig1teal}{called} \textcolor{fig1ochre}{one}\textcolor{fig1teal}{.} \textcolor{fig1ochre}{Her} \textcolor{fig1ochre}{parents} \textcolor{fig1teal}{said} \textcolor{fig1teal}{"}\textcolor{fig1ochre}{No}\textcolor{fig1teal}{,} \textcolor{fig1ochre}{Lucy}\textcolor{fig1teal}{.} \textcolor{fig1ochre}{Lucy} \textcolor{fig1teal}{was} \textcolor{fig1ochre}{feeling} \textcolor{fig1teal}{very} \textcolor{fig1teal}{sad}\textcolor{fig1teal}{.} \textcolor{fig1ochre}{Whenever} \textcolor{fig1teal}{she} \textcolor{fig1ochre}{should} \textcolor{fig1ochre}{play} \textcolor{fig1teal}{with} \textcolor{fig1ochre}{toys}\textcolor{fig1teal}{,} \textcolor{fig1teal}{the} \textcolor{fig1teal}{water} \textcolor{fig1ochre}{f}\textcolor{fig1ochre}{am} \textcolor{fig1ochre}{playing}\textcolor{fig1teal}{,} \textcolor{fig1ochre}{sometimes} \textcolor{fig1ochre}{just} \textcolor{fig1teal}{like} \textcolor{fig1ochre}{ice}\textcolor{fig1ochre}{bows}\textcolor{fig1teal}{.} \textcolor{fig1teal}{But} \textcolor{fig1ochre}{Lucy} \textcolor{fig1ochre}{still} \textcolor{fig1ochre}{couldn}\textcolor{fig1teal}{'t} \textcolor{fig1ochre}{want} \textcolor{fig1teal}{to} \textcolor{fig1ochre}{ignore} \textcolor{fig1teal}{it}\textcolor{fig1teal}{.}   \textcolor{fig1ochre}{Mom}\textcolor{fig1teal}{my} \textcolor{fig1ochre}{and} \textcolor{fig1ochre}{Daddy} \textcolor{fig1ochre}{stopped} \textcolor{fig1teal}{in} \textcolor{fig1teal}{the} \textcolor{fig1teal}{park} \textcolor{fig1teal}{and} \textcolor{fig1ochre}{suddenly}\textcolor{fig1teal}{,} \textcolor{fig1ochre}{Lucy} \textcolor{fig1teal}{saw} \textcolor{fig1teal}{a} \textcolor{fig1ochre}{giant} \textcolor{fig1ochre}{p}\textcolor{fig1ochre}{ur}\textcolor{fig1ochre}{rying} \textcolor{fig1teal}{down}\textcolor{fig1teal}{!} \textcolor{fig1teal}{She} \textcolor{fig1ochre}{pointed} \textcolor{fig1teal}{up} \textcolor{fig1ochre}{pointing} \textcolor{fig1teal}{at} \textcolor{fig1ochre}{the} \textcolor{fig1ochre}{billboard} \textcolor{fig1ochre}{saying} \textcolor{fig1teal}{"}\textcolor{fig1ochre}{M}\textcolor{fig1ochre}{a}\textcolor{fig1ochre}{".} \textcolor{fig1ochre}{Jenny} \textcolor{fig1teal}{was} \textcolor{fig1teal}{so} \textcolor{fig1teal}{excited} \textcolor{fig1teal}{that} \textcolor{fig1teal}{she} \textcolor{fig1ochre}{shouted} \textcolor{fig1ochre}{every}\textcolor{fig1ochre}{Yes} \textcolor{fig1teal}{of} \textcolor{fig1teal}{the} \textcolor{fig1ochre}{crane}\textcolor{fig1ochre}{."}  \textcolor{fig1ochre}{Daddy} \textcolor{fig1ochre}{looked} \textcolor{fig1ochre}{nodded}\textcolor{fig1teal}{.} \textcolor{fig1teal}{It} \textcolor{fig1teal}{was} \textcolor{fig1ochre}{quite} \textcolor{fig1ochre}{expensive}\textcolor{fig1teal}{,} \textcolor{fig1teal}{but} \textcolor{fig1teal}{she} \textcolor{fig1teal}{saw} \textcolor{fig1teal}{it} \textcolor{fig1ochre}{just} \textcolor{fig1ochre}{again} \textcolor{fig1teal}{and} \textcolor{fig1teal}{asked} \textcolor{fig1ochre}{if} \textcolor{fig1teal}{it} \textcolor{fig1teal}{wanted}
\end{samplebox}
\begin{samplebox}{\normalfont\textbf{$\ours$ (PR), Round: 2} \hfill \normalfont\scriptsize \textcolor{darkgray}{Committed: \textbf{94/128}}}
\scriptsize\linespread{0.9}\selectfont
\textcolor{fig1teal}{Once} \textcolor{fig1teal}{upon} \textcolor{fig1teal}{a} \textcolor{fig1teal}{time}\textcolor{fig1teal}{,} \textcolor{fig1teal}{there} \textcolor{fig1teal}{was} \textcolor{fig1teal}{a} \textcolor{fig1teal}{little} \textcolor{fig1teal}{girl} \textcolor{fig1teal}{called} \textcolor{fig1ochre}{Emma}\textcolor{fig1teal}{.} \textcolor{fig1ochre}{Mom}\textcolor{fig1teal}{my} \textcolor{fig1teal}{said} \textcolor{fig1teal}{"}\textcolor{fig1ochre}{No}\textcolor{fig1teal}{,} \textcolor{fig1ochre}{baby}\textcolor{fig1teal}{.} \textcolor{fig1ochre}{Emma} \textcolor{fig1teal}{was} \textcolor{fig1ochre}{feeling} \textcolor{fig1teal}{very} \textcolor{fig1teal}{sad}\textcolor{fig1teal}{.} \textcolor{fig1ochre}{As} \textcolor{fig1teal}{she} \textcolor{fig1teal}{was} \textcolor{fig1ochre}{filled} \textcolor{fig1teal}{with} \textcolor{fig1ochre}{soap}\textcolor{fig1teal}{,} \textcolor{fig1teal}{the} \textcolor{fig1teal}{water} \textcolor{fig1teal}{was} \textcolor{fig1teal}{so} \textcolor{fig1ochre}{happy}\textcolor{fig1teal}{,} \textcolor{fig1teal}{it} \textcolor{fig1teal}{looked} \textcolor{fig1teal}{like} \textcolor{fig1teal}{a} \textcolor{fig1ochre}{hug}\textcolor{fig1teal}{.} \textcolor{fig1teal}{But} \textcolor{fig1ochre}{M} \textcolor{fig1ochre}{still} \textcolor{fig1ochre}{didn}\textcolor{fig1teal}{'t} \textcolor{fig1ochre}{want} \textcolor{fig1teal}{to} \textcolor{fig1ochre}{see} \textcolor{fig1teal}{it}\textcolor{fig1teal}{.}   \textcolor{fig1ochre}{Mom}\textcolor{fig1teal}{my} \textcolor{fig1teal}{started} \textcolor{fig1teal}{to} \textcolor{fig1ochre}{look} \textcolor{fig1teal}{in} \textcolor{fig1teal}{the} \textcolor{fig1teal}{park} \textcolor{fig1teal}{and} \textcolor{fig1ochre}{suddenly}\textcolor{fig1teal}{,} \textcolor{fig1ochre}{she} \textcolor{fig1teal}{saw} \textcolor{fig1teal}{a} \textcolor{fig1teal}{little} \textcolor{fig1ochre}{green} \textcolor{fig1ochre}{bob} \textcolor{fig1ochre}{floating} \textcolor{fig1teal}{down}\textcolor{fig1teal}{!} \textcolor{fig1teal}{She} \textcolor{fig1teal}{looked} \textcolor{fig1teal}{up} \textcolor{fig1ochre}{and} \textcolor{fig1teal}{at} \textcolor{fig1teal}{it} \textcolor{fig1teal}{and} \textcolor{fig1teal}{said} \textcolor{fig1teal}{"}\textcolor{fig1ochre}{Hello} \textcolor{fig1ochre}{shrimp}\textcolor{fig1ochre}{".} \textcolor{fig1teal}{She} \textcolor{fig1teal}{was} \textcolor{fig1teal}{so} \textcolor{fig1teal}{excited} \textcolor{fig1teal}{that} \textcolor{fig1teal}{she} \textcolor{fig1ochre}{pulled} \textcolor{fig1teal}{it} \textcolor{fig1teal}{out} \textcolor{fig1teal}{of} \textcolor{fig1teal}{the} \textcolor{fig1ochre}{ground}\textcolor{fig1teal}{.}  \textcolor{fig1teal}{The} \textcolor{fig1ochre}{fish} \textcolor{fig1teal}{smiled}\textcolor{fig1teal}{.} \textcolor{fig1teal}{It} \textcolor{fig1teal}{was} \textcolor{fig1ochre}{quite}\textcolor{fig1ochre}{ucky}\textcolor{fig1teal}{,} \textcolor{fig1teal}{but} \textcolor{fig1teal}{she} \textcolor{fig1teal}{saw} \textcolor{fig1teal}{it} \textcolor{fig1teal}{was} \textcolor{fig1ochre}{smiling} \textcolor{fig1teal}{and} \textcolor{fig1teal}{asked} \textcolor{fig1ochre}{if} \textcolor{fig1teal}{it} \textcolor{fig1teal}{wanted}
\end{samplebox}
\begin{samplebox}{\normalfont\textbf{$\ours$ (PR), Round: 3} \hfill \normalfont\scriptsize \textcolor{darkgray}{Committed: \textbf{111/128}}}
\scriptsize\linespread{0.9}\selectfont
\textcolor{fig1teal}{Once} \textcolor{fig1teal}{upon} \textcolor{fig1teal}{a} \textcolor{fig1teal}{time}\textcolor{fig1teal}{,} \textcolor{fig1teal}{there} \textcolor{fig1teal}{was} \textcolor{fig1teal}{a} \textcolor{fig1teal}{little} \textcolor{fig1teal}{girl} \textcolor{fig1teal}{called} \textcolor{fig1ochre}{baby}\textcolor{fig1teal}{.} \textcolor{fig1teal}{Mom}\textcolor{fig1teal}{my} \textcolor{fig1teal}{said} \textcolor{fig1teal}{"}\textcolor{fig1ochre}{No}\textcolor{fig1teal}{,} \textcolor{fig1ochre}{daddy}\textcolor{fig1teal}{.} \textcolor{fig1teal}{She} \textcolor{fig1teal}{was} \textcolor{fig1teal}{feeling} \textcolor{fig1teal}{very} \textcolor{fig1teal}{sad}\textcolor{fig1teal}{.} \textcolor{fig1teal}{As} \textcolor{fig1teal}{she} \textcolor{fig1teal}{was} \textcolor{fig1teal}{playing} \textcolor{fig1teal}{with} \textcolor{fig1teal}{water}\textcolor{fig1teal}{,} \textcolor{fig1teal}{the} \textcolor{fig1teal}{water} \textcolor{fig1teal}{was} \textcolor{fig1teal}{so} \textcolor{fig1teal}{big}\textcolor{fig1teal}{,} \textcolor{fig1teal}{it} \textcolor{fig1teal}{looked} \textcolor{fig1teal}{like} \textcolor{fig1teal}{a} \textcolor{fig1ochre}{flood}\textcolor{fig1teal}{.} \textcolor{fig1teal}{But} \textcolor{fig1teal}{her} \textcolor{fig1ochre}{daddy} \textcolor{fig1ochre}{wasn}\textcolor{fig1teal}{'t} \textcolor{fig1teal}{seem} \textcolor{fig1teal}{to} \textcolor{fig1ochre}{accept} \textcolor{fig1teal}{it}\textcolor{fig1teal}{.}   \textcolor{fig1teal}{Mom}\textcolor{fig1teal}{my} \textcolor{fig1teal}{started} \textcolor{fig1teal}{to} \textcolor{fig1teal}{look} \textcolor{fig1teal}{in} \textcolor{fig1teal}{the} \textcolor{fig1teal}{park} \textcolor{fig1teal}{and} \textcolor{fig1ochre}{then}\textcolor{fig1teal}{,} \textcolor{fig1teal}{she} \textcolor{fig1teal}{saw} \textcolor{fig1teal}{a} \textcolor{fig1teal}{little} \textcolor{fig1ochre}{boat} \textcolor{fig1ochre}{fish}\textcolor{fig1ochre}{ing} \textcolor{fig1teal}{down}\textcolor{fig1teal}{!} \textcolor{fig1teal}{She} \textcolor{fig1teal}{looked} \textcolor{fig1teal}{up} \textcolor{fig1ochre}{straight} \textcolor{fig1teal}{at} \textcolor{fig1teal}{it} \textcolor{fig1teal}{and} \textcolor{fig1teal}{said} \textcolor{fig1teal}{"}\textcolor{fig1ochre}{Hi}\textcolor{fig1ochre}{uck}\textcolor{fig1teal}{!"} \textcolor{fig1teal}{She} \textcolor{fig1teal}{was} \textcolor{fig1teal}{so} \textcolor{fig1teal}{excited} \textcolor{fig1teal}{that} \textcolor{fig1teal}{she} \textcolor{fig1ochre}{pulled} \textcolor{fig1teal}{it} \textcolor{fig1teal}{out} \textcolor{fig1teal}{of} \textcolor{fig1teal}{the} \textcolor{fig1teal}{water}\textcolor{fig1teal}{.}  \textcolor{fig1teal}{The} \textcolor{fig1teal}{girl} \textcolor{fig1teal}{smiled}\textcolor{fig1teal}{.} \textcolor{fig1teal}{It} \textcolor{fig1teal}{was} \textcolor{fig1teal}{a} \textcolor{fig1ochre}{surprise}\textcolor{fig1teal}{,} \textcolor{fig1teal}{but} \textcolor{fig1teal}{she} \textcolor{fig1teal}{saw} \textcolor{fig1teal}{it} \textcolor{fig1teal}{was} \textcolor{fig1ochre}{friendly} \textcolor{fig1teal}{and} \textcolor{fig1teal}{asked} \textcolor{fig1teal}{if} \textcolor{fig1teal}{it} \textcolor{fig1teal}{wanted}
\end{samplebox}
\begin{samplebox}{\normalfont\textbf{$\ours$ (PR), Round: 4} \hfill \normalfont\scriptsize \textcolor{darkgray}{Committed: \textbf{128/128}}}
\scriptsize\linespread{0.9}\selectfont
\textcolor{fig1teal}{Once} \textcolor{fig1teal}{upon} \textcolor{fig1teal}{a} \textcolor{fig1teal}{time}\textcolor{fig1teal}{,} \textcolor{fig1teal}{there} \textcolor{fig1teal}{was} \textcolor{fig1teal}{a} \textcolor{fig1teal}{little} \textcolor{fig1teal}{girl} \textcolor{fig1teal}{called}\textcolor{fig1teal}{my}\textcolor{fig1teal}{.} \textcolor{fig1teal}{Mom}\textcolor{fig1teal}{my} \textcolor{fig1teal}{said} \textcolor{fig1teal}{"}\textcolor{fig1teal}{No}\textcolor{fig1teal}{,} \textcolor{fig1teal}{Sammy}\textcolor{fig1teal}{.} \textcolor{fig1teal}{She} \textcolor{fig1teal}{was} \textcolor{fig1teal}{feeling} \textcolor{fig1teal}{very} \textcolor{fig1teal}{sad}\textcolor{fig1teal}{.} \textcolor{fig1teal}{As} \textcolor{fig1teal}{she} \textcolor{fig1teal}{was} \textcolor{fig1teal}{playing} \textcolor{fig1teal}{with} \textcolor{fig1teal}{water}\textcolor{fig1teal}{,} \textcolor{fig1teal}{the} \textcolor{fig1teal}{water} \textcolor{fig1teal}{was} \textcolor{fig1teal}{so} \textcolor{fig1teal}{big}\textcolor{fig1teal}{,} \textcolor{fig1teal}{it} \textcolor{fig1teal}{looked} \textcolor{fig1teal}{like} \textcolor{fig1teal}{a} \textcolor{fig1teal}{flood}\textcolor{fig1teal}{.} \textcolor{fig1teal}{But} \textcolor{fig1teal}{her} \textcolor{fig1teal}{feet} \textcolor{fig1teal}{couldn}\textcolor{fig1teal}{'t} \textcolor{fig1teal}{seem} \textcolor{fig1teal}{to} \textcolor{fig1teal}{make} \textcolor{fig1teal}{it}\textcolor{fig1teal}{.}   \textcolor{fig1teal}{Mom}\textcolor{fig1teal}{my} \textcolor{fig1teal}{started} \textcolor{fig1teal}{to} \textcolor{fig1teal}{look} \textcolor{fig1teal}{in} \textcolor{fig1teal}{the} \textcolor{fig1teal}{park} \textcolor{fig1teal}{and} \textcolor{fig1teal}{suddenly}\textcolor{fig1teal}{,} \textcolor{fig1teal}{she} \textcolor{fig1teal}{saw} \textcolor{fig1teal}{a} \textcolor{fig1teal}{little} \textcolor{fig1teal}{girl} \textcolor{fig1teal}{was} \textcolor{fig1teal}{sliding} \textcolor{fig1teal}{down}\textcolor{fig1teal}{!} \textcolor{fig1teal}{She} \textcolor{fig1teal}{looked} \textcolor{fig1teal}{up} \textcolor{fig1teal}{over} \textcolor{fig1teal}{at} \textcolor{fig1teal}{it} \textcolor{fig1teal}{and} \textcolor{fig1teal}{said} \textcolor{fig1teal}{"}\textcolor{fig1teal}{Hi}\textcolor{fig1teal}{elcome}\textcolor{fig1teal}{!"} \textcolor{fig1teal}{She} \textcolor{fig1teal}{was} \textcolor{fig1teal}{so} \textcolor{fig1teal}{excited} \textcolor{fig1teal}{that} \textcolor{fig1teal}{she} \textcolor{fig1teal}{took} \textcolor{fig1teal}{it} \textcolor{fig1teal}{out} \textcolor{fig1teal}{of} \textcolor{fig1teal}{the} \textcolor{fig1teal}{water}\textcolor{fig1teal}{.}  \textcolor{fig1teal}{The} \textcolor{fig1teal}{girl} \textcolor{fig1teal}{smiled}\textcolor{fig1teal}{.} \textcolor{fig1teal}{It} \textcolor{fig1teal}{was} \textcolor{fig1teal}{a} \textcolor{fig1teal}{cold}\textcolor{fig1teal}{,} \textcolor{fig1teal}{but} \textcolor{fig1teal}{she} \textcolor{fig1teal}{saw} \textcolor{fig1teal}{it} \textcolor{fig1teal}{was} \textcolor{fig1teal}{smiling} \textcolor{fig1teal}{and} \textcolor{fig1teal}{asked} \textcolor{fig1teal}{if} \textcolor{fig1teal}{it} \textcolor{fig1teal}{wanted}
\end{samplebox}
\caption{Posterior Refinement trajectory on TinyStories, Sample-2.}
\label{fig:refine_sample_tinystories_2}
\end{figure}